\documentclass[10pt,letterpaper,twoside]{article} 

\usepackage{amsmath,amsthm,amssymb,bbold,mathrsfs} 
\usepackage{graphicx}

\usepackage[title]{appendix}
\usepackage{enumitem}
\setlist[itemize]{itemsep=0.0cm}
\usepackage[hidelinks]{hyperref}
\usepackage[numbers]{natbib}

\newtheorem{thm}{Theorem}[section]
\newtheorem{cor}{Corollary}[section]
\newtheorem{lem}{Lemma}[section]
\newtheorem{prop}{Proposition}[section]
\newtheorem{assumption}{Assumption}[section]
\newtheorem{cond}{Condition}[section]

\newtheorem{rem}{Remark}[section]

\numberwithin{equation}{section}

\def\1{\mathbf{1}}
\def\I{\mathbb{1}}
\def\asto{\overset{a.s.}{\to}}
\def\argmin{\mathop{\arg\min}}
\def\argmax{\mathop{\arg\max}}
\def\e{e}
\def\fe{\phi}
\def\gma{\gamma}
\def\k{\kappa}
\def\re{\Re}
\def\rn{\Re^n}
\def\tr{\top}
 
\def\C{C}
\def\E{\mathbb{E}}

\def\L{\mathcal{L}}
\def\M{\mathcal{M}}
\def\N{\mathcal{N}}

\def\Pr{\mathbf{P}}
\def\S{\mathcal{S}}
\def\A{\mathcal{A}}
\def\Z{\mathcal{Z}}
\def\sp{\text{\rm span}}
\def\cl{\text{\rm cl}}

\def\P{P}
\def\Gm{\Gamma}

\def\Tl{T^{(\lambda)}}
\def\Pl{P^{(\lambda)}}

\makeatletter
\newcommand\appendix@section[1]{%
  \refstepcounter{section}%
  \orig@section*{Appendix \@Alph\c@section: #1}%
  \addcontentsline{toc}{section}{Appendix \@Alph\c@section: #1}%
}
\let\orig@section\section
\g@addto@macro\appendix{\let\section\appendix@section}
\makeatother

\begin{document} 

\markboth{Convergence of some Gradient-based TD Algorithms}{}

\title{On Convergence of some Gradient-based Temporal-Differences Algorithms for Off-Policy Learning\thanks{This research was supported by a grant from Alberta Innovates---Technology Futures.}}

\author{Huizhen Yu\thanks{RLAI Lab, Department of Computing Science, University of Alberta, Canada (\texttt{janey.hzyu@gmail.com})}}
\date{}
\maketitle

\begin{abstract}
We consider off-policy temporal-difference (TD) learning methods for policy evaluation in Markov decision processes with finite spaces and discounted reward criteria, and we present a collection of convergence results for several gradient-based TD algorithms with linear function approximation. 
The algorithms we analyze include: (i)~two basic forms of two-time-scale gradient-based TD algorithms, which we call GTD and which minimize the mean squared projected Bellman error using stochastic gradient-descent; (ii)~their ``robustified'' biased variants; (iii)~their mirror-descent versions which combine the mirror-descent idea with TD learning; and (iv)~a single-time-scale version of GTD that solves minimax problems formulated for approximate policy evaluation. We consider primarily constrained algorithms which confine their iterates in bounded sets; for the single-time-scale GTD algorithm, we also analyze its unconstrained version. 

We derive convergence results for three types of stepsizes: constant stepsize, slowly diminishing stepsize, as well as the standard type of diminishing stepsize with a square-summable condition. For the first two types of stepsizes, we apply the weak convergence method from stochastic approximation theory to characterize the asymptotic behavior of the algorithms, and for the standard type of stepsize, we analyze the algorithmic behavior with respect to a stronger mode of convergence, almost sure convergence. 
Our convergence results are for the aforementioned TD algorithms with three general ways of setting their $\lambda$-parameters: (i) state-dependent $\lambda$; (ii) a recently proposed scheme of using history-dependent $\lambda$ to keep the eligibility traces of the algorithms bounded while allowing for relatively large values of $\lambda$; and (iii) a composite scheme of setting the $\lambda$-parameters that combines the preceding two schemes and allows a broader class of generalized Bellman operators to be used for approximate policy evaluation with TD methods.
\end{abstract}


\newpage
\tableofcontents

\newpage

\section{Introduction}

We consider off-policy temporal-difference (TD) learning methods for policy evaluation in Markov decision processes (MDPs) with finite spaces and discounted reward criteria. Off-policy TD learning extends on-policy model-free TD learning \cite{Sut88,tr-disc} (see also the books \cite{BET,SUB}) to cases where stationary policies of interest are evaluated using data collected without executing the policies.
It is more flexible than on-policy learning and can be useful not only as a computational tool to solve MDPs, but also as an aid to building experience-based knowledge representations for autonomous agents in AI applications \cite{Sut09}.
The specific class of algorithms that we consider in this technical report is the class of gradient-based off-policy TD algorithms with linear function approximation. Our purpose is to analyze several such algorithms proposed in the literature, and present a collection of convergence results for a broad range of choices of stepsizes and other important algorithmic parameters.

The algorithms that we will analyze include the following:\vspace*{-5pt}
\begin{itemize}
\item Two two-time-scale gradient-based TD algorithms proposed and studied by Sutton et al.~\cite{gtd09,gtd08} and Maei~\cite{maei11}. These two algorithms use stochastic gradient-descent to minimize the mean squared projected Bellman error, a convex quadratic objective function, for approximate policy evaluation, thereby overcoming the divergence issue in off-policy TD learning. They have been called GTD2, TDC, as well as GTD($\lambda$) in the early works just mentioned. Here we shall refer to them as GTDa and GTDb, respectively, and refer to both algorithms as GTD algorithms.
\item A single-time-scale version of GTDa that solves minimax problems formulated for approximate policy evaluation. 
This algorithm was also considered in those early works on GTD just mentioned. However, that it is trying to solve a minimax problem equivalent to the projected-Bellman-error minimization was pointed out later by Liu et al.~\cite{pmtd3} (see also Mahadevan et al.~\cite{pmrl}). The latter viewpoint facilitates convergence analysis of the algorithm, by placing it in a more general class of stochastic approximation algorithms for solving minimax problems. 
\item The mirror-descent versions of GTD and TD. Combining the mirror-descent idea of Nemirovsky and Yudin \cite{Nem83} with TD learning was proposed by Mahadevan and Liu~\cite{pmtd2} (see also \cite{pmrl}). 
\item ``Robustified'' biased variants of the preceding algorithms. These algorithms use a ``robustification'' procedure to mitigate the high-variance issue in off-policy learning, at the price of introducing biases in this procedure. They are similar to the biased variant algorithms considered by the author \cite{etd-wkconv} for the emphatic TD (ETD) algorithm proposed by Sutton et al.~\cite{SuMW14}. In the present context, as we will show, for two-time-scale GTD algorithms, the variant algorithms can be viewed as approximate gradient algorithms, and for the single-time-scale GTDa, its variant tries to solve minimax problems that approximate the ones GTDa tries to solve.\vspace*{-5pt}
\end{itemize}
We will analyze primarily constrained algorithms which confine their iterates in bounded sets. Only for the single-time-scale GTDa algorithm, we will also analyze its unconstrained version under certain conditions.

We will present convergence results for three types of stepsizes: constant stepsize, slowly diminishing stepsize,  as well as the standard type of diminishing stepsize with a square-summable condition. For the first two types of stepsizes, we apply the weak convergence method from stochastic approximation theory \cite{KuY03} to characterize the asymptotic behavior of the algorithms. For the third, standard type of stepsize, we analyze the algorithmic behavior with respect to a stronger mode of convergence, almost sure convergence, by using general results on stochastic approximation~\cite{Bor08,KuY03}.

Our convergence results are for the aforementioned TD algorithms with three general ways of setting the $\lambda$-parameters in TD learning: \vspace*{-5pt}
\begin{itemize}
\item state-dependent $\lambda$ \cite{Sut95,SUB}; 
\item a case of history-dependent $\lambda$ as proposed recently by Yu et al.~\cite{gbe_td17},  which can keep the eligibility traces in the off-policy TD algorithms bounded while allowing for relatively large values of $\lambda$; 
\item a composite scheme of setting the $\lambda$-parameters~\cite{Yu-siam-lstd,gbe_td17}, which combines the preceding two schemes and allows a broader class of generalized Bellman operators to be used for approximate policy evaluation with TD methods.\vspace*{-3pt}
\end{itemize}

To our knowledge, for off-policy gradient-based TD algorithms with linear function approximation, there are few prior convergence results that address the case of general nonzero $\lambda$-parameters. Although such algorithms with constant or state-dependent $\lambda$ have been proposed and investigated (see e.g., \cite{maei11,gtd10,pmrl}), the analyses given earlier \cite{pmtd3,maei11,pmrl,gtd09,gtd08} have only proved convergence for the case where $\lambda = 0$ and the data consist of i.i.d.\ state transitions. But the assumption on i.i.d.\ data is unrealistic for reinforcement learning even in the case $\lambda = 0$. Moreover, these analyses cannot be extended to the case of positive $\lambda$, where the algorithms need to use non-i.i.d.\ off-policy data in order to gather information about the multistep Bellman operator with respect to which the mean squared projected Bellman error is defined. To our knowledge, the only prior convergence result that applies to off-policy data is given by Karmakar and Bhatnagar~\cite{KaB15}. It is a convergence result for the two-time-scale TDC algorithm (GTDb as we call it) with the standard type of diminishing stepsize, as an application of the theoretical results developed in their work~\cite{KaB15} for two-time-scale differential inclusions. The result is for $\lambda = 0$, although its arguments can be applied in the case where $\lambda$ is small enough so that the eligibility traces produced in the algorithm are bounded.

Our results on gradient-based TD algorithms differ from that of \cite{KaB15} not only in the range of algorithms and parameter settings they cover, but also in the proof approaches by which they are derived. Specifically, we combine the ordinary-differential-equation (ODE) based proof methods in stochastic approximation theory \cite{KuY03} with special properties of the eligibility traces and the ergodicity of the joint state and eligibility trace process under various settings of the $\lambda$-parameters mentioned above. Those properties were derived in the author's early works \cite{Yu-siam-lstd,etd-wkconv} for state-dependent $\lambda$ and in the recent work \cite{gbe_td17} for the special case of history-dependent $\lambda$ mentioned above. The properties of the joint state and eligibility trace process are not considered in \cite{KaB15} since it treats the case $\lambda = 0$ where the eligibility traces are simply functions of states, but these properties are important for convergence analysis of those TD algorithms that use general nonzero $\lambda$-parameters. The ODE-based line of analysis we use is less general than the differential inclusion-based method studied in \cite{KaB15}, however. As future work, it can be worthwhile to use the latter approach to handle even more flexible ways of choosing history-dependent $\lambda$ than the one we consider in this work.

Another difference between our work and \cite{KaB15} is that we analyze primarily constrained algorithms, as mentioned earlier. The result given in \cite{KaB15} is for the unconstrained two-time-scale TDC under the assumption that the iterates are almost surely bounded. With constraints, we do not need such assumptions, and instead we could simply require that the constraint sets are large enough so that the algorithms can estimate the gradients correctly. The presence of constraint sets helps us avoid some theoretical difficulties in convergence analysis. However, extra work is also needed to ensure that the constraint sets do not interfere with the algorithms to prevent them from achieving the original goals that they are designed for. We take care of such issues in our analyses, especially for mirror-descent GTD/TD algorithms and for the single- and two-time-scale GTDa algorithms, which are not as straightforward as the two-time-scale GTDb algorithm.

This technical report is organized as follows. 
Section~\ref{sec-prel} is about the preliminaries: We first describe the off-policy policy evaluation problem and the two two-time-scale GTD algorithms. We then explain the role of the $\lambda$-parameters in TD learning, and discuss the properties of the eligibility traces and the properties of the joint state and eligibility trace process, in order to prepare the stage for convergence analyses.
In Sections~\ref{sec-3}-\ref{sec-4}, we present convergence results for slowly diminishing stepsize and constant stepsize, which are derived with the weak convergence method. Section~\ref{sec-3} is for the two two-time-scale GTD algorithms and their biased variants. Section~\ref{sec-4} is for the mirror-descent GTD/TD algorithms, as well as a single-time-scale GTDa algorithm and its biased variant, both of which solve minimax problems for policy evaluation. In this section we also use the minimax problem formulation to strengthen the results of Section~\ref{sec-3} for the two-time-scale GTDa algorithm and its biased variant.
In Section~\ref{sec-5}, we consider standard stepsize conditions and present almost sure convergence results for 
both the two-time-scale and single-time-scale algorithms, including a result for the unconstrained single-time-scale GTDa for certain choices of the $\lambda$-parameters. We then conclude the paper in Section~\ref{sec-6} with a brief discussion about these results and open questions.
For quick access to our convergence results in Sections \ref{sec-3}-\ref{sec-5}, the convergence theorems will be listed at the beginning of each of those sections.

\section{Preliminaries} \label{sec-prel}

In this section we first introduce the off-policy policy evaluation problem, and describe two basic forms of the gradient-based TD algorithms which were proposed and studied in \cite{maei11, gtd10,gtd09,gtd08}. 
We then explain generalized Bellman operators, how they relate to the $\lambda$-parameters of the TD algorithms and to the objectives of these algorithms.
We also specify two ways of choosing these $\lambda$-parameters, for the algorithms that we will analyze in the paper.
These background materials will be given in Section~\ref{sec-problem}. 
In Section~\ref{sec-property-stproc}, the second half of this section, we present materials to prepare the stage for analyzing the gradient-based TD algorithms as stochastic approximation algorithms in the rest of this paper. These materials are about the state-trace process, the random process that underlies and drives the TD algorithms, and its important properties for convergence analysis (including, among others, ergodicity and uniform integrability properties). 

\subsection{Problem Setup and Two Basic Forms of Algorithms} \label{sec-problem}

The gradient-based TD algorithms we consider belong to the class of model-free, temporal-differences based learning algorithms for evaluating a stationary policy in a Markov decision process (MDP). We shall consider MDP with finite spaces. For the purpose of this paper, however, we do not need the full MDP framework.%footnote starts
\footnote{For references on MDP, see the excellent textbook~\cite{Puterman94}.}
It is adequate to consider two Markov chains on a finite state space $\S$.%footnote starts
\footnote{The states in these Markov chains need not correspond to the states of the MDP; they can correspond to state-action pairs. Depending on whether one evaluates the value function for states or state-actions pairs in the MDP, the Markov chains here correspond to slightly different processes in the MDP. (For the details of these correspondences, see~\cite[Examples 2.1, 2.2]{Yu-siam-lstd}.) But the analysis is the same, so for notational simplicity, we have chosen not to introduce action variables in the paper.}
%footnote ends
The first Markov chain has transition matrix $\P$, and the second $\P^o$.
Whatever physical mechanisms that induce the two chains shall be denoted by $\pi$ and $\pi^o$, and referred to as the target policy and behavior policy, respectively. The second Markov chain we can observe; however, what we want is to evaluate the system performance with respect to (w.r.t.) the first Markov chain that we do not observe---the ``off-policy'' learning case.

The performance of the target policy $\pi$ is defined w.r.t.\ a discounted total reward criterion as follows.
A one-stage reward function $r_\pi : \S \to \Re$ specifies the expected reward $r_\pi(s)$ at each state $s \in \S$. 
Each state is also associated with a state-dependent discount factor $\gamma(s) \in [0,1], s \in \S$. 
The expected discounted total rewards for each initial state $s \in \S$ is defined by
\begin{equation} \label{def-vpi}
 \textstyle{v_{\pi}(s) : = \E^\pi_s \big[ \, r_\pi(S_0) +  \sum_{n=1}^{\infty} \gamma(S_{1})\, \gamma(S_{2}) \, \cdots \, \gamma(S_{n})  \cdot r_\pi(S_n) \big].}
\end{equation} 
Here the notation $\E^\pi_s$ indicates that the expectation is taken w.r.t.\ the Markov chain $\{S_n\}_{n \geq 0}$ induced by $\pi$ and with the initial state $S_0 =s$. The function $v_\pi$ in (\ref{def-vpi}) is called the value function of $\pi$, and it is well-defined under Condition~\ref{cond-pol}(i) given below, which we shall assume throughout the paper. 

Denote by $\Gm$ the $|\S| \times |\S|$ diagonal matrix with the discount factors $\gamma(s)$ as its diagonal entries. 
{\samepage
\begin{cond}[Conditions on the target and behavior policies] \label{cond-pol} \hfill\vspace*{-4pt}
\begin{itemize}
\item[\rm (i)] $\P$ is such that the inverse $(I - \P \Gm)^{-1}$ exists, and 
\item[\rm (ii)] $\P^o$ is such that for all $s,s' \in \S$, $\P^o_{ss'} = 0 \Rightarrow \P_{ss'}=0$,  and moreover, $\P^o$ is irreducible.
\end{itemize}
\end{cond}
}

The second part of this condition is for the behavior policy $\pi^o$ that generates the observed Markov chain. It will be needed when we describe the off-policy learning algorithms. 
 
By standard MDP theory (see e.g., \cite{Puterman94}), under Condition~\ref{cond-pol}(i), $v_\pi$ satisfies uniquely the linear equation (expressed in matrix/vector notation with $v_\pi$, $r_\pi$ viewed as $|\S|$-dimensional vectors): 
\begin{equation} \label{eq-Bellman0}
v_\pi = r_{\pi} + \P \Gm \, v_\pi \qquad  \text{(i.e., \  $v_\pi = (I - \P \Gm)^{-1} r_\pi)$}.
\end{equation}
It is known as the Bellman equation (or dynamic programming equation) for $\pi$. Besides this equation,
$v_\pi$ also satisfies a broad family of generalized Bellman equations, which have $v_\pi$ as their unique solution and like (\ref{eq-Bellman0}), express $v_\pi$ as the sum of two rewards terms, with the first term representing the expected rewards received prior to a certain (randomized stopping) time and the second term those received afterwards. (We shall discuss these equations further in the subsequent Section~\ref{sec-choice-lambda}.)

TD algorithms compute $v_\pi$ by solving such a Bellman equation for $\pi$. Which Bellman equation to solve is determined by certain parameters, which we call the $\lambda$-parameters, used by the algorithms in their iterative computation of eligibility traces (which are iterates that carry information about the past states).
We shall give more details about this correspondence between $\lambda$ and the generalized Bellman equations in Section~\ref{sec-choice-lambda}, after we describe two basic forms of the GTD algorithms. For now, we will focus on the overall structure of the computation problem tackled by the gradient-based TD algorithms. Think of $v_\pi = \Tl v_\pi$ as one generalized Bellman equation that an algorithm chooses to solve. The operator $\Tl$ is an affine operator on $\re^{|S|}$ and similar to (\ref{eq-Bellman0}), can be expressed as
\begin{equation} \label{eq-Tl0}
  \Tl v = r_{\pi}^{(\lambda)} + \Pl v, \qquad \forall \, v \in \re^{|S|},
\end{equation}  
for a vector $r_{\pi}^{(\lambda)} \in \re^{|S|}$ and a substochastic matrix $\Pl$. We shall refer to $\Tl$ as a generalized Bellman operator for $\pi$.
The TD algorithms we consider try to find an approximate solution to the linear equation 
$$v = \Tl v,$$ 
by solving an optimization problem on a lower dimensional space using linear function approximation. 
Let us describe first the approximation architecture and then the formulation of the optimization problem.

Let $\phi : \S \to \re^d$ be a given function that maps each state to a $d$-dimensional feature vector (it will be taken for granted that $\phi$ is non-trivial; i.e., $\phi(s) \not = 0$ for at least one state $s$).
Write $\phi=(\phi_1, \ldots, \phi_d)$, and denote the subspace spanned by these component functions $\phi_i$ by $\L_\phi$.
To approximate $v_\pi$, the TD algorithms look for some function $v \in \L_\phi$ that satisfies the generalized Bellman equation approximately: $v \approx \Tl v$. The functions in the approximation subspace $\L_\phi$ are parameterized as $v(s) = \phi(s)^\tr \theta$, $s \in \S$, for parameters $\theta \in \re^d$.
(We treat $\phi(s)$ and $\theta$ as column vectors; the symbol $^\tr$ stands for transpose.) 
We do not require the functions $\phi_i$ to be linearly independent; for this reason, another subspace will be useful later. This is the subspace in $\re^d$ spanned by the feature vectors, $\sp \{ \phi(s) \mid s \in \S\}$; below we shall write it as $\sp \{\phi(\S)\}$ for short. In matrix notation, $\L_\phi$ is the column space of the matrix $\Phi$ that has the feature vectors $\phi(s)$ as its rows; i.e.,
$$\Phi  = \left[ \begin{array}{c} \vdots \\ \phi(s)^\tr \\ \vdots \end{array} \right] \qquad \text{or} \qquad \Phi^\tr  = \left[ \begin{array}{ccc} \cdots & \phi(s)  &\cdots \end{array} \right].$$
Any approximate value function $v \in \L_\phi$ can be written as $v = \Phi \theta$ for some $\theta \in \re^d$---note that $\theta$ is uniquely determined by $v$ if $\theta \in \sp \{\phi(\S)\}$.

Let us now describe the optimization problem that the gradient-based TD algorithms try to solve in order to find an approximation of $v_\pi$ in $\L_\phi$.

\subsubsection{The objective function} \label{sec-obj}
To find a function $v \in \L_\phi$ with $v \approx \Tl v$, the two original gradient-based TD algorithms, GTDa and GTDb, to be described shortly, try to minimize an objective function of the form
\begin{equation} \label{eq-J}
  J(\theta) = \tfrac{1}{2} \| \Pi_\xi (\Tl v_\theta - v_\theta) \|^2_\xi, \qquad \text{where} \ \ v_\theta = \Phi \theta, \, \theta \in \re^d.
\end{equation}  
Here $\Pi_\xi$ denotes projection onto the approximation subspace $\L_\phi$ w.r.t.\ the weighted Euclidean norm $\|\cdot\|_\xi$ given by $\| v\|_\xi^2 = \sum_{s \in \S} \xi_s v(s)^2$, for a positive $|\S|$-dimensional vector $\xi$ with components $\xi_s$. 
The objective $J(\theta)$ measures the magnitude of the ``Bellman error'' $\Tl v_\theta - v_\theta$ on the subspace $\L_\phi$. If the projected Bellman equation
$v_\theta = \Pi_\xi \Tl v_\theta$ has a unique solution, the relation between this solution and $v_\pi$ can be characterized using the oblique projection viewpoint and related approximation error bounds (see the early work \cite{bruno-oblproj,yb-errbd} and a summary in more general terms given in the recent work \cite[Appendix B]{gbe_td17}). In this paper, since our focus is on convergence properties of the algorithms and since the problem $\inf_\theta J(\theta)$ always has an optimal solution, we do not require the projected Bellman equation $v_\theta = \Pi_\xi \Tl v_\theta$ to have a unique solution or any solution. (When it has no solution or multiple solutions, the quality of the approximate value function from the minimization of $J(\theta)$ could be a concern, though.)  

In this paper, we shall take $\xi$ to be the invariant probability distribution of the Markov chain with transition matrix $\P^o$ induced by the behavior policy $\pi^o$ (i.e., $\xi^\tr \P^o = \xi^\tr$); such a distribution exists and is unique under Condition~\ref{cond-pol}(ii). 
This choice of $\xi$ is mostly for notational simplicity: our analyses extend to cases where $\xi$ does not coincide with the invariant distribution of $\P^o$, but the algorithms in those cases have additional weighting terms and are notationally more cumbersome. 

Later we will also discuss objective functions of the form $J(\theta) + p(\theta)$, where $p(\theta)$ is some smooth convex function that serves as a regularizer. It will be seen that to handle this additional $p(\theta)$ term, little extra effort is needed in the convergence analysis. So, for notational simplicity, we will take $J(\theta)$ to be the objective function in the first half of the paper, and discuss the regularized objective function after we have presented the main convergence proof arguments. One can also consider a mixed objective function by combing the projected Bellman errors for multiple Bellman operators, for instance, $J_1(\theta) + J_2(\theta)$ with $J_1, J_2$ defined like $J$ but for two different $\Tl$. The convergence analysis of the gradient-based TD algorithms for such mixed objectives is essentially the same as that for $J$, so we will focus on the latter for simplicity.

Let us now work out two expressions of $\nabla J(\theta)$, which are used respectively by the two GTD algorithms, before we describe these algorithms.
Let $\langle \cdot, \cdot \rangle_\xi$ denote the inner product on the Euclidean space $(\re^{|\S|}, \|\cdot\|_\xi)$; i.e., $\langle v, v'\rangle_\xi = \sum_{s \in \S} \xi_s v(s) v'(s)$. (The notation $\langle \cdot, \cdot \rangle$ will be used for the usual inner product in Euclidean spaces.) 
For $i \leq d$, let $\Phi_i$ denote the $i$-th column of $\Phi$, and $\theta_i$ the $i$-th component of $\theta$. Since $v_\theta = \Phi \theta = \sum_{i=1}^d \theta_i \Phi_i$, the partial derivative of $J$ w.r.t.\ each $\theta_i$ is 
$$ \nabla_{\theta_i} J(\theta) = \langle \Pi_\xi \big(\Tl v_\theta - v_\theta \big), \, \Pi_\xi \big(\Pl - I\big) \Phi_i \rangle_\xi.$$
(Recall that $\Pl$ is the substochastic matrix in the affine operator $\Tl$; cf.\ (\ref{eq-Tl0}).) Observe two facts. First, for any $v, v'$,
$$ \langle \Pi_\xi v, \Pi_\xi v' \rangle_\xi = \langle v, \Pi_\xi v' \rangle_\xi = \langle \Pi_\xi v, v' \rangle_\xi.$$
Second, for any $v$, there is a unique $x \in \sp \{\phi(\S)\}$ with $\Phi x = \Pi_\xi v$, and therefore, given $\theta$, there is a unique solution $x_\theta$ to the linear equation (in $x$),
\begin{equation} \label{eq-xtheta}
 \Phi x = \Pi_\xi \big(\Tl v_\theta - v_\theta \big), \quad x \in \sp\{\phi(\S)\},
\end{equation} 
which is also the unique solution to the equivalent linear equation%footnote starts
\footnote{To see (\ref{eq-xtheta2}) is equivalent to (\ref{eq-xtheta}), note that for any $v$, $\Phi x = \Pi_\xi v$ if and only if w.r.t.\ $\langle \cdot, \cdot \rangle_\xi$, 
$v - \Phi x$ is perpendicular to the approximation subspace $\L_\phi$, which is true if and only if $\langle \Phi_i, v - \Phi x \rangle_\xi = 0$ for all $\Phi_i$, $i \leq d$ (since $\L_\phi$ is the column space of $\Phi$). The latter system of linear equations is the same as the first equation in (\ref{eq-xtheta2}) when we set $v = \Tl v_\theta - v_\theta$.}
%footnote ends 
\begin{equation} \label{eq-xtheta2}
   \Phi^\tr \Xi \, \Phi x = \Phi^\tr \Xi \, \big(\Tl v_\theta - v_\theta \big), \quad x \in \sp\{\phi(\S)\},
\end{equation}
where $\Xi$ denotes the $|\S| \times |\S|$ diagonal matrix with the components of $\xi$ on its diagonal. 

Using the above facts, we can write
\begin{equation} \label{eq-Jgrad1c}
 \nabla_{\theta_i} J(\theta) = \langle \Phi x_\theta, \, \big(\Pl - I\big) \Phi_i \rangle_\xi = x_\theta^\tr \cdot \Phi^\tr \Xi \big(\Pl - I\big) \Phi_i,
\end{equation} 
which gives the expression of the gradient as
\begin{equation} \label{eq-Jgrad1}
   \nabla J(\theta) = \big( \Phi^\tr \Xi \big(\Pl - I\big) \Phi \big)^\tr x_\theta.
\end{equation}
Alternatively, we can write 
\begin{align}
\nabla_{\theta_i} J(\theta) & = - \langle \Pi_\xi \big(\Tl v_\theta - v_\theta \big), \, \Pi_\xi  \Phi_i \rangle_\xi  + \langle \Pi_\xi \big(\Tl v_\theta - v_\theta \big), \, \Pi_\xi \Pl \Phi_i \rangle_\xi  \notag \\
& =  - \langle \Tl v_\theta - v_\theta, \, \Phi_i \rangle_\xi + \langle \Phi x_\theta, \, \Pl \Phi_i \rangle_\xi  \notag \\
& = - \Phi_i^\tr \Xi \big(\Tl v_\theta - v_\theta \big) + x_\theta^\tr \cdot \Phi^\tr \Xi \, \Pl \Phi_i      \label{eq-Jgrad2c}
\end{align} 
(where in the second equality we used $\Pi_\xi \Phi_i = \Phi_i$). This gives another expression of the gradient:
\begin{equation} \label{eq-Jgrad2}
  \nabla J(\theta) = - \Phi^\tr \Xi \big(\Tl v_\theta - v_\theta \big) + \big( \Phi^\tr \Xi \, \Pl \Phi \big)^\tr x_\theta.
\end{equation}
In principle one can derive other gradient expressions and formulate corresponding gradient-based algorithms; we shall, however, focus on the expressions (\ref{eq-Jgrad1}) and (\ref{eq-Jgrad2}) only.

\subsubsection{GTDa and GTDb} \label{sec-gtdalg}

We now describe two basic forms of GTD algorithms. As mentioned earlier, they can only observe the Markov chain $\{S_n\}_{n \geq 0}$ 
induced by the behavior policy $\pi^o$, instead of the target policy $\pi$. Upon each state transition $(S_n, S_{n+1})$, they receive a random reward $R_{n+1}$ that is a function of the transition, $r(S_t, S_{t+1})$, plus a zero-mean finite-variance noise term whose distribution is determined by the state transition. The reward function $r(\cdot)$ relates to the target policy's one-stage reward $r_\pi$ as $r_{\pi}(s) = \E^\pi_s [ r(s, S_1) ]$ for all $s \in \S$.
The rewards $\{R_n\}$ and the states $\{S_n\}$ are all that the algorithms can observe.

Define $\rho(s, s') = \P_{ss'}/\P^o_{ss'}$ for $s, s' \in \S$. They are the importance sampling ratios that can be used to compensate for the differences in the dynamics of the two Markov chains. We assume that the algorithms know these ratios (this is the case for standard value function or state-action value function estimation, as well as for the simulation context where both $\P$ and $\P^o$ are known). 
To simplify notation, for $n \geq 0$, we write 
$$\rho_n = \rho(S_n, S_{n+1}), \qquad \gamma_n=\gamma(S_n).$$
For any given approximate value function $v$ on $\S$, we write $\delta_n(v)$ for the (scalar) temporal-difference term calculated based on the observed random transition $(S_n, S_{n+1})$ and reward $R_{n+1}$:
\begin{equation} \label{eq-deltan}
 \delta_n(v) = \rho_n \big( R_{n+1} + \gma_{n+1} v(S_{n+1}) - v(S_n) \big).
\end{equation}
The conditional expectation of $\delta_n(v)$ given the history $\{S_0, \ldots, S_n\}$ measures the difference between the two sides of the Bellman equation (\ref{eq-Bellman0}) for the state $S_n$ when $v_\pi$ in (\ref{eq-Bellman0}) is replaced by $v$.

Using the states $\{S_n\}$, a sequence of eligibility trace vectors $\e_n \in \re^d$ is calculated iteratively by both GTD algorithms according to this formula:
given an initial $\e_0 \in \re^d$, for $n \geq 1$,
\begin{equation} \label{eq-gtd-e}
   \e_{n} = \lambda_n \gma_n \rho_{n-1} e_{n-1} + \phi(S_n).
\end{equation}
Here $\lambda_n \in [0,1], n \geq 1$, are the $\lambda$-parameters we referred to earlier. They are important parameters in TD learning. Not only do they affect the behavior of the algorithms, but they also determine the Bellman operator $\Tl$ appearing in the objective function $J(\theta)$. In the next subsection we shall describe the choices of these parameters that we consider in this paper, and explain what are the associated Bellman operators $\Tl$.

The eligibility traces (or traces, for short) are combined with temporal-differences terms by the algorithms to generate a sequence of iterates $(\theta_n, x_n) \in \re^d \times \re^d$, starting from some initial $(\theta_0, x_0)$.
In particular, let $v_\theta = \Phi \theta$ as before. The first algorithm, GTDa, calculates the sequence iteratively according to
\begin{align}
   \theta_{n+1} & = \theta_n + \alpha_n \,  \rho_n \big(\phi(S_n) - \gma_{n+1} \phi(S_{n+1}) \big) \cdot \e_n^\tr x_n  ,  \label{gtd1-th}\\
   x_{n+1} & = x_n + \beta_n \big(\e_n \delta_n(v_{\theta_n}) - \phi(S_n) \phi(S_n)^\tr x_n\big). \label{gtd-x}
\end{align} 
The second algorithm, GTDb, has the same formula for $\{x_n\}$, but calculates $\{\theta_n\}$ according to
\begin{equation}
 \theta_{n+1} = \theta_n + \alpha_n \big(\e_n \delta_n(v_{\theta_n}) -   \rho_n (1 - \lambda_{n+1})  \gma_{n+1}  \phi(S_{n+1}) \cdot \e_n^\tr x_n \big). \label{gtd2-th}
\end{equation}
In the above $\{\alpha_n\}$ and $\{\beta_n\}$ are stepsizes, with $\alpha_n < < \beta_n$. (We will consider a broad range of stepsizes and we defer the precise stepsize conditions till later sections where we analyze the algorithms.) Although it can be hard, for readers unfamiliar with TD algorithms, to see how the preceding formulae relate to the gradient $\nabla J(\theta)$, the GTD algorithms do correspond to applying gradient-descent to minimize $J(\theta)$ for the two gradient expressions (\ref{eq-Jgrad1}) and (\ref{eq-Jgrad2}), respectively.
The idea of the two algorithms, roughly speaking, is to let the $\theta$-iterates evolve at a slow time-scale and the $x$-iterates at a fast time-scale. 
As $\theta_n$ varies slowly, the fast-evolving $x$-iterates aim to track the solution $x_\theta$ of (\ref{eq-xtheta2}) for $\theta = \theta_n$, the ``current'' $\theta$-iterate. The information about $x_\theta$ is then used to perform stochastic gradient-descent in the $\theta$-space to minimize $J(\theta)$.

To gain more intuition and insights about the algorithms, we suggest the reader consult the original derivations given in e.g.,~\cite[Chap.\ 7]{maei11}.%footnote starts
\footnote{GTDa and GTDb here are called  GTD2 and TDC, respectively, in~\cite{maei11}, for the case $\lambda = 0$. For the case of nonzero $\lambda$, GTDb here is called GTD($\lambda$) for value function estimation and GQ($\lambda$) for state-action value function estimation in~\cite{maei11}. More precisely, the GQ($\lambda$) algorithm does not coincide exactly with GTDb for estimating state-action values; it differs from the latter in a term with a conditional mean of zero, which does not make any difference in convergence analysis, however.}
%footnote ends
We need to also point out, however, that these derivations have issues. For example, they involve various expectations such as $\E [ \e_n \delta_n(v) ]$ that are taken for granted to be independent of the time $n$. But before we know the properties of the process $\{(S_n, \e_n)\}$, it is not clear w.r.t.\ which probability distribution one can define $\E [ \e_n \delta_n(v) ]$ so that it is independent of $n$. Indeed, even with constant $\lambda_n = \lambda$ for all $n$ and a stationary state process $\{S_n\}$, it does not immediately follow just from these that $\{(S_n, \e_n)\}$ has to have a stationary distribution, let alone a unique one. 

We shall discuss the properties of the state-trace process $\{(S_n, \e_n)\}$ in Section~\ref{sec-property-stproc}. As can be seen from (\ref{eq-gtd-e}), this process depends on the $\lambda$-parameters in the algorithms. So let us first explain the relation between the $\lambda$-parameters and the generalized Bellman operators $\Tl$, since this will at least let us complete the definition of the objective function $J(\theta)$. We will then focus the discussion on the process $\{(S_n, \e_n)\}$ and its many properties that will be needed---actually, for some choices of the $\lambda$'s that we consider, $\Tl$ also depends on the property of the process $\{(S_n, \e_n)\}$. As to the connection between the above algorithms and the gradient $\nabla J(\theta)$, it will be seen first in Prop.~\ref{prop-mean-fn}, Section~\ref{sec-erg}, after we explain $\Tl$ and the ergodicity property of $\{(S_n, \e_n)\}$.

\subsubsection{Choices of $\lambda$-parameters and associated Bellman operators} \label{sec-choice-lambda}

We will consider in this paper three ways of setting the $\lambda$-parameters in (\ref{eq-gtd-e}) for the trace iterates $\{\e_n\}$. We discuss the first two in this subsection (the third one builds upon them and will be discussed in Section~\ref{sec-cp-extension}.) 
These two choices are state-dependent $\lambda$ and a case of history-dependent $\lambda$ with special properties: 
\vspace*{-3pt}
\begin{itemize}
\item[(i)] State-dependent $\lambda$ \cite{Sut95,SUB}, where $\lambda_n = \lambda(S_n)$ for a given function $\lambda : \S \to [0,1]$.
\item[(ii)] History-dependent $\lambda$ as introduced in \cite{gbe_td17}, where we choose $\lambda_n$ based on the previous trace $\e_{n-1}$ directly, in order to make $\{\e_n\}$ bounded. In particular, we introduce additional memory states $y_n, n \geq 0$, to summarize the history of past states up to time $n$. We let $y_n$ evolve in a Markovian way and choose $\lambda_n$ based on the current $y_n$ and the previous trace $\e_{n-1}$ as follows: 
\begin{equation}
 y_n = f(y_{n-1}, S_n), \qquad \lambda_n=\lambda(y_n, \e_{n-1})  \label{eq-histlambda}
\end{equation}  
where $f$ and $\lambda$ are some given functions, whose properties will be given shortly.
\end{itemize}

Although state-dependent $\lambda$ is a special case of history-dependent $\lambda$ (e.g., take $y_n = S_n$ in (ii)), generality is not our purpose here. The primary purpose of choosing $\lambda$ based on (ii), as explained in \cite{gbe_td17}, is to exploit the flexibility of history-dependent $\lambda$ to bound the traces $\{\e_n\}$ easily, while allowing for a large range of $\lambda$ values. The latter is important because the choice of $\lambda$ affects the choice of the generalized Bellman operator $\Tl$ appearing in the objective function $J(\theta)$, and in turn, this choice of $\Tl$ affects approximation error. Bounding the traces is also important, as it facilitates convergence of the algorithms. Thus, instead of the most general history-dependent $\lambda$, we shall focus on the special case (ii) under additional conditions studied in \cite{gbe_td17}. These conditions concern the memory states and the function $\lambda(\cdot)$, and they will be needed in the next subsection to ensure certain desired properties of the state-trace process:

\begin{cond}[Evolution of memory states in (\ref{eq-histlambda})] \label{cond-mem} 
The memory states $y_n$ take values in a finite space $\M$, and under the behavior policy $\pi^o$, the Markov chain $\{(S_n, y_n)\}$ on $\S \times \M$ has a single recurrent class.
\end{cond} 

\begin{cond}[Condition for $\lambda(\cdot)$ in (\ref{eq-histlambda})] \label{cond-lambda} 
The function $\lambda : \M \times \re^d \to [0,1]$ in (\ref{eq-histlambda}) satisfies the following. For some norm $\|\cdot\|$ on $\rn$ and for each memory state $y  \in \M$:\vspace*{-0.2cm}
\begin{itemize}
\item[\rm (i)] For any $\e , \e' \in \re^d$, $\|\lambda(y, \e) \, \e - \lambda(y, \e') \, \e' \| \leq \| \e - \e'\|$.
\item[\rm (ii)] For some constant $\C_y$, $\| \gamma(s') \rho(s, s') \cdot \lambda(y, \e)  \, \e\| \leq \C_y$ for all $\e \in \re^d$ and all possible state transitions $(s, s')$ that can lead to the memory state $y$.
\end{itemize}
\end{cond}

Several existing off-policy algorithms, Tree-Backup~\cite{offpolicytd-pss}, Retrace~\cite{offpolicytd-mshb} and ABQ~\cite{abq}, choose state or state transition-dependent $\lambda$ accordingly to keep the trace iterates bounded (in fact, these works motivated the history-dependent $\lambda$ described above). Such choices of $\lambda$ satisfy the above conditions, since one can simply let $y$ be a state or state transition and let $\lambda(y,\e)$ be a function of $y$ only. One disadvantage of these choices, however, is that they are too conservative and often result in small values of $\lambda$. A few examples of memory states and function $\lambda(y,\e)$ that satisfy the above conditions are given in \cite[Section~2.2]{gbe_td17}.

In the rest of this subsection, let us explain, at a level of detail adequate for the purpose of this paper, what are the generalized Bellman operators $\Tl$ for the target policy associated with the preceding two ways of setting $\lambda$.
As mentioned above, corresponding to different choices of $\lambda$ are different Bellman operators $\Tl$ in the objective function $J(\theta)$. With different $\Tl$, the solutions of the minimization problem $\inf_{\theta} J(\theta)$ are also different and can have different approximation biases (see e.g.,~\cite[Appendix B]{gbe_td17}). 
For the purpose of this paper, however, the details of $\Tl$ do not matter, because our focus is on the stochastic approximation aspects of the algorithms and what we care about is whether the average dynamics of the algorithms can be characterized by mean ODEs that are related to the minimization of $J(\theta)$ for the associated $\Tl$. The two cases of choosing $\lambda$ mentioned above share many properties in common. Once their common properties are made clear, as will be done here and in the next subsection, the two cases can be treated together in most of our convergence analysis of the algorithms.
For this reason, regarding the Bellman operators $\Tl$, we shall recount only the facts that we will need in this paper (for a detailed study, see the paper \cite{gbe_td17}).

As mentioned earlier, different choices of $\lambda$ induce different Bellman operators $\Tl$ for the target policy.
They are members of a broad family of generalized Bellman operators $T_\tau$ associated with randomized stopping times $\tau$ \cite[Section 3.1]{gbe_td17}, which are all contractive operators that have $v_\pi$ as their unique fixed point (see \cite[Theorem 3.1 and Appendix A]{gbe_td17}). Such an operator takes the general form of
\begin{equation} \label{eq-gbe0}
  (T_\tau v)(s) = \E^\pi \big[ R^\tau + \gamma_1^\tau v(S_\tau) \mid S_0 = s \big], \qquad s \in \S, \ \forall \, v \in \re^{|\S|},
\end{equation}  
where $\E^\pi$ denotes expectation over the randomized stopping time $\tau$ and the states $\{S_t\}$ generated according to the target policy, $R^\tau$ is the total discounted rewards received prior to the time $\tau$ of stopping, $S_\tau$ is the state at time $\tau$, and $\gamma_1^\tau$ is a shorthand for the product of discount factors $\gamma_1 \cdots \gamma_\tau$. These generalized Bellman equations and operators are a consequence of the strong Markov property of Markov chains \cite[Theorem 3.3]{Num84}. We refer the reader to \cite[Section 3.1]{gbe_td17} for a fuller account of the framework and the mathematical notions and derivations involved.

For state-dependent $\lambda$, $\Tl$ corresponds to a randomized stopping time $\tau$ for the Markov chain $\{S_n\}$ under the target policy, where $\tau$ is such that 
$$\tau \geq 1, \qquad \Pr (\tau = n \mid \tau > n-1, S_0, \ldots, S_n) = 1 - \lambda(S_n) \ \ \ \text{for} \ n \geq 1.$$
(I.e., the probability of stopping at time $n$ given that the system has not stopped yet is $1 - \lambda(S_n)$.) The associated operator $\Tl$ can be expressed in several equivalent ways. Besides the general form (\ref{eq-gbe0}) above, we can write $\Tl$ as%footnote starts
\footnote{This follows from (\ref{eq-gbe0}) by taking conditional expectation over $\tau$.} 
%footnote ends
\begin{equation} \label{eq-expT1}
  (\Tl v)(s) = \E^\pi \left[  \sum_{n=0}^\infty \lambda_1^n \cdot \gamma_1^n \, r_\pi(S_n) + \sum_{n=1}^\infty \lambda_1^{n-1} (1 - \lambda_n) \cdot \gamma_1^n \, v(S_n) \, \Big| \, S_0 = s \right], \quad s \in \S, \ \forall \, v \in \re^{|\S|},
\end{equation}  
where we used the shorthand notation $\lambda_1^n = \lambda_1 \cdots \lambda_n$ with $\lambda_1^0=1$. 
We can also write $\Tl$ explicitly in terms of $\lambda$ and the model parameters as%footnote starts
\footnote{This follows from (\ref{eq-expT1}) by a direct calculation using the definition $\lambda_n = \lambda(S_n)$.} 
%footnote ends
\begin{equation} \label{eq-expT2}
 \Tl v = (I - P \Gamma \Lambda)^{-1} r_\pi + (I - P \Gamma \Lambda)^{-1} P \Gamma (I - \Lambda) \, v,
\end{equation}
where $\Lambda$ denotes the $|\S|\times |\S|$ diagonal matrix with diagonal entries $\lambda(s), s \in \S$. (Thus, the substochastic matrix in (\ref{eq-Tl0}) has the explicit expression $\Pl = (I - P \Gamma \Lambda)^{-1} P \Gamma (I - \Lambda)$.)

In the case of history-dependent $\lambda$, it is shown in \cite[Section 3.2]{gbe_td17} under Conditions~\ref{cond-pol}-\ref{cond-lambda} that $\Tl$ is also a generalized Bellman operator (for the target policy) corresponding to a certain randomized stopping time $\tau$. But this random time $\tau$ now depends on the behavior policy in a much more complex way than in the case of state-dependent $\lambda$. Among others, it depends on the dynamics of the traces under the behavior policy. As such, we generally cannot write $\Tl$ explicitly in terms of the model parameters $P$, $P^o$ and the function $\lambda$. We express $\Tl$ in other ways, in order to relate it to the algorithms that employ history-dependent $\lambda$. In particular, an expression of $\Tl$ similar to (\ref{eq-expT1}) will be useful in our subsequent analysis:
\begin{equation} \label{eq-expT3}
  (\Tl v)(s) = \E^\pi_\zeta \left[  \sum_{n=0}^\infty \lambda_1^n \cdot \gamma_1^n \, r_\pi(S_n) + \sum_{n=1}^\infty \lambda_1^{n-1} (1 - \lambda_n) \cdot \gamma_1^n \, v(S_n) \, \Big| \, S_0 = s \right], \quad s \in \S, \ \forall \, v \in \re^{|\S|}.
\end{equation}
In the above $\E^\pi_\zeta$ denotes expectation with respect to the probability measure of the following process:\vspace*{-3pt}
\begin{itemize}
\item The states $\{S_n\}$ are generated \emph{under the target policy} $\pi$.
\item For $n \geq 1$, the memory state $y_n$, the parameter $\lambda_n$, and the trace $\e_n$ are calculated according to (\ref{eq-histlambda}) and (\ref{eq-gtd-e}), respectively.%footnote starts
\footnote{The randomized stopping time $\tau$ is generated according to the following rule: $\tau \geq 1$ and for $n \geq 1$, $\Pr (\tau = n \mid \tau > n -1, y_0, \e_0, S_0, S_1, \ldots, S_n) = 1 - \lambda_n$ (i.e., the probability of stopping at time $n$ given that the system has not stopped yet is $1 - \lambda_n$, which is similar to the case of state-dependent $\lambda$). The random time $\tau$ does not appear in the expression (\ref{eq-expT3}) of $\Tl$, because (\ref{eq-expT3}) is an equivalent form of (\ref{eq-gbe0}) after taking conditional expectation over $\tau$.}
%footnote ends
\item The initial state, memory state and trace $(S_0, y_0, \e_0)$ are distributed according to $\zeta$, the unique invariant probability measure of the process $\{(S_n, y_n, \e_n)\}$ \emph{under the behavior policy} $\pi^o$.\vspace*{-3pt}
\end{itemize}
The existence and uniqueness of the invariant probability measure $\zeta$ just mentioned is ensured under Condition~\ref{cond-pol} on the two policies and Conditions \ref{cond-mem}-\ref{cond-lambda} on the memory states and the function $\lambda(\cdot)$. Further details will be explained in the next subsection (see \cite[Section 3.2]{gbe_td17} for the derivations of the preceding results).

Another expression that will be useful later is an expression of the ``Bellman error'' $\Tl v - v$ in terms of temporal-differences terms (this expression follows from (\ref{eq-expT2}) by rearranging terms):
\begin{equation} \label{eq-expT4}
(\Tl v - v)(s) = \E^\pi_\zeta \left[  \sum_{n=0}^\infty \lambda_1^n \cdot \gamma_1^n \, \big( r_\pi(S_n) + \gamma_{n+1} v(S_{n+1}) - v(S_n) \big) \, \Big| \,  S_0 = s \right], \quad s \in \S, \ \forall \, v \in \re^{|\S|}.
\end{equation}
The same expression (ignoring the subscript $\zeta$ of $\E$) also holds for the case of state-dependent $\lambda$.

\subsection{Properties of State-Trace Process} \label{sec-property-stproc}

The purpose of this subsection is to prepare the stage for analyzing the asymptotic behavior of the gradient-based TD algorithms using stochastic approximation theory.  We collect here important properties of the state-trace process that our subsequent analysis will rely on. 
Specifically, in Section~\ref{sec-erg}, we first discuss ergodicity properties of the state-trace process. We then consider several functions on the state-trace space that appear in the TD algorithms, and we derive their expectations w.r.t.\ the stationary state-trace process, which can be related to expressions of the gradient $\nabla J(\theta)$. In Section~\ref{sec-more-traceproperty}, we include more properties of the trace iterates, which will be needed later in analyzing the average dynamics of the algorithms and proving their convergence.

Most of the results in this subsection were proved earlier by the author \cite{Yu-siam-lstd,etd-wkconv,gbe_td17}. There are some small differences in the setup of the problem considered in those earlier works; these differences are nonessential and do not affect the conclusions obtained. For clarity, however, we will give additional details to bridge the gap, when this can be done quickly without repeating long proofs.

\subsubsection{Some notations and definitions}

Let us first introduce some notations and definitions that we will need below and throughout the paper. In most of our analysis, we will treat the two cases of $\lambda$ together. For brevity, let us collect the conditions given earlier for each case of $\lambda$ in a single assumption.

\begin{assumption} \label{cond-collective} 
Condition~\ref{cond-pol} holds. In the case of history-dependent $\lambda$ given in (\ref{eq-histlambda}), Conditions~\ref{cond-mem}-\ref{cond-lambda} also hold. 
\end{assumption}

By the state-trace process, we mean $\{(S_n, \e_n)\}$ for the case of state-dependent $\lambda$, and $\{(S_n, y_n, \e_n)\}$ (including the memory states $y_n$) for the case of history-dependent $\lambda$, generated under the behavior policy $\pi^o$. The state-trace process is a Markov chain with the weak Feller property---this means that, with $X_n = (S_n, \e_n)$ or $(S_n, y_n, \e_n)$ (depending on the case of $\lambda$), $\E [ f(X_1) \mid X_0 = x]$ is a continuous function of $x$ for any bounded continuous function $f$ \cite[Prop.\ 6.1.1]{MeT09}. (Using the definitions of the traces and memory states, one can verify that this is the case.) Weak Feller Markov chains have nice ergodicity properties \cite{Mey89}, which helped us in obtaining some of the ergodicity properties of the state-trace process that will be discussed shortly.

Let $\I(\cdot)$ denote the indicator function. For each initial condition $x$ of $(S_0, \e_0)$ or $(S_0, y_0, \e_0)$, define random probability measures $\mu_{x, n}, n \geq 0$, on the state-trace space by
$$\mu_{x, n}(D) = \textstyle{\frac{1}{n+1} \sum_{i = 0}^n} \I\big( (S_i, \e_i) \in D \big) \quad \text{or}  \ \ \ \mu_{x, n}(D) = \textstyle{\frac{1}{n+1} \sum_{i = 0}^n} \I\big( (S_i, y_i, \e_i) \in D \big)$$ 
for all Borel subsets $D$ of the state-trace space.%footnote starts
\footnote{We take the topology on the state-trace space to be the product topology, with the discrete topology on the space of states/memory states and with the usual topology on $\re^d$, the trace space. The Borel sigma-algebra on the state-trace space is generated by this topology.}
%footnote 
We refer to them as the occupation probability measures of the state-trace process. 
Their convergence to the unique invariant probability measure of the state-trace process is crucial for our convergence analysis of the TD algorithms. 
Here the sense of convergence for these probability measures is weak convergence, and it is defined as follows: if $\mu$ and $\{\mu_n\}$ are probability measures on the state-trace space and as $n \to \infty$, $\int f d\mu_{n} \to \int f d\mu$ for all bounded continuous functions $f$, then the sequence $\{\mu_n\}$ is said to converge weakly to $\mu$. 

Let $Z_n=(S_n, \e_n, S_{n+1})$ in the case of state-dependent $\lambda$, and $Z_n= (S_n,  y_n, \e_n, S_{n+1})$ in the case of history-dependent $\lambda$. Denote the space of $Z_n$ in each case by the same notation $\Z$. (This is to prepare for handling temporal-differences terms, which involve state transitions.)
Among (vector-valued) functions on $\Z$, a set of them will be important in our analysis: these are functions $f(z)$ that are Lipschitz continuous in the trace variable, uniformly w.r.t.\ the other components of $z$. We have $z = (s, \e, s')$ or $z = (s, y, \e, s')$ (depending on the case of $\lambda$), and since the state space $\S$ and the memory state space $\M$ are finite, such a function $f(z)$ is just one that is Lipschitz continuous in $\e$ for each $(s, s')$ or $(s, y, s')$. So in what follows, when referring to such a function $f$, we will simply say $f$ is Lipschitz continuous in the trace variable $\e$.

Regarding other notations and terminologies, the abbreviation ``a.s.'' stands for ``almost surely,'' and ``$\Pr_x$-a.s.'' for ``almost surely w.r.t.\ the probability measure $\Pr_x$,'' where the subscript $x$ indicates that the process under consideration starts from the initial condition $x$. 
We shall use $\|\cdot\|$ to denote the sup-norm and $\|\cdot\|_2$ the standard Euclidean norm. For a sequence of random variables $X_n$, we say $X_n$ converges in mean to a random variable $X$, if $\E [ \| X_n - X \| ] \to 0$ as $n \to \infty$.

\subsubsection{Ergodicity properties} \label{sec-erg}

The ergodicity results given in Theorem~\ref{thm-erg} below are proved essentially in \cite[Theorems 3.1,~3.3]{Yu-siam-lstd} for the case of state-dependent $\lambda$, and in \cite[Theorem 3.2]{gbe_td17} for the case of history-dependent $\lambda$ we consider.%footnote starts
\footnote{The paper \cite{Yu-siam-lstd} analyzed the state-trace process in the case of constant $\lambda$. Its proof arguments and conclusions, however, extend to the case of state-dependent $\lambda$, under Condition~\ref{cond-pol}. In fact, this extension was incorporated in the convergence analysis of the more complex ETD algorithm \cite{yu-etdarx}, and that is why we do not repeat here the proof arguments for this extension. The paper \cite{gbe_td17} analyzed the state-trace process in the case of history-dependent $\lambda$ under Assumption~\ref{cond-collective}.}
%footnote ends
The first part of the theorem concerns the existence and uniqueness of an invariant probability measure of the state-trace process, and the convergence of the occupation probability measures. Without these ergodicity properties, the behavior of the TD algorithms will be quite different, indeed much more complex, and will have to be analyzed using more advanced stochastic approximation theory for differential inclusions (which are beyond the scope of the present paper).

The second part of the theorem will be used, among others, to characterize the average dynamics of the algorithms. We need it in the case of state-dependent $\lambda$. In that case, the functions appearing in the algorithms can be unbounded, but they have the Lipschitz continuity property required in the theorem.  Note that for bounded continuous functions $f(z)$, by the weak convergence of occupation probability measures given in Theorem~\ref{thm-erg}(i),  the conclusions in the part (ii) automatically hold. (However, even in the case of history-dependent $\lambda$, we will still need the function $f(z)$ to be Lipschitz continuous in the trace variable, in order to show that it satisfies a certain ``averaging condition'' that is stronger than the convergence-in-mean ensured by Theorem~\ref{thm-erg} and is needed in the subsequent convergence analysis. See Prop.~\ref{prop-trace-averaging}(i) and Remark~\ref{rmk-Lipschitzcont} in Section~\ref{sec-more-traceproperty}.)

%\smallskip
\begin{thm}[Ergodicity of the state-trace process] \label{thm-erg}
Under Assumption~\ref{cond-collective}, the following hold: \vspace*{-3pt}
\begin{itemize}
\item[\rm (i)] The state-trace process is a weak Feller Markov chain and has a unique invariant probability measure $\zeta$. For each initial condition $x$ of the process, the occupation probability measures $\{\mu_{x,n}\}$ converge weakly to $\zeta$, $\Pr_x$-a.s.
\item[\rm (ii)] Let $\E_\zeta$ denote expectation w.r.t.\ the stationary state-trace process with initial distribution $\zeta$. Then $\E_\zeta \big[ \|f(Z_0)\| \big] < \infty$ for any vector-valued function $f(z)$ that is Lipschitz continuous in the trace variable $\e$. Furthermore, for such function $f$, given each initial condition of $Z_0$, as $n \to \infty$, $\frac{1}{n} \sum_{i=0}^{n-1} f(Z_i)$ converges to $\E_\zeta [f(Z_0) ]$ in mean and almost surely.
\end{itemize}
\end{thm}
%\smallskip

%\smallskip
Next, w.r.t.\ the stationary state-trace process with initial distribution $\zeta$, we derive expressions of the expectation $\E_\zeta [ f(Z_0) ]$ for several functions $f$ involved in the GTD algorithms.
These expressions are related to the expressions (\ref{eq-Jgrad1c})-(\ref{eq-Jgrad2}) of the gradient $\nabla J(\theta)$ and will appear in the mean ODEs associated with the algorithms. 
To state the results concisely, let 
$$\bar{\delta}(s, s', v) = \rho(s) \big( r(s, s') +  \gma(s') v(s') - v(s) \big), \qquad \bar{\delta}_0(v) = \bar{\delta}(S_0, S_1, v).$$
(Recall $r(s,s')$ is the mean reward associated with the transition $(s,s')$; cf.\ Section~\ref{sec-gtdalg}. The above are temporal-difference terms without noises in the rewards.)
Recall that $\Pl$ is the substochastic matrix in the generalized Bellman operator $\Tl$ (cf.\ (\ref{eq-Tl0})), $\phi_i$ is the $i$-th component of the function $\phi$, and $\Phi_i$ the $i$-th column of the matrix $\Phi$.

%\smallskip
\begin{prop} \label{prop-mean-fn}
Under Assumption~\ref{cond-collective}, we have
\begin{align}
 \E_\zeta \big[ \phi(S_0) \phi(S_0)^\tr \big] & = \Phi^\tr \Xi \, \Phi, \label{mean-exp1} \\
\E_\zeta \big[ \e_0 \, \bar{\delta}_0(v) \big] & = \Phi^\tr \Xi \, (\Tl v - v),  \qquad \forall \, v \in \re^{|\S|}, \label{mean-exp2} \\
\E_\zeta \big[ \e_0 \cdot \rho_0 \big(\phi_i(S_0) - \gma_{1}  \phi_i(S_{1}) \big) \big] & = \Phi^\tr \Xi \, (I - \Pl) \Phi_i, \qquad  1 \leq i \leq d, \label{mean-exp3} \\
 \E_\zeta \big[ \e_0 \cdot \rho_0 (1 - \lambda_{1})  \gma_{1}  \phi_i(S_{1}) \big] & = \Phi^\tr \Xi \, \Pl \Phi_i, \qquad 1 \leq i \leq d. \label{mean-exp4}
\end{align}
\end{prop}
\smallskip

To prove this proposition, it is convenient to extend the stationary state-trace process whose time is indexed by $n \geq 0$, to a double-ended stationary state-trace process $\{(S_n, \e_n)\}$ or $\{(S_n, y_n, \e_n)\}$ with $- \infty < n < \infty$. Let $\Pr_\zeta$ denote the probability measure of the latter process. We shall keep using $\E_\zeta$ to denote expectation with respect to $\Pr_\zeta$.
The following lemma gives an expression of $\e_0$ in this stationary process. It will facilitate our calculation of $\E_\zeta [ f(Z_0) ]$ for various functions $f$ in the proposition.

Regarding notation in the lemma and in what follows, for $k \leq m$, let $\rho_k^m = \prod_{i=k}^m \rho_i$, $\lambda_k^m = \prod_{i=k}^m \lambda_i$, $\gamma_k^m = \prod_{i=k}^m  \gamma_i$, and in addition, adopt the convention that $\lambda_k^m=\gamma_k^m=\rho_k^{m}=1$ if $k > m$. 
Let $\1$ denote the $d$-dimensional vector of all $1$'s.

\smallskip
\begin{lem}[An expression for stationary traces] \label{lem-exp-e}  
Let Assumption~\ref{cond-collective} hold. Then $\Pr_\zeta$-almost surely, $\sum_{n=1}^{\infty} \lambda_{1-n}^0  \gamma_{1-n}^0 \rho_{-n}^{-1}  \fe(S_{-n})$ is well-defined and finite, and
\begin{equation} 
  \e_0 =  \fe(S_0) + \textstyle{ \sum_{n=1}^{\infty} \lambda_{1-n}^0  \gamma_{1-n}^0 \rho_{-n}^{-1} \, \fe(S_{-n})}. 
\label{eq-e}
\end{equation}
\end{lem}

\begin{proof}
For the case of history-dependent $\lambda$ we consider, this is proved in \cite[Lemma 3.1]{gbe_td17}.
We give the proof for the case of state-dependent $\lambda$. The beginning part of the proof is same as that of \cite[Lemma 3.1]{gbe_td17} and similar to that of \cite[Lemma 4.2]{Yu-siam-lstd} for the case of constant $\lambda$.
Under Condition~\ref{cond-pol}(i), $(I - \P \Gamma)^{-1} = \sum_{n=0}^\infty (\P \Gamma)^n$ and therefore,
$$ \textstyle{ \E_\zeta \big[ \sum_{n=1}^\infty \gamma_{1-n}^0 \rho_{-n}^{-1} \big]} =  \textstyle{ \sum_{n=1}^\infty \E_\zeta \big[ \gamma_{1-n}^0 \rho_{-n}^{-1} \big]} 
   =  \sum_{n=1}^\infty \xi^\tr\! (\P \Gamma)^n \1 < \infty
$$
where the first equality follows from the monotone convergence theorem, and the second equality follows from a direct calculation together with the fact that the marginal of $\zeta$ on $\S$ coincides with $\xi$, the unique invariant probability measure of $\{S_n\}$ under Condition~\ref{cond-pol}(ii).
Since $\lambda_{1-n}^0 \leq 1$ for all $n$, the above relation implies that 
\begin{equation} \label{eq-prfc0}
\textstyle{ \E_\zeta \big[ \sum_{n=1}^{\infty} \lambda_{1-n}^0  \gamma_{1-n}^0 \rho_{-n}^{-1} \, \| \fe(S_{-n}) \| \big]  \leq \max_{s \in \S} \| \fe(s) \| \cdot \E_\zeta \big[ \sum_{n=1}^{\infty}  \gamma_{1-n}^0 \rho_{-n}^{-1} \big] < \infty}.
\end{equation}
It then follows from a theorem on integration \cite[Theorem 1.38, p.\ 28-29]{Rudin66} that $\Pr_\zeta$-almost surely, the infinite series $\sum_{n=1}^{\infty} \lambda_{1-n}^0  \gamma_{1-n}^0 \rho_{-n}^{-1}  \fe(S_{-n})$ converges to a finite limit. 

We now prove (\ref{eq-e}). Since under Condition~\ref{cond-pol}(i), $\sum_{n=m}^\infty (\P \Gamma)^n$ converges to the zero matrix as $m \to \infty$, it follows from an argument very similar to the above that
\begin{equation} \label{eq-prfc1}
\textstyle{ \E_\zeta \left[ \big\| \sum_{n=m}^{\infty} \lambda_{1-n}^0  \gamma_{1-n}^0 \rho_{-n}^{-1} \, \fe(S_{-n}) \big\| \right]} \to 0, \qquad \text{as} \ m \to \infty.
\end{equation} 
Unfolding the iteration (\ref{eq-gtd-e}) for $\e_0$ backwards in time, we have that for all $m \geq 1$,
\begin{equation} 
 \textstyle{\e_0 = \fe(S_0) + \sum_{n=1}^{m-1} \lambda_{1-n}^0  \gamma_{1-n}^0 \rho_{-n}^{-1} \, \fe(S_{-n})  +  \lambda_{1-m}^0  \gamma_{1-m}^0 \rho_{-m}^{-1}  \,\e_{-m}.} \notag
\end{equation} 
Let $\hat \e_0$ denote the expression on the right-hand side (r.h.s.) of (\ref{eq-e}). Then, for all $m \geq 1$,
\begin{equation} \label{eq-prfc2}
 \E_\zeta \big[ \| \e_0 - \hat \e_0 \| \big] \leq \textstyle{ \E_\zeta \left[ \big\| \sum_{n=m}^{\infty} \lambda_{1-n}^0  \gamma_{1-n}^0 \rho_{-n}^{-1} \, \fe(S_{-n}) \big\| \right] + \E_\zeta \left[ \big\| \lambda_{1-m}^0  \gamma_{1-m}^0 \rho_{-m}^{-1}  \,\e_{-m} \big\| \right]}.
\end{equation} 
Let $m \to \infty$ in the r.h.s.\ of (\ref{eq-prfc2}). The first term converges to $0$ by (\ref{eq-prfc1}). For the second term, since $\lambda_{1-m}^0 \leq 1$, it is bounded above by
\begin{align*}
   \E_\zeta \left[ \big\|  \gamma_{1-m}^0 \rho_{-m}^{-1}  \,\e_{-m} \big\| \right] & = \E_\zeta \left[ \| \e_{-m} \| \cdot \E_\zeta [  \gamma_{1-m}^0 \rho_{-m}^{-1} \mid S_{-m}, \e_{-m} ]  \right] \\
   & \leq \E_\zeta \left[  \| \e_{-m} \| \cdot \1^\tr (P \Gamma)^m \1 \right] = \E_\zeta \left[ \| \e_{0} \| \right ] \cdot \1^\tr (P \Gamma)^m \1  \  \overset{m \to \infty}{\longrightarrow} \ 0,
\end{align*}
where the last equality uses the stationarity of the process, and the convergence to zero follows from the fact $\E_\zeta \left[ \| \e_{0} \| \right ] < \infty$ (Theorem~\ref{thm-erg}(ii)) and the fact that under Condition~\ref{cond-pol}(i), $(P \Gamma)^m$ converges to a zero matrix. Hence, letting $m \to \infty$ in the r.h.s.\ of (\ref{eq-prfc2}), we obtain 
$\E_\zeta \big[ \| \e_0 - \hat \e_0 \| \big] = 0$ and consequently, $\e_0 = \hat \e_0$, $\Pr_\zeta$-a.s., which proves (\ref{eq-e}).
\end{proof}

We now use the preceding lemma to calculate the expectations in Prop.~\ref{prop-mean-fn}.

\begin{proof}[Proof of Proposition~\ref{prop-mean-fn}]
Since $\xi$ is the invariant distribution of $\{\S_n\}$, the marginal of $\zeta$ on $\S$ coincides with $\xi$. Then (\ref{mean-exp1}) obviously holds. 
Using the expression of $\e_0$ given in Lemma~\ref{lem-exp-e}, the verifications of (\ref{mean-exp2})-(\ref{mean-exp4}) are similar and follow from direct calculations. 
We do this first for the case of state-dependent $\lambda$. Consider the double-ended stationary state-trace process.
First, let us calculate $\E_\zeta \big[ \e_0 \cdot \rho_0 f(S_0, S_1) \big]$ for an arbitrary bounded function $f$ on $\S^2$. 
We have
\begin{align}
 \E_\zeta \big[ \e_0 \cdot \rho_0 f(S_0, S_1) \big] & = \textstyle{  \sum_{n=0}^\infty \E_\zeta \Big[ \lambda_{1-n}^0  \gamma_{1-n}^0 \rho_{-n}^{-1} \, \fe(S_{-n}) \cdot \rho_0 f(S_0, S_1) \Big]} \notag \\
     & = \textstyle{   \sum_{n=0}^\infty \E_\zeta \Big[ \lambda_{1}^n  \gamma_{1}^n \rho_{0}^{n-1} \fe(S_0) \cdot \rho_n  f(S_n, S_{n+1}) \Big]} \notag  \\
     & =  \textstyle{  \sum_{n=0}^\infty \E_\zeta \Big[ \fe(S_0) \cdot \E_\zeta \left[ \lambda_{1}^n  \gamma_{1}^n \rho_{0}^{n} \,  f(S_n, S_{n+1}) \mid S_0 \right] \Big] },
\notag 
\end{align} 
where the first equality follows from applying Lemma~\ref{lem-exp-e} and then changing the order of expectation and summation (which, in view of (\ref{eq-prfc0}), is justified by the dominated convergence theorem), and the second equality uses the stationarity of the double-ended state-trace process. From the change of measure performed through $\rho_{0}^{n}$, we have
$$  \E_\zeta \left[ \lambda_{1}^n  \gamma_{1}^n \rho_{0}^{n} \,  f(S_n, S_{n+1}) \mid S_0 \right] = \E^\pi \left[ \lambda_{1}^n  \gamma_{1}^n f(S_n, S_{n+1}) \mid S_0 \right]$$
and therefore,
\begin{align} 
  \E_\zeta \big[ \e_0 \cdot \rho_0 f(S_0, S_1) \big] & = \textstyle{ \sum_{s \in \S} \xi_s \fe(s) \cdot \sum_{n=0}^\infty \E^\pi \big[ \lambda_{1}^n  \gamma_{1}^n f(S_n, S_{n+1}) \mid S_0 = s\big]} \notag \\
  & = \textstyle{ \sum_{s \in \S} \xi_s \fe(s) \cdot  \E^\pi \big[  \sum_{n=0}^\infty \lambda_{1}^n  \gamma_{1}^n f(S_n, S_{n+1}) \mid S_0 = s\big]} \label{eq-prfcal1}
\end{align}  
where the change of the order of summation and expectation in (\ref{eq-prfcal1}) is justified by the dominated convergence theorem, since by Condition~\ref{cond-pol}(i), we have $\E^\pi \big[  \sum_{n=0}^\infty \lambda_{1}^n  \gamma_{1}^n \| f(S_n, S_{n+1}) \| \mid S_0 = s\big] \leq \| f\| \cdot \sum_{n=0}^\infty \1^\tr (P \Gamma)^n \1 < \infty$.

Equations (\ref{mean-exp2})-(\ref{mean-exp4}) then follow from (\ref{eq-prfcal1}) by specializing the function $f$ to the corresponding function in each case. 
In particular, for a given $v \in \re^{|\S|}$, corresponding to $\bar{\delta}_0(v)$ in (\ref{mean-exp2}), we let $f(s,s') = r(s, s') +  \gma(s') v(s') - v(s)$. Then by (\ref{eq-prfcal1}) and the expression of $\Tl v - v$ given in (\ref{eq-expT4}), we have $\E_\zeta \big[ \e_0 \cdot \rho_0 f(S_0, S_1) \big] = \Phi^\tr \Xi \, (\Tl v - v)$, verifying (\ref{mean-exp2}). We omit the similar calculation for (\ref{mean-exp4}). 
Equation~(\ref{mean-exp3}) is implied by (\ref{mean-exp2}), as can be seen by letting $r(s,s') \equiv 0$ and $v=\Phi_i$. 
This proves the proposition for the case of state-dependent $\lambda$.

In the case of history-dependent $\lambda$ we consider, (\ref{mean-exp2}) is proved in \cite[Theorem 3.2]{gbe_td17}. As just mentioned, (\ref{mean-exp3}) is implied by (\ref{mean-exp2}). Let us verify (\ref{mean-exp4}). In this case the same proof of \cite[Theorem 3.2]{gbe_td17} (which works with the double-ended stationary state-trace process $\{(S_n, y_n, \e_n)\}_{- \infty < n < \infty}$ and is similar to the preceding proof leading to (\ref{eq-prfcal1})) shows that for any bounded function $f$ on $\S$, 
$$ \E_\zeta \big[ \e_0 \cdot \rho_0 (1 - \lambda_{1}) f(S_1) \big] = \textstyle{ \E^{\pi}_\zeta \Big[ \fe(S_0) \cdot  \E^\pi_\zeta \big[ \sum_{n=0}^\infty  \lambda_{1}^{n}  \gamma_{1}^{n}  \cdot (1 - \lambda_{n+1}) f(S_{n+1}) \mid S_0 \big] \Big]},$$
where $\E^\pi_\zeta$ is expectation w.r.t.\ the process described immediately after (\ref{eq-expT3}), the same process w.r.t.\ which $\Tl$ is defined.
Letting $f(s) = \gamma(s) \phi_i(s)$ for each $i \leq d$, we then have
\begin{align*}
  \E_\zeta \big[ \e_0 \cdot \rho_0 (1 - \lambda_{1}) \, \gma_1 \phi_i(S_1) \big] 
  & = \textstyle{ \E^\pi_\zeta \Big[ \fe(S_0) \cdot  \E^\pi_\zeta \left[ \sum_{n=0}^\infty \lambda_{1}^{n}  \gamma_{1}^{n+1}  \cdot (1 - \lambda_{n+1}) \, \fe_i(S_{n+1}) \mid S_0 \right] \Big]}.
\end{align*}
On the other hand, from the expression of $\Tl$ in (\ref{eq-expT3}), we see that the substochastic matrix $\Pl$ in the affine operator $\Tl$ is given by
$$ (\Pl v)(s)  =  \textstyle{ \E^\pi_\zeta \left[ \sum_{n=0}^\infty \lambda_1^{n} (1 - \lambda_{n+1})  \cdot \gamma_1^{n+1} \, v(S_{n+1})  \mid S_0 = s\right]}, \quad s \in \S, \ \forall \, v \in \re^{|\S|}.$$
It then follows that 
$\E_\zeta \big[ \e_0 \cdot \rho_0 (1 - \lambda_{1})  \gma_{1}  \phi_i(S_{1}) \big] = \textstyle{ \sum_{s \in \S} \xi_s \fe(s) \cdot (\Pl \phi_i)(s)} = \Phi^\tr \Xi \, \Pl \Phi_i$. This proves the proposition for the case of history-dependent $\lambda$.
\end{proof}

\subsubsection{More properties of traces} \label{sec-more-traceproperty}
In the rest of this subsection, we state several more properties of the trace iterates, which will be needed in the subsequent convergence analysis.
Some of the properties given in the part (i) of the following proposition were in fact used to prove the ergodicity of the state-trace process given in Theorem~\ref{thm-erg}.

One important property of the trace iterates $\{\e_n\}$ is its uniform integrability---a set $\mathcal{X}$ of $\re^m$-valued random variables is said to be uniformly integrable if $\lim_{a \to \infty} \sup_{X \in \mathcal{X}} \E [ \| X\| \I( \| X\| > a) ] = 0$. This property is much stronger than and entails that $\sup_{ n \geq 0} \E [ \| \e_n \| ] < + \infty$ \cite[Theorem 10.3.5]{Dud02}. In the case of history-dependent $\lambda$, since $\{\e_n\}$ lies in a pre-determined bounded set by the choice of $\{\lambda_n\}$ (cf.\ Condition~\ref{cond-lambda}(ii)), trivially $\{\e_n\}$ is uniformly integrable. So the part (ii) of the proposition below concerns only the case of state-dependent $\lambda$. 

%\smallskip
\begin{prop} \label{prop-trace}
Under Assumption~\ref{cond-collective}, the traces $\{\e_n\}$ have the following properties:\vspace*{-3pt}
\begin{itemize}
\item[\rm (i)] Let $\{\e_n\}$ and $\{\hat \e_n\}$ be generated by the iteration (\ref{eq-gtd-e}), using the same trajectory of states $\{S_n\}$ and also the same initial memory state $y_0$ in the case of history-dependent $\lambda$, but with different initial $\e_0$ and $\hat \e_0$, respectively. Then $\e_n - \hat \e_n \asto 0$. Moreover, for any initial $\e_0, \hat \e_0$ in a compact set $E$, $\E [ \| \e_n - \hat \e_n \| ] \leq c_n$ for some constant $c_n$ that depends on the set $E$ and satisfies $c_n \to 0$ as $n \to \infty$.
\item[\rm (ii)] In the case of state-dependent $\lambda$, for any given initial $\e_0$, $\{\e_n \}$ is uniformly integrable.
\end{itemize}
\end{prop}

\begin{proof}
In (i), that $\e_n - \hat \e_n \asto 0$ is proved in \cite[Lemma 2.1]{gbe_td17} for the case of history-dependent $\lambda$, using the continuity condition in Condition~\ref{cond-lambda}(i), the supermartingale convergence theorem, and the fact that $(P \Gamma)^n$ converges to a zero matrix under Condition~\ref{cond-pol}(i). For state-dependent $\lambda$, it follows from the same argument (which is similar to the proof of \cite[Lemma 3.2]{Yu-siam-lstd} for constant $\lambda$).
The proof, for both cases of $\lambda$, also shows that for all $n$,
$$\E [ \| \e_n - \hat \e_n \| ] \leq \| \e_0 - \hat \e_0 \| \cdot \1^\tr\! (P \Gamma)^n \1.$$
Since $\lim_{n \to \infty} \1^\tr\! (P \Gamma)^n \1 = 0$ and $\sup_{\e, \hat \e \in E} \| \e - \hat \e\| < \infty$ if $E$ is compact, the constants $c_n$ in (i) can be chosen as $c_n = \sup_{\e, \hat \e \in E} \| \e - \hat \e\| \cdot \1^\tr\! (P \Gamma)^n \1$ for all $n$.
 
We now show (ii) is true. We can apply the same proof of a similar result \cite[Prop.\ 2(i), p.\ 25]{etd-wkconv}, which establishes the uniform integrability of the trace iterates produced by the ETD algorithm, provided that we can show the following approximation property of what we call truncated traces. (The proof of \cite[Prop.\ 2(i)]{etd-wkconv} uses only the approximation property just mentioned and does not care about how $\{\e_n\}$ is generated. This is why we can apply that proof exactly to the case considered here.) Specifically, for each integer $K \geq 1$, define truncated traces $\{\tilde \e_{n,K}\}$ as follows:
\begin{equation} \label{eq-trunc-e}
   \tilde \e_n = \e_n \ \ \ \text{if} \ n \leq K;  \qquad \tilde \e_{n,K} = \textstyle{ \sum_{i = 0}^K \lambda_{n-i+1}^n \gamma_{n-i+1}^n \rho_{n-i}^{n-1}  \phi(S_{n-i})} \ \ \ \text{if} \ n > K.
\end{equation}
By definition $\{\tilde \e_{n,K}\}$ lies in a bounded set depending on the initial $\e_0$ and $K$.%footnote starts
\footnote{To see this, notice that for $n > K$, the truncated traces $\tilde \e_{n,K}$ do not depend on the initial $\e_0$, and being the sum of $K+1$ functions of the states, they lie in a bounded set determined by $K$ (since the state space is finite). For all $n \leq K$, $\tilde \e_{n,K}$ also lies in a bounded set, which is determined by $K$ and the initial $\e_0$.}
%footnote ends
To use these bounded truncated traces for establishing the uniform integrability of the original traces $\{\e_n\}$, the approximation property we need is that given an initial $\e_0$, for each $K$, there is a constant $L_K$ that depends on $K$ and $\e_0$ such that
\begin{equation} \label{eq-trunc-e-appr}
   \textstyle{\sup_{n \geq 0} \E \big[ \| \e_n - \tilde \e_{n,K} \| \big] \leq L_K}, \quad \text{and} \ \ L_K \downarrow 0 \ \   \text{as} \ K \to + \infty
\end{equation}   
(i.e., $L_K$ decreases to 0 as $K \to \infty$). 
For $n > K$, with $c = \max \{ \| \e_0 \|,  \max_{s \in \S} \| \phi(s)\| \} < \infty$, we have
\begin{align*}
  \E \big[ \| \e_n - \tilde \e_{n,K} \| \big] & \leq   \sum_{i = K+1}^{n-1}  \E \big[  \lambda_{n-i+1}^n \gamma_{n-i+1}^n \rho_{n-i}^{n-1} \| \phi(S_{n-i}) \| \big] + 
\E \big[ \lambda_{1}^n \gamma_{1}^n \rho_{0}^{n-1} \| \e_0\| \big] \\
  & \leq c \cdot \sum_{i = K+1}^{n}  \E \big[ \gamma_{n-i+1}^n \rho_{n-i}^{n-1} \big] \leq c \cdot \sum_{i = K+1}^n \1^\tr\! (P \Gamma)^i \1 \leq c \cdot \sum_{i = K+1}^\infty \1^\tr\! (P \Gamma)^i \1.
\end{align*}
Now let $L_K$ be the constant in the last expression above. Then the bound in (\ref{eq-trunc-e-appr}) holds and moreover, since under Condition~\ref{cond-pol}(i) the matrix $\sum_{i = K+1}^\infty (P \Gamma)^i$ converges to a zero matrix as $K \to \infty$, we have $L_K \downarrow 0$ as $K \to \infty$ as required in (\ref{eq-trunc-e-appr}). Feeding this approximation property of truncated traces to the proof of \cite[Prop.\ 2(i), p.\ 25]{etd-wkconv}, we obtain the uniform integrability of trace variables stated in (ii).
\end{proof}

The next proposition, particularly, its part (i), will be important for convergence analysis. It will be used to show that the algorithms satisfy certain ``averaging conditions'' that will allow us to relate their average dynamics to respective mean ODEs. Part (i) is similar to \cite[Prop.\ 3]{etd-wkconv} for the ETD algorithm, and its proof is also similar, albeit simpler. Part (ii) will be used in the analysis of the two-time-scale algorithms at their fast time-scale. In the proposition, the notation $\lim_{m, n \to \infty}$ or $\lim_{m, n, j \to \infty}$ means that the limit is taken as $m, n$  or $m, n, j$ go to infinity in any possible way.

\begin{prop} \label{prop-trace-averaging}
Let Assumption~\ref{cond-collective} hold. Let $f(z)$ be a vector-valued function that is Lipschitz continuous in the trace variable $\e$, and let $\bar f = \E_\zeta [ f(Z_0) ]$. 
Then the following hold:\vspace*{-3pt}
\begin{itemize}
\item[\rm (i)] For each compact set $D \subset \Z$,
$$ \textstyle{ \lim_{m, n \to \infty} \frac{1}{m} \sum_{i=n}^{n+m-1} \big( f(Z_i) - \bar f \, \big) \I(Z_n \in D) = 0} \ \ \ \text{in mean}.$$
\item[\rm (ii)] For a given sequence $\{c_i\}$ of positive numbers with $c_i \to 0$ as $i \to \infty$, 
\begin{align*}
  \textstyle{ \lim_{m, n \to \infty} \frac{1}{m} \sum_{i=n}^{n+m-1} \E \big[ c_i \cdot \| f(Z_i) \|\big] = 0}, \qquad
   \textstyle{ \lim_{m, n, j \to \infty}  \frac{1}{m} \sum_{i=n}^{n+m-1} \E \big[c_j \cdot \| f(Z_i) \| \big] = 0}. 
\end{align*} 
\end{itemize}
\end{prop}

\begin{proof}
For $z$ to be in a given compact set $D \subset \Z$, there is only a finite number of values that the state and memory state components of $z$ can take (since the spaces of states and memory states are finite). 
This means that to prove (i) for each compact set $D$, it is sufficient to prove it for each $D$ of the form $\{(s, \e, s') \mid \e \in  E \}$ or $\{(s, y, \e, s') \mid \e \in  E \}$, where $E$ is an arbitrary compact set in $\re^d$, and $(s,s')$ or $(s,y,s')$ (depending on the case of $\lambda$) is an arbitrary, possible state transition or an arbitrary, possible triple of memory state and state transition. To simplify notation, consider such a pair $(s,s')$ or triple $(s,y,s')$, and denote it by $\bar z^o$. 
Denote $Z^o_i = (S_i, S_{i+1})$ for the case of state-dependent $\lambda$ and $Z^o_i = (S_i, y_i, S_{i+1})$ for the case of history-dependent $\lambda$. 
Then, as just discussed, to prove (i), it is sufficient to show that for an arbitrary compact set $E \subset \re^d$ and two arbitrary sequences of integers $m_j, n_j \to \infty$, 
\begin{equation} \label{eq-prfb01}
 \lim_{j \to \infty} \, \frac{1}{m_j} \sum_{i=n_j}^{n_j+m_j-1} \big( f(Z_i) - \bar f \, \big) \cdot \I\big(\e_{n_j} \in E, Z^o_{n_j} = \bar z^o \big) = 0  \ \ \ \text{in mean}.
\end{equation} 

To prove (\ref{eq-prfb01}), we first define, for each $j$, auxiliary trace variables $\e_i^j, i \geq n_j$, to decompose the difference term $f(Z_i) - \bar f$ in (\ref{eq-prfb01}). Specifically, we define $\e_i^j$ as follows:\vspace*{-3pt}
\begin{enumerate}
\item[(a)] Fix a point $\bar e \in E$. 
\item[(b)] For each $j \geq 1$, define a sequence of traces, $\e_i^{j}$, $i \geq n_j$, by using the same recursion that defines the traces $\{\e_i\}$, based on the same trajectory of $\{S_i\}$ or $\{(S_i, y_i)\}$ (for each case of $\lambda$, respectively), but starting at the iteration $n_j$ with the initial $\e_{n_j}^{j} = \bar \e$.\vspace*{-3pt}  
\end{enumerate}
Let $Z_i^j =(S_i, e_i^j, S_{i+1})$ or $Z_i^j =(S_i, y_i, e_i^j, S_{i+1})$ (for each case of $\lambda$, respectively). We write the summation in (\ref{eq-prfb01}) as
$$\sum_{i=n_j}^{n_j+m_j-1} \big( f( Z_i) - \bar f \big) = \sum_{i=n_j}^{n_j+m_j-1} \big( f(Z^j_i) - \bar f \, \big) + \sum_{i=n_j}^{n_j+m_j-1} \big( f(Z_i) - f(Z^j_i) \big).$$
Equation~(\ref{eq-prfb01}) will follow immediately if we show that
\begin{align} 
 & \lim_{j \to \infty} \, \frac{1}{m_j} \sum_{i=n_j}^{n_j+m_j-1} \big( f(Z^j_i) - \bar f \, \big) \cdot \I\big(\e_{n_j} \in E, Z^o_{n_j} = \bar z^o\big) = 0  \ \ \ \text{in mean}, \label{eq-prfb02} \\
 & \lim_{j \to \infty} \, \frac{1}{m_j} \sum_{i=n_j}^{n_j+m_j-1} \big( f(Z_i) - f(Z^j_i) \big) \cdot \I\big(\e_{n_j} \in E, Z^o_{n_j} = \bar z^o\big) = 0  \ \ \ \text{in mean}. \label{eq-prfb03}
\end{align}

To prove (\ref{eq-prfb02}), note that by construction, for all $j$, conditioned on the event $Z^o_{n_j} = \bar z^o$, the process $\{Z^j_{n_j + i}\}_{ i \geq 0}$ has the same probability distribution: it is just the distribution of the process $\hat{Z}_i = (\hat S_i, \hat \e_i, \hat S_{i+1}), i \geq 0,$ or the process $\hat{Z}_i = (\hat S_i, \hat y_i, \hat \e_i, \hat S_{i+1}), i \geq 0$ (depending on the case of $\lambda$), for a state-trace process $\{(\hat S_i, \hat \e_i)\}$ or $\{(\hat S_i, \hat y_i, \hat \e_i)\}$ that starts from $\hat \e_0 = \bar e$ and $(\hat S_0, \hat S_1) = \bar z^o$ or $(\hat S_0, \hat y_0, \hat S_1) = \bar z^o$ (depending on the case of $\lambda$). Therefore, the convergence-in-mean stated in (\ref{eq-prfb02}) is equivalent to 
\begin{equation} \label{eq-prfb04}
  \textstyle{ \lim_{m \to \infty} \, \frac{1}{m} \sum_{i=0}^{m-1} \big( f(\hat Z_i) - \bar f \, \big) = 0} \ \ \ \text{in mean}.
\end{equation}  
Since the function $f(z)$ is Lipschitz continuous in the trace variable $\e$, (\ref{eq-prfb04}) is true by Theorem~\ref{thm-erg}(ii). Therefore, (\ref{eq-prfb02}) holds. 

We now prove (\ref{eq-prfb03}).
Again, by the Lipschitz continuity of $f(z)$ in the trace variable $\e$, it is sufficient to show 
\begin{equation} \label{eq-prfb05}
\lim_{j \to \infty} \frac{1}{m_j} \sum_{i=n_j}^{n_j+m_j-1}  \big\| \e_i - \e^j_i \big\| \cdot \I\big(\e_{n_j} \in E, \, Z^o_{n_j} = \bar z^o\big) = 0  \ \ \ \text{in mean}.
\end{equation} 
By Prop.~\ref{prop-trace}(i), for the given compact set $E$, starting from time $n_j$ and initial $\e_{n_j}, \e^j_{n_j} \in E$, we have $\E \big\| \e_{n_j + i} - \e^j_{n_j + i} \big\| \leq c_i$ for some constant $c_i$ that satisfies $c_i \to 0$ as $i \to \infty$. This implies (\ref{eq-prfb05}), and hence (\ref{eq-prfb03}) also holds. We have thus proved (i).

To prove (ii), note that since $f(z)$ is Lipschitz continuous in the trace variable $\e$, there is some constant $\ell$ such that $\| f(z) \| \leq \ell (1 + \| \e\|)$ for all $z$. Hence, $\E \big[ \| f(Z_i) \| \big] \leq  \ell + \ell \, \E \big[ \| \e_i \| \big]$ for all $i \geq 0$. 
We also have $\sup_{i \geq 0} \E \big[ \| \e_i\| \big] < \infty$ by Prop.~\ref{prop-trace}(ii), and $\lim_{i \to \infty} c_i = 0$ by assumption. 
These facts together clearly imply (ii). 
\end{proof}
%\vspace*{2pt}

\begin{rem}[About the Lipschitz continuity condition on $f(\cdot)$] \rm \label{rmk-Lipschitzcont}
Note that for a bounded measurable function $f$, Prop.~\ref{prop-trace-averaging}(ii) holds trivially. For Prop.~\ref{prop-trace-averaging}(i), however, our proof (in particular, the proof of (\ref{eq-prfb03})) needs the function $f$ to be Lipschitz continuous in the trace variable, even if $f$ is bounded and continuous, and it is not clear to us if this Lipschitz continuity condition can be relaxed. 
\end{rem}

\section{Convergence Analyses I} \label{sec-3}

This section starts our convergence analysis of several gradient-based TD algorithms. 
The first two subsections are for the two-time-scale GTDa and GTDb, respectively. In the subsequent Section~\ref{sec-4}, we will analyze several other algorithms, including single-time-scale algorithms.
The purposes of this section are: (i) to present specific convergence results for GTD (Sections~\ref{sec-gtd1},~\ref{sec-gtd2}); (ii) to discuss a way to ``robustify'' the algorithms against the high-variance issue in off-policy learning (Section~\ref{sec-bias-vrt}), and to discuss a composite scheme of setting the $\lambda$-parameters that extends the ones given in Section~\ref{sec-choice-lambda} (Section~\ref{sec-cp-extension}), both of which can also be applied to other algorithms later; and (iii) to show the overall proof structure and the main proof arguments, so that when we analyze other algorithms, we can focus on their differences from GTD, without repeating similar proofs. 

In this and next sections, we use the weak convergence method from stochastic approximation theory \cite[Chap.\ 8]{KuY03} to analyze the algorithms for a broad range of stepsizes. The advantage of the weak convergence method is that the conditions involved are much weaker than those for almost sure convergence. As a result, one can characterize the asymptotic behavior of the algorithms for slowly diminishing stepsizes and constant stepsizes, which are often used in practice. In comparison, almost sure convergence, which we will address in Section~\ref{sec-5}, requires more stringent conditions, and the range of stepsizes for convergence is narrower---especially narrow for the case of state-dependent $\lambda$ where the trace iterates can be unbounded and have unbounded variances.

In most part of this paper, we will analyze constrained versions of the algorithms, where the iterates are confined in some bounded sets. This is not a serious limit in practice if the constraint sets are large enough; in addition, constraints are certainly needed if constant stepsizes are used.

We find it better to state the convergence theorems close to where they are finally proved. 
For quick access to the convergence results of this section, we list them here:
\vspace*{-0.1cm}
\begin{itemize}
\item GTDa: Theorem~\ref{thm-gtd1-wk-dim} (for diminishing stepsize), Theorem~\ref{thm-gtd1-wk-constant} (for constant stepsize), and  Theorem~\ref{thm-gtd1-average} (for averaged iterates in the constant-stepsize case);
\item GTDb: Theorem~\ref{thm-gtd2-wk};
\item ``robustified'' biased variants of GTD: Theorem~\ref{thm-biasgtd1-wk};
\item GTD with a composite scheme of setting $\lambda$: Theorem~\ref{thm-composite-gtd}.
\end{itemize}
We remark that for the (constrained) GTDa and GTDb algorithms, the convergence theorems we present in this section have the same conclusions and are proved essentially with the same arguments. However, in the case of GTDa, these theorems can be improved further by using a minimax problem formulation, and we will address that in Section~\ref{sec-gtda-revisit} (see Theorem~\ref{thm-gtda-relaxBx}) when we have all the tools needed.

\subsection{GTDa} \label{sec-gtd1}

In this subsection we analyze a constrained version of GTDa under Assumption~\ref{cond-collective}:
\begin{align}
   \theta_{n+1} & = \Pi_{B_\theta} \Big( \theta_n + \alpha_n \, \rho_n \big(\phi(S_n) - \gma_{n+1} \phi(S_{n+1}) \big) \cdot \e_n^\tr x_n \Big),  \label{eq-gtd1a-th}\\
   x_{n+1} & = \Pi_{B_x} \Big( x_n + \beta_n \big(\e_n \delta_n(v_{\theta_n}) - \phi(S_n) \phi(S_n)^\tr x_n\big) \Big), \label{eq-gtd1a-x}
\end{align} 
where we take $B_\theta$ and $B_x$ to simply be closed balls in $\re^d$ centered at the origin and with large radii, and $\Pi_{B_\theta}$ and $\Pi_{B_x}$ denote the projection on the constraint sets $B_\theta$ and $B_x$, respectively, w.r.t.\ $\|\cdot\|_2$ (the usual Euclidean norm). We will later impose a condition on $B_x$ (cf.\ (\ref{eq-cond-Bx0})) to ensure that it is large enough so that the algorithm can carry out gradient-descent to minimize $J(\theta)$ over $B_\theta$. A much weaker condition on $B_x$ also works for GTDa (under which GTDa computes gradients of a function different from but related to $J$) and will be discussed in Section~\ref{sec-gtda-revisit}.

We choose the constraint sets this way because of a nice property: the projection operations will not take the iterates outside the subspace $\sp \{\phi(\S)\}$ if the iterates are initially in this subspace. This is due to the nature of TD algorithms, as can be see from the above formulae and the formula (\ref{eq-gtd-e}) for calculating traces:
$$ \e_{n} = \lambda_n \gma_n \rho_{n-1} e_{n-1} + \phi(S_n).$$
It is especially convenient to have the $x$-iterates lying in $\sp \{\phi(\S)\}$: It avoids all the nuisances due to the possible linear dependence between the component functions of the feature map $\phi$, so we can focus the convergence analysis on issues that are important.

The GTDa algorithm (\ref{eq-gtd1a-th})-(\ref{eq-gtd1a-x}) is a two-time-scale algorithm. Its $\theta$-iterates vary at a slow time-scale determined by the stepsizes $\{\alpha_n\}$, while its $x$-iterates vary at a fast time-scale determined by the stepsizes $\{\beta_n\}$. We will consider two types of stepsizes: diminishing stepsize and constant stepsize. The latter refers to the case where $\alpha_n = \alpha$ and $\beta_n = \beta$ for all $n$, with $\alpha < < \beta$. For diminishing stepsizes, we impose the following condition:

\begin{assumption}[Condition on diminishing stepsizes for two-time-scale algorithms] \label{cond-large-stepsize} \hfill \\
The (deterministic) nonnegative sequences $\{\alpha_n\}$ and $\{\beta_n\}$ satisfy that 
$$\sum_{n \geq 0} \alpha_n = \sum_{n \geq 0} \beta_n =\infty, \quad \alpha_n,\, \beta_n \to 0 \ \  \text{and} \ \ \alpha_n/\beta_n \to 0  \ \text{as} \ n \to \infty.$$ 
Moreover, for some sequences of integers $m_n \to \infty$ and $m'_n \to \infty$,
\begin{equation} \label{cond-w-stepsize}
\lim_{n \to \infty} \, \sup_{0 \leq j \leq m_n} \left| \frac{\alpha_{n+j}}{\alpha_n} - 1 \right| = 0, \qquad \lim_{n \to \infty} \, \sup_{0 \leq j \leq m'_n} \left| \frac{\beta_{n+j}}{\beta_n} - 1 \right| = 0.
\end{equation}
\end{assumption}
\smallskip

The condition $\alpha_n/\beta_n \to 0$ above shows the separation between the time-scale of the $\theta$-iterates and that of the $x$-iterates. The condition (\ref{cond-w-stepsize}) above allows large stepsizes that diminish very slowly, such as $n^{-c}, c \in (0,1]$. This stepsize condition is used in the book \cite[Chap.\ 8]{KuY03} to analyze stochastic approximation algorithms with the weak convergence method; the results from that book will be used in this section.

The structure of the proof is the same for the two-time-scale GTDa, GTDb, and other two-time-scale algorithms that we will consider later. We first use the ergodicity of the state-trace process to identify the desired mean ODEs associated with the algorithms. We then use the ergodicity property and other properties of the trace iterates given in Section~\ref{sec-property-stproc} to show that the conditions of the analyses in \cite[Chap.\ 8]{KuY03} (for diminishing as well as constant stepsizes) are satisfied by each algorithm. The theorems in \cite[Chap.\ 8]{KuY03} ensure that the asymptotic behavior of the algorithm is characterized by the asymptotic behavior of the solutions of its associated mean ODE. 
There are actually two mean ODEs associated with a two-time-scale algorithm, one for each time-scale. 
We apply the theorems in \cite[Chap.\ 8]{KuY03} for single-time-scale algorithms twice to analyze the behavior of the algorithm at each time-scale, feeding the result on the fast time-scale to the analysis of the slow time-scale.
For constrained algorithms, almost no conditions are imposed on the mean ODE for the slow time-scale. This allows us to establish the connection between the behavior of an algorithm and the solutions of its slow-time-scale ODE, without worrying about what those ODE solutions actually are. We can study the ODE separately to get its solution properties and translate them into the convergence properties of the algorithm. This last proof step is simpler and quicker to carry out, so when presenting our proofs, we will first tackle this step by discussing the desired mean ODEs and deriving their solution properties.

This subsection is relatively long. After we explain the main proof arguments for GTDa in this subsection, we will apply the same arguments to other algorithms in later subsections, so they will be much shorter than this one.

\subsubsection{Associated mean ODEs and their solution properties} \label{sec-gtd1-odes}

We first rewrite the GTDa algorithm in a form that is more concise and convenient for analysis.
Recall we defined the random variable $Z_n = (S_n, \e_n, S_{n+1})$ for state-dependent $\lambda$ and $Z_n = (S_n, y_n, \e_n, S_{n+1})$ for history-dependent $\lambda$. From now on, let us write $Z_n = (S_n, y_n, \e_n, S_{n+1})$ in both cases, treating the memory states $y_n$ as dummy variables in the case of state-dependent $\lambda$.
Let $\Z$ denote the space of $z$ as before. 
Let $\omega_{n+1} = \rho_n \big(R_{n+1} - r(S_n, S_{n+1}) \big)$, which is a noise term in the observed one-stage reward and which, given the history $(Z_0, \ldots, Z_{n})$ and $(R_1, \ldots, R_{n})$, has zero-mean.
We rewrite the GTDa algorithm (\ref{eq-gtd1a-th})-(\ref{eq-gtd1a-x}) as
\begin{align}
     \theta_{n+1} & = \Pi_{B_\theta} \big( \theta_n + \alpha_n \, g(\theta_n, x_n, Z_n) \big), \label{eq-gtd1-altth}\\
      x_{n+1} & = \Pi_{B_x} \big( x_n + \beta_n (k(\theta_n, x_n, Z_n) + \e_n \omega_{n+1} ) \big), \label{eq-gtd1-altx}
\end{align}
where with $z = (s, y, \e, s')$, the functions $g(\theta, x, z)$ and $k(\theta, x, z)$ are given by
\begin{equation} \label{eq-gtd1-gk}
   g(\theta, x, z) =  \rho(s, s') \big(\phi(s) - \gma(s') \phi(s') \big) \cdot \e^\tr\! x, \qquad k(\theta, x, z) = \e \, \bar{\delta}(s, s', v_{\theta}) - \phi(s) \phi(s)^\tr x.
\end{equation} 
Recall that $\bar{\delta}(s, s', v_{\theta})$ above is the temporal-difference term with the noise in the reward subtracted: 
$\bar{\delta}(s, s', v_{\theta}) =  \rho(s) \big( r(s, s') +  \gma(s') v_\theta(s') - v_\theta(s) \big)$, 
and we use the shorthand $\bar \delta_0(v_\theta)$ for $\bar{\delta}(S_0, S_1, v_{\theta})$.

As mentioned earlier, there are two ODEs associated with the algorithm, one for each time-scale. The slow-time-scale one depends on the solution to the fast-time-scale one. Let us work out what these ODEs are---at this stage they are only desired ODEs that we wish to associate with the algorithm, for their real connection with the algorithm can only be established after we carry out the analysis using stochastic approximation theory in the following subsections.

For the fast time-scale, by Prop.~\ref{prop-mean-fn}, for each fixed $(\theta, x)$, 
\begin{align}
  \bar k(\theta, x) : = \E_\zeta \big[ k(\theta, x, Z_0) \big] & = \E_\zeta \big[ \e_0 \, \bar{\delta}_0(v_\theta) \big] - \E_\zeta \big[ \phi(S_0) \phi(S_0)^\tr \big] \cdot x \notag \\
    & = \Phi^\tr \Xi \, (\Tl v_\theta - v_\theta) - \Phi^\tr \Xi \, \Phi \, x.  \label{eq-gtd1-bk}
\end{align}
For the function $g$ in the slow-time-scale iteration, by (\ref{mean-exp3}) in Prop.~\ref{prop-mean-fn}, for fixed $(\theta, x)$,
\begin{align}   
 \bar g(\theta, x): = \E_\zeta \big[ g(\theta, x, Z_0) \big] & = \E_\zeta \big[  \rho_0 \big(\phi(S_0) - \gma_{1}  \phi(S_{1}) \big) \cdot \e_0^\tr x \big] \notag \\
       & = (\Phi^\tr \Xi \, (I - \Pl) \, \Phi)^\tr x.  \label{eq-gtd1-meang}  
\end{align}
Recall from (\ref{eq-xtheta})-(\ref{eq-xtheta2}) that for any given $\theta$, there is a unique solution $x_\theta$ to the linear equation
\begin{equation} \label{eq-xtheta3}
  \bar k(\theta, x) = 0, \ \ \ x \in \sp \{\phi(\S)\}. 
\end{equation}  
Let us define a function $\bar x(\theta) = x_\theta$. 
By (\ref{eq-gtd1-meang}), 
\begin{equation}
\bar g(\theta, \bar x(\theta))  = \E_\zeta \big[ g(\theta, \bar x(\theta), Z_0) \big] = (\Phi^\tr \Xi \, (I - \Pl) \, \Phi)^\tr x_\theta =: \bar g(\theta),   \label{eq-gtd1-bg}
\end{equation}
where we have let $\bar g(\theta)$ stand for $\bar g(\theta, \bar x(\theta))$ in order to simplify notation for the subsequent analysis. 
From the expression (\ref{eq-Jgrad1}) of the gradient $\nabla J(\theta)$, we see that
$$\bar g(\theta) = - \nabla J(\theta).$$ 

At the fast time-scale determined by the stepsizes $\{\beta_n\}$, we can view the algorithm (\ref{eq-gtd1-altth})-(\ref{eq-gtd1-altx}) as the single-time-scale algorithm,
\begin{align}
     \theta_{n+1} & = \Pi_{B_\theta} \big( \theta_n + \beta_n \cdot (\alpha_n/\beta_n) \, g(\theta_n, x_n, Z_n) \big), \label{eq-gtd1-altth0}\\
      x_{n+1} & = \Pi_{B_x} \big( x_n + \beta_n (k(\theta_n, x_n, Z_n) + \e_n \omega_{n+1} ) \big). \label{eq-gtd1-altx0}
\end{align}
Since $\alpha_n < < \beta_n$, we wish the terms $(\alpha_n/\beta_n) \, g(\theta_n, x_n, Z_n)$ are negligible, and we want to relate the average dynamics of the algorithm to the mean ODE,
\begin{equation} \label{eq-gtd1-odefast}
\qquad \left( \! \begin{array}{c} \dot{\theta}(t) \\ \dot x(t) \end{array} \!\right) 
 =  \left( \! \begin{array}{l} 0 \\
 \bar k \big(\theta(t), x(t) \big) + z(t)  \end{array} \!\right), \qquad \theta(0) \in B_\theta, \  z(t) \in - \N_{B_x}(x(t)),
\end{equation} 
where $\N_{B_x}(x(t))$ is the normal cone of $B_x$ at $x(t)$, and $z(t)$ is the boundary reflection term. 
Specifically, if $x(t)$ is in the interior of $B_x$, the reflection term $z(t)$ is zero. 
If $x(t)$ is on the boundary of $B_x$, $z(t)$ cancels out the projection of $\bar k \big(\theta(t), x(t) \big)$ on $\N_{B_x}(x(t))$; namely, 
$$ z(t) = - \Pi_D \big( \bar k (\theta(t), x(t)) \big) \qquad \text{with} \ D = \N_{B_x}(x(t))$$
(the projection is w.r.t.\ $\| \cdot\|_2$).
This boundary reflection term is the ``minimal force'' needed to keep the solution $x(\cdot)$ in the constraint set $B_x$ (cf.\ \cite[Chap.\ 4.3]{KuY03}).

At the slow time-scale determined by the stepsize $\{\alpha_n\}$, we can view the $\theta$-iteration (\ref{eq-gtd1-altth}) as the single-time-scale iteration:
\begin{equation} \label{eq-gtd1-thnoise}
  \theta_{n+1}  = \Pi_{B_\theta} \big( \theta_n + \alpha_n (g(\theta_n, \bar x (\theta_n), Z_n)  + \Delta_n ) \big),  \quad \text{for} \ \ \Delta_n  = g(\theta_n, x_n, Z_n) - g(\theta_n, \bar x (\theta_n), Z_n).
\end{equation}
We wish that $x_n$ ``tracks'' $\bar x(\theta_n)$ so that the $\Delta_n$ terms have negligible effects on the average dynamics of the $\theta$-iterates, thereby allowing us to associate the following mean ODE with the $\theta$-iteration above: 
\begin{equation} \label{eq-gtd1-odeslow}
  \dot \theta(t) = \bar g\big(\theta(t)\big) + z(t),  \quad z(t) \in - \N_{B_\theta}\big(\theta(t)\big),
\end{equation}
where $\N_{B_\theta}(\theta(t))$ is the normal cone of $B_\theta$ at $\theta(t)$, and $z(t)$ is the boundary reflection term; i.e., $z(t)$ equals minus the projection of $\bar g(\theta(t))$ on $\N_{B_\theta}(\theta(t))$. 

As mentioned earlier, if we can establish the connection between the algorithm and these two ODEs, then the asymptotic properties of the solutions to these ODEs will tell us the asymptotic behavior of the algorithm. So let us examine first the solution properties of these two ODEs.

For the fast-time-scale ODE (\ref{eq-gtd1-odefast}), with $\theta(\cdot) \equiv \theta \in B_\theta$, note that the solution $x(\cdot)$ stays in the subspace $\sp \{\phi(\S)\}$ if $x(0) \in \sp \{\phi(\S)\}$,%footnote starts
\footnote{We have $\bar k(\theta, x) \in \sp \{\phi(\S)\}$ (cf.\ the expression (\ref{eq-gtd1-bk}) of $\bar k$). Therefore, its projection on the ball $B_x$ lies in $\sp \{\phi(\S)\}$ and so does its boundary reflection term. This shows $x(\cdot)$ stays in $\sp \{\phi(\S)\}$ if $x(0) \in \sp \{\phi(\S)\}$.} 
%footnote ends
and we wish $x(t)$ to converge to $\bar x(\theta)=x_\theta$ as $t \to \infty$ (so that there is hope for the iterates $x_n$ to track $\bar x (\theta_n)$ as desired for the slow-time-scale $\theta$-iteration).
For this we need $B_x$ to be large enough so that%footnote starts
\footnote{A sufficient condition for (\ref{eq-cond-Bx0}) is that 
$\text{radius}(B_x) \geq \textstyle{\sup_{\theta \in B_\theta}} \big\| \Phi^\tr \Xi \, (\Tl v_\theta - v_\theta) \big\|_2 /c$, where $c$ is the smallest nonzero eigenvalue of the matrix $\Phi^\tr \Xi \,\Phi$. 
}
%footnote ends 
\begin{equation} \label{eq-cond-Bx0}
B_x \supset \{x_\theta \mid \theta \in B_\theta\}. 
\end{equation}
Now notice that for each $\theta$, $\bar k(\theta, x) = - \nabla f(x)$ for the convex quadratic function $f(x) =  \tfrac{1}{2} x^\tr C x -  x^\tr b$, where $C =  \Phi^\tr \Xi \,\Phi$ and $b = \Phi^\tr \Xi \, (\Tl v_\theta - v_\theta)$, and $f(x)$ attains its minimum at $x_\theta$. So the fast-time-scale ODE is simply solving the problem $\inf_{x \in B_x} f(x)$ with gradient descent, and it is straightforward to derive its solution properties stated in the following lemma. (We defer the proof details until after we have discussed the slow-time-scale ODE, which has a similar analysis.)

Let $\cl \, D$ denote the closure of a set $D$. The limit set of the ODE (\ref{eq-gtd1-odefast}) is by definition the set
$$\bigcap_{ \bar t \geq 0} \, \cl \,\Big\{ \big(\theta(t), x(t)\big) \, \big| \,    \theta(0) \in B_\theta, \, x(0) \in B_x, \, t \geq \bar t  \, \Big\}$$
where each set in the intersection is the closure of the set of points that can be reached after time $\bar t$ by a solution $(\theta(\cdot), x(\cdot))$ of the ODE from some initial condition in $B_\theta \times B_x$. Since we are interested only in those solutions with $x(0), x(t) \in \sp \{\phi(\S)\}$ for all $t$, we focus on a smaller limit set corresponding to those solutions, which is  the set $\cap_{ \bar t \geq 0} \, \cl \,\big\{ \big(\theta(t), x(t)\big) \, \big| \,    \theta(0) \in B_\theta, \, x(0) \in B_x \cap \sp \{\phi(\S)\}, \, t \geq \bar t  \, \big\}$. 

\begin{lem} \label{lem-Bx}
Let the constraint set $B_x$ satisfy (\ref{eq-cond-Bx0}). Then for all initial $x(0) \in B_x \cap \sp \{\phi(\S)\}$ and $\theta(0) \in B_\theta$,
the limit set of the ODE (\ref{eq-gtd1-odefast}) is $\big\{ \big(\theta, \bar x(\theta)\big) \mid \theta \in B_\theta \big\}$.
\end{lem}

Now for the slow time-scale, consider the desired mean ODE (\ref{eq-gtd1-odeslow}).
Let $\Theta_{\text{opt}}$ denote the compact set of optimal solutions of $\inf_{\theta \in B_\theta} J(\theta)$:
\begin{equation} \label{eq-gtd-opt}
  \Theta_{\text{opt}} = \argmin_{\theta \in B_\theta} J(\theta).
\end{equation}  

\begin{lem} \label{lem-gtd1-odeslow}
The limit set of the ODE (\ref{eq-gtd1-odeslow}) is
$\bigcap_{\, \bar t \geq 0}  \cl \, \big\{\theta(t)  \, \big| \,   \theta(0) \in B_\theta, \, t \geq \bar t  \, \big\}  = \Theta_{\text{\rm opt}}$.
\end{lem}

To prove the lemma, we need the following facts concerning convex sets, which will also be useful in later sections:

\begin{lem} \label{lem-decomp-vec-cones}
Let $D$ be a closed convex set in $\re^d$ and let $\theta \in D$. For any $h \in \re^d$, let $y$ be the projection of $h$ on $\N_{D}(\theta)$, the normal cone of $D$ at $\theta$. Then $y- h$ is the point in the convex set $\N_{D}(\theta)- h$ that has the minimal $\|\cdot\|_2$-norm. Moreover, $h - y$ lies in the tangent cone of $D$ at $\theta$, and $\langle y, h - y \rangle = 0$.
\end{lem}

The first statement of Lemma~\ref{lem-decomp-vec-cones} is obvious: Since $y$ is the closest point in $\N_{D}(\theta)$ to $h$, by translation, $y - h$ is the closest point in $\N_{D}(\theta)- h$ to the origin. The second statement about $h - y$ comes from a property of projections on a polar pair of cones \cite[Ex.\ 12.22, p.\ 546]{RoW98}.%footnote starts
\footnote{This theorem asserts that given a closed convex cone $K$, $h$ can be decomposed uniquely as the sum of two orthogonal vectors, one in $K$ and the other in the polar of $K$, and these two vectors are given by the projections of $h$ on the two cones, respectively.}
%footnote ends

Here are the implications of Lemma~\ref{lem-decomp-vec-cones} for a solution $\theta(\cdot)$ of the ODE (\ref{eq-gtd1-odeslow}). Take $D = B_\theta$, $\theta=\theta(t)$, $h = \bar g(\theta(t))$, and $y = -z(t)$; recall that our boundary reflection term $z(t)$ corresponds precisely to $ - \Pi_{\N_D(\theta)} h$. Therefore, 
\begin{equation} \label{eq-prfe1}
     \langle z(t), \, \bar g(\theta(t) + z(t) \rangle = 0,
\end{equation}
and the vector  $- \bar g(\theta(t)) - z(t) =  \nabla J(\theta(t)) - z(t)$ is the point in $\nabla J(\theta(t)) + \N_{B_{\theta}}(\theta(t))$ that has the minimal $\|\cdot\|_2$-norm.
Notice that for $\theta \in B_{\theta}$, the set $\nabla J(\theta) + \N_{B_{\theta}}(\theta)$ consists of the subgradients $\partial f(\theta)$ of the convex function $f(\theta) = J(\theta) + \delta_{B_\theta}(\theta)$ at the point $\theta$ (here $\delta_{B_\theta}(\theta)$ denotes the indicator function that takes the value $0$ on $B_{\theta}$ and $+ \infty$ outside $B_{\theta}$). For the problem $\inf_\theta f(\theta)$ (which is equivalent to $\inf_{\theta \in B_\theta} J(\theta)$), at a non-optimal $\theta$, $0 \not \in \partial f(\theta)$, but unlike an arbitrary subgradient, the minimal-norm subgradient has the property that its negation always points in a descent direction of $f$. At an optimal $\theta$, $0 \in \partial f(\theta)$ by the optimality condition, and then, of course, the minimal-norm subgradient must be the zero vector. Thus we see that the path of a solution $\theta(\cdot)$ of (\ref{eq-gtd1-odeslow}) follows descent directions of the objective function $f$ that we want to minimize.

Let us now prove Lemma~\ref{lem-gtd1-odeslow} for the desired mean ODE at the slow time-scale. 

\begin{proof}[Proof of Lemma~\ref{lem-gtd1-odeslow}]
Since $\bar g(\theta) = - \nabla J(\theta)$, the function $V(\theta)= J(\theta)$ serves as a Lyapunov function for the ODE (\ref{eq-gtd1-odeslow}). In particular, we have
\begin{align}
  \dot V(\theta(t)) & = \big\langle \nabla J\big(\theta(t)\big), - \nabla J\big(\theta(t)\big) + z(t) \big\rangle \notag \\
  & = - \big\| \nabla J\big(\theta(t)\big) - z(t) \big\|^2_2 +  \big\langle z(t), - \nabla J\big(\theta(t)\big) + z(t) \big\rangle \notag \\
  & = - \big\| \nabla J\big(\theta(t)\big) - z(t) \big\|^2_2,  \label{eq-prfe2}
\end{align}  
where the last equality follows from (\ref{eq-prfe1}). As discussed before the proof, $\nabla J\big(\theta(t)\big) - z(t)$ is the minimal-norm subgradient of $f$ at $\theta(t)$. So, if $\theta(t) \in \Theta_{\text{opt}}$, $\dot V\big(\theta(t)\big) = 0$ by the optimality condition; if $\theta(t)  \not\in \Theta_{\text{opt}}$, then $\nabla J\big(\theta(t)\big) - z(t) \not=0$ and hence $\dot V\big(\theta(t)\big) < 0$. 
Moreover, for any $0 < \epsilon \leq \sup_{\theta \in B_\theta} J(\theta) - \inf_{\theta \in B_\theta} J(\theta)$, there exists $\eta_\epsilon > 0$ such that%footnote starts
\footnote{To derive (\ref{eq-prfe3}), consider the convex function $f(\theta) = J(\theta) + \delta_{B_\theta}(\theta)$ mentioned before the proof. As a set-valued mapping, the subgradient mapping $\partial f$ is outer semicontinuous and therefore its graph, $\text{gph} \, \partial f = \{(\theta, y) \mid y \in \partial f(\theta)\}$, is closed \cite[Theorem 5.7]{RoW98}. Then the set $\text{gph} \, \partial f \cap \big( \{ \theta \in B_\theta \mid J(\theta) \geq \inf_{\theta \in B_\theta} J(\theta) + \epsilon \} \times \re^d \big)$ is also closed, and on this set, the continuous function $(\theta, y) \mapsto \| y\|^2_2$ attains the minimal value, which must be nonzero (since no optimal $\theta$ is involved in this set). We can take this value to be $\eta_\epsilon$ in (\ref{eq-prfe3}).}
%footnote ends
\begin{equation} \label{eq-prfe3}
 \dot V\big(\theta(t)\big) \leq - \eta_\epsilon, \quad \text{if} \  J(\theta(t)) \geq \textstyle{\inf_{\theta \in B_\theta}} J(\theta) +\epsilon.
\end{equation} 
Since $V$ is bounded on $B_\theta$, this implies that there exists a time $t_\epsilon$ such that all solutions of (\ref{eq-gtd1-odeslow}) satisfy that 
$$J(\theta(t)) \leq \textstyle{ \inf_{\theta \in B_\theta}} J(\theta)  + \epsilon, \qquad \forall \, t \geq t_\epsilon.$$ 
Thus the limit set of the ODE must be $\cap_{\epsilon > 0} \{ \theta \in B_\theta \mid J(\theta) \leq \inf_{\theta \in B_\theta} J(\theta)  + \epsilon\} = \Theta_{\text{opt}}$.
\end{proof}

The proof of Lemma~\ref{lem-Bx} is similar:
 
\begin{proof}[Proof of of Lemma~\ref{lem-Bx}] 
For each $\theta \in B_\theta$, with $\theta(\cdot) \equiv \theta$, consider a solution $x(\cdot)$ of the ODE (\ref{eq-gtd1-odefast}) with $x(0) \in \sp \{\phi(\S)\}$. 
We can write $\bar k(\theta, x) = - C(x - x_\theta)$ for $C = \Phi^\tr \Xi \, \Phi$. 
Let $V_\theta(x) = \tfrac{1}{2} \| x - x_\theta \|_2^2$. Then with a calculation similar to (\ref{eq-prfe2}), we have 
\begin{align*}
  \dot V_\theta(x(t)) = \langle x(t) - x_\theta, \, - C(x(t) - x_\theta) + z(t) \rangle = - \langle x(t) - x_\theta, C(x(t) - x_\theta) \rangle + \langle x(t) - x_\theta, \, z(t) \rangle.
\end{align*}   
Since $x(t) \in \sp \{\phi(\S)\}$, the first term on the r.h.s.\ is bounded above by $- c \, \| x(t) - x_\theta \|_2^2$, where $c > 0$ is the smallest nonzero eigenvalue of the symmetric positive semidefinite matrix $C$. The second term on the r.h.s.\ is nonpositive, because $x_\theta \in B_x$ under our assumption and by the definition of the boundary reflection term, $ - z(t) \in \N_{B_x}(x(t))$, which implies $\langle x - x(t), - z(t) \rangle \leq 0$ for all $x \in B_x$.
Thus, we have $\dot V_\theta(x(t)) \leq - c \| x(t) - x_\theta \|_2^2$. Since on $B_x$, $\{V_\theta \mid \theta \in B_\theta\}$ are uniformly bounded, this implies that for any $\epsilon > 0$, there is a time $t_\epsilon$ independent of $\theta$ such that $\| x(t) - x_\theta \|_2^2  \leq \epsilon$ for all $t \geq t_\epsilon$. The conclusion of the lemma then follows. 
\end{proof}

We have now obtained solution properties of the desired mean ODEs. Next, we need to show that everything we ``wished'' in the above actually holds for the GTDa algorithm, so that the $\theta_n$ iterates produced by the algorithm will behave like the solutions $\theta(\cdot)$ of the slow-time-scale ODE and converge, in a sense to be made precise, to the optimal solution set $\Theta_{\text{opt}}$.

\subsubsection{Conditions to verify} \label{sec-gtd1-cond}

We use the weak convergence method from stochastic approximation theory to analyze the convergence properties of the algorithm for constant and slowly diminishing stepsizes. Specifically, in the terminology of \cite{KuY03}, our algorithm is of the type that involves ``exogenous noises,'' meaning that the evolution of $\{Z_n\}$ is not influenced by the iterates $\{(\theta_n, x_n)\}$ themselves. We will apply the theorems in \cite[Chap.\ 8]{KuY03} for this type of algorithms, and in this subsection, we first use the results of Section~\ref{sec-property-stproc} to verify several major conditions that will be needed when applying those theorems. 

We list here only the conditions for the case of diminishing stepsize. The conditions for the constant-stepsize case differ only in notation and will be addressed in Section~\ref{sec-gtd1conv-constant} when we analyze that case. 
We are going to show that the algorithm fulfills the following seven conditions. They are closely related to the conditions in \cite[Theorems 8.2.3 and 8.6.1]{KuY03}. (We will apply \cite[Theorems 8.2.3]{KuY03} twice, once for each time-scale, in our convergence proof.)\vspace*{-3pt}
\begin{itemize}
\item[(i)] The following sets of random variables are uniformly integrable:%footnote starts
\footnote{In this condition, the random variables in the first (third) set correspond to the terms that appear after the stepsizes $\alpha_n$ ($\beta_n$) in the $\theta$-iteration ($x$-iteration). For different algorithms, these two sets in the condition (i) can have different expressions; for example, for GTDb, the first set will also involve the noise terms $\e_n \omega_{n+1}$.} 
%footnote ends
$$\big\{g(\theta_n, x_n, Z_n) \big\}_{n\geq0}, \quad \big\{g(\theta, x, Z_n) \mid n \geq 0, (\theta, x) \in B_\theta \times B_x \big\},$$ 
$$\big\{k(\theta_n,x_n,Z_n)+\e_n\omega_{n+1}\big\}_{n\geq0}, \quad \big\{k(\theta, x, Z_n) \mid n \geq 0, (\theta, x) \in B_\theta \times B_x \big\}.$$
\item[(ii)] The set of random variables $\{Z_n\}_{n \geq 0}$ is tight.%footnote starts
\footnote{This means that for any $\epsilon > 0$, there exists a compact set $D$ such that $\inf_{n} \Pr(Z_n \in D) \geq 1 - \epsilon$.}
%footnote ends
\item[(iii)] The functions $g(\theta, x, z)$ and $k(\theta, x, z)$ are continuous in $(\theta, x)$ uniformly in $z \in D$, for any compact set $D \subset \Z$.
\item[(iv)] There exists a continuous function $\bar k(\theta, x)$ such that for each $(\theta, x) \in B_\theta \times B_x$, the convergence-in-mean stated below in Lemma~\ref{lem-conv-mean-cond1}(i) holds for $\{k(\theta, x, Z_n)\}$. Also, the convergence-in-mean stated in the first part of Lemma~\ref{lem-conv-mean-cond1}(iii) holds for $\{g(\theta, x, Z_n)\}$ and each $(\theta, x) \in B_\theta \times B_x$.
\item[(v)] 
There exists a continuous function $\bar x(\theta)$ such that $\{ (\theta, \bar x(\theta)) \mid \theta \in B_\theta \}$ is the limit set of the projected ODE $\dot \theta(t) = 0, \dot x(t) = \bar k(\theta(t), x(t)) + z(t)$, where $z(t) \in - \N_{B_x}(x(t))$ is the boundary reflection term, for all initial conditions $\theta(0) \in B_\theta$ and $x(0) \in B_x \cap \sp \{\phi(\S)\}$.%footnote starts
\footnote{Normally, this condition is required to hold for all initial $x(0) \in B_x$, but here we are interested only in those solutions from $x(0) \in \sp \{\phi(\S)\}$ and they all stay in the subspace $\sp \{\phi(\S)\}$ for the ODE under our consideration, as discussed earlier.}
%footnote ends
\end{itemize}\vspace*{-3pt}
More conditions for the algorithm at its slow time-scale:\vspace*{-3pt}
\begin{itemize}
\item[(vi)] 
The following sets of random variables are uniformly integrable:
$$\big\{g(\theta_n, \bar x(\theta_n), Z_n)\big\}_{n \geq 0}, \qquad \big\{ g(\theta, \bar x(\theta), Z_n) \mid n \geq 0, (\theta, x) \in B_\theta \times B_x \big\}.$$
\item[(vii)] For $\{g(\theta, \bar x(\theta), Z_n)\}$ and some continuous function $\bar g(\theta)$,  the convergence-in-mean stated below in Lemma~\ref{lem-conv-mean-cond1}(ii) holds for each $\theta \in B_\theta$.\vspace*{-3pt}
\end{itemize}

Among the above conditions, the continuity condition (iii) and the condition (v) on the fast-time-scale ODE are obviously satisfied: (iii) follows simply from the definitions of the functions $g$ and $k$ (cf.\ (\ref{eq-gtd1-gk})), and (v) follows from Lemma~\ref{lem-Bx} and the continuity of the function $\bar x(\theta) = x_\theta$. 
In the case of history-dependent $\lambda$, since the trace iterates lie in a given bounded set and hence $\{Z_n\}$ lie in a compact set,  the conditions (i)-(ii) also hold trivially. 
In the case of state-dependent $\lambda$, the tightness condition (ii) follows from the fact $\sup_{n \geq 0} \E [ \| \e_n\| ] < \infty$ (which is implied by Prop.~\ref{prop-trace}(ii)), together with the Markov inequality.  The uniform integrability required by the condition (i) is ensured by the uniform integrability of the traces $\{\e_n\}$ established in Prop.~\ref{prop-trace}(ii). The proof of this uses the uniform integrability of $\{\e_n\}$ together with the following simple facts: the iterates $\theta_n$ and $x_n$ are confined in bounded sets, the number of states is finite, the functions $g, k$ are Lipschitz continuous in the trace variable $\e$ uniformly w.r.t.\ $(\theta, x)$ in a bounded set and w.r.t.\ other components of $z$, and the noise terms $\{\omega_{n+1}\}$ have finite variances. We omit the proof because it is straightforward and the details are the same as those given for a similar result \cite[Prop.\ 2(ii)-(iv)]{etd-wkconv} in the analysis of the ETD algorithm with the weak convergence method.

We now address the two remaining convergence-in-mean conditions (iv) and (vii), which are stated in the lemma below. The functions $\bar k(\theta, x)$, $\bar g(\theta)$, and $\bar x(\theta)$ involved here are as defined in Section~\ref{sec-gtd1-odes}. The main result we need is the first two parts of the lemma. We will use them to relate the average dynamics of the algorithm to the desired mean ODEs at the fast and slow time-scales, respectively. The third part of the lemma will also be needed in the analysis of the fast time-scale, to show that the $(\alpha_n/\beta_n) \, g(\theta_n, x_n, Z_n)$ terms in (\ref{eq-gtd1-altth0}) are negligible as wished (the first and second relations in this part will be used respectively in the diminishing-stepsize and constant-stepsize cases).
In the lemma, $\E_n$ denotes conditional expectation given the history $(Z_0, \ldots, Z_n)$ and $(R_1, \ldots, R_n)$. 
Recall that the notation ``$\lim_{m, n \to \infty}$'' means that the limit is taken as $m$ and $n$ go to infinity in any possible way, and the notation ``$\lim_{m, n, j \to \infty}$'' is defined likewise.

\begin{lem} \label{lem-conv-mean-cond1}
Under Assumption~\ref{cond-collective}, the following hold for each $\theta \in B_\theta$, $x \in B_x$:\vspace*{-5pt}
\begin{itemize}
\item[\rm (i)] For each compact set $D \subset \Z$,
$$ \textstyle{ \lim_{m, n \to \infty} \frac{1}{m} \sum_{i=n}^{n+m-1} \E_n \big[ k(\theta, x, Z_i) - \bar k(\theta, x) \big] \I(Z_n \in D) = 0} \ \ \ \text{in mean}.$$
\item[\rm (ii)] For each compact set $D \subset \Z$,
$$ \textstyle{ \lim_{m, n \to \infty} \frac{1}{m} \sum_{i=n}^{n+m-1} \E_n \big[ g(\theta, \bar x(\theta), Z_i) - \bar g(\theta) \big] \I(Z_n \in D) = 0} \ \ \ \text{in mean}.$$
\item[\rm (iii)]If $\{\alpha_n\}$ and $\{\beta_n\}$ satisfy Assumption~\ref{cond-large-stepsize}, then
$$ \textstyle{ \lim_{m, n \to \infty} \frac{1}{m} \sum_{i=n}^{n+m-1} \E_n \big[ (\alpha_i/\beta_i) \cdot g(\theta, x, Z_i) \big] = 0} \ \ \ \text{in mean}.$$
For any positive sequence $c_j \to 0$ as $j \to \infty$,
$$ \textstyle{ \lim_{m, n, j \to \infty} \frac{1}{m} \sum_{i=n}^{n+m-1} \E_n \big[ c_j \cdot g(\theta, x, Z_i) \big] = 0} \ \ \ \text{in mean}.$$
\end{itemize}
\end{lem}

\begin{proof}
Note first that by Jensen's inequality, it is sufficient to prove the convergence-in-mean statements (i)-(iii) with the conditional expectation $\E_n$ removed. (E.g., for (i), it is sufficient to show that $\lim_{m, n \to \infty} \frac{1}{m} \sum_{i=n}^{n+m-1} \big( k(\theta, x, Z_i) - \bar k(\theta, x) \big) \I(Z_n \in D) = 0$ in mean.)
Now for each fixed $\theta$ and $x$, as functions of $z$, $k(\theta, x, z)$, $g(\theta, x, z)$, and $g(\theta, \bar x(\theta), z)$ are Lipschitz continuous in the trace variable $\e$ (cf.~(\ref{eq-gtd1-gk})), and moreover, by 
(\ref{eq-gtd1-bk})-(\ref{eq-gtd1-bg}), 
$\E_\zeta \big[ k(\theta, x, Z_0) \big] = \bar k(\theta, x)$ and $\E_\zeta \big[ g(\theta, \bar x(\theta), Z_0) \big] = \bar g(\theta)$. 
So (i) and (ii) follow from Prop.~\ref{prop-trace-averaging}(i), and (iii) follows from Prop.~\ref{prop-trace-averaging}(ii). In particular, the first statement in (iii) follows from the first equation in Prop.~\ref{prop-trace-averaging}(ii) with $c_i = \alpha_i/\beta_i$, since $\lim_{i \to \infty} \alpha_i/\beta_i = 0$ by Assumption~\ref{cond-large-stepsize}.
\end{proof}

\begin{rem} \rm \label{rmk-proofgtd}
Before we proceed to state and prove two convergence theorems for  GTDa with diminishing and constant stepsizes, we remark that besides the seven conditions listed above, when analyzing the slow time-scale (in particular, to show that the noise terms $\Delta_n$ in (\ref{eq-gtd1-thnoise}) have negligible effects), our proofs will also use a Lipschitz continuity property of the function $g$: for $(\theta, x) \in B_\theta \times B_x$ and $z$ in a compact set, $g(\theta, x, z)$ is Lipschitz continuous in $x$ uniformly w.r.t.\ $(\theta, z)$ (cf.\ the definition (\ref{eq-gtd1-gk}) of $g$). This condition is different from the continuity condition (iii) above, and we use it to make the proofs a little shorter---in fact, the condition (iii) is sufficient.
Besides this Lipschitz continuity property, the characterization of the limit set of the slow-time-scale mean ODE given in Lemma~\ref{lem-gtd1-odeslow} will also be used in the proofs.
The analysis applies to two-time-scale constrained TD algorithms other than GTDa, as long as they satisfy these same conditions. Lemma~\ref{lem-gtd1-odeslow} is used only at the end of the proofs; if an algorithm satisfies all the other conditions but the limit set of its slow-time-scale mean ODE is not $\Theta_{\text{\rm opt}}$, then with that limit set in place of $\Theta_{\text{\rm opt}}$, the conclusions for GTDa also hold for that algorithm. 
\end{rem}

\subsubsection{Convergence proof for the diminishing-stepsize case} \label{sec-gtd1conv-dim}

Recall $\Theta_{\text{\rm opt}}$ is the set of optimal solutions of $\inf_{\theta \in B_\theta} J(\theta)$.
For a set $D \subset \re^d$, let $N_\epsilon(D)$ denote the $\epsilon$-neighborhood of $D$. For any $T > 0$ and $n \geq 0$, define an integer $m(n,T)$ that points to the last iteration before the sum of the stepsizes $\alpha_i$ between it and the $n$-th iteration exceeds $T$:
\begin{equation} \label{eq-mT}
   m(n,T) = \min \big\{ m \geq n \mid \textstyle{\sum_{i=n}^{m+1}} \alpha_i > T \big\}.
\end{equation}   
To refer to an iteration $i$ with $n \leq i \leq m(n,T)$, we will simply write $i \in [n, m(n, T) ]$.

We are now ready to prove our first convergence result for the constrained GTDa algorithm (\ref{eq-gtd1a-th})-(\ref{eq-gtd1a-x}) with diminishing stepsizes.
Denote by $\text{dist}(\theta, \Theta_{\text{opt}})$ the distance between $\theta$ and the optimal solution set $\Theta_{\text{opt}}$.
The convergence is, first of all, in the sense that $\text{\rm dist}( \theta_n,  \Theta_{\text{\rm opt}})$ converges to $0$ in mean, and it implies the convergence of $\theta_n$ to $\Theta_{\text{\rm opt}}$ in probability. But more strongly than this kind of convergence of $\theta_n$, the powerful weak convergence method allows one to assert that as $n \to \infty$, increasingly more consecutive $\theta$-iterates must all lie in an arbitrarily small neighborhood of $\Theta_{\text{\rm opt}}$ with an arbitrarily high probability.
This is what the following theorem says with the relation (\ref{eq-thm-wkconv1}).

\begin{thm} \label{thm-gtd1-wk-dim}
Let Assumption~\ref{cond-collective} hold. Let $\{(\theta_n, x_n)\}$ be generated by the two-time-scale constrained GTDa algorithm (\ref{eq-gtd1a-th})-(\ref{eq-gtd1a-x}), with diminishing stepsizes that satisfy Assumption~\ref{cond-large-stepsize}, with the constraint set $B_x$ satisfying the condition (\ref{eq-cond-Bx0}), and with initial $x_0, \e_0 \in \sp  \{\phi(\S)\}$.
Then for each given initial condition, $\lim_{n \to \infty} \E \big[ \text{\rm dist}( \theta_n,  \Theta_{\text{\rm opt}} )\big] = 0$. Moreover,
there exists a sequence of positive numbers $T_n \to \infty$ such that for any $\epsilon > 0$,
\begin{equation} \label{eq-thm-wkconv1}
  \limsup_{n \to \infty} \Pr \big( \theta_i \not\in N_\epsilon(\Theta_{\text{\rm opt}}), \ \text{some} \ i \in [n, m(n, T_n) ] \big)  = 0.
\end{equation}  
\end{thm}

\begin{proof} 
First, we consider $(\theta_n, x_n)$ together at the fast time-scale determined by $\{\beta_n\}$. Our goal is to show 
\begin{equation} \label{eq-prfa0}
   \limsup_{n \to \infty} \E \left[ \big\| x_n - \bar x (\theta_n) \big\| \right] = 0.
\end{equation}
We will then use (\ref{eq-prfa0}) in analyzing the $\theta$-iterates at the slow time-scale determined by $\{\alpha_n\}$.

At the fast time-scale, let $g_n(\theta_n, x_n, Z_n) = \frac{\alpha_n}{\beta_n} g(\theta_n, x_n, Z_n)$, and rewrite the algorithm in the form of a single-time-scale algorithm as
\begin{align} \label{eq-prfa11}
   \left( \! \begin{array}{c}  \theta_{n+1} \\ x_{n+1} \end{array} \!\right) = \Pi_{B_\theta \times B_x}
  \! \left( \! \begin{array}{l} \theta_n + \beta_n \, g_n(\theta_n, x_n, Z_n)  \\
   x_n + \beta_n \big( k(\theta_n, x_n, Z_n) + \e_n \omega_{n+1} \big) \end{array} \! \right).
\end{align}
To analyze its asymptotic behavior, we apply \cite[Theorem 8.2.3]{KuY03} for single-time-scale algorithms with diminishing stepsizes and exogenous noises.
The desired mean ODE is (\ref{eq-gtd1-odefast}), namely,
\begin{equation} \label{eq-prfa-ode}
 \left( \! \begin{array}{c} \dot{\theta}(t) \\ \dot x(t) \end{array} \!\right) 
 =  \left( \! \begin{array}{l} 0 \\
 \bar k \big(\theta(t), x(t) \big) + z(t)  \end{array} \!\right),
\end{equation} 
where $\theta(t) \in B_\theta$ and $z(t) \in - \N_{B_x}(x(t))$ is the boundary reflection term. Lemma~\ref{lem-conv-mean-cond1}(i) shows that the function $\bar k(\theta, x)$ in the ODE indeed characterizes the average dynamics of the $x$-iterates in (\ref{eq-prfa11}), while the first relation in Lemma~\ref{lem-conv-mean-cond1}(iii) shows that the constant zero function in the ODE characterizes the average dynamics of the $\theta$-iterates in (\ref{eq-prfa11}). These two parts of Lemma~\ref{lem-conv-mean-cond1} thus fulfill the convergence-in-mean condition (A.8.2.7) in \cite[Theorem 8.2.3]{KuY03}. Regarding other conditions in that theorem, the simple constraint set $B_\theta \times B_x$ satisfies the condition required, and so does the stepsize $\{\beta_n\}$ due to Assumption~\ref{cond-large-stepsize}. The rest of the conditions correspond to the conditions (i)-(iii) we gave in Section~\ref{sec-gtd1-cond}. (In particular, (A.8.2.1) and (A.8.2.5) correspond to our uniform integrability condition (i); (A.8.2.4) corresponds to our tightness condition (ii); and (A.8.2.3) corresponds to our continuity condition (iii).) These conditions are all satisfied by (\ref{eq-prfa11}), as discussed earlier.

Thus we can apply \cite[Theorem 8.2.3]{KuY03} to the iterates (\ref{eq-prfa11}) and obtain the following conclusions. 
Consider the sequence of continuous-time processes $\{(\theta^{n}(\cdot), x^{n}(\cdot))\}_{n \geq 0}$ where for each $n$, $\theta^{n}(\cdot)$ and $x^{n}(\cdot)$ are piecewise constant interpolations of the $\theta$-iterates and $x$-iterates, respectively, with the $n$-th iterates placed at $t=0$:
$$
 \theta^n(t) = \theta_{n+m}   \ \ \  \text{and} \ \ x^n(t) = x_{n+m} \qquad \text{for} \  t \in \textstyle{\big[\sum_{i=0}^{m-1} \beta_{n+i} \, , \,  \sum_{i=0}^{m} \beta_{n+i} \big)}, \ m \geq 0.
$$
Then by \cite[Theorem 8.2.3]{KuY03}, for each subsequence of $\{(\theta^{n}(\cdot), x^{n}(\cdot))\}_{n \geq 0}$, there is a further subsequence $\{(\theta^{n_j}(\cdot), x^{n_j}(\cdot))\}_{j \geq 0}$ that converges in distribution%footnote starts
\footnote{Involved here is the convergence of probability measures on the space of $\re^{2d}$-valued functions on $[0, +\infty)$ or $(-\infty, + \infty)$ that are right-continuous with left limits, where the space of these functions is endowed with the Skorohod metric, which makes it a complete separable metric space \cite[p.\ 238-240]{KuY03}.}
%footnote ends
to $(\theta(\cdot), x(\cdot))$, where $(\theta(\cdot), x(\cdot))$ has absolutely continuous paths satisfying the ODE (\ref{eq-prfa-ode}) and moreover, almost surely, its path lies in the limit set of the ODE (\ref{eq-prfa-ode}). 
As discussed in Section~\ref{sec-gtd1-odes}, with the initial $x_0, \e_0 \in \sp \{\phi(\S)\}$, we are only interested in those solutions of the ODE with initial $x(0)$ and hence the entire solution $x(\cdot)$ in $\sp \{\phi(\S)\}$. For all initial conditions $x(0) \in B_x \cap \sp \{\phi(\S)\}$, by our assumption on $B_x$, the limit set of the ODE (\ref{eq-prfa-ode}) is the set $\{\big(\theta, \bar x(\theta) \big) \mid \theta \in B_\theta \}$ (Lemma~\ref{lem-Bx}).
So, almost surely, $\theta(\cdot) \equiv \theta(0) \in B_\theta$ and $x(\cdot) \equiv \bar x(\theta(0))$.

Let us now prove (\ref{eq-prfa0}).
Suppose it is not true. Then $\limsup_{n \to \infty} \E \left[ \big\| x_n - \bar x (\theta_n) \big\| \right] > 0$, so there must exist some $\epsilon > 0$ and a subsequence $\{n_j\}_{j \geq 1}$ of integers such that 
\begin{equation} \label{eq-prfa1} 
 \E \big[ \big\| x_{n_j} - \bar x (\theta_{n_j}) \big\| \big] > \epsilon, \qquad \forall \, j \geq 1.
\end{equation} 
Consider the sequence of piecewise constant interpolated processes $\{(\theta^{n_j}(\cdot), x^{n_j}(\cdot))\}_{j \geq 1}$. By the preceding proof, there exists a further subsequence, which we denote by $\{(\theta^{n'_j}(\cdot), x^{n'_j}(\cdot))\}_{j \geq 1}$, that converges in distribution to a process $(\theta(\cdot), x(\cdot))$ with the properties described above.
More conveniently, as in \cite[Chap.\ 8.2, p.\ 255]{KuY03}, let us work with the Skorohod representation of these processes. 
Specifically, by \cite[Theorem 11.7.1]{Dud02}, there exist some probability space and processes $\{(\tilde \theta^{n'_j}(\cdot), \tilde x^{n'_j}(\cdot))\}_{j \geq 1}$ and $(\tilde \theta(\cdot), \tilde x(\cdot))$ defined on that probability space such that\vspace*{-3pt} 
\begin{itemize}
\item[(i)] each of the tilde processes has the same probability distribution as the corresponding process without tilde;
\item[(ii)] $\big\{\big(\tilde \theta^{n'_j}(\cdot), \tilde x^{n'_j}(\cdot)\big)\big\}_{j \geq 1}$ converges almost surely to $(\tilde \theta(\cdot), \tilde x(\cdot))$.\vspace*{-3pt}
\end{itemize}
Because the limit process $(\tilde \theta(\cdot), \tilde x(\cdot))$ has continuous paths, the convergence of $(\tilde \theta^{n'_j}(\cdot), \tilde x^{n'_j}(\cdot))$ to $(\tilde \theta(\cdot), \tilde x(\cdot))$ in (ii) is uniform on any compact interval of $t$.
We then have $\tilde x(\cdot) \equiv \bar x(\tilde \theta(0))$ a.s.\ and 
$$\big\|\tilde \theta^{n'_j}(0) -  \tilde \theta(0) \big\| + \big\|\tilde x^{n'_j}(0) - \bar x\big(\tilde \theta(0) \big) \big\| \asto 0, \qquad \text{as} \  j \to \infty.$$
Since $\bar x(\theta)$ is a continuous function of $\theta$, the above implies that
$\|\tilde x^{n'_j}(0) - \bar x \big( \tilde \theta^{n'_j}(0) \big)\| \asto 0$ as $j \to \infty$. In turn, since all the processes lie in the bounded set $B_\theta \times B_x$, 
by the bounded convergence theorem, we have $\lim_{j \to \infty} \E \big[ \big\| \tilde x^{n'_j}(0) - \bar x \big( \tilde \theta^{n'_j}(0) \big) \big\| \big] = 0$, which is the same as 
$$\lim_{j \to \infty} \E \big[ \big\| x^{n'_j}(0) - \bar x \big(\theta^{n'_j}(0) \big) \big\| \big]  = \lim_{j \to \infty} \E \big[ \big\| x_{n'_j} - \bar x \big(\theta_{n'_j} \big) \big\| \big] = 0,$$  
because $(\tilde \theta^{n'_j}(\cdot), \tilde x^{n'_j}(\cdot))$ has the same probability distribution as $(\theta^{n'_j}(\cdot), x^{n'_j}(\cdot))$. But the last equality above contradicts (\ref{eq-prfa1}). Therefore, (\ref{eq-prfa0}) must hold.

We now consider the slow time-scale determined by $\{\alpha_n\}$ and rewrite the $\theta$-iterates (\ref{eq-gtd1a-th}) of the algorithm as
\begin{align} \label{eq-prfa5}
 \theta_{n+1} & = \Pi_{B_\theta} \big( \theta_n + \alpha_n (g(\theta_n, \bar x (\theta_n), Z_n)  + \Delta_n ) \big), \\
\text{where} \qquad \Delta_n & = g(\theta_n, x_n, Z_n) - g(\theta_n, \bar x (\theta_n), Z_n). \label{eq-prfa52}
\end{align}
We view this as a single-time-scale algorithm and apply again \cite[Theorem 8.2.3]{KuY03} or more precisely, the proof of that theorem. 
As before, the constraint set $B_\theta$ and the stepsize $\{\alpha_n\}$ satisfy the conditions in the theorem, and the uniform integrability and tightness conditions in the theorem are entailed by our conditions (i)-(ii) and (vi) listed in Section~\ref{sec-gtd1-cond}, which are all satisfied by the algorithm. Another condition is the continuity of $g(\theta, \bar x(\theta), z)$ in $\theta$ uniformly in $z$ in a compact set, which corresponds to the condition (A.8.2.3) in \cite[Theorem 8.2.3]{KuY03} and which is implied by our condition (iii) and the uniform continuity of the function $\bar x(\theta)$ on the bounded set $B_\theta$. What remains to be shown is that there is a mean ODE corresponding to (\ref{eq-prfa5}).

The desired mean ODE is (\ref{eq-gtd1-odeslow}), namely,
\begin{equation} \label{eq-prfa6}
  \dot \theta(t) = \bar g\big(\theta(t)\big) + z(t),  \quad z(t) \in - \N_{B_\theta}\big(\theta(t)\big),
\end{equation}  
where $z(t)$ is the boundary reflection term.
Lemma~\ref{lem-conv-mean-cond1}(ii) relates the averages of $\{g(\theta, \bar x (\theta), Z_n)\}$ in (\ref{eq-prfa5}) to the function $\bar g(\theta)$ in the ODE (\ref{eq-prfa6}), and it fulfills the condition (A.8.2.7) in \cite[Theorem 8.2.3]{KuY03}. 
The $\Delta_n$ terms in (\ref{eq-prfa5})-(\ref{eq-prfa52}) are noise terms. We need to show that eventually their effects become negligible so that the average dynamics of the algorithm is characterized by the desired mean ODE. Specifically, if we can show for two arbitrary, fixed positive numbers $t < T$,
$$ \limsup_{n \to \infty} \E \left[ \Big\| \sum_{i= m(n, t)}^{m(n,T)} \alpha_i \Delta_i \Big\| \right] = \limsup_{n \to \infty} \E \left[ \Big\| \sum_{i= m(n, t)}^{m(n,T)} \alpha_i \big(g(\theta_i, x_i, Z_i) - g(\theta_i, \bar x(\theta_i), Z_i) \big) \Big\| \right] = 0,$$
then we can proceed as in the proof of \cite[Theorem 8.2.3]{KuY03} to obtain its conclusions.%footnote starts
\footnote{The term $\Delta_n$ here corresponds to the noise term $\beta_n$ in \cite[Theorem 8.2.3]{KuY03}. Our condition on $\Delta_n$ is weaker than the condition (A.8.2.6) in \cite[Theorem 8.2.3]{KuY03} for noise terms. This is why, instead of directly applying the theorem, we need to follow its proof. The part of the proof related to the treatment of the noise terms appears in \cite[p.\ 251-254]{KuY03} (presented there are the proof arguments for the constant-stepsize case, which, apart from notational differences, are essentially the same as the arguments for the diminishing-stepsize case). What we need to do here is to use the property of the noise terms to eliminate the $B(\cdot)$ term in the equation (2.6) of that proof \cite[p.\ 252]{KuY03}. The condition we give above is sufficient for this purpose.}
%footnote ends
For the preceding relation to hold, it is sufficient that
\begin{equation} \label{eq-prfa2}
 \limsup_{n \to \infty} \sum_{i= m(n, t)}^{m(n,T)} \alpha_i \, \E \!\left[  \big\| g(\theta_i, x_i, Z_i) - g(\theta_i, \bar x(\theta_i), Z_i) \big\| \right] = 0.
\end{equation} 
Let us prove (\ref{eq-prfa2}) using (\ref{eq-prfa0}) and the uniform integrability of $\{g(\theta_n, x_n, Z_n)\}$ and $\{g(\theta_n, \bar x (\theta_n), Z_n)\}$.

For any $\epsilon$, by the tightness of $\{\e_n\}$ (implied by Prop.~\ref{prop-trace}(ii)), there exists $K_\epsilon > 0$ such that 
$$\inf_{n \geq 0} \Pr(\|\e_n \| > K_\epsilon) < \epsilon.$$ 
Since $\{g(\theta_n, x_n, Z_n)\}$ and $\{g(\theta_n, \bar x (\theta_n), Z_n)\}$ are uniformly integrable (as verified in the previous subsection), by \cite[Theorem~10.3.5]{Dud02}, there exist constants $\eta_\epsilon$ which depend on $\epsilon$, such that
\begin{equation} \label{eq-prfa3}
 \sup_{n \geq 0} \E \!\left[ \big\|g(\theta_n, x_n, Z_n)\big\| \cdot \I( \| \e_n \| > K_\epsilon) \right] < \eta_\epsilon, \qquad  \sup_{n \geq 0} \E \!\left[ \big \| g(\theta_n, \bar x (\theta_n), Z_n) \big\| \cdot \I( \| \e_n \| > K_\epsilon)  \right] < \eta_\epsilon,
\end{equation} 
and $\eta_\epsilon \to 0$ as $\epsilon \to 0$.

The function $g(\theta, x, z)$ is Lipschitz continuous in $x$ uniformly w.r.t.\ $(\theta, z) \in B_\theta \times D$ for a compact set $D$ in the space of $z$. Hence, for some constant $c_\epsilon$ depending on $K_\epsilon$,
$$ \E \!\left[ \big\| g(\theta_n, x_n, Z_n) - g(\theta_n, \bar x(\theta_n), Z_n) \big\| \cdot \I \big( \| \e_n \| \leq K_\epsilon \big) \right] \leq 
c_\epsilon \, \E \!\left[ \| x_n - \bar x(\theta_n) \| \right], \quad \forall \, n \geq 0.$$
Combining this with (\ref{eq-prfa0}), we have
\begin{equation} \label{eq-prfa4}
 \lim_{n \to \infty} \E \! \left[ \big\| g(\theta_n, x_n, Z_n) - g(\theta_n, \bar x(\theta_n), Z_n) \big\| \cdot \I \big( \| \e_n \| \leq K_\epsilon \big) \right] = 0.
\end{equation} 
Now $ \sum_{i= m(n, t)}^{m(n,T)} \alpha_i \, \E \!\left[ \big\| g(\theta_i, x_i, Z_i) - g(\theta_i, \bar x(\theta_i), Z_i) \big\| \right]$ is bounded above by
\begin{align*}
   & (T - t + \alpha_{m(n,t)}) \cdot \sup_{m(n, t) \leq i \leq m(n,T)} \E \!\left[ \big\| g(\theta_i, x_i, Z_i) - g(\theta_i, \bar x(\theta_i), Z_i) \big\| \cdot \I \big( \| \e_i \| \leq K_\epsilon \big) \right] \\
   & + (T - t+ \alpha_{m(n,t)}) \cdot \sup_{m(n, t) \leq i \leq m(n,T)} \E \!\left[ \left( \big\|g(\theta_i, x_i, Z_i)\big\|   + \big\| g(\theta_i, \bar x (\theta_i), Z_i) \big\| \right) \cdot \I( \| \e_i \| > K_\epsilon) \right].
\end{align*}   
(To see this, recall that by the definition (\ref{eq-mT}) of $m(n, \cdot)$, $\sum_{i= m(n, t)}^{m(n,T)} \alpha_i \leq T-t + \alpha_{m(n,t)}$.) So, letting first $n \to \infty$ and then $\epsilon \to 0$, and using first (\ref{eq-prfa3}) and (\ref{eq-prfa4}) (together with the fact that the stepsize is diminishing), and using then the fact $\lim_{\epsilon \to 0} \eta_\epsilon = 0$, we obtain (\ref{eq-prfa2}).

Thus we have fulfilled the conditions to apply the proof of \cite[Theorem 8.2.3]{KuY03}. From it we obtain the conclusion that there exists a 
a sequence of positive numbers $T_n \to \infty$ such that for any $\epsilon > 0$,
$$ \limsup_{n \to \infty} \Pr \big( \theta_i \not\in N_\epsilon(L_{B_\theta}), \ \text{some} \ i \in [n, m(n, T_n) ] \big)  = 0, 
$$
where $L_{B_\theta}$ is the limit set of the ODE (\ref{eq-prfa6}). 
This is just the set $\Theta_{\text{opt}}$ by Lemma~\ref{lem-gtd1-odeslow}, so we obtain (\ref{eq-thm-wkconv1}). Finally, since $\{\theta_n\}$ is bounded and (\ref{eq-thm-wkconv1}) implies $\text{dist}(\theta_n, \Theta_{\text{opt}}) \to 0$ in probability, the convergence is also in mean by \cite[Theorem 10.3.6]{Dud02}: $\lim_{n \to \infty} \E \big[ \text{dist}(\theta_n, \Theta_{\text{opt}}) \big]= 0$.
\end{proof}

\subsubsection{Convergence proof for the constant-stepsize case} \label{sec-gtd1conv-constant}

The two-time-scale GTDa algorithm (\ref{eq-gtd1a-th})-(\ref{eq-gtd1a-x}) can be carried out using constant stepsizes instead of diminishing stepsizes, i.e., with $\alpha_n = \alpha, \beta_n=\beta$ for all $n$ and $\alpha < < \beta$. An analogue of Theorem~\ref{thm-gtd1-wk-dim} is given below for this constant stepsize algorithm. Imagine that we run the algorithm simultaneously for all possible stepsize parameters. The theorem concerns the asymptotic property of the algorithm, as the stepsize parameters $\alpha, \beta$ approach $0$ in a way such that $\alpha/\beta \to 0$. In the theorem the shorthand notation ``$i \in [0, T_{\alpha,\beta}/\alpha ]$'' refers to an iteration $i$ with $0 \leq i \leq T_{\alpha,\beta}/\alpha$.

\begin{thm} \label{thm-gtd1-wk-constant}
Let Assumption~\ref{cond-collective} hold. Let $\{(\theta^{\alpha,\beta}_n, x^{\alpha,\beta}_n)\}$ be generated by the two-time-scale GTDa algorithm (\ref{eq-gtd1a-th})-(\ref{eq-gtd1a-x}) with constant stepsizes $\alpha_n = \alpha$ and $\beta_n = \beta$ for all $n$, where $\alpha << \beta$. 
Moreover, let the constraint set $B_x$ satisfy the condition (\ref{eq-cond-Bx0}), and let the initial $x^{\alpha,\beta}_0=x_0, \e_0 \in \sp  \{\phi(\S)\}$.
Then for each given initial condition, the following holds: For any integers $n_{\alpha}$ such that $\alpha \, n_{\alpha} \to \infty$ as $\alpha \to 0$,
there exist positive numbers $\{T_{\alpha, \beta} \mid \alpha, \beta > 0\}$ with $T_{\alpha, \beta} \to \infty$ as $(\beta, \alpha/\beta) \to 0$, such that for any $\epsilon > 0$,
\begin{equation} \label{eq-wk-constant}
  \limsup_{\beta \to 0, \, \alpha/\beta \to 0} \Pr \big( \theta^{\alpha,\beta}_{n_\alpha + i} \not\in N_\epsilon(\Theta_{\text{\rm opt}}), \ \text{some} \ i \in [0, T_{\alpha,\beta}/\alpha ] \big)  = 0.
\end{equation}  
\end{thm}
\smallskip

The proof of Theorem~\ref{thm-gtd1-wk-constant} given below parallels that of Theorem~\ref{thm-gtd1-wk-dim} and uses essentially the same arguments. 
We will refer to the previous proof frequently to avoid repetition. First of all, we want to make sure that several conditions are met by the algorithm, so that we can apply the theorems from \cite{KuY03} for constant stepsize algorithms in our proof. These conditions are the same as the conditions given in Section~\ref{sec-gtd1-cond}, except that in the uniform integrability condition, iterates generated with different stepsize parameters all need to be considered. In particular, Conditions (ii)-(v) and (vii) remain the same. In Conditions (i) and (vi), the iterates $(\theta_n, x_n)$ now need to be indexed also by the stepsize parameters used to generate them. Specifically, besides those sets in Conditions (i) and (vi) that involve fixed $(\theta, x)$, the following sets of random variables are required to be uniformly integrable:
\begin{align*}
 & \big\{g(\theta_n^{\alpha,\beta}, x_n^{\alpha,\beta}, Z_n) \mid n \geq 0, \alpha > 0, \beta > 0 \big\}, \qquad \big\{k(\theta_n^{\alpha,\beta},x_n^{\alpha,\beta},Z_n)+\e_n\omega_{n+1} \mid n \geq 0, \alpha > 0, \beta > 0 \big\}, \\
& \big\{g(\theta_n^{\alpha,\beta}, \bar x(\theta_n^{\alpha,\beta}), Z_n) \mid  n \geq 0, \alpha > 0, \beta > 0  \big\}.
\end{align*}
As before, that these sets are uniform integrable follows from the uniform integrability of $\{\e_n\}$ given in Prop.~\ref{prop-trace}(ii).
Regarding other conditions used in the proof below, Remark~\ref{rmk-proofgtd} given at the end of Section~\ref{sec-gtd1-cond} applies here as well.

\begin{proof}
Consider an arbitrary sequence of stepsize parameters $(\alpha^{(m)}, \beta^{(m)}), m \geq 1$, that satisfies $(\alpha^{(m)}, \beta^{(m)}) \to 0$ and $\alpha^{(m)}/\beta^{(m)} \to 0$ as $m \to \infty$. For each $m$, let $n_m = n_\alpha$ for $\alpha = \alpha^{(m)}$, and let $\big(\theta_n^{(m)}, \theta_n^{(m)}\big), n \geq 0$, denote the iterates generated by the algorithm with the stepsizes $(\alpha^{(m)}, \beta^{(m)})$.  

As before we view the algorithm as a single-time-scale algorithm, first for the fast time-scale and then for the slow time-scale. At the fast time-scale, for each $m$, we let $g_m = \frac{\alpha^{(m)}}{\beta^{(m)}} \cdot g$, and rewrite the algorithm as
\begin{align} \label{eq-prfd1}
   \left( \! \begin{array}{c}  \theta_{n+1}^{(m)} \\ x_{n+1}^{(m)} \end{array} \!\right) = \Pi_{B_\theta \times B_x}
  \! \left( \! \begin{array}{l} \theta_n^{(m)} + \beta^{(m)} \, g_m(\theta_n^{(m)}, x_n^{(m)}, Z_n)  \\
   x_n^{(m)} + \beta^{(m)} \big( k(\theta_n^{(m)}, x_n^{(m)}, Z_n) + \e_n \omega_{n+1} \big) \end{array} \! \right).
\end{align}
At the slow time-scale, for each $m$, we rewrite the $\theta$-iterates of the algorithm as
\begin{align} \label{eq-prfd2}
 \theta_{n+1}^{(m)} & = \Pi_{B_\theta} \big( \theta_n^{(m)} + \alpha^{(m)} (g(\theta_n^{(m)}, \bar x (\theta_n^{(m)}), Z_n)  + \Delta_n^{(m)} ) \big), \\
\text{where} \qquad \Delta_n^{(m)} & = g(\theta_n^{(m)}, x_n^{(m)}, Z_n) - g(\theta_n^{(m)}, \bar x (\theta_n^{(m)}), Z_n). \notag
\end{align}
In analyzing (\ref{eq-prfd1}), our goal is to show the following analogue of (\ref{eq-prfa0}):
\begin{equation} \label{eq-prfd3}
   \limsup_{m \to \infty} \sup_{i \geq n_m} \E \! \left[ \big\| x_i^{(m)} - \bar x \big(\theta_i^{(m)}\big) \big\| \right] = 0.
\end{equation}
We will then use this relation to bound the effects of the noise terms $\Delta_n^{(m)}$ in (\ref{eq-prfd2}) at the slow time-scale, similarly to the previous diminishing-stepsize case.

We consider first (\ref{eq-prfd1}) at the fast time-scale. Its associated mean ODE is again (\ref{eq-prfa-ode}).  We proceed with arguments paralleling those given in the previous proof, except that in place of \cite[Theorem 8.2.3]{KuY03}, we now apply \cite[Theorem 8.2.2]{KuY03} for single-time-scale, constant stepsize algorithms and we verify that the conditions in \cite[Theorem 8.2.2]{KuY03} are met.%footnote starts
\footnote{The verification step is the same as that given immediately after (\ref{eq-prfa-ode}) in the previous proof, except that we now use the second (instead of the first) relation in Lemma~\ref{lem-conv-mean-cond1}(iii) to characterize the average dynamics of the $\theta$-iterates at the fast time-scale.}
%footnote ends 
With the latter theorem, we obtain the following conclusions. 
Let $n'_m, m \geq 1$, be any sequence of integers satisfying $\beta^{(m)} n'_m \to \infty$ as $m \to \infty$.
For each $m$, let $(\theta^{(m)}(\cdot), x^{(m)}(\cdot))$ be piecewise constant interpolations of the iterates $(\theta^{(m)}_{n'_m+i}, x^{(m)}_{n'_m+i})$ given by 
$$
 \theta^{(m)}(t) = \theta^{(m)}_{n'_m+i}   \ \ \  \text{and} \ \ x^{(m)}(t) = x^{(m)}_{n'_m+i} \qquad \text{for} \  t \in \big[ i \beta^{(m)}, (i+1)  \beta^{(m)} \big),  \ i \geq 0.
$$
Then, by \cite[Theorem 8.2.2]{KuY03}, for each subsequence of $\{(\theta^{(m)}(\cdot), x^{(m)}(\cdot))\}_{m \geq 1}$, there is a further subsequence that converges in distribution to 
a process $(\theta(\cdot), x(\cdot))$, where $(\theta(\cdot), x(\cdot))$ has the same property as before. Namely, $(\theta(\cdot), x(\cdot))$ has absolutely continuous paths satisfying the ODE (\ref{eq-prfa-ode}) and moreover, almost surely, its path lies in the limit set of the ODE (\ref{eq-prfa-ode}). 
As before, with the initial $x_0, \e_0 \in \sp \{\phi(\S)\}$, we are only interested in those solutions of the ODE with the initial $x(0) \in B_x \cap \sp \{\phi(\S)\}$. For such initial conditions, by our assumption on $B_x$, the limit set of the ODE (\ref{eq-prfa-ode}) is $\{\big(\theta, \bar x(\theta) \big) \mid \theta \in B_\theta \}$ (Lemma~\ref{lem-Bx}).
So, almost surely, $\theta(\cdot) \equiv \theta(0) \in B_\theta$ and $x(\cdot) \equiv \bar x(\theta(0))$. 

We use the preceding conclusions to prove (\ref{eq-prfd3}). Suppose (\ref{eq-prfd3}) does not hold. Then there exists $\epsilon > 0$ and a subsequence $m_j$ and integers $n'_j \geq n_{m_j}$ such that
\begin{equation} \label{eq-prfd4}
       \E \!\left[ \big\| x_{n'_j}^{(m_j)} - \bar x \big(\theta_{n'_j}^{(m_j)}\big) \big\| \right] > \epsilon, \qquad \forall \, j \geq 1. 
\end{equation}
Now in the preceding conclusions, let $n'_{m_j}$  be $n'_j$ (note that $n'_j$ satisfies the requirement $\beta^{(m_j)}  n'_j \to \infty$ as $j \to \infty$, because $\beta^{(m_j)}/\alpha^{(m_j)} \to \infty$ and $\alpha^{(m_j)} n'_j  \geq \alpha^{(m_j)} n_{m_j} \to \infty$). By the preceding proof, from the subsequence of interpolated processes $\{(\theta^{(m_j)}(\cdot), x^{(m_j)}(\cdot))\}_{j \geq 1}$, we can select a further subsequence that has the weak convergence properties described above. Denote the latter subsequence by $\{(\theta^{(m'_j)}(\cdot), x^{(m'_j)}(\cdot))\}$. Then, as in the previous proof, working with the Skorohod representation of these processes allows us to conclude that 
$$   \E \!\left[ \big\| x^{(m'_j)}_{n'_j} - \bar x\big(\theta^{(m'_j)}_{n'_j} \big) \big\| \right] \to 0 \ \ \text{as} \ j \to \infty.$$
This contradicts (\ref{eq-prfd4}); therefore, (\ref{eq-prfd3}) must hold.

We now consider (\ref{eq-prfd2}) at the slow time-scale, and apply the proof of \cite[Theorem 8.2.2]{KuY03} to connect it with the desired mean ODE (\ref{eq-prfa6}). We proceed with arguments paralleling those given immediately after (\ref{eq-prfa6}) in the previous proof. As before, in order for the conclusions of \cite[Theorem 8.2.2]{KuY03} to hold for (\ref{eq-prfd2}), we need to show that the noise terms $\Delta_n^{(m)}$ in (\ref{eq-prfd2}) are well-behaved. Specifically,
we want to show that for two arbitrary, fixed positive numbers $t < T$,
$$ \limsup_{m \to \infty} \E \!\left[ \Big\| \sum_{i = n_m+ \lfloor t/\alpha^{(m)} \rfloor }^{n_m + \lfloor T/\alpha^{(m)} \!\rfloor } \alpha^{(m)}  \Delta_i^{(m)} \Big\| \right] = 0.$$
Clearly, it is sufficient if we show
$$ \limsup_{m \to \infty}   \sum_{i = n_m+ \lfloor t/\alpha^{(m)} \rfloor }^{n_m + \lfloor T/\alpha^{(m)} \!\rfloor } \alpha^{(m)} \cdot \E \!\left[  \big\| g\big(\theta_i^{(m)}, x_i^{(m)}, Z_i\big) - g\big(\theta_i^{(m)}, \bar x(\theta_i^{(m)}), Z_i\big) \big\| \right] = 0.$$
To prove this, we proceed as in the last part of the previous proof (that is, the proof immediately after (\ref{eq-prfa2})), using (\ref{eq-prfd3}) in place of (\ref{eq-prfa0}), and also using, as before, the uniform integrability of the sets 
$\big\{g(\theta_n^{\alpha,\beta}, x_n^{\alpha,\beta}, Z_n) \mid n \geq 0, \alpha > 0, \beta > 0 \big\}$ and $\big\{g(\theta_n^{\alpha,\beta}, \bar x(\theta_n^{\alpha,\beta}), Z_n) \mid  n \geq 0, \alpha > 0, \beta > 0  \big\}$, together with the fact that for $z$ in a compact set and $\theta \in B_\theta$, the function $g(\theta, x, z)$ is Lipschitz continuous in $x$ uniformly w.r.t.\ $(\theta, z)$.

Having verified the required conditions in order to apply the proof of \cite[Theorem 8.2.2]{KuY03}, we obtain the conclusions of \cite[Theorem 8.2.2]{KuY03} for the sequences of iterates $\{\big(\theta_n^{(m)}, \theta_n^{(m)}\big)\}_{n \geq n_m}$, $m \geq 1$.  
In particular, since $\Theta_{\text{opt}}$ is the limit set of the ODE (\ref{eq-prfa6}) (Lemma~\ref{lem-gtd1-odeslow}), we have that for any $T > 0$ and $\epsilon > 0$,
\begin{equation} \label{eq-prfd5}
  \limsup_{m \to \infty} \Pr \big( \theta^{(m)}_{n_m + i} \not\in N_\epsilon(\Theta_{\text{opt}}), \ \text{some} \ i \in [0, T/\alpha^{(m)} ] \big)  = 0.
\end{equation}  
To finish the proof, we now show that (\ref{eq-prfd5}), together with the arbitrariness of the stepsize sequence $\{(\alpha^{(m)}, \beta^{(m)})\}$, implies the desired conclusion (\ref{eq-wk-constant}). To this end, for each $(\alpha, \beta, \epsilon, T)$, with $c = \alpha/\beta$, define
$$f(\beta,  c, \epsilon, T) = \Pr \big( \theta^{\alpha, \beta}_{n_{\alpha} + i} \not\in N_\epsilon(\Theta_{\text{opt}}), \ \text{some} \ i \in [0, T/\alpha ] \big).$$
Consider the sets $A_j = \{ (\beta, c) \mid \beta \leq 1/j, \, c \leq 1/j \}$, $j \geq 1$.
For each $T, \epsilon > 0$, by (\ref{eq-prfd5}), 
$$\textstyle{\sup_{(\beta, c) \in A_j} f(\beta, c, \epsilon, T) \to 0} \quad \text{as} \  j \to \infty.$$
(Otherwise, we would be able to choose a sequence $\{(\alpha^{(m)}, \beta^{(m)})\}$ with $\beta^{(m)} \to 0$ and $\alpha^{(m)}/\beta^{(m)} \to 0$ that violates (\ref{eq-prfd5}); a contradiction.) This implies that for each $m \geq 1$, there exists $j_m$ such that
\begin{equation} \label{eq-prfd6}
   \textstyle{\sup_{(\beta, c) \in A_j} f(\beta, c, \frac{1}{m}, m) \leq \frac{1}{m}}, \qquad \forall j \geq j_m.
\end{equation}   
Choose a strictly increasing sequence $j_m, m \geq 1$, with the above property. 
Define sets $B_m = A_{j_{m}} \setminus A_{j_{m+1}}$.  Then define the positive numbers $T_{\alpha,\beta}$ in the theorem as follows: 
for all $(\alpha, \beta)$ such that $(\beta, \alpha/\beta) \in B_m$, let $T_{\alpha, \beta} = m$. (Define $T_{\alpha,\beta}$ arbitrarily if $(\alpha, \beta)$ belongs to none of the sets $\{B_m\}$.)
With this definition of $T_{\alpha,\beta}$, let us show that (\ref{eq-wk-constant}) hold. For any $\epsilon > 0$, there exists $\bar m$ with $\epsilon > 1/\bar m$. 
If $(\alpha, \beta)$ is such that $(\beta, \alpha/\beta) \in B_m$ for some $m \geq \bar m$,
the probability in the left-hand side (l.h.s.) of (\ref{eq-wk-constant}) is 
$$\Pr \big( \theta^{\alpha, \beta}_{n_{\alpha} + i} \not\in N_\epsilon(\Theta_{\text{opt}}), \, \text{some} \ i \in [0, T_{\alpha,\beta}/\alpha ] \big) = f(\beta, \tfrac{\alpha}{\beta}, \epsilon, m) \leq f(\beta, \tfrac{\alpha}{\beta}, \tfrac{1}{m}, m) \leq \frac{1}{m},$$
where the last inequality follows from (\ref{eq-prfd6}) and our construction of $B_m$. If we let $\beta \to 0$ and $\alpha/\beta \to 0$, then the corresponding points $(\beta, \alpha/\beta)$ enter sets $B_m$ with $m$ increasing to $+\infty$. So the preceding relation shows that the limit on the l.h.s.\ of (\ref{eq-wk-constant}) equals $0$.
\end{proof}

\subsubsection{Averaged iterates for the constant-stepsize case} \label{sec-aveite}

We mention a result about the averaged iterates for the constant-stepsize case $\alpha_n = \alpha, \beta_n=\beta$ with $\alpha < < \beta$. It complements Theorem~\ref{thm-gtd1-wk-constant} by characterizing the asymptotic behavior of averaged iterates for a fixed pair of stepsize parameters $(\alpha, \beta)$. Averaged iterates, such as 
\begin{equation} \label{eq-aveite}
  \bar \theta^{\alpha,\beta}_n = \textstyle{\frac{1}{n - n_0} \sum_{i=n_0}^{n-1} \theta^{\alpha,\beta}_i}
\end{equation}  
for some given $n_0 \geq 0$, often behave better than the original iterates. Theoretical results regarding the acceleration of convergence by averaging are discussed in \cite{PoJ92} and \cite[Chap.\ 10]{KuY03} (see also the references in the latter book). In our case, we do not yet have a rate of convergence analysis for the algorithms, so we cannot prove the acceleration of convergence for the averaged iterates.%footnote starts
\footnote{However, we did observe in a numerical study of ETD that averaging has great advantages in both the diminishing-stepsize case and the constant-stepsize case. For details of the study, see the companion note to the paper~\cite{etd-wkconv}.}
%footnote ends
The result below about the averaged iterates (\ref{eq-aveite}) is similar to a result proved in \cite{etd-wkconv} in the ETD context and follows from the same analysis.

Let $X_n = (S_n, y_n, \e_n, \theta^{\alpha,\beta}_n, x^{\theta,\beta}_n)$ (treat $y_n$ as a dummy variable in the case of state-dependent~$\lambda$). 
Consider the $m$-step version of the Markov chain $\{X_n\}$: $\{(X_n, X_{n+1}, \ldots X_{n+m-1})\}_{n\geq 0}$. 
Denote the set of its invariant probability measures by $\mathcal{I}^{\alpha,\beta}_m$. (The set $\mathcal{I}^{\alpha,\beta}_m$ is always nonempty, and it can have multiple elements, despite that the state-trace process has a unique invariant probability measure.)
For the invariant probability measures in $\mathcal{I}^{\alpha,\beta}_m$, consider their marginal distributions on the space of the $m$ $\theta$-variables, and let $\bar{\mathcal{I}}^{\alpha,\beta}_m$ denote the set of all these marginals.  
For $\epsilon > 0$, write $N'_\epsilon(\Theta_{\text{\rm opt}})$ for the open $\epsilon$-neighborhood of $\Theta_{\text{\rm opt}}$,
and write $\big[N'_\epsilon(\Theta_{\text{\rm opt}})\big]^m$ for the Cartesian product of $m$ copies of  $N'_\epsilon(\Theta_{\text{\rm opt}})$.

The following theorem is analogous to \cite[Theorem 8, Section 3.3]{etd-wkconv}. The second part of the theorem bounds, in expectation, the maximal deviation from $\Theta_{\text{\rm opt}}$ of $m$ consecutive averaged iterates. The bound is expressed in terms of the minimal probability mass that the distributions in $\bar{\mathcal{I}}^{\alpha,\beta}_m$ assign to $\big[N'_\epsilon(\Theta_{\text{\rm opt}})\big]^m$. The first part of the theorem says that for an arbitrarily small $\epsilon$, this minimal probability mass increases to $1$, as the stepsize parameters $(\alpha, \beta)$, as well as their ratio, approach zero.

\begin{thm} \label{thm-gtd1-average}
In the setting of Theorem~\ref{thm-gtd1-wk-constant}, consider the sequence of averaged iterates $\{\bar \theta^{\alpha,\beta}_n\}$ given by (\ref{eq-aveite}). Then the following hold for any $\epsilon > 0$ and $m \geq 1$:\vspace*{-3pt}
\begin{itemize}
\item[\rm (i)] $ \liminf_{\beta \to 0, \, \alpha/\beta \to 0} \, \inf_{\mu \in \bar{\mathcal{I}}^{\alpha,\beta}_m} \mu \big( \big[N'_\epsilon(\Theta_{\text{\rm opt}})\big]^m \big) = 1$
and more strongly, with $m_\alpha = \lceil \tfrac{m}{\alpha} \rceil$,
$$\liminf_{\beta \to 0, \, \alpha/\beta \to 0} \, \inf_{\mu \in \bar{\mathcal{I}}^{\alpha,\beta}_{m_\alpha}} \mu \big( \big[N'_\epsilon(\Theta_{\text{\rm opt}})\big]^{m_\alpha} \big) = 1.$$
\item[\rm (ii)] For each initial condition of the algorithm and each stepsize pair $(\alpha, \beta)$,
$$ \limsup_{n \to \infty} \E \Big[ \max_{n \leq i < n+m} \!\text{\rm dist} \big( \bar \theta_i^{\alpha,\beta}, \Theta_{\text{\rm opt}} \big) \Big] \leq \epsilon \kappa^{\alpha,\beta}_m + 2 \, r_{B_\theta} ( 1 - \kappa_m^{\alpha,\beta}),$$
where $r_{B_\theta}$ is the radius of $B_\theta$, and $\kappa_m^{\alpha,\beta} = \inf_{\mu \in \bar{\mathcal{I}}^{\alpha,\beta}_m} \mu \big( \big[N'_\epsilon(\Theta_{\text{\rm opt}})\big]^m \big)$.
\end{itemize}
\end{thm} 
%\smallskip

We briefly explain how the theorem is proved. The starting point of this analysis is the observation that with constant stepsizes, the iterates jointly with the states/memory states and traces, $\{(S_n, y_n, \e_n, \theta^{\alpha,\beta}_n, x^{\theta,\beta}_n)\}$, form a weak Feller Markov chain. By exploiting this fact---recall that weak Feller Markov chains have nice ergodicity properties \cite{Mey89,MeT09}, one can combine the convergence result of Theorem~\ref{thm-gtd1-wk-constant}%footnote starts
\footnote{More precisely, we use a version of Theorem~\ref{thm-gtd1-wk-constant} in which for each $(\alpha, \beta)$, the initial $(S_0, y_0, \e_0, \theta^{\alpha,\beta}_0, x^{\theta,\beta}_0)$ is distributed according to an invariant probability measure of the Markov chain $\{(S_n, y_n, \e_n, \theta^{\alpha,\beta}_n, x^{\theta,\beta}_n)\}$.}
%footnote ends 
with ergodicity properties of weak Feller Markov chains \cite{Mey89} and with the ergodicity property of the state-trace process given in Section~\ref{sec-property-stproc}, to obtain the characterization of the asymptotic behavior of the averaged iterates $\{\bar \theta^{\alpha,\beta}_n\}$ given in the above theorem. 
Because such an analysis has already been carried out in detail, in the context of the ETD algorithm, in \cite[Sections~3.3,~4.3]{etd-wkconv} and the same proof arguments can be applied here, we do not repeat the proofs. We only remark that the analysis given in \cite[Sections~3.3,~4.3]{etd-wkconv} for ETD concerns the distances from the $\theta$-iterates to a single desired limit point. These distances are replaced, in the case here, by the distances from the $\theta$-iterates to the set $\Theta_{\text{opt}}$. Because the distance function, $\text{dist}(\theta, \Theta_{\text{opt}})$, is a convex function of $\theta$ (since the set $\Theta_{\text{opt}}$ is convex), the same analysis given in \cite{etd-wkconv} applies here with this replacement of the distant function.

\subsection{GTDb} \label{sec-gtd2}
Consider now a constrained version of the two-time-scale GTDb algorithm (\ref{gtd-x})-(\ref{gtd2-th}):
\begin{align}
\theta_{n+1} & =  \Pi_{B_\theta} \Big( \theta_n + \alpha_n \big(\e_n \delta_n(v_{\theta_n}) -   \rho_n (1 - \lambda_{n+1})  \,\gma_{n+1}  \phi(S_{n+1})  \cdot \e_n^\tr x_n \big) \Big), \label{eq-gtd2-th}\\
x_{n+1} & = \Pi_{B_x} \Big( x_n + \beta_n \big(\e_n \delta_n(v_{\theta_n}) - \phi(S_n) \phi(S_n)^\tr x_n\big) \Big). \label{eq-gtd2-x}
\end{align}
It differs from the previous constrained GTDa algorithm only in the $\theta$-iteration. Under the same conditions for GTDa, we show that it is minimizing $J(\theta)$ over $B_\theta$ with stochastic gradient-descent and has the same convergence properties.

The analysis is almost identical to that of GTDa given in the previous subsection. First, we write the iterates equivalently as 
\begin{align}
     \theta_{n+1} & = \Pi_{B_\theta} \big( \theta_n + \alpha_n ( g(\theta_n, x_n, Z_n) + \e_n \omega_{n+1}) \big), \label{eq-gtd2a-th}\\
      x_{n+1} & = \Pi_{B_x} \big( x_n + \beta_n ( k(\theta_n, x_n, Z_n) + \e_n \omega_{n+1}) \big), \label{eq-gtd2a-x}
\end{align}
where $Z_n$ and the function $k(\cdot)$ are the same as in the case of GTDa, and only the function $g(\cdot)$ is now different. It is defined according to (\ref{eq-gtd2-th}) as follows: with $z=(s, y, \e, s')$ (treat $y$ as a dummy variable in the case of state-dependent $\lambda$), for state-dependent $\lambda$,
\begin{equation}
   g(\theta, x, z) = \e \, \bar{\delta}(s, s', v_{\theta}) - \rho(s, s') \big(1 - \lambda(s')\big)  \gma(s') \phi(s') \cdot \e^\tr\! x,  
\end{equation} 
and for history-dependent $\lambda$, 
\begin{equation} \label{eq-gtd2-g}
  g(\theta, x, z) = \e \, \bar{\delta}(s, s', v_{\theta}) -  \rho(s, s') \big(1 - \lambda(f(y, s'), \e)\big)  \gma(s') \phi(s') \cdot \e^\tr\! x,
\end{equation} 
where $f$ is the function in (\ref{eq-histlambda}) that determines the next memory state.
We then consider the mean ODEs associated with the algorithm. For the fast time-scale, we calculated $\bar k(\theta, x) = \E_\zeta \big[ k(\theta, x, Z_0) \big]$ already in Section~\ref{sec-gtd1-odes}. For the slow time-scale, let $\bar x(\theta) = x_\theta$ as before (cf.\ (\ref{eq-xtheta})-(\ref{eq-xtheta2}) and (\ref{eq-xtheta3})), and let us calculate $\bar g(\theta) = \E_\zeta \big[ g(\theta, \bar x(\theta), Z_0) \big]$ for each fixed $\theta$.
Using the expressions (\ref{mean-exp2}) and (\ref{mean-exp4}) given in Prop.~\ref{prop-mean-fn} and using also the second expression (\ref{eq-Jgrad2}) of $\nabla J(\theta)$, we have
\begin{align*}
 \bar g(\theta) = \E_\zeta \big[ g(\theta, \bar x(\theta), Z_0) \big] & =  \E_\zeta \big[\e_0 \, \bar{\delta}_0(v_\theta) \big] -  \E_\zeta \big[  \, \rho_0 (1 - \lambda_{1})  \gma_{1}  \phi(S_{1}) \cdot \e_0^\tr \bar x(\theta)   \big]  \\
 & = \Phi^\tr \Xi \, (\Tl v_\theta - v_\theta) - \big(\Phi^\tr \Xi \, \Pl \Phi \big)^\tr x_\theta  \\
 & = - \nabla J(\theta),
\end{align*} 
which is also the same as in the GTDa case. Therefore, if we impose the same condition (\ref{eq-cond-Bx0}) on the constraint set $B_x$, the mean ODEs for the fast and slow time-scales are the same ODEs (\ref{eq-gtd1-odefast}) and (\ref{eq-gtd1-odeslow}) for GTDa that we discussed in Section~\ref{sec-gtd1-odes}.

Next we check if the algorithm satisfies the conditions listed in Section~\ref{sec-gtd1-cond} and in Remark~\ref{rmk-proofgtd}; if so, then we can apply the same convergence proof for GTDa. Clearly all the conditions that do not involve the function $g$ are satisfied, since only $g$ is different from the previous case. Those involving $g$ are also satisfied, and it is straightforward to verify them using the above definition of $g$ and the fact that for each $(\theta, x)$, the function $g(\theta, x, \cdot)$ is Lipschitz continuous in the trace variable $\e$. So we will omit the details except for one subtle point in the case of history-dependent $\lambda$. In that case, $g(\theta, x, z)$ is given by (\ref{eq-gtd2-g}), and for $g(\theta, x, \cdot)$ to be Lipschitz continuous in the trace variable $\e$, we need the term $\rho(s, s') \lambda(f(y, s'), \e) \gma(s') \phi(s') \cdot \e^\tr x$, as a function of $\e$ for given $x$ and $(s, y, s')$, to be Lipschitz continuous. This follows from Condition~\ref{cond-lambda}(i), which ensures the Lipschitz continuity of the function $\e \mapsto \lambda(\bar y, \e) \e$ for each given memory state $\bar y$, in the case of history-dependent $\lambda$.

We can now conclude that the convergence theorems we proved for GTDa also hold for GTDb:

\begin{thm} \label{thm-gtd2-wk}
Consider the two-time-scale GTDb algorithm (\ref{eq-gtd2-th})-(\ref{eq-gtd2-x}). Then in the case of diminishing (constant) stepsize, under the same conditions in Theorem~\ref{thm-gtd1-wk-dim} (Theorem~\ref{thm-gtd1-wk-constant}), the conclusions of Theorem~\ref{thm-gtd1-wk-dim} (Theorems~\ref{thm-gtd1-wk-constant} and \ref{thm-gtd1-average}) hold. 
\end{thm}

\subsection{Biased Variant of GTD} \label{sec-bias-vrt}

For the case of state-dependent $\lambda$, we discuss in this subsection a biased variant algorithm that can ``robustfy'' the off-policy TD algorithms against the high variance issue in off-policy learning. The price it pays for being more robust is that some bias is introduced, so that instead of approaching the set $\Theta_{\text{opt}}$, it can only approach a neighborhood of $\Theta_{\text{opt}}$, with the size of the neighborhood depending on the degree of its ``bias.''  The idea has been applied to the ETD algorithm in \cite{etd-wkconv} and to the off-policy LSTD algorithms in \cite{gbe_td17}, with its effectiveness observed experimentally and reported in those references.
In the present case of gradient-based TD algorithms, their variant algorithms can be regarded as approximate (stochastic) gradient-descent algorithms. 
 
In what follows, let us first describe the biased variant of GTD, before we explain further its motivation. A convergence analysis of the variant algorithm will be given after that, starting with its associated mean ODEs.

The biased variant we consider is, for the GTDa algorithm (\ref{eq-gtd1a-th})-(\ref{eq-gtd1a-x}), given by
\begin{align}
   \theta_{n+1} & = \Pi_{B_\theta} \Big( \theta_n + \alpha_n \,  \rho_n \big(\phi(S_n) - \gma_{n+1} \phi(S_{n+1}) \big) \cdot h(\e_n)^\tr x_n  \Big), \label{eq-gtd1b-th} \\
   x_{n+1} & = \Pi_{B_x} \Big( x_n + \beta_n \big(h(\e_n) \cdot \delta_n(v_{\theta_n}) - \phi(S_n) \phi(S_n)^\tr x_n\big) \Big), \label{eq-gtdb-x}
\end{align} 
where we have replaced $\e_n$ in the iterates with a certain \emph{bounded} continuous function $h(\e_n)$ of $\e_n$, in order to robustify GTDa.
A similar variant for the GTDb algorithm (\ref{eq-gtd2-th})-(\ref{eq-gtd2-x}) is 
\begin{equation}
  \theta_{n+1} = \Pi_{B_\theta} \Big( \theta_n + \alpha_n \big( h(\e_n)  \delta_n(v_{\theta_n}) -  \rho_n (1 - \lambda_{n+1})  \,\gma_{n+1}  \phi(S_{n+1}) \cdot h(\e_n)^\tr x_n \big) \Big), \label{eq-gtd2b-th}
\end{equation}
with the same iteration (\ref{eq-gtdb-x}) for $x_{n+1}$. 

Regarding the function $h$, besides boundedness, we require it to be Lipschitz continuous. 
In order to relate the degree of its ``bias'' to the behavior of the variant algorithm, we will consider a family of functions $h_K$ parametrized by $K > 0$, each of which can serve as $h$, and for each $K > 0$, we require that $h_K : \rn \to \rn$ is a bounded Lipschitz continuous function such that
\begin{equation} \label{eq-h}
  \| h_K(\e) \| \leq \| \e \| \ \  \forall \, \e \in \re^d, \quad \text{and} \quad h_K(\e) = \e \ \ \text{if} \ \| \e \| \leq K.
\end{equation}
The parameter $K$ reflects the degree of bias: the larger $K$ is, the smaller the bias. We shall also assume that
\begin{equation} \label{eq-h2}
 h_K(\e) \in \sp\{\phi(\S)\}, \qquad \forall \, \e \in \sp\{\phi(\S)\}.
\end{equation}
This is mostly for simplicity: when the matrix $\Phi$ has full column rank, (\ref{eq-h2}) is satisfied trivially; when $\Phi$ does not have full column rank, it allows us to keep exploiting the nice subspace $\sp\{\phi(\S)\}$ and focus the analysis on important issues, without being sidetracked. A simple example of functions with the properties (\ref{eq-h})-(\ref{eq-h2}) are functions that downscale $\e$ when $\|\e\|_2$ is too large. Simply truncating $\e$ componentwise using some large thresholds also works when $\Phi$ has full column rank.

Let us now explain the motivation for using bounded $h(\e_n)$ in place of $\e_n$ in the algorithms. With state-dependent (or constant) $\lambda$, in general, $\{\e_n\}$ can have unbounded variances and are also naturally unbounded in common off-policy learning situations (see a related result and its discussion in \cite[Prop. 3.1 and Footnote 3, p.\ 3320-3322]{Yu-siam-lstd}). One can understand this unboundedness behavior based on the ergodicity of the state-trace process: if the invariant probability measure $\zeta$ has an unbounded support on the trace space, then, since the process is ergodic, $\{\e_n\}$ will surely visit every part of the support however far away it is from the origin and however small its probability mass is under $\zeta$. When $\{\e_n\}$ visits such a distant part, its magnitude can stay large for many consecutive iterations since $\e_n$'s are calculated iteratively. This will result in drastic changes in the $(\theta_n, x_n)$-iterates, causing the algorithms to ``forget'' rapidly what they have ``learned,'' if no measure is taken to prevent such drastic changes from happening. If we replace $\e_n$ by $h(\e_n)$ in the $(\theta_n, x_n)$-iterates, a ``bias'' is introduced when $\e_n$ visits the distant part of its space on which $h(\e)$ ``downsizes'' $\e$. But if that part has a small probability mass under $\zeta$, the ``total bias'' introduced by using $h(\e_n)$ will also be small. 
This is the motivation of the above variant algorithm, and it is also reflected in the convergence analysis that we give below (cf.\ the proof of Lemma~\ref{lem-bias-appr-lim} below).

\subsubsection{Associated mean ODEs} \label{sec-vrt-ode}

We start by calculating the functions involved in the desired mean ODEs. We do this for the biased variant of GTDa; the case of GTDb is similar. 
We shall assume Assumption~\ref{cond-collective} and the conditions (\ref{eq-h})-(\ref{eq-h2}) on the Lipschitz continuous functions $h, h_K$ throughout the analysis; to be concise, we will not mention these conditions explicitly in the lemmas we derive below.
First, we rewrite the algorithm (\ref{eq-gtd1b-th})-(\ref{eq-gtdb-x}) as
\begin{align*}
     \theta_{n+1} & = \Pi_{B_\theta} \big( \theta_n + \alpha_n \, g_h(\theta_n, x_n, Z_n) \big),  \\
      x_{n+1} & = \Pi_{B_x} \big( x_n + \beta_n (k_h(\theta_n, x_n, Z_n) + h(\e_n) \, \omega_{n+1} ) \big), 
\end{align*}
where, with $z=(s,\e,s')$,
\begin{equation} \label{eq-biasgtd1-gk}
   g_h(\theta, x, z) =  \rho(s, s') \big(\phi(s) - \gma(s') \phi(s') \big) \cdot h(\e)^\tr x, \quad k_h(\theta, x, z) = h(\e) \, \bar{\delta}(s, s', v_{\theta}) - \phi(s) \phi(s)^\tr x.
\end{equation}
For each fixed $(\theta, x)$, since $g_h(\theta, x, z)$ and $k_h(\theta, x, z)$ are bounded functions of $z$, $\E_\zeta \big[ g_h(\theta, x, Z_0) \big]$ and $\E_\zeta \big[ k_h(\theta, x, Z_0) \big]$ are well-defined and finite. For the fast time-scale, we define
\begin{align}
  \bar k_h(\theta, x)  := \E_\zeta \big[ k_h(\theta, x, Z_0) \big]  
  & = \E_\zeta \big[ h(\e_0) \, \bar{\delta}_0(v_\theta) \big] - \E_\zeta \big[ \phi(S_0) \phi(S_0)^\tr \big] \, x \notag \\
  &  = \E_\zeta \big[ h(\e_0) \, \bar{\delta}_0(v_\theta) \big] - \Phi^\tr \Xi \, \Phi \, x.   \label{eq-gtd1b-bk} 
\end{align}
It is an affine function of $(\theta, x)$. Observe the following:

\begin{lem} \label{lem-vrt-xth}
For each $\theta$, the linear system of equations, $\bar k_h(\theta, x) = 0, x \in \sp \{\phi(\S)\}$,
has a unique solution $\bar x_h(\theta)$, and it is a continuous function of $\theta$.
\end{lem}
\begin{proof}
First, we show $\E_\zeta \big[ h(\e_0) \, \bar{\delta}_0(v_\theta) \big] \in \sp \{\phi(\S)\}$ (otherwise $\bar k_h(\theta, x) = 0, x \in \sp \{\phi(\S)\}$ does not have a solution). In view of the condition (\ref{eq-h2}) on $h(\cdot)$, it is sufficient to show that w.r.t.\ the invariant probability measure $\zeta$ of the state-trace process $\{(S_n, \e_n)\}$, the set $\S \times \sp \{\phi(\S)\}$ has measure $1$.  
By Theorem~\ref{thm-erg}(i), from any initial condition of $(S_0, \e_0)$, the sequence of occupation probability measures converge almost surely to $\zeta$. On the other hand, if we let $\e_0 \in \sp \{\phi(\S)\}$, then by the definition (\ref{eq-gtd-e}) of $\{\e_n\}$, $\e_n  \in \sp \{\phi(\S)\}$ for all $n$, so the occupation probability measures all assign measure $1$ to the set $\S \times \sp \{\phi(\S)\}$. It then follows that $\zeta$ must also assign measure $1$ to this set. Consequently, $\E_\zeta \big[ h(\e_0) \, \bar{\delta}_0(v_\theta) \big] \in \sp \{\phi(\S)\}$. From this it follows that the equation $\Phi^\tr \Xi \, \Phi \, x = \E_\zeta \big[ h(\e_0) \, \bar{\delta}_0(v_\theta) \big]$ has a unique solution in $\sp \{\phi(\S)\}$, or equivalently, the solution to $\bar k_h(\theta, x) = 0, x \in \sp \{\phi(\S)\}$ is unique. 

To see that the solution $\bar x_h(\theta)$ is continuous in $\theta$, note that $\|\bar x_h(\theta)\|_2 \leq \big\| \E_\zeta \big[ h(\e_0) \, \bar{\delta}_0(v_\theta) \big] \big\|_2 /c$, where $c$ is the smallest nonzero eigenvalue of the matrix $\Phi^\tr \Xi \, \Phi$. Since $\E_\zeta \big[ h(\e_0) \, \bar{\delta}_0(v_\theta) \big]$ is an affine function of $\theta$, the preceding bound on $\|\bar x_h(\theta)\|_2$ implies that for any point $\theta'$, as $\theta \to \theta'$, $\bar x_h(\theta)$ lies in a bounded set, and then the limit of any of its convergent subsequences must be a solution to $\bar k_h(\theta', x) = 0, x \in \sp \{\phi(\S)\}$, by the continuity of $\bar k_h(\cdot)$. But the latter equation has a unique solution $\bar x_h(\theta')$, so we must have $\bar x_h(\theta) \to \bar x_h(\theta')$ as $\theta \to \theta'$. This proves the continuity of $\bar x_h(\cdot)$.
\end{proof}

Henceforth, let $\bar x_h(\theta)$ be as in Lemma~\ref{lem-vrt-xth}.
The desired mean ODE for the fast time-scale is
\begin{equation} \label{eq-gtdvrt-odefast}
\qquad \left( \! \begin{array}{c} \dot{\theta}(t) \\ \dot x(t) \end{array} \!\right) 
 =  \left( \! \begin{array}{l} 0 \\
 \bar k_h \big(\theta(t), x(t) \big) + z(t)  \end{array} \!\right), \qquad \theta(0) \in B_\theta, \  z(t) \in - \N_{B_x}(x(t)),
\end{equation} 
where $z(t)$ is the boundary reflection term. Similar to the condition (\ref{eq-cond-Bx0}) on $B_x$ for GTDa, we shall require%footnote starts
\footnote{Note that a sufficient condition for (\ref{eq-cond-vrtBx0}) is that
$\text{radius}(B_x) \geq  \textstyle{\sup_{\theta \in B_\theta}} \big\| \E_\zeta \big[ h(\e_0) \, \bar{\delta}_0(v_\theta) \big] \big\|_2\, /c$,
where $c$ is the smallest nonzero eigenvalue of the matrix $\Phi^\tr \Xi \,\Phi$.}
%footnote ends 
\begin{equation} \label{eq-cond-vrtBx0}
  B_x \supset \{ \bar x_h(\theta) \mid \theta \in B_\theta \}.
\end{equation} 
This condition ensures that for an initial condition $\theta(0)=\theta$ and $x(0) \in \sp \{\phi(\S)\}$, the solution $x(t)$ of the ODE~(\ref{eq-gtdvrt-odefast}) converges to $\bar x_h(\theta)$ as $t \to \infty$. Moreover, we have an analogue of Lemma~\ref{lem-Bx} about the limit set of the ODE (\ref{eq-gtdvrt-odefast}), which follows from the same proof of Lemma~\ref{lem-Bx} (with $\bar k_h, \bar x_h$ in place of $\bar k, \bar x$) and which will be needed in analyzing the variant algorithm at the fast time-scale:

%\smallskip
\begin{lem} \label{lem-vrtBx}
Let the constraint set $B_x$ satisfy (\ref{eq-cond-vrtBx0}). Let $\bar x_h(\theta)$ be as in Lemma~\ref{lem-vrt-xth}. Then for all initial $x(0) \in B_x \cap \sp \{\phi(\S)\}$ and $\theta(0) \in B_\theta$,
the limit set of the ODE (\ref{eq-gtdvrt-odefast}) is
$$  \bigcap_{ \bar t \geq 0} \, \cl \,\Big\{ \big(\theta(t), x(t)\big) \, \Big| \,    \theta(0) \in B_\theta, \, x(0) \in B_x \cap \sp \{\phi(\S)\}, \, t \geq \bar t  \, \Big\}  = \big\{ \big(\theta, \bar x_h(\theta)\big) \mid \theta \in B_\theta \big\}.$$
\end{lem}

For the slow time-scale, we define, for each $\theta$,
\begin{align}   
 \bar g_h(\theta)  :=  \E_\zeta \big[ g_h(\theta, \bar x_h(\theta), Z_0) \big] =  \E_\zeta \big[ \,  \rho_0 \big(\phi(S_0) - \gma_{1}  \phi(S_{1}) \big) \cdot h(\e_0)^\tr \bar x_h(\theta) \big].   \label{eq-gtd1b-bg}
\end{align}
It is a continuous function of $\theta$, since the function $\bar x_h(\theta)$ is continuous.
The desired mean ODE is
\begin{equation} \label{eq-gtdvrt-odeslow}
  \dot \theta(t) = \bar g_h\big(\theta(t)\big) + z(t),  \quad z(t) \in - \N_{B_\theta}\big(\theta(t)\big),
\end{equation}
where $z(t)$ is the boundary reflection term. 
We want to characterize the relation between its limit set and the set $\Theta_{\text{opt}}$, when $h$ is chosen from the family of functions $h_K, K > 0$ mentioned earlier and with large values of $K$ (which correspond to small biases). 

Let $\bar x_K(\cdot)$ and $\bar g_K(\cdot)$ stand for the function $\bar x_h(\cdot)$ and $\bar g_h(\cdot)$, respectively, when $h = h_K$. 
The next lemma shows the approximate gradient-descent nature of the variant algorithm.

%\smallskip
\begin{lem}  \label{lem-bias-appr-lim}
$ \lim_{ K \to \infty} \sup_{\theta \in B_\theta} \big\| \bar g_K(\theta) + \nabla J(\theta) \big\| = 0.$
\end{lem} 

\begin{proof}
First, recall an expression of the gradient $\nabla J(\theta)$ from the analysis of the mean ODEs of GTDa in Section~\ref{sec-gtd1-odes} (cf.\ (\ref{eq-gtd1-bk})-(\ref{eq-gtd1-bg})): for each $\theta \in \re^d$,  
$\nabla J(\theta) = - \E_\zeta \big[ \, \rho_0 \big(\phi(S_0) - \gma_{1}  \phi(S_{1}) \big) \cdot \e_0^\tr x_\theta  \big]$, 
where $x_\theta$ is the unique solution of the equation $ \Phi^\tr \Xi \, \Phi \, x = \E_\zeta \big[ \e_0 \, \bar{\delta}_0(v_\theta) \big]$ in $\sp\{\phi(\S)\}$.
Then by (\ref{eq-gtd1b-bg}),
\begin{align*}
    \bar g_K(\theta) + \nabla J(\theta)  & =  \E_\zeta \big[ \, \rho_0 \big(\phi(S_0) - \gma_{1}  \phi(S_{1})  \big) \cdot h_K(\e_0)^\tr \bar x_K(\theta)  \big]
      -    \E_\zeta \big[ \,  \rho_0 \big(\phi(S_0) - \gma_{1}  \phi(S_{1}) \big) \cdot \e_0^\tr x_\theta \big] \\
      & =  \E_\zeta \big[ \, \rho_0 \big(\phi(S_0) - \gma_{1}  \phi(S_{1}) \big) \cdot h_K(\e_0)^\tr \big( \bar x_K(\theta) - x_\theta \big) \big] \\
      & \quad \
           +  \E_\zeta \big[ \,  \rho_0 \big(\phi(S_0) - \gma_{1}  \phi(S_{1}) \big) \cdot  \big(h_K(\e_0) - \e_0   \big)^\tr x_\theta \big].
\end{align*}
So to prove the lemma, it suffices to prove that as $K \to \infty$,
\begin{equation} \label{eq-prff1}
       \sup_{\theta \in B_\theta}  \big\|   \bar x_K(\theta) - x_\theta \big\| \to 0 \quad
\text{and} \quad  \E_\zeta \big[ \big\| h_K(\e_0) - \e_0 \big\|  \big]  \to 0.
\end{equation}
Since $\E_\zeta \big[ \big\| h_K(\e_0) - \e_0 \big\|  \big] \leq 2 \E_\zeta \big[ \| \e_0 \| \I(\| \e_0\| > K) \big]$ by (\ref{eq-h}), the second relation in (\ref{eq-prff1}) follows from $\E_\zeta[ \| \e_0 \| ] < \infty$ (Theorem~\ref{thm-erg}(ii)). To show the first relation in (\ref{eq-prff1}), note that by the definition of $\bar x_K(\theta)$ (cf.\ Lemma~\ref{lem-vrt-xth}), $\bar x_K(\theta) - x_\theta$ is the unique solution of the following equation (in $x$) in the subspace $\sp \{\phi(\S)\}$:
\begin{equation} \label{eq-prff2}
   \Phi^\tr \Xi \, \Phi \, x = \E_\zeta \big[ \big(h_K(\e_0)  - \e_0 \big) \cdot \bar{\delta}_0(v_\theta) \big].
\end{equation}
This implies that $\| \bar x_K(\theta) - x_\theta\|_2 \leq \big\| \E_\zeta \big[ \big(h_K(\e_0)  - \e_0 \big) \cdot \bar{\delta}_0(v_\theta) \big] \big\|_2 /c$, where $c$ is the smallest nonzero eigenvalue of the matrix $\Phi^\tr \Xi \, \Phi$. On the other hand, 
$$ \sup_{\theta \in B_\theta} \E_\zeta \big[ \big\| h_K(\e_0)  - \e_0 \big\| \cdot |\bar{\delta}_0(v_\theta)| \big]  \to 0 \quad \text{as} \ K \to \infty,$$
since $\sup_{\theta \in B_\theta} \sup_{s, s' \in \S} | \bar{\delta}(s, s', v_\theta) | < \infty$ and $\lim_{K \to \infty} \E_\zeta \big[ \big\|h_K(\e_0)  - \e_0 \big\| \big]  = 0$ as just proved. Therefore, $\sup_{\theta \in B_\theta} \big\|\bar x_K(\theta) - x_\theta \big\| \to 0$ as $K \to \infty$, proving the first relation in (\ref{eq-prff1}). The proof is now complete.
\end{proof}

Below is an analogue of Lemma~\ref{lem-gtd1-odeslow} about the limit set of the ODE (\ref{eq-gtdvrt-odeslow}), which will be needed in analyzing the variant algorithm at the slow time-scale:

\begin{lem} \label{lem-gtdvrt-odeslow}
For any $\epsilon > 0$, there exists $K_\epsilon > 0$ such that if $K \geq K_\epsilon$, then with $\bar g = \bar g_K$,
the limit set of the ODE (\ref{eq-gtdvrt-odeslow}) satisfies
$\bigcap_{\, \bar t \geq 0}  \cl \, \big\{\theta(t)  \, \big| \,   \theta(0) \in B_\theta, \, t \geq \bar t  \, \big\}  \subset N_\epsilon ( \Theta_{\text{\rm opt}}).$
\end{lem}

\begin{proof}
Consider a solution $\theta(\cdot)$ of (\ref{eq-gtdvrt-odeslow}) with $\bar g_h = \bar g_K$ for some $K$. At time $t$, denote $D=\N_{B_{\theta}}(\theta(t))$ and recall that the boundary reflection term $z(t) = - \Pi_{D} \bar g_K(\theta(t))$. Define $\tilde z(t) = - \Pi_{D} \big(- \nabla J(\theta(t))\big)$.
Then, by the non-expansiveness of the projection operator $\Pi_D$,
\begin{equation} 
   \big\| z(t) - \tilde z(t) \big\|_2 \leq  \big\| \bar g_K(\theta(t)) + \nabla J(\theta(t)) \big\|_2. \label{eq-prff3}
\end{equation}   
We now calculate $\dot V(\theta(t))$ for the function $V(\theta) = J(\theta)$:
\begin{align}
 \dot V(\theta(t)) & = \big\langle \nabla J(\theta(t)), \, \bar g_K(\theta(t)) + z(t) \big\rangle \notag \\
  & =   \big\langle \nabla J(\theta(t)), \,  - \nabla J (\theta(t)) + \tilde z(t) \big\rangle + \big\langle \nabla J(\theta(t)), \, \bar g_K(\theta(t)) + z(t) + \nabla J (\theta(t)) - \tilde z(t)  \big\rangle  \notag \\
  & \leq \big\langle \nabla J(\theta(t)), \,  - \nabla J (\theta(t)) + \tilde z(t) \big\rangle  + \big\|  \nabla J(\theta(t)) \big\|_2 \cdot \big\| \, \bar g_K(\theta(t)) + z(t)  + \nabla J(\theta(t)) -  \tilde z(t) \big\|_2  \notag \\
  & \leq \big\langle \nabla J(\theta(t)), \,  - \nabla J (\theta(t)) + \tilde z(t) \big\rangle  + \big\|  \nabla J(\theta(t)) \big\|_2 \cdot   2 \, \big\| \, \bar g_K(\theta(t)) + \nabla J(\theta(t)) \big\|_2 \, , \label{eq-prff4}
\end{align}
where the last inequality follows from (\ref{eq-prff3}) and the triangle inequality.
As we showed in the proof of Lemma~\ref{lem-gtd1-odeslow} for GTDa, the first term in (\ref{eq-prff4}) is always nonpositive and moreover,
for any $0 < \epsilon \leq \sup_{\theta \in B_\theta} J(\theta) - \inf_{\theta \in B_\theta} J(\theta)$, there exists $\eta_\epsilon > 0$ such that
\begin{equation} \label{eq-prff5}
   \big\langle \nabla J(\theta(t)), \,  - \nabla J (\theta(t)) + \tilde z(t) \big\rangle \leq - \eta_\epsilon  \quad \text{if} \  J(\theta(t)) \geq \textstyle{\inf_{\theta \in B_\theta}} J(\theta) +\epsilon.
\end{equation}
For the second term in (\ref{eq-prff4}), we can upper-bound it by 
$$\Delta_K : = \textstyle{\sup_{\theta \in B_\theta}} \big\| \nabla J(\theta) \big\|_2 \cdot 2 \, \textstyle{\sup_{\theta \in B_\theta}} \big\| \bar g_K(\theta) + \nabla J(\theta) \big\|_2,$$
which, by Lemma~\ref{lem-bias-appr-lim}, converges to $0$ as $K \to +\infty$. Therefore, we can choose $\bar K_\epsilon$ sufficiently large so that $\eta_\epsilon - \Delta_K \geq \eta_\epsilon/2$ for all $K \geq \bar K_\epsilon$. Together with (\ref{eq-prff4}) and (\ref{eq-prff5}), this shows that for any solution of the ODE corresponding to such $K$, there exists a time $\tau_\epsilon$ (which depends on $\eta_\epsilon$ but not on $K$) such that whenever $\theta(t)$ is outside the set 
$E_\epsilon := \{ \theta \mid J(\theta) \leq \inf_{\theta \in B_\theta} J(\theta) + \epsilon \} \cap B_\theta$, it will visit $E_\epsilon$ before the amount of time $\tau_\epsilon$ has elapsed.

Now assume that at some moment $t_0$, $\theta(t_0) \in E_\epsilon$. Let us check how far away from the set $E_\epsilon$ the solution $\theta(t)$ can wander before it revisits this set. During such an excursion, by (\ref{eq-prff4}), the maximal amount of increase in $V(\theta)$ relative to $V(\theta(t_0))$ is bounded by $\tau_{\epsilon} \Delta_K$. Hence, once the solution hits $E_\epsilon$, it will not leave the set $E_{\ell}$, where $\ell =  \tau_\epsilon \Delta_K + \epsilon$.

Finally, to finish the proof, consider any given $\epsilon' > 0$ (it is the given $\epsilon$ in the lemma). There exists a sufficiently small $\tilde \epsilon > 0$ such that $E_{\tilde \epsilon} \subset N_{\epsilon'}(\Theta_{\text{opt}})$.%footnote starts
\footnote{Otherwise, there would be a sequence of points $\{\theta_i\}$ in $B_\theta$ that converges to $\theta_\infty$ and has the property that $\lim_{i \to \infty} J(\theta_i) = J(\theta_\infty) = \inf_{\theta \in B_\theta} J(\theta)$ and yet $\theta_\infty \not\in \Theta_{\text{opt}}$, a contradiction.}
%footnote ends
We first choose $\epsilon < \tilde \epsilon$ and then choose a threshold $\bar K' \geq \bar K_\epsilon$ sufficiently large so that $\ell < \tilde \epsilon$ and hence $E_{\ell} \subset E_{\tilde \epsilon}$ for all $K \geq \bar K'$. (This is possible because in the expression of $\ell$ above, we have $\Delta_K \to 0$ as $K \to \infty$ by Lemma~\ref {lem-bias-appr-lim}.)
It then follows that for all $K \geq \bar K'$, the limit set of the ODE~(\ref{eq-gtdvrt-odeslow}) with $\bar g = \bar g_K$ must be contained in $E_{\tilde \epsilon} \subset N_{\epsilon'}(\Theta_{\text{opt}})$.
\end{proof}

\subsubsection{Convergence properties}

We are now ready to address the convergence properties of the variant algorithm. 
It is straightforward to check that all the conditions used in our convergence proof of GTDa are satisfied by the variant algorithm, by using the definition of the functions $g_h, k_h$ and by using the results about the desired mean ODEs for the fast and slow time-scales that we just derived.%footnote starts
\footnote{In particular, among the conditions listed in Section~\ref{sec-gtd1-cond}, the uniform integrability conditions are trivially satisfied because the boundedness of the function $h(\e)$ makes $g_h(\theta, x, z)$ and $k_h(\theta, x, z)$ bounded on $ B_\theta \times B_x \times \Z$. The functions $g_h(\theta, x, z)$ and $k_h(\theta, x, z)$ are simply affine functions of $(\theta, x)$ for each $z$. By their definitions, the uniform continuity condition (iii) on $g_h$ and $k_h$ are also trivially satisfied, and so does the Lipschitz continuity condition on $g_h(\theta, \cdot, z)$ mentioned in Remark~\ref{rmk-proofgtd}, which is required in the convergence proof. For the ``averaging conditions'' (iv) and (vii), as the proof of Lemma~\ref{lem-conv-mean-cond1} showed, they are satisfied if for each $\theta$ and $x$, the functions $k_h(\theta, x, z)$ and $g_h(\theta, \bar x(\theta), z)$ are Lipschitz continuous in the trace variable $\e$. This is true since $h(\e)$ is Lipschitz continuous by assumption.}
%footnote ends
We can thus apply the convergence proof for GTDa to the variant algorithm for each choice $h = h_K$, with the corresponding condition (\ref{eq-cond-vrtBx0}) on the constraint set $B_x$. (We remark that the threshold for the radius will depend on $K$, but is bounded above for all $K$ because of (\ref{eq-h}).)
The result is that the conclusions of Theorems~\ref{thm-gtd1-wk-dim}-\ref{thm-gtd1-wk-constant} for GTDa hold also in this case, with the limit set of the slow-time-scale ODE~(\ref{eq-gtdvrt-odeslow}) in place of the set $\Theta_{\text{\rm opt}}$. We then combine these conclusions with Lemma~\ref{lem-gtdvrt-odeslow} to obtain Theorem~\ref{thm-biasgtd1-wk}(i)-(ii) given below.

The theorem is stated for both GTDa and GTDb variants---the analysis of the GTDb variant is essentially the same as that of the GTDa variant we gave.
The conclusions about the biased variant algorithms are very similar to those about GTD given earlier (cf.\ Theorems~\ref{thm-gtd1-wk-dim}-\ref{thm-gtd2-wk}). The only difference is that here, in order for the algorithms to approach a given small neighborhood of $\Theta_{\text{\rm opt}}$, $K$ needs to be sufficiently large (so that the bias is sufficiently small). 

The property of the averaged iterates is addressed in the part (iii) of the theorem, where $\N'_\epsilon(\Theta_{\text{\rm opt}})$ is the open $\epsilon$-neighborhood of $\Theta_{\text{\rm opt}}$, and the set $\bar{\mathcal{I}}^{\alpha,\beta}_m$ of probability measures is as defined in Section~\ref{sec-aveite} before Theorem~\ref{thm-gtd1-average}.  Like the latter theorem, the proof of this part is essentially the same as that given in \cite[Section 4.3]{etd-wkconv} in the context of ETD and therefore omitted. (The result in the part (iii) is analogous to \cite[Theorem~9, Section~3.3]{etd-wkconv} for biased variants of ETD.) Compared with Theorem~\ref{thm-gtd1-average} for GTDa, the only difference is that here, for a desired small $\epsilon$-neighborhood $\N'_\epsilon(\Theta_{\text{\rm opt}})$, in order for the minimal probability mass $\kappa_m^{\alpha,\beta}$ for this set to increase to $1$ as $(\alpha, \beta) \to 0$ and $\alpha/\beta \to 0$, $K$ needs to be sufficiently large (i.e., the bias needs to be sufficiently small).

\begin{thm} \label{thm-biasgtd1-wk}
For the case of state-dependent $\lambda$, consider the two-time-scale biased GTD algorithm (\ref{eq-gtd1b-th})-(\ref{eq-gtdb-x}) or (\ref{eq-gtdb-x})-(\ref{eq-gtd2b-th}). In the algorithm, let the function $h \in \{ h_K \mid K > 0\}$, a family of bounded Lipschitz continuous functions that satisfy (\ref{eq-h})-(\ref{eq-h2}).
In addition, let the same conditions in Theorem~\ref{thm-gtd1-wk-dim} (Theorem~\ref{thm-gtd1-wk-constant}) hold for the case of diminishing (constant) stepsize, except that the condition on the constrained set $B_x$ is replaced by (\ref{eq-cond-vrtBx0}) for the chosen $h=h_K$.
Then for each initial condition of the algorithm, the following hold:\vspace*{-3pt}
\begin{itemize}
\item[\rm (i)] In the case of diminishing stepsize, for each $\epsilon > 0$, there exists $K_\epsilon > 0$ such that if $K \geq K_\epsilon$, then 
it holds for some sequence of positive numbers $T_n \to \infty$ that 
\begin{equation} \label{eq-bias-wk-dim}
  \limsup_{n \to \infty} \Pr \big( \theta_i \not\in N_\epsilon(\Theta_{\text{\rm opt}}), \ \text{some} \ i \in [n, m(n, T_n) ] \big)  = 0.
\end{equation}  
\item[\rm (ii)] In the case of constant stepsize, for each $\epsilon > 0$, there exists $K_\epsilon > 0$ with the following property. If $K \geq K_\epsilon$, then for any integers $n_{\alpha}$ that satisfy $\alpha \, n_{\alpha} \to \infty$ as $\alpha \to 0$,
there exist positive numbers $\{T_{\alpha, \beta} \mid \alpha, \beta > 0\}$ with $T_{\alpha, \beta} \to \infty$ as $\beta \to 0$ and $\alpha/\beta \to 0$, such that,
\begin{equation} \label{eq-bias-wk-constant1}
  \limsup_{\beta \to 0, \, \alpha/\beta \to 0} \Pr \big( \theta^{\alpha,\beta}_{n_\alpha + i} \not\in N_\epsilon(\Theta_{\text{\rm opt}}), \ \text{some} \ i \in [0, T_{\alpha,\beta}/\alpha ] \big)  = 0.
\end{equation}
\item[\rm (iii)] In the case of constant stepsize, consider the averaged iterates $\{\bar \theta^{\alpha,\beta}_n\}$ defined by (\ref{eq-aveite}). Then 
for each $\epsilon > 0$, there exists $K_\epsilon > 0$ such that if $K \geq K_\epsilon$, then it holds for all $m \geq 1$ that
$\liminf_{\beta \to 0, \, \alpha/\beta \to 0} \, \inf_{\mu \in \bar{\mathcal{I}}^{\alpha,\beta}_m} \mu \big( \big[N'_\epsilon(\Theta_{\text{\rm opt}})\big]^m \big) = 1$,
and more strongly, with $m_\alpha = \lceil \tfrac{m}{\alpha} \rceil$,
$$\liminf_{\beta \to 0, \, \alpha/\beta \to 0} \, \inf_{\mu \in \bar{\mathcal{I}}^{\alpha,\beta}_{m_\alpha}} \mu \big( \big[N'_\epsilon(\Theta_{\text{\rm opt}})\big]^{m_\alpha} \big) = 1.$$
Furthermore, regardless of the choice of $K$, for each $\epsilon > 0$ and each stepsize pair $(\alpha, \beta)$,
$$ \limsup_{n \to \infty} \E \Big[ \max_{n \leq i < n+m} \!\text{\rm dist} \big( \bar \theta_i^{\alpha,\beta}, \Theta_{\text{\rm opt}} \big) \Big] \leq \epsilon \kappa^{\alpha,\beta}_m + 2 \, r_{B_\theta} ( 1 - \kappa_m^{\alpha,\beta}),$$
where $r_{B_\theta}$ is the radius of $B_\theta$, and $\kappa_m^{\alpha,\beta} = \inf_{\mu \in \bar{\mathcal{I}}^{\alpha,\beta}_m} \mu \big( \big[N'_\epsilon(\Theta_{\text{\rm opt}})\big]^m \big)$. 
\end{itemize}
\end{thm}
%\smallskip

\subsection{Extension to a Composite Scheme of Setting $\lambda$-Parameters} \label{sec-cp-extension}

In this subsection we explain a composite scheme of setting the $\lambda$-parameters, 
which builds upon the state-dependent and history-dependent $\lambda$ considered so far in the paper but is more general than them.
With this composite scheme, a larger set of generalized Bellman operators can be utilized for the approximate policy evaluation task.
We will explain this and show how the GTD algorithms can be simply modified to minimize the projected-Bellman-error objective function $J(\theta)$, for the Bellman operators $\Tl$ associated with this scheme. A convergence theorem for the GTD algorithms will be given at the end. 

We shall use primarily the GTDa algorithm as an illustrative example, and mention the GTDb algorithm near the end, which is very similar. One can apply the same arguments given below to other algorithms discussed earlier or to be discussed in the next section. 

\subsubsection{Multiple concurrent trace processes and associated Bellman operators} 

To explain the composite scheme of setting $\lambda$ and its underlying idea, it is convenient to start with an illustrative example of an algorithm that uses such a scheme. Consider the GTDa algorithm (\ref{eq-gtd1a-th})-(\ref{eq-gtd1a-x}). 
Let $(\theta_n, x_n)$ be generated according to these same formulae, but with the trace iterates $\e_n$ calculated differently as follows. 
Partition the state space $\S$ into disjoint subsets, say, $\S = \cup_{i=1}^\ell \S_i$. 
Associate each subset $\S_i$ with a way of setting the $\lambda$-parameters $\lambda_n^{(i)}$ and with a trace sequence $\e_n^{(i)}$ generated according to the recursion
\begin{equation} \label{eq-decomp-e}
    \e^{(i)}_{n} = \lambda_n^{(i)} \gma_n \rho_{n-1} \e^{(i)}_{n-1} + \phi(S_n) \cdot \I(S_n \in \S_i).
\end{equation}
Then take their sum to be the trace $\e_n$ used in the algorithm:
\begin{equation} \label{eq-comp-e}
     \e_n = \textstyle{\sum_{i=1}^\ell} \e^{(i)}_n.
\end{equation} 
As will be seen shortly, the idea of the composite scheme is quite general. Here, for the purpose of reusing the previous analyses, however, we shall restrict attention to the case where for each $i \leq \ell$, the choice of $\lambda_n^{(i)}$ is in the two classes (state-dependent and history-dependent $\lambda$) discussed in Section~\ref{sec-choice-lambda} and obeys the conditions given there. Specifically, either $\lambda_n^{(i)} = \lambda^{(i)}(S_n)$ for a given function $\lambda^{(i)}$, or $\lambda_n^{(i)}$ is determined through (\ref{eq-histlambda}) based on the memory states and the previous trace \emph{in the $i$-th process}: 
\begin{equation}
  y_n = f(y_{n-1}, S_n), \qquad \lambda^{(i)}_n=\lambda^{(i)}(y_n, \e^{(i)}_{n-1}). \label{eq-cp-histlambda}
\end{equation}
In the latter case, we require that the memory states $\{y_n\}$ satisfy Condition~\ref{cond-mem}, and the $\lambda^{(i)}$ function satisfies Condition~\ref{cond-lambda}. 

The idea behind (\ref{eq-decomp-e})-(\ref{eq-comp-e}) is that by decomposing the trace $\e_n$ in this manner, one can compose a generalized Bellman operator $\Tl$ from selected component mappings of other Bellman operators and minimize the projected-Bellman-error objective function $J(\theta)$ for that $\Tl$. As a simple example, if $\S = \S_1 \cup \S_2$ and $T^{(0)}, T^{(1)}$ are the Bellman operators (for the target policy) associated with the constant $\lambda = 0$ and $\lambda = 1$, respectively, one can employ in $J(\theta)$ a composite $\Tl$ given by
$$ (\Tl v)(s) = (T^{(0)} v)(s) \ \ \ \text{if} \ s \in \S_1;  \qquad (\Tl v)(s) = (T^{(1)} v)(s) \ \ \ \text{if} \ s \in \S_2.$$
To minimize $J(\theta)$ for this $\Tl$ using GTDa, one just modifies the definition of $\e_n$ in the algorithm as shown above; namely, let $\e_n = \e^{(1)}_n + \e^{(2)}_n$, where $\e^{(1)}_n$ and $\e^{(2)}_n$ are updated according to (\ref{eq-decomp-e}) with $\lambda_n^{(1)} = 0$ and $\lambda_n^{(2)} = 1$, respectively. Such composite $\Tl$ is not possible with the form of state-dependent or history-dependent $\lambda$-parameters we considered earlier. By using such Bellman operators in the objective function $J(\theta)$, one can better regulate what information to use for approximating the value function $v_\pi$ at specific states. 

This idea of setting $\lambda$ and its implementation through (\ref{eq-decomp-e})-(\ref{eq-comp-e}) were proposed by the author in \cite{yb-bellmaneq}. The properties of the associated state-trace processes and LSTD algorithms were further analyzed in \cite{Yu-siam-lstd,gbe_td17}, with the second reference extending the scheme to include history-dependent $\lambda$'s. We refer the reader to \cite[Section 3]{gbe_td17} (especially Section 3.1 and the discussion before Corollary 3.1 in Section 3.2 therein) for the details of the generalized Bellman operator $\Tl$ associated with (\ref{eq-decomp-e})-(\ref{eq-comp-e}). As before, for the purpose of this paper, we do not need the details of $\Tl$, and what we care about is that the algorithms with the above redefinition of the trace iterates are indeed minimizing $J(\theta)$ for the corresponding Bellman operator $\Tl$. So we will only describe $\Tl$ at an abstract level in terms of other Bellman operators considered earlier, before we proceed to discuss the convergence properties of the GTD algorithms.

With the new definition (\ref{eq-decomp-e})-(\ref{eq-comp-e}) for the trace iterates, instead of a single state-trace process, there are now $\ell$ concurrent processes, $\big\{ \big(S_n, y_n, \e_n^{(i)} \big) \big\}$, $i \leq \ell$, which share the same state and memory state variables (treat the memory states $y_n$ as dummy variables if state-dependent $\lambda$'s are used in the recursion (\ref{eq-decomp-e}) for the $i$-th trace process).
Since the choice of $\lambda$ in each of these state-trace processes is assumed to satisfy the same conditions as required before, there is a Bellman operator $T^{(\lambda^{(i)})}$ associated with the $i$-th process for each $i \leq \ell$ (cf.\ Section~\ref{sec-choice-lambda}), and the theorems given in Section~\ref{sec-property-stproc} all hold for each of these $\ell$ state-trace processes individually.

Consider a Bellman operator $\Tl$ composed from the component mappings of $T^{(\lambda^{(i)})}$, $i \leq \ell$, as follows: for any $v \in \re^{|\S|}$,
\begin{equation} \label{eq-Tcomp}
  (\Tl v)(s)  := \big( T^{(\lambda^{(i)})} v \big)(s)  \quad \text{for} \ s \in \S_i.
\end{equation}
This is the generalized Bellman operator corresponding to the above composite scheme of setting the $\lambda$-parameters. Let us now proceed to show that the correspondingly modified GTD algorithms minimize the projected-Bellman-error objective function $J(\theta)$, for this Bellman operator $\Tl$, and that the convergence properties we derived earlier GTD continue to hold for the modified algorithms.

\subsubsection{Convergence analysis for GTD}

To present the analysis, we continue to use the GTDa algorithm as an illustrative example. (The corresponding GTDb algorithm will be addressed at the end; see (\ref{eq-gtd2c-th})-(\ref{eq-gtd2c-x}).)
Substituting the new definition (\ref{eq-comp-e}) for $\e_n$, GTDa now becomes
\begin{align}
   \theta_{n+1} & = \Pi_{B_\theta} \Big( \theta_n + \alpha_n \, \textstyle{\sum_{i=1}^\ell} \rho_n \big(\phi(S_n) - \gma_{n+1} \phi(S_{n+1}) \big) \cdot (\e_n^{(i)})^\tr x_n  \Big),  \label{eq-gtd1c-th}\\
   x_{n+1} & = \Pi_{B_x} \Big( x_n + \beta_n \big( \textstyle{ \sum_{i=1}^\ell} \e_n^{(i)} \delta_n(v_{\theta_n}) - \phi(S_n) \phi(S_n)^\tr x_n\big) \Big). \label{eq-gtd1c-x}
\end{align} 
The first step is to identify the mean ODEs. By Theorem~\ref{thm-erg}, for each $i \leq \ell$, the Markov chain $\big\{ \big(S_n, y_n, \e_n^{(i)} \big) \big\}$ has a unique invariant probability measure $\zeta^{(i)}$. Since conditioned on the states and memory states, the evolution of traces in each of the $\ell$ processes is independent of one another, the joint state-trace process $\{(S_n, y_n, \e_n^{(1)}, \ldots, \e_n^{(\ell)} )\}$ also has a unique invariant probability measure $\tilde \zeta$, which is determined by the collection $\zeta^{(i)}$, $i \leq \ell$. Denote expectation w.r.t.\ the stationary joint state-trace process by $\E_{\tilde \zeta}$. To identify the functions in the mean ODEs, we calculate the expectations of several functions of the states and traces, for those functions involved in the algorithm, w.r.t.\ the stationary joint state-trace process $\{(S_n, y_n, \e_n^{(1)}, \ldots, \e_n^{(\ell)} )\}$.

Consider first the functions involved in the iteration (\ref{eq-gtd1c-x}) at the fast time-scale. 
We have $\E_{\tilde \zeta} \big[ \phi(S_0) \phi(S_0)^\tr \big] = \Phi^\tr \Xi \, \Phi$ as before.
For each fixed $\theta$ and each $i \leq \ell$, applying Prop.~\ref{prop-mean-fn} to the $i$-th state-trace process, we have
\begin{equation} \label{eq-cp-calc1}
 \E_{\tilde \zeta} \big[ \e_0^{(i)} \, \bar \delta_0(v_{\theta}) \big] = \E_{\zeta^{(i)}} \big[ \e_0^{(i)} \, \bar \delta_0(v_{\theta}) \big] =  \big(\Phi^{(i)}\big)\!^\tr \Xi \, \big(T^{(\lambda^{(i)})} v_\theta - v_\theta \big),
\end{equation} 
where $\Phi^{(i)}$ is the matrix whose $s$-th row equals $\phi(s)^\tr$ if $s \in \S_i$ and equals the zero vector otherwise (cf.\ the definition (\ref{eq-decomp-e}) of $\e_n^{(i)}$). Now by the definition of $\Phi^{(i)}$ and the definition (\ref{eq-Tcomp}) of $\Tl$, 
$$ \textstyle{\sum_{i = 1}^\ell} \big(\Phi^{(i)}\big)\!^\tr \Xi \, \big(T^{(\lambda^{(i)})} v_\theta - v_\theta \big)  = \Phi^\tr \Xi \, \big(\Tl v_\theta - v_\theta \big),$$
so we obtain that for each fixed $(\theta, x)$,
\begin{equation} \label{eq-cp-calc2}
   \E_{\tilde \zeta} \big[ \textstyle{\sum_{i = 1}^\ell}  \e_0^{(i)} \, \bar \delta_0(v_{\theta}) \big] - \E_{\tilde \zeta} \big[ \phi(S_0) \phi(S_0)^\tr \big] x =  \Phi^\tr \Xi \, \big(\Tl v_\theta - v_\theta \big)  - \Phi^\tr \Xi \, \Phi \, x  = : \bar k(\theta, x).
\end{equation}
Note that the function $\bar k(\theta, x)$ here is exactly the one for GTDa in Section~\ref{sec-gtd1-odes}, apart from involving a different $\Tl$.

Thus, for the mean ODE associated with the fast time-scale, we can proceed as in Section~\ref{sec-gtd1-odes}: For each $\theta$, define $\bar x(\theta)$ to be the unique solution $x_\theta$ to $\bar k(\theta, x) = 0, x \in \sp \{\phi(\S)\}$. The mean ODE is the same as (\ref{eq-gtd1-odefast}) (apart from a different definition of $\Tl$). If we place the same condition (\ref{eq-cond-Bx0}) on the constraint set $B_x$ (with $x_\theta$ being defined w.r.t.\ the Bellman operator $\Tl$ here), then as before, for all initial conditions $x(0) \in \sp \{\phi(\S)\}$, the solution $x(t)$ of the ODE converges to $\bar x(\theta(0))$ and the limit set of the ODE is characterized by Lemma~\ref{lem-Bx}.

Consider now the iteration (\ref{eq-gtd1c-th}) at the slow time-scale. In terms of the function $g$ defined in (\ref{eq-gtd1-gk}) earlier, with $Z^{(i)}_n = (S_n, y_n, \e_n^{(i)}, S_{n+1})$, we can write (\ref{eq-gtd1c-th}) equivalently as
\begin{equation}
  \theta_{n+1}  = \Pi_{B_\theta} \big( \theta_n + \alpha_n \, \textstyle{\sum_{i = 1}^\ell} \, g(\theta_n, x_n, Z^{(i)}_n) \big). \label{eq-gtd1c-altth}
\end{equation}
For each fixed $(\theta, x)$ and each $i \leq \ell$, by Prop.~\ref{prop-mean-fn} applied to the $i$-th state-trace process, we have
\begin{align}
  \E_{\tilde \zeta} \big[ g(\theta, x, Z^{(i)}_n) \big] & = \E_{\zeta^{(i)}} \big[ g(\theta, x, Z^{(i)}_n) \big] \notag \\
      & =   \E_{\zeta^{(i)}} \big[ \,  \rho_0 \big(\phi(S_0) - \gma_{1}  \phi(S_{1}) \big) \cdot (\e_0^{(i)})^\tr x \big] \notag \\
       & = \big[ \big(\Phi^{(i)}\big)\!^\tr \Xi \, (I - P^{(\lambda^{(i)})})  \, \Phi \big]^\tr x, \label{eq-cp-calc3}
\end{align}
where $\Phi^{(i)}$ is as defined above, and $P^{(\lambda^{(i)})}$ is the substochastic matrix in the operator $T^{(\lambda^{(i)})}$. 
Thus,
\begin{align}
    \E_{\tilde \zeta} \big[ \, \textstyle{\sum_{i = 1}^\ell} \, g(\theta, x, Z^{(i)}_n) \big]  & =  \textstyle{\sum_{i = 1}^\ell} \, \big[ \big(\Phi^{(i)}\big)\!^\tr \Xi \, (I - P^{(\lambda^{(i)})})  \, \Phi \big]^\tr x  \notag \\
    & = \big[ \Phi^\tr \Xi \, (I - \Pl) \, \Phi \big]^\tr x, \label{eq-cp-calc4}
\end{align}
where $\Pl$ is the substochastic matrix in the affine operator $\Tl$, and the second equality follows from the definition of $\Phi^{(i)}$ and the definition (\ref{eq-Tcomp}) of $\Tl$. Then for $x = \bar x(\theta)$, (\ref{eq-cp-calc4}) yields
\begin{equation}
     \E_{\tilde \zeta} \big[ \, \textstyle{\sum_{i = 1}^\ell} \, g(\theta, \bar x(\theta), Z^{(i)}_n) \big] = \big[ \Phi^\tr \Xi \, (I - \Pl) \, \Phi \big]^\tr \bar x(\theta)  =: \bar g(\theta)
\end{equation}
and $\bar g(\theta) = - \nabla J(\theta)$ by the gradient expression (\ref{eq-Jgrad1}). This is also the same as in Section~\ref{sec-gtd1-odes}, apart from that a different operator $\Tl$ in involved in defining the functions $\bar g$ and $J$ here. The mean ODE for the slow time-scale is thus the same as (\ref{eq-gtd1-odeslow}), and its limit set, by Lemma~\ref{lem-gtd1-odeslow}, is the set 
\begin{equation} \label{eq-cp-opt}
\Theta_{\text{\rm opt}} = \argmin_{\theta \in B_\theta} J(\theta) = \argmin_{\theta \in B_\theta} \tfrac{1}{2}  \big\| \Pi_\xi \big(\Tl v_{\theta} - v_{\theta} \big) \big\|_\xi^2,
\end{equation}
for the Bellman operator $\Tl$ defined in (\ref{eq-Tcomp}).

Finally, to apply the same convergence proofs given in Sections~\ref{sec-gtd1conv-dim}-\ref{sec-gtd1conv-constant}, we need to check if the algorithm here meets the conditions required by those analyses, including uniform integrability, continuity, and averaging conditions, etc.
To see which conditions to verify, it is convenient to work with the same formulae used earlier for the GTDa algorithm, with the trace now being the ``total trace'' $\e_n = \sum_{i = 1}^\ell \e^{(i)}_n$:
\begin{align*}
     \theta_{n+1} & = \Pi_{B_\theta} \big( \theta_n + \alpha_n \, g(\theta_n, x_n, Z_n) \big), \\
      x_{n+1} & = \Pi_{B_x} \big( x_n + \beta_n (k(\theta_n, x_n, Z_n) + \e_n \omega_{n+1} ) \big).
\end{align*}  
In the above, we take $Z_n$ to be $(S_n, y_n, \e_n, S_{n+1})$ together with $\e_n^{(i)}, i \leq \ell$. We take the functions $g, k$ to be the ones defined earlier in (\ref{eq-gtd1-gk}), whose values depend only on the state transition and the ``total trace'' $\e_n$. Since these are the same functions as before and only the trace iterates are different, what we need to verify are just those conditions that involve the traces iterates $\{\e_n\}$.

Consider first the uniform integrability conditions (i) and (vi) listed in Section~\ref{sec-gtd1-cond}. As mentioned there, these conditions are satisfied if $\{\e_n\}$ is uniformly integrable. In the case here, by Prop.~\ref{prop-trace}(ii), for each $i \leq \ell$, $\{\e_n^{(i)}\}$ is uniformly integrable. Then, since $\e_n = \sum_{i = 1}^\ell \e_n^{(i)}$, as a simple consequence of the definition of uniform integrability, $\{\e_n\}$ is uniformly integrable as well. So the uniform integrability conditions (i) and (vi) are met. Similarly, the tightness condition (ii) listed in Section~\ref{sec-gtd1-cond} is met, because Prop.~\ref{prop-trace}(ii) implies the tightness of each $\{\e_n^{(i)}\}$, $i \leq \ell$, and as a result, both $\{\e_n\}$ and the joint trace process $\{(\e_n, \e_n^{(1)}, \ldots, \e_n^{(\ell)})\}$ are also tight. 

What remain to be verified are the averaging conditions (iv) and (vii) listed in Section~\ref{sec-gtd1-cond}, as well as Lemma~\ref{lem-conv-mean-cond1}(iii), which corresponds to part of the averaging conditions needed for analyzing the $\theta$-iterates at the fast time-scale. These conditions are also satisfied, thanks to the properties of the $\ell$ concurrent state-trace processes. We prove this in the lemma below, which is an analogue of Lemma~\ref{lem-conv-mean-cond1}. In the lemma $\Z$ denotes the space of $Z_n$. 

{\samepage
\begin{lem} \label{lem-composite-conv-mean-cond}
In the above context, for each $\theta \in B_\theta$, $x \in B_x$, the following hold:\vspace*{-3pt}
\begin{itemize}
\item[\rm (i)] For each compact set $D \subset \Z$,
$$ \textstyle{ \lim_{m, n \to \infty} \frac{1}{m} \sum_{i=n}^{n+m-1} \E_n \big[ k(\theta, x, Z_i) - \bar k(\theta, x) \big] \I(Z_n \in D) = 0} \ \ \ \text{in mean}.$$
\item[\rm (ii)] For each compact set $D \subset \Z$,
$$ \textstyle{ \lim_{m, n \to \infty} \frac{1}{m} \sum_{i=n}^{n+m-1} \E_n \big[ g(\theta, \bar x(\theta), Z_i) - \bar g(\theta) \big] \I(Z_n \in D) = 0} \ \ \ \text{in mean}.$$
\item[\rm (iii)]If $\{\alpha_n\}$ and $\{\beta_n\}$ satisfy Assumption~\ref{cond-large-stepsize}, then
$$ \textstyle{ \lim_{m, n \to \infty} \frac{1}{m} \sum_{i=n}^{n+m-1} \E_n \big[ (\alpha_i/\beta_i) \cdot g(\theta, x, Z_i) \big] = 0} \ \ \ \text{in mean}.$$
For any positive sequence $c_j \to 0$ as $j \to \infty$,
$$ \textstyle{ \lim_{m, n, j \to \infty} \frac{1}{m} \sum_{i=n}^{n+m-1} \E_n \big[ c_j \cdot g(\theta, x, Z_i) \big] = 0} \ \ \ \text{in mean}.$$
\end{itemize}
\end{lem}}

\begin{proof}
For $j \leq \ell$, let $Z_n^{(j)} = (S_n, y_n, \e_n^{(j)}, S_{n+1})$ (treat the memory state $y_n$ as a dummy variable if the $j$-th trace process is generated with state-dependent $\lambda$). Let $D^j$ be the projection of the compact set $D$ on the space of $Z_n^{(j)}$; note that $D^j$ is also compact.
Write $k(\theta, x, Z_n) = \sum_{j = 1}^\ell k^o(\theta, x, Z_n^{(j)})$, where the function $k^o$ is given by 
$k^o\big(\theta, x,  (s, y, \e, s') \big) = \e \, \bar{\delta}(s, s', v_{\theta}) - \tfrac{1}{\ell} \phi(s) \phi(s)^\tr x$.
For each $(\theta, x)$, let $\bar k^o_j(\theta, x) = \E_{\zeta^{(j)}} [ k^o(\theta, x, Z_0^{(j)}) ]$, 
and since the function $k^o(\theta, x, \cdot)$ is Lipschitz continuous in the trace variable $\e$, applying Prop.~\ref{prop-trace-averaging}(i) to each of the $\ell$ state-trace processes, we have that for each $j \leq \ell$ and the compact set $D^j$, 
\begin{equation} \label{eq-prfcp1}
     \textstyle{ \lim_{m, n \to \infty} \frac{1}{m} \sum_{i=n}^{n+m-1} \big( k^o(\theta, x, Z_i^{(j)}) -  \bar k^o_j(\theta, x)  \big) \cdot \I(Z_n^{(j)} \in D^j) = 0} \ \ \ \text{in mean}.
\end{equation}
Combining the relation (\ref{eq-prfcp1}) for all $j \leq \ell$, and using also the fact that $D^j$'s are the projections of the set $D$, we obtain
\begin{equation} \label{eq-prfcp2}
     \textstyle{ \lim_{m, n \to \infty} \frac{1}{m} \sum_{i=n}^{n+m-1} \sum_{j =1}^\ell \big( k^o(\theta, x, Z_i^{(j)}) -  \bar k^o_j(\theta, x)  \big) \cdot \I(Z_n \in D) = 0} \ \ \ \text{in mean}.
\end{equation}
Since $\sum_{j =1}^\ell k^o(\theta, x, Z_i^{(j)}) = k(\theta, z, Z_i)$ and $\sum_{j =1}^\ell \bar k^o_j(\theta, x) = \bar k(\theta, x)$ as we calculated earlier (cf.\ (\ref{eq-cp-calc1})-(\ref{eq-cp-calc2})), we see that (\ref{eq-prfcp2}) is just the desired relation in the part (i) of the lemma except for the conditional expectation $\E_n$. The part (i) then follows from (\ref{eq-prfcp2}) by Jensen's inequality.

The part (ii) follows from a similar argument, by writing $g(\theta, \bar x(\theta), Z_n) = \sum_{j = 1}^\ell g(\theta, \bar x(\theta), Z_n^{(j)})$ and using the calculations given in (\ref{eq-cp-calc3})-(\ref{eq-cp-calc4}). 
Similarly, for the part (iii), we write $g(\theta, x, Z_n) = \sum_{j = 1}^\ell g(\theta, x, Z_n^{(j)})$ and apply Prop.~\ref{prop-trace-averaging}(ii) to the function $g(\theta, x, \cdot)$ and each of the $\ell$ state-trace processes, and we then add up the results to obtain the part (iii).
\end{proof}

We have now verified all the conditions required in the convergence proofs, so we can conclude that the convergence theorems given earlier for GTDa hold in the present context as well.
Regarding GTDb, if we employ the preceding composite scheme of setting $\lambda$, then with the ``total trace'' being $\e_n = \sum_{i=1}^\ell \e_n^{(i)}$,
the algorithm becomes 
\begin{align}
\theta_{n+1} & =  \Pi_{B_\theta} \Big( \theta_n + \alpha_n \Big[ \e_n \delta_n(v_{\theta_n}) -   \textstyle{ \sum_{i=1}^\ell}  \rho_n \big(1 - \lambda_{n+1}^{(i)} \big) \,\gma_{n+1}  \phi(S_{n+1}) \cdot (\e_n^{(i)})^\tr x_n  \Big] \Big), \label{eq-gtd2c-th}\\
x_{n+1} & = \Pi_{B_x} \Big( x_n + \beta_n \big(\e_n \delta_n(v_{\theta_n}) - \phi(S_n) \phi(S_n)^\tr x_n\big) \Big). \label{eq-gtd2c-x}
\end{align}
The analysis of this algorithm follows the same line of argument given above, apart from some non-essential differences in calculations and definitions.
We now summarize the results for both GTD algorithms in the following theorem.

\begin{thm} \label{thm-composite-gtd}
Consider the two-time-scale GTDa algorithm (\ref{eq-gtd1c-th})-(\ref{eq-gtd1c-x}) or GTDb algorithm (\ref{eq-gtd2c-th})-(\ref{eq-gtd2c-x}). Assume the following:\vspace*{-3pt}
\begin{itemize}
\item[\rm (i)] Condition~\ref{cond-pol} holds.
\item[\rm (ii)] For each $i \leq \ell$, the $i$-th trace process $\{\e_n^{(i)}\}$ is generated by (\ref{eq-decomp-e}) either with state-dependent $\lambda$-parameters or with history-dependent $\lambda$-parameters of the form (\ref{eq-cp-histlambda}), where the memory states $\{y_n\}$ and the function $\lambda^{(i)}(\cdot)$ satisfy Conditions~\ref{cond-mem} and~\ref{cond-lambda}, respectively.
\item[\rm (iii)] The constraint set $B_x$ satisfies the condition (\ref{eq-cond-Bx0}) for the Bellman operator $\Tl$ defined by (\ref{eq-Tcomp}).
\item[\rm (iv)] The initial $x_0 \in \sp \{\phi(\S)\}$ and for each $i \leq \ell$, the initial $\e_0^{(i)} \in \sp \{\phi(\S)\}$.\vspace*{-3pt} 
\end{itemize}
Then for diminishing stepsizes $\{\alpha_n\}$, $\{\beta_n\}$ that satisfy Assumption~\ref{cond-large-stepsize}, the conclusions of Theorem~\ref{thm-gtd1-wk-dim} hold, and for constant stepsizes $\alpha_n = \alpha, \beta_n= \beta$ with $\alpha < < \beta$, the conclusions of Theorems~\ref{thm-gtd1-wk-constant} and~\ref{thm-gtd1-average} hold, where the set $\Theta_{\text{\rm opt}}$ in the theorems is now given by (\ref{eq-cp-opt}) for the Bellman operator $\Tl$ in (\ref{eq-Tcomp}).
\end{thm} 
%\smallskip

\begin{rem} \rm \label{rmk-biasvrt-cp}
Finally, we note that one can also apply the ``robustification'' idea discussed in the previous Section~\ref{sec-bias-vrt} to derive biased variants for the GTD algorithms that employ the composite scheme of setting $\lambda$. Convergence properties of these ``robustified'' variants can be analyzed similarly, by combining the arguments given in the present and previous subsections.
\end{rem}

\section{Convergence Analyses II} \label{sec-4}

This section continues our analyses of gradient-based TD algorithms using the weak convergence method. 
Section~\ref{sec-mdtds} is about the Mirror-Descent TD algorithms, and Section~\ref{sec-gtds} is about the single-time-scale GTDa algorithm and its biased variant, which use a minimax approach to solve the minimization problem for off-policy policy evaluation. 
Section~\ref{sec-gtda-revisit} revisits the two-time-scale GTDa algorithm and its biased variant, and relaxes the conditions on their constraint set $B_x$ by using the minimax problem formulation and related results derived in Section~\ref{sec-gtds}.
As before, we consider constrained algorithms with slowly diminishing or constant stepsize. 
For the objective function, we will treat the case where it has an additional regularizer that is smooth and convex, to demonstrate that the convergence analysis remains essentially the same. 

We will use the results of the previous sections to show that the average dynamics of the algorithms are characterized by their associated mean ODEs. Most of the efforts in our convergence analysis will be spend on the study of solution properties of these mean ODEs. In particular, for the Mirror-Descent TD algorithms, we derive a suitable form of the constraint sets to ensure desired convergence properties. For the single-time-scale GTDa algorithm, we use a general theory on differential inclusion with maximal monotone mappings to analyze the limit set of its mean ODE, and we also show that its ``robustified'' biased variant solves minimax problems that approximate the one solved by GTDa.

As before, convergence theorems will be stated close to where they are finally proved. For quick access to the results of this section, here is the list:
\vspace*{-3pt}
\begin{itemize}
\item Mirror-Descent GTD and Mirror-Descent TD: Theorems~\ref{thm-mdgtd-wk} and~\ref{thm-mdtd}, respectively;
\item single-time-scale GTDa and its biased variant: Theorems~\ref{thm-gtds} and~\ref{thm-gtds-bvrt}, respectively;
\item an improved result for two-time-scale GTDa and its biased variant: Theorem~\ref{thm-gtda-relaxBx}.
\end{itemize}

\subsection{Mirror-Descent TD Algorithms} \label{sec-mdtds}
In this subsection we consider the mirror-descent variant of the TD algorithms. As mentioned in the introduction, combining the mirror-descent idea \cite{Nem83} with TD learning was proposed in \cite{pmtd2} (see also \cite{pmrl}). We shall call the variant algorithms MD-GTD for GTD algorithms, and MD-TD for ordinary TD algorithms. 

In what follows, let $\psi^*: \re^d \to \re$ be a differentiable convex function such that $\psi^*(\theta^*)$ grows to $+\infty$ at a faster than linear rate, as $\| \theta^*\| \to +\infty$. Such a function is said to be co-finite and has the property that $\lim_{i \to \infty} \| \nabla \psi^*(\theta^*_i)  \| = + \infty$ for every sequence $\{\theta^*_i\}$ such that $\lim_{i \to \infty} \| \theta^*_i  \| = + \infty$ \cite[Lemma 26.7]{Roc70}. 
The $\theta^*$-space on which $\psi^*$ is defined is dual to the $\theta$-space we have been working on thus far.%footnote starts
\footnote{The notation $\psi^*$ may look cumbersome with its superscript. We use it because dual to the convex function $\psi^*$ is its conjugate on the $\theta$-space, which is usually denoted by the same symbol but without the superscript $^*$, although our analysis will not involve this conjugate function.}
%footnote ends 
One can think of a point $\theta^*$ being a ``mirror image'' of the point $\theta = \nabla \psi^*(\theta^*)$. 
The co-finiteness of $\psi^*$ implies that for any $\theta$, its pre-image under $\nabla \psi^*$ is non-empty: $(\nabla \psi^*)^{-1} (\theta) \not= \emptyset$.%footnote starts
\footnote{Because $(\nabla \psi^*)^{-1} (\theta)$ is precisely the nonempty set of optimal solutions to $\inf_{\theta^*} \big\{ \psi^*(\theta^*) - \langle \theta, \theta^* \rangle \big\}$, whose objective function is not only convex and finite everywhere but also coercive due to the co-finiteness of $\psi^*$.}
%footnote ends
Thus every $\theta \in \re^d$ is ``accessible'' from the $\theta^*$-space via its mirror images.

\subsubsection{MD-GTD} \label{sec-mdgtd}

Let $J_p(\theta) : = J(\theta) + p(\theta)$, where $J(\theta)$ is the projected-Bellman-error objective function as before, and $p(\theta): \re^d \to \re$ is a differentiable convex function that serves as a regularizer. In the convergence analysis given later, we will require that $\inf_\theta J_p(\theta)$ has a unique optimal solution, which can be ensured by a strictly convex $p(\cdot)$, for instance. 

Note that to minimize $J_p$ instead of $J$, one can modify the GTD algorithms by simply subtracting the gradient term $\alpha_n \nabla p(\theta_n)$ in the $\theta$-iteration. In the case of GTDa, for example, change (\ref{eq-gtd1a-th}) to:
\begin{equation}
    \theta_{n+1}  = \Pi_{D_{\theta}} \Big( \theta_n + \alpha_n \, \rho_n \big(\phi(S_n) - \gma_{n+1} \phi(S_{n+1}) \big) \cdot  \e_n^\tr x_n   - \alpha_n \nabla p(\theta_n) \Big)  \label{eq-gtd1p-th}.
\end{equation}    

The two-time-scale MD-GTDa algorithm we consider is given by
\begin{align}
 x_{n+1}  & = \Pi_{B_x} \Big( x_n + \beta_n \big(\e_n \delta_n(v_{\theta_n}) - \phi(S_n) \phi(S_n)^\tr x_n\big) \Big), \label{eq-mdgtd1-x} \\
    \theta_{n+1}^* & = \Pi_{D_{\theta^*}} \Big( \theta_n^* + \alpha_n \,  \rho_n \big(\phi(S_n) - \gma_{n+1} \phi(S_{n+1}) \big)  \cdot \e_n^\tr x_n  - \alpha_n \nabla p(\theta_n) \Big)  \label{eq-mdgtd1-thstar},\\
       \theta_{n+1} & = \nabla \psi^*(\theta_{n+1}^*) \label{eq-mdgtd1-th},
\end{align}
with $\theta_0 = \nabla \psi^*(\theta^*_0)$ for some given initial $\theta^*_0$.  
This algorithm calculates $x_n$ like GTDa, at the fast time-scale determined by $\{\beta_n\}$. At the slow time-scale determined by $\{\alpha_n\}$, the algorithm obtains $\theta_n$ indirectly from its ``mirror image'' $\theta_n^*$, which is updated according to (\ref{eq-mdgtd1-thstar}), a formula that would be identical to the slow-time-scale iteration (\ref{eq-gtd1p-th}) in GTDa if we remove all the superscripts $^*$ in it. The set $D_{\theta^*} \subset \re^d$ is some large enough compact set to constrain the $\theta^*$-iterates. We assume for now that it is a large convex set with a smooth boundary; a specific choice of $D_{\theta^*}$ that can ensure the desired convergence will be discussed later.
Note that if $\{\theta_n^*\}$ is bounded, then so is $\{\theta_n\}$.%footnote starts
\footnote{Since any convex differentiable function is continuously differentiable \cite[Corollary 25.5.1]{Roc70}, $\nabla \psi^*$ is continuous on $\re^d$ and therefore maps bounded sets to bounded sets.\label{footnote-convexdiff}}
%footnote ends

The MD-GTDb algorithm uses the same formulae (\ref{eq-mdgtd1-x}) and (\ref{eq-mdgtd1-th}) to compute $x_{n+1}$ and $\theta_{n+1}$. To compute $\theta^*_{n+1}$, it uses a formula that resembles the slow-time-scale iteration (\ref{eq-gtd2-th}) in GTDb:
\begin{equation}
   \theta^*_{n+1} = \Pi_{D_{\theta^*}} \Big( \theta_n^* + \alpha_n  \big(\e_n \delta_n(v_{\theta_n}) -  \rho_n (1 - \lambda_{n+1})  \,\gma_{n+1}  \phi(S_{n+1}) \cdot \e_n^\tr x_n  \big) - \alpha_n \nabla p(\theta_n) \Big), \label{eq-mdgtd2-thstar}
\end{equation}
where $\theta_n = \nabla \psi^*(\theta_n^*)$.
The convergence analyses of the two MD-GTD algorithms are essentially the same. We will use MD-GTDa to show how we carry out this analysis, because it is simpler notation-wise. (One can also apply the mirror-descent idea on the $x$-iteration, working with the $x^*$-iterates in addition to the $\theta^*$-iterates. We shall not consider these fancier algorithms here, although the line of analysis would be similar.)

\bigskip
\noindent {\bf Associated mean ODEs}
\smallskip

Let us focus on the iterates $\{(\theta^*_n, x_n)\}$ and discuss the mean ODEs associated with them, for the MD-GTDa algorithm (\ref{eq-mdgtd1-x})-(\ref{eq-mdgtd1-th}). Using the functions $k(\cdot), g(\cdot)$ defined earlier for GTDa in (\ref{eq-gtd1-gk}),
we write (\ref{eq-mdgtd1-x})-(\ref{eq-mdgtd1-th}) equivalently as follows:
with $\theta_n = \nabla \psi^*(\theta_n^*)$,
\begin{align}
  \theta_{n+1}^*  & = \Pi_{D_{\theta^*}} \big( \theta^*_n + \alpha_n \, g(\theta_n, x_n, Z_n) - \alpha_n \nabla p(\theta_n) \big), \label{eq-mdgtd1-thstar-alt} \\
  x_{n+1}  & =   \Pi_{B_x} \Big( x_n + \beta_n \big( k( \theta_n, x_n, Z_n) + \e_n \omega_{n+1} \big) \Big). \label{eq-mdgtd1-xalt}
\end{align}  
Consider first the fast time-scale determined by $\{\beta_n\}$. 
Based on the reasoning and calculation given in Sections~\ref{sec-gtd1}, the desired mean ODE for this time-scale is
\begin{equation} \label{eq-mdgtd1-odefast}
\qquad \left( \! \begin{array}{c} \dot \theta^*(t) \\ \dot x(t) \end{array} \!\right) 
 =  \left( \! \begin{array}{l} 0 \\
 \bar k \big(\nabla \psi^*(\theta^*(t)), x(t) \big) + z(t)  \end{array} \!\right), \qquad \theta^*(0) \in D_{\theta^*}, \  z(t) \in - \N_{B_x}(x(t)),
\end{equation} 
where $z(t)$ is the boundary reflection term, and the function $\bar k(\theta, x)$ is as defined before in (\ref{eq-gtd1-bk}). Note that in this ODE, the function $\bar k(\nabla \psi^*(\cdot), \cdot)$ is continuous, since the gradient mapping $\nabla \psi^*$ is continuous (cf.\ Footnote~\ref{footnote-convexdiff}).
As before, for each $\theta$, let $x_\theta$ be the unique solution to the linear system of equations, $\bar k(\theta, x) = 0, x \in \sp \{\phi(\S)\}$. 
For each $\theta^*$,  define 
$$\bar x(\theta^*) : = x_\theta \ \ \ \text{for} \  \theta = \nabla \psi^*(\theta^*).$$ 
Then, as shown in Sections~\ref{sec-gtd1}, by making the radius of $B_x$ sufficiently large, it can be ensured that for any initial $x(0) \in \sp \{\phi(\S)\}$, the solution of the ODE (\ref{eq-mdgtd1-odefast}) has the desired behavior that as $t \to \infty$, $x(t)$ converges to $x_\theta$ for $\theta = \nabla \psi^*(\theta^*(0))$, or equivalently, $x(t) \to \bar x(\theta^*(0))$. Moreover, it can be ensured that the limit set of (\ref{eq-mdgtd1-odefast}) is $\{(\theta^*, \bar x(\theta^*) \mid \theta^* \in D_{\theta^*} \}$, similar to Lemma~\ref{lem-Bx} and with the same proof.

Let us defer the precise condition on the radius of $B_x$ and proceed to consider the desired mean ODE for the slow time-scale.
For fixed $\theta^*$, with $\theta = \nabla \psi^*(\theta^*)$, we have
\begin{equation} 
  \E_\zeta \big[ \,g\big(\theta, \bar x(\theta^*), Z_0\big) \big] =  \E_\zeta \big[ \,g\big(\theta, x_\theta, Z_0\big) \big] =  - \nabla J \big(\theta \big),
\end{equation}
as calculated earlier in (\ref{eq-gtd1-bg}). Thus, with $\bar g(\theta) = - \nabla J(\theta)$ as before, define 
$$\bar g_p(\theta) = \bar g(\theta) - \nabla p(\theta) = - \nabla J(\theta) - \nabla p(\theta) = - \nabla J_p(\theta), $$
and we see that the desired mean ODE for the slow time-scale is given by
\begin{equation} \label{eq-mdgtd1-odeslow}
  \dot \theta^*(t) = (\bar g_p \circ \nabla \psi^*) \big(\theta^*(t) \big) + z(t),  \quad z(t) \in - \N_{D_{\theta^*}}\big(\theta^*(t)\big),
\end{equation}
where $\N_{D_{\theta^*}}\big(\theta^*(t)\big)$ denotes the normal cone of $D_{\theta^*}$ at the point $\theta^*(t)$, and $z(t)$ is the boundary reflection term. 

Next we examine the solution properties of this ODE (\ref{eq-mdgtd1-odeslow}). We would like it to yield solutions that lead to points $\theta^*$ that are ``mirror images'' of those points $\theta$ with small values of $J_p(\theta)$. Based on the idea of mirror-descent methods \cite{Nem83}, we know this is the case when there is no constraint set $D_{\theta^*}$. So we will aim at finding a suitable choice of the constraint set $D_{\theta^*}$, so that it does not ``ruin'' the good property of mirror-descent.

\bigskip
\noindent {\bf Choice of the constraint set and the solution property of the corresponding ODE}
\smallskip

To study the solution property of the ODE~(\ref{eq-mdgtd1-odeslow}), let us pick an optimal solution $\theta_{\text{opt}}$ of the problem $\inf_{\theta} J(\theta)$, and consider a candidate Lyapunov function for the ODE, which is from the original idea behind mirror descent methods \cite[p.\ 87]{Nem83}:
\begin{equation} \label{eq-mdV}
   V(\theta^*) = \psi^*(\theta^*) - \langle \theta^*, \theta_{\text{opt}}\rangle.
\end{equation}   
The co-finiteness of $\psi^*$ ensures that $V$ is bounded below on $\re^d$.
Let $\theta^*(\cdot)$ be a solution of (\ref{eq-mdgtd1-odeslow}), and let $\theta(t) = \nabla \psi^*\big(\theta^*(t)\big)$. 
We have
\begin{align}
  \dot V\big(\theta^*(t)\big) & = \big\langle   \dot \theta^*(t), \, \nabla \psi^*(\theta^*(t)) - \theta_{\text{opt}} \big\rangle \notag \\
     & = \big\langle - \!\nabla J_p( \theta(t)) + z(t) , \,   \theta(t) - \theta_{\text{opt}} \big\rangle \notag \\
     & \leq J_p( \theta_{\text{opt}}) - J_p(\theta(t))  + \big\langle  z(t)  , \, \theta(t) - \theta_{\text{opt}} \big\rangle, \label{eq-md-dV}
\end{align}
where the last inequality follows from the inequality $J_p( \theta_{\text{opt}}) \geq J_p(\theta(t)) + \big\langle \nabla J_p \big( \theta(t) \big), \, \theta_{\text{opt}} - \theta(t)  \big\rangle$ implied by the convexity of the function $J_p$. 
Thus, when the boundary reflection term $z(t) = 0$, as is the case when $\theta^*(t)$ is in the interior of $D_{\theta^*}$, 
we have the desired relation that $\dot V\big(\theta^*(t)\big) < 0$ if $\theta(t)$ is not an optimal solution of $\inf_\theta J_p(\theta)$. 
We would like this relation to hold also when $\theta^*(t)$ is on the boundary of $D_{\theta^*}$, so that $V$ can serve as a Lyapunov function. 
This leads us to make a simple but natural choice of the constraint set $D_{\theta^*}$ as follows.

Let $D_{\theta^*}$ be a level set of $\psi^*$,
$D_{\theta^*} = \{ \theta^* \mid \psi^*(\theta^*) \leq \ell \}$, for some finite number $\ell > \inf_{\theta^*} \psi^*(\theta^*)$. 
Then the normal cone of $D_{\theta^*}$ at a boundary point $\theta^*$ is just $\N_{D_{\theta^*}}(\theta^*) = \big\{ a \nabla \psi^*(\theta^*) \mid a \geq 0 \big\}$. If $\theta^*(t)$ is on the boundary of $D_{\theta^*}$, the boundary reflection term $z(t) = - a \nabla \psi^*\big(\theta^*(t)\big) = - a \, \theta(t)$ for some $a \geq 0$, and the second term in (\ref{eq-md-dV}) becomes
$$  - a \big\langle \theta(t) , \,  \theta(t) - \theta_{\text{opt}}\big\rangle \leq - a \, \| \theta(t) \|_2 \cdot \big( \| \theta(t) \|_2 - \| \theta_{\text{opt}} \|_2 \big).$$
The upperbound on the r.h.s.\ is nonpositive if we make $D_{\theta^*}$ sufficiently large so that all its boundary points $\theta^*$ satisfy $\|\nabla \psi^*(\theta^*) \|_2 \geq \| \theta_{\text{opt}} \|_2$ (which is possible since $\|\nabla \psi^*(\theta^*) \|_2 \to + \infty$ as $\|\theta^* \|_2 \to + \infty$; cf.~the co-finiteness property of $\psi^*$ mentioned at the beginning). Then we will have the desired relation $\dot V\big(\theta^*(t)\big) \leq J_p( \theta_{\text{opt}}) - J_p(\theta(t))$ at all points $\theta^*(t)$.

We can now state precisely our assumptions on the constraint sets $D_{\theta^*}$ and $B_x$.
Define
\begin{equation} \label{eq-tildeThopt}
   \tilde \Theta_{\text{opt}} = \argmin_{\theta \in \re^d} J_p(\theta). 
\end{equation}   

\begin{assumption}[Conditions on the constraint sets of MD-GTD algorithms] \label{cond-mdgtd} \hfill
\vspace*{-3pt}
\begin{itemize}
\item[\rm (i)] The constraint set $D_{\theta^*} = \{ \theta^* \mid \psi^*(\theta^*) \leq \ell \}$, where $\ell$ is large enough so that for some optimal solution $ \theta_{\text{\rm opt}} \in \tilde \Theta_{\text{\rm opt}}$, $\|\nabla \psi^*(\theta^*) \|_2 \geq \| \theta_{\text{\rm opt}} \|_2$ for all points $\theta^*$ on the boundary of $D_{\theta^*}$.
\item[\rm (ii)] With $D_\theta = \big\{ \nabla \psi^*(\theta^*) \mid \theta^* \in D_{\theta^*} \big\}$, the constraint set $B_x$ satisfies
$B_x \supset \{ x_\theta \mid \theta \in D_\theta \}$.%footnote starts
\footnote{As before, a sufficient condition is that $\text{\rm radius}(B_x) \geq \textstyle{\sup_{\theta \in D_\theta }} \big\| \Phi^\tr \Xi \, (\Tl v_\theta - v_\theta) \big\|_2 /c$, where $c > 0$ is the smallest nonzero eigenvalue of the matrix $\Phi^\tr \Xi \,\Phi$.}
%footnote ends
\end{itemize}
\end{assumption}

\smallskip
Note that in the above, $\theta_{\text{opt}} \in D_\theta$ always (hence $D_\theta \cap \tilde \Theta_{\text{\rm opt}} \not=\emptyset$; i.e., $D_{\theta^*}$ includes at least one ``mirror image'' of an optimal $\theta$).
This is a generic property of the function $\psi^*$, as the next lemma shows.

\begin{lem} \label{lem-Dth-inclusion}
Given $\bar \theta$, let $D_{\theta^*} = \{ \theta^* \mid \psi^*(\theta^*) \leq \ell \}$ where $\ell$ is such that $\|\nabla \psi^*(\theta^*) \|_2 \geq \| \bar \theta \|_2$ for all points $\theta^*$ on the boundary of $D_{\theta^*}$. Then there exists at least one point $\bar \theta^* \in D_{\theta^*}$ with $\nabla \psi^*(\bar \theta^*) = \bar \theta$. In particular, in Assumption~\ref{cond-mdgtd}, $\theta_{\text{\rm opt}} \in D_\theta$.
\end{lem}
\begin{proof}
As discussed at the beginning of the section, the co-finiteness of $\psi^*$ implies that the set of points $\{ \bar \theta^* \mid \nabla \psi^*(\bar \theta^*) = \bar \theta \}$ is nonempty. If all these points are in $D_{\theta^*}$, we are done. Otherwise, some point $\bar \theta^*$ from this set lies outside $D_{\theta^*}$. Let us find a point $\theta^*$ on the boundary of $D_{\theta^*}$ with $\nabla \psi^*(\theta^*) = \bar \theta$, thereby proving the lemma. Consider the boundary point $\theta^* = \Pi_{D_{\theta^*}} \bar \theta^*$. 
Then $\bar \theta^* - \theta^* \in \N_{D_{\theta^*}}(\theta^*)$, so $\bar \theta^* - \theta^* = a \nabla \psi^*(\theta^*)$ for some $a > 0$.  
Let $f(t) = \psi^*(\bar \theta^* + t (\theta^* - \bar \theta^*))$ for $t \in \re$.
We have 
\begin{align*}
f'(0)  \!= \!\langle \nabla \psi^*(\bar \theta^*), \,  \theta^* - \bar \theta^* \rangle\! = \!- a  \langle \bar \theta, \, \nabla \psi^*(\theta^*) \rangle,
\quad f'(1) \! = \!\langle \nabla \psi^*(\theta^*), \theta^* - \bar \theta^* \rangle \!= - a \, \| \nabla \psi^*(\theta^*) \|_2^2.
\end{align*}
Since $f$ is convex, $f'(0) \leq f'(1)$ by \cite[Theorem 24.1]{Roc70} and hence 
$\langle \bar \theta, \, \nabla \psi^*(\theta^*) \rangle \geq  \| \nabla \psi^*(\theta^*) \|_2^2$. But $\|\nabla \psi^*(\theta^*) \|_2 \geq \| \bar \theta \|_2$ since $\theta^*$ is on the boundary of $D_{\theta^*}$. These two inequalities, together with the Cauchy-Schwarz inequality, imply $\nabla \psi^*(\theta^*)  = \bar \theta$. This completes the proof.
\end{proof}

The next lemma is about the behavior of the solutions to the ODE~(\ref{eq-mdgtd1-odeslow}) when the constraint set $D_{\theta^*}$ satisfies Assumption~\ref{cond-mdgtd}(i). We will use its second part on the limit set of the ODE to obtain a convergence theorem for the MD-GTD algorithms. 

%\smallskip
\begin{lem} \label{lem-mdgtd-odeslow-lim}
Under Assumption~\ref{cond-mdgtd}(i) on the constraint set $D_{\theta^*}$, every solution $\theta^*(\cdot)$ of the ODE~(\ref{eq-mdgtd1-odeslow}) satisfies that 
$\theta^*(t)$ converges to the compact set $( \nabla \psi^*)^{-1} (\tilde\Theta_{\text{\rm opt}}) \cap D_{\theta^*}$ as $t \to \infty$. 
If, in addition, $\tilde\Theta_{\text{\rm opt}} = \{ \theta_{\text{\rm opt}}\}$ (i.e., $\inf_\theta J_p(\theta)$ has a unique optimal solution), then
the limit set of the ODE~(\ref{eq-mdgtd1-odeslow}) is $( \nabla \psi^*)^{-1} (\theta_{\text{\rm opt}}) \cap D_{\theta^*}$.
\end{lem}

\begin{proof}
Consider the function $V(\cdot)$ defined in (\ref{eq-mdV}) for the point $\theta_{\text{opt}}$ in Assumption~\ref{cond-mdgtd}(i). 
Let $\theta^*(\cdot)$ be a solution of (\ref{eq-mdgtd1-odeslow}), and let $\theta(t) = \nabla \psi^*(\theta^*(t))$ for $t \geq 0$. As discussed earlier, by (\ref{eq-md-dV}) and Assumption~\ref{cond-mdgtd}(i) on $D_{\theta^*}$,
\begin{equation} \label{eq-prfmd1}
   \dot V\big(\theta^*(t)\big) \leq J_p( \theta_{\text{opt}}) - J_p(\theta(t)) \leq 0.
\end{equation}   
From this it follows that as $t \to \infty$, $\theta^*(t)$ converges to the compact set%footnote starts
\footnote{To see this, let $\bar \theta^* \in D_{\theta^*}$ be any point such that for some sequence of times $t_n \to \infty$, $\theta^*(t_n)$ converges to $\bar \theta^*$ as $n \to \infty$. 
If $J_p(\nabla \psi^*( \bar \theta^*)) > J_p( \theta_{\text{opt}})$, then since $J_p \circ \nabla \psi^*$ is a continuous function, there exist sufficiently small $\epsilon,  \delta > 0$ such that $J_p(\nabla \psi^*(\theta^*)) > J_p( \theta_{\text{opt}}) + \epsilon$ for all $\theta^* \in N_\delta(\bar \theta^*)$ (the $\delta$-neighborhood of $\bar \theta^*$). Then, since $V$ is bounded below on $D_{\theta^*}$, by (\ref{eq-prfmd1}), the total amount of time the solution $\theta^*(\cdot)$ spends in $N_\delta(\bar \theta^*)$ must be finite. But the solution $\theta^*(\cdot)$ is Lipschitz continuous (because by Lemma~\ref{lem-decomp-vec-cones}, $\bar g_p(\nabla \psi^*(\theta^*(t))) + z(t)$ is the projection of $\bar g_p(\nabla \psi^*(\theta^*(t)))$ on the tangent cone of $D_{\theta^*}$ at $\theta^*(t)$, and therefore, for all $t$, $\|\bar g_p(\nabla \psi^*(\theta^*(t))) + z(t)\|_2 \leq \|\bar g_p(\nabla \psi^*(\theta^*(t)))\|_2 \leq \sup_{\theta^* \in D_{\theta^*}} \| \bar g_p(\nabla \psi^*(\theta^*)) \|_2 < \infty$). 
So, whenever its path enters inside the smaller neighborhood $N_{\delta/2}(\bar \theta^*)$, there is some fixed amount of time, say $\Delta$, that it has to spend inside $N_\delta(\bar \theta^*)$ before it can exit $N_\delta(\bar \theta^*)$. This means it will not visit $N_\delta(\bar \theta^*)$ any more after some finite time $t$, a contradiction to $\theta^*(t_n) \to \bar \theta^*$ with $t_n \to \infty$. Therefore, we must have $J_p(\nabla \psi^*( \bar \theta^*)) = J_p( \theta_{\text{opt}})$.
}
%footnote ends 
$$\{ \theta^* \in D_{\theta^*} \mid  J_p( \theta_{\text{opt}}) - J_p(\nabla \psi^*(\theta^*)) = 0 \}   = ( \nabla \psi^*)^{-1} (\tilde\Theta_{\text{\rm opt}}) \cap D_{\theta^*}.$$

Now consider the limit set of (\ref{eq-mdgtd1-odeslow}) in the case $\tilde\Theta_{\text{\rm opt}} = \{ \theta_{\text{\rm opt}}\}$. 
Clearly, it contains the set $( \nabla \psi^*)^{-1} (\theta_{\text{\rm opt}}) \cap D_{\theta^*}$, since for any $\bar \theta^*$ in the latter set, $\bar g_p(\nabla \psi^*(\bar \theta^*)) = - \nabla J_p(\theta_{\text{\rm opt}}) = 0$ by the optimality condition and hence $\theta^*(\cdot) \equiv \bar \theta^*$ is a solution of (\ref{eq-mdgtd1-odeslow}). 
To show the other direction of inclusion, 
define level sets $D_\eta, E_\epsilon$ for $\eta, \epsilon  > 0$ by
$$  D_\eta : = \big\{ \theta^* \in  D_{\theta^*} \mid V(\theta^*) \leq \textstyle{\inf_{\tilde \theta^*}} V(\tilde \theta^*) + \eta \big\}, \quad E_\epsilon : = \big\{ \theta^* \in  D_{\theta^*} \mid J_p(\nabla \psi^*(\theta^*)) \leq J_p( \theta_{\text{opt}}) + \epsilon \big\}. $$
Note that $(\nabla \psi^*)^{-1} (\theta_{\text{\rm opt}})$ is the set of optimal solutions to $\inf_{\theta^*} V(\theta^*)$, and it is nonempty, convex and compact. 
The sets $D_\eta$ are also convex and compact, with $\cap_{\eta > 0} D_\eta = (\nabla \psi^*)^{-1} (\theta_{\text{\rm opt}}) \cap D_{\theta^*}$.
Since the function $J_p \circ \nabla \psi^*$ is continuous, the sets $E_\epsilon$ are also compact and moreover, by the uniqueness of $\theta_{\text{opt}}$, $\cap_{\epsilon > 0} E_\epsilon = (\nabla \psi^*)^{-1} (\theta_{\text{\rm opt}}) \cap D_{\theta^*}$.

For each $\eta$, there must exist $\epsilon$ sufficiently small such that $E_\epsilon \subset D_\eta$. 
Otherwise, we could find a sequence of points $\theta^*_n \in E_{\epsilon_n}$ with $\epsilon_n \to 0$, but with $V(\theta^*_n) - \inf_{\theta^*} V(\theta^*) > \eta$ for all $n$. Then any limit point $\theta^*_\infty$ of this bounded sequence $\{\theta^*_n\}$ must lie in $\cap_{\epsilon > 0} E_\epsilon$ and hence in $(\nabla \psi^*)^{-1} (\theta_{\text{\rm opt}})$, yet it must also satisfy $V(\theta^*_\infty) - \inf_{\theta^*} V(\theta^*) \geq \eta$, which is impossible.

For any given $\eta$, consider this set $E_\epsilon \subset D_\eta$.
Since $V$ is bounded on $D_{\theta^*}$, (\ref{eq-prfmd1}) implies that there exists a finite time $t_\epsilon$ such that every solution of (\ref{eq-mdgtd1-odeslow}) must visit the set $E_\epsilon$ at least once before time $t_\epsilon$;
i.e., there is a time $t_0 \leq t_\epsilon$ with $\theta^*(t_0) \in E_\epsilon$. Since $V(\theta^*(t))$ is nonincreasing and $E_\epsilon \subset D_\eta$, 
this means that $\theta^*(t) \in D_\eta$ for all $t \geq t_\epsilon$ and all solutions $\theta^*(\cdot)$ of (\ref{eq-mdgtd1-odeslow}). 
Consequently, the limit set of (\ref{eq-mdgtd1-odeslow}) must satisfy
$\textstyle{\bigcap_{ \,\bar t \geq 0}} \, \cl \, \big\{\theta^*(t)  \, \big| \,   \theta^*(0) \in D_{\theta^*},  \, t \geq \bar t  \, \big\} \subset \cap_{\eta > 0} D_\eta = (\nabla \psi^*)^{-1} (\theta_{\text{\rm opt}}) \cap D_{\theta^*}.$
This prove the other direction of inclusion and completes the proof.
\end{proof}

%\smallskip
\noindent {\bf Convergence results}
\smallskip

We now use the preceding results about the limit sets of the two ODEs to characterize the asymptotic behavior of MD-GTD. Besides Assumption~\ref{cond-mdgtd} on the constraint sets $D_{\theta^*}$ and $B_x$, we shall also assume, as in the second part of Lemma~\ref{lem-mdgtd-odeslow-lim}, that $\inf_\theta J_p(\theta)$ has a unique optimal solution $\theta_{\text{\rm opt}}$. 
It is straightforward to verify that the iterates $\{(\theta^*_n, x_n)\}$ generated by MD-GTD satisfy all the conditions for applying the convergence proofs given earlier for GTDa.%footnote starts
\footnote{Although the conditions are now on the $(\theta^*, x)$-space instead of the $(\theta, x)$-space, they are essentially the same as before. For example, the ``averaging condition'' (iv) listed in Section~\ref{sec-gtd1-cond} is now a convergence-in-mean condition on $\{k(\nabla \psi^*(\theta^*), x, Z_n)\}$ for each $(\theta^*, x)$. This is the same as the ``averaging condition'' on $\{k(\theta, x, Z_n)\}$ for each $\theta = \nabla \psi^*(\theta^*)$ and each $x$, which we already treated in Lemma~\ref{lem-conv-mean-cond1}. As another example, in the case here, we need that for $(\theta^*, x) \in D_{\theta^*} \times B_x$ and $z$ in a compact set, the function $g(\nabla \psi^*(\theta^*), x, z) - \nabla p(\nabla \psi^*(\theta^*))$ is Lipschitz continuous in $x$ uniformly w.r.t.\ $(\theta^*, z)$ (cf.\ Remark~\ref{rmk-proofgtd}). This is the same as the function $g(\theta, x, z)$ being Lipschitz continuous in $x$ uniformly w.r.t.\ $(\theta, z)$ in a compact set, the same property that we verified earlier in the GTDa case.\label{footnote-mdgtd-cond}}
%footnote ends  
Thus the conclusions of Theorems~\ref{thm-gtd1-wk-dim}-\ref{thm-gtd1-wk-constant} hold for $\{(\theta^*_n, x_n)\}$. In particular, about the asymptotic behavior of $\{\theta_n^*\}$, the conclusions of those theorems hold with $\Theta_{\text{\rm opt}}$ replaced by $( \nabla \psi^*)^{-1} (\theta_{\text{\rm opt}}) \cap D_{\theta^*}$, the limit set of the ODE (\ref{eq-mdgtd1-odeslow}). We can then translate these results about the asymptotic behavior of $\{\theta_n^*\}$ to that of $\{\theta_n\}$, by using the following observation: Since the gradient mapping $\nabla \psi^*$ is continuous and the set $D_{\theta^*}$ is bounded, for any $\epsilon > 0$, there exists $\eta_{\epsilon} > 0$ (with $\eta_{\epsilon} \to 0$ as $\epsilon \to 0$), such that 
\begin{equation} \label{eq-mdgtd-thstar2th}
\theta^* \in N_{\eta_\epsilon} \big(( \nabla \psi^*)^{-1} (\theta_{\text{\rm opt}}) \cap D_{\theta^*} \big) \quad \Longrightarrow \quad \theta = \nabla \psi^*(\theta^*) \in  N_{\epsilon} \big(\theta_{\text{\rm opt}} \big).
\end{equation}
This gives us the convergence results for MD-GTD with both diminishing and constant stepsizes stated in the following convergence theorem. 

The theorem also includes conclusions about the averaged $\theta$-iterates in the case of constant stepsize. The line of proof for this part is the same as that explained after Theorem~\ref{thm-gtd1-average} in Section~\ref{sec-aveite}. It uses the observation that with a given pair of constant stepsizes $(\alpha, \beta)$, the MD-GTD iterates together with the states and traces, $(S_n, y_n, \e_n, \theta^*_n, \theta_n, x_n)$, jointly form a weak Feller Markov chain (this proof is similar to the proof of \cite[Lemma 6]{etd-wkconv}). The proof then combines this observation with those conclusions about the asymptotic behavior of the $\theta$-iterates, and with ergodicity properties of weak Feller Markov chains. As before, we omit the proof, since it is essentially the same proof given in \cite[Section 4.3]{etd-wkconv}.

%\smallskip
\begin{thm} \label{thm-mdgtd-wk}
Consider the MD-GTDa algorithm (\ref{eq-mdgtd1-x})-(\ref{eq-mdgtd1-th}) or the MD-GTDb algorithm given by (\ref{eq-mdgtd1-x}), (\ref{eq-mdgtd1-th}), and (\ref{eq-mdgtd2-thstar}),  with the initial $x_0, \e_0 \in \sp \{\phi(\S)\}$. Let Assumptions~\ref{cond-collective} and~\ref{cond-mdgtd} hold, and let $\theta_{\text{\rm opt}}$ be the unique optimal solution to $\inf_\theta J_p(\theta)$. 
Then for diminishing stepsizes $\{\alpha_n\}$, $\{\beta_n\}$ that satisfy Assumption~\ref{cond-large-stepsize}, the conclusions of Theorem~\ref{thm-gtd1-wk-dim} about the $\theta$-iterates hold with $\theta_{\text{\rm opt}}$ in place of $\Theta_{\text{\rm opt}}$. For the case of constant stepsize, the conclusions of both Theorem~\ref{thm-gtd1-wk-constant} (about the $\theta$-iterates) and Theorem~\ref{thm-gtd1-average} (about the averaged $\theta$-iterates) hold with $\theta_{\text{\rm opt}}$ in place of $\Theta_{\text{\rm opt}}$.
\end{thm}

\begin{rem}[About the condition on the minimization problem] \rm \label{rmk-mdgtd-cond}
For the convergence of GTD algorithms, we do not need any condition on the minimization problem $\inf_{\theta \in B_\theta} J_p(\theta)$; in particular, we do not need $B_\theta$ to contain an optimal solution of the unconstrained minimization problem $\inf_\theta J_p(\theta)$. The condition in the above theorem for MD-GTD is more stringent, as it requires not only that $D_\theta$ contains an optimal solution of the unconstrained problem $\inf_\theta J_p(\theta)$, but also that $\inf_\theta J_p(\theta)$ has a unique optimal solution. This condition is used in Lemma~\ref{lem-mdgtd-odeslow-lim} to ensure that the limit set of the ODE~(\ref{eq-mdgtd1-odeslow}) has desirable properties---if the latter can be obtained under a weaker condition on the minimization problem, the condition in the theorem can be correspondingly weakened.
\end{rem}

\begin{rem}[Biased variants] \rm \label{rmk-mdgtd2}
We mention that in the case of state-dependent $\lambda$, one can also consider biased variants of the MD-GTD algorithms. Such variants can be analyzed by combining the preceding analysis with the reasoning given in Section~\ref{sec-bias-vrt}. 
\end{rem}

\begin{rem}[MD-GTD with the composite scheme of setting $\lambda$] \rm \label{rmk-mdgtd-cp}
When the $\lambda$'s are set according to the composite scheme described in Section~\ref{sec-cp-extension}, the changes in the MD-GTD algorithms are similar to those in the GTD algorithms described earlier. In particular, with $\e_n = \sum_{i=1}^\ell \e_n^{(i)}$, where $\e_n^{(i)}$ is updated according to (\ref{eq-decomp-e}) for each $i \leq \ell$, MD-GTDa is still given by (\ref{eq-mdgtd1-x})-(\ref{eq-mdgtd1-th}), and MD-GTDb becomes
\begin{align}
     x_{n+1}  & = \Pi_{B_x} \Big( x_n + \beta_n \big(\e_n \delta_n(v_{\theta_n}) - \phi(S_n) \phi(S_n)^\tr x_n\big) \Big),  \\
    \theta^*_{n+1} & = \Pi_{D_{\theta^*}} \Big( \theta_n^* + \alpha_n  \big(\e_n \delta_n(v_{\theta_n}) -   \textstyle{ \sum_{i=1}^\ell}  \rho_n (1 - \lambda^{(i)}_{n+1})  \,\gma_{n+1}  \phi(S_{n+1}) \cdot (\e_n^{(i)})^\tr  x_n \big) - \alpha_n \nabla p(\theta_n) \Big), \\
    \theta_{n+1} & = \nabla \psi^*(\theta_{n+1}^*).
\end{align}    
By combining the preceding analysis with the reasoning given in Section~\ref{sec-cp-extension}, it is straightforward to show that the conclusions of Theorem~\ref{thm-mdgtd-wk} hold for these MD-GTD algorithms as well, where the Bellman operator $\Tl$ in $J(\theta)$ and in Assumption~\ref{cond-mdgtd} is now given by (\ref{eq-Tcomp}).
\end{rem}

\subsubsection{MD-TD}

In the rest of this subsection, we discuss an algorithm that combines the mirror-descent method with the original TD method. This MD-TD algorithm is given by
\begin{equation}  \label{eq-mdtd}
    \theta_{n+1}^*  = \Pi_{D_{\theta^*}} \big( \theta_n^* + \alpha_n  \e_n  \delta_n(v_{\theta_n})  \big),  \qquad 
     \theta_{n+1}  = \nabla \psi^*(\theta_{n+1}^*),
\end{equation}
with $\theta_0 = \nabla \psi^*(\theta_0^*)$ for some given initial $\theta_0^*$.
It is a single-time-scale algorithm like TD, and computes the $\theta^*$-iterates using a formula that resembles the TD algorithm.

We now derive the mean ODE associated with the $\theta^*$-iterates, and we will introduce a few conditions to ensure desired solution properties. 
By Prop.~\ref{prop-mean-fn}, for each fixed $\theta$,
\begin{equation} \label{eq-mdtd-bg}
  \E_\zeta \big[ \e_0 \bar \delta_0(v_\theta) \big] = \Phi^\tr \Xi \, (\Tl v_\theta - v_\theta) = : \bar g(\theta).
\end{equation}   
So the projected mean ODE associated with (\ref{eq-mdtd}) is
\begin{equation} \label{eq-mdtd-ode}
      \dot \theta^*(t) = (\bar g \circ \nabla \psi^*) \big(\theta^*(t) \big) + z(t),  \quad z(t) \in - \N_{D_{\theta^*}}\big(\theta^*(t)\big),
\end{equation}
where $z(t)$ is the boundary reflection term.
We shall impose a strong condition on $\bar g$, in order to ensure desired convergence:

\begin{assumption} \label{cond-mdtd} 
The linear equation $\bar g(\theta) = \Phi^\tr \Xi \, (\Tl\!(\Phi \theta)  - \Phi \theta) = 0$ has a unique solution $\theta_{\text{\rm TD}}$, and the matrix $C=\Phi^\tr \Xi \, (\Pl - I) \Phi$ is negative definite (i.e., for some $c > 0$, $\theta^\tr C \theta \leq - c \| \theta \|_2^2$ for all $\theta \in \re^d$).
\end{assumption}

Suppose Assumption~\ref{cond-mdtd} holds. In addition, suppose the constraint set $D_{\theta^*}$ is a level set of $\psi^*$ that satisfies Assumption~\ref{cond-mdgtd}(i) with $\theta_{\text{TD}}$ in place of $\theta_{\text{opt}}$. Note that by Lemma~\ref{lem-Dth-inclusion}, $\theta_{\text{TD}} \in D_\theta : = \{ \nabla \psi^*(\theta^*) \mid \theta^* \in D_{\theta^*} \}$ (i.e., in $D_{\theta^*}$ there is at least one ``mirror image'' of $\theta_{\text{TD}}$, through which the algorithm can ``access'' $\theta_{\text{TD}}$).

As a candidate Lyapunov function, consider the function
$$ V (\theta^*) = \psi^*(\theta^*) - \langle \theta^*, \theta_{\text{TD}}\rangle.$$
Let $\theta^*(\cdot)$ be a solution of (\ref{eq-mdtd-ode}), and let $\theta(t) = \nabla \psi^*\big(\theta^*(t)\big)$. 
Note that since $\bar g(\theta_{\text{TD}}) = 0$, $\bar g(\theta) = C (\theta - \theta_{\text{TD}})$.
Then similar to (\ref{eq-md-dV}) in the previous MD-GTD case, 
we have
\begin{align}
 \dot V\big(\theta^*(t)\big) & = \big\langle \bar g( \theta(t))  + z(t), \,   \nabla \psi^*(\theta^*(t)) - \theta_{\text{TD}}  \big\rangle  \notag \\
   & =  \big\langle C (\theta(t) - \theta_{\text{TD}}), \, \theta(t) - \theta_{\text{TD}} \big\rangle  + \big\langle  z(t)  , \, \theta(t) - \theta_{\text{\text{TD}}} \big\rangle \notag \\
   & \leq \big\langle C (\theta(t) - \theta_{\text{TD}}), \, \theta(t) - \theta_{\text{TD}} \big\rangle \notag \\
   & \leq - c \, \| \theta(t) - \theta_{\text{TD}}\|_2^2, \label{eq-mdtd-dV}
 \end{align}
for some constant $c > 0$.
In the above, the second equality uses the fact $\bar g(\theta) = C (\theta - \theta_{\text{TD}})$. The second to last inequality uses the fact that our choice of the constraint set $D_{\theta^*}$ ensures that the term
$\big\langle  z(t)  , \, \theta(t) - \theta_{\text{\text{TD}}} \big\rangle \leq 0$ if $\theta^*(t)$ is on the boundary of $D_{\theta^*}$ (cf.\ the discussion before Assumption~\ref{cond-mdgtd}(i) in the previous case), so this term is always nonpositive. The last inequality follows from the negative definiteness of the matrix $C$ under Assumption~\ref{cond-mdtd}. Thus $\dot V\big(\theta^*(t)\big) < 0$ for all $\theta(t) \not= \theta_{\text{TD}}$ as desired.

Furthermore, similar to Lemma~\ref{lem-mdgtd-odeslow-lim}, we have that the limit set of the ODE~(\ref{eq-mdtd-ode}) is%footnote starts
\footnote{The limit set contains $(\nabla \psi^*)^{-1}( \theta_{\text{TD}}) \cap D_{\theta^*}$ since each point in the latter set corresponds to a constant solution to the ODE. The other direction of inclusion follows from the last part of the proof of Lemma~\ref{lem-mdgtd-odeslow-lim}, with $\theta_{\text{TD}}$ in place of $\theta_{\text{opt}}$ and with the function $\theta \mapsto - c \, \| \theta - \theta_{\text{TD}}\|_2^2$ serving as the function $J_p$ in the proof.}
%footnote ends 
\begin{equation} \label{eq-mdtd-odelim}
 \textstyle{\bigcap_{ \,\bar t \geq 0}} \, \cl \, \big\{\theta^*(t)  \, \big| \,   \theta^*(0) \in D_{\theta^*},  \, t \geq \bar t  \, \big\} = (\nabla \psi^*)^{-1}( \theta_{\text{TD}}) \cap D_{\theta^*} .
\end{equation}
As $\theta^*(t)$ approaches this limit set, the corresponding $\theta(t)=\nabla \psi^*(\theta^*(t))$ converges to $\theta_{\text{\rm TD}}$, since for any $\epsilon > 0$, there exists $\eta_{\epsilon} > 0$ such that 
\begin{equation} \label{eq-mdtd-thstar2th}
\theta^* \in N_{\eta_\epsilon} \big(( \nabla \psi^*)^{-1} (\theta_{\text{\rm TD}}) \cap D_{\theta^*} \big) \quad \Longrightarrow \quad \theta = \nabla \psi^*(\theta^*) \in  N_{\epsilon} \big(\theta_{\text{\rm TD}} \big)
\end{equation}
(cf.\ the discussion preceding (\ref{eq-mdgtd-thstar2th})).
Likewise, if the iterates $\theta_n^*$ of the MD-TD algorithm approach the above limit set, the corresponding iterates $\theta_n = \nabla \psi^*(\theta_n^*)$ approach the point $\theta_{\text{TD}}$. 

We now proceed to the convergence proof for MD-TD. We first establish the connection between the average dynamics of the algorithm and the ODE~(\ref{eq-mdtd-ode}) using stochastic approximation theory, which will allow us to translate the solution property of the mean ODE (\ref{eq-mdtd-ode}) derived above to the asymptotic behavior of the $\theta^*$-iterates. We then use (\ref{eq-mdtd-thstar2th}) to obtain convergence properties for the $\theta$-iterates. For the first step of this proof, we apply \cite[Theorem 8.2.3]{KuY03} in the case of diminishing stepsize and \cite[Theorem 8.2.2]{KuY03} in the case of constant stepsize.  The conditions of these theorems are similar to those listed in Section~\ref{sec-gtd1-cond}, and it is straightforward to show that they are satisfied.%footnote starts
\footnote{We omit the details here, because in the subsequent Section~\ref{sec-gtds-conv}, for another single-time-scale algorithm (GTDa), we will describe similar conditions and explain how to verify them. In particular, see Footnote~\ref{footnote-gtds-cond} therein.}
%footnote ends
This gives us the convergence results about the $\theta$-iterates generated by MD-TD stated in the following convergence theorem. Those conclusions in the theorem about the averaged $\theta$-iterates, in the constant-stepsize case, are then obtained through a procedure similar to that described before Theorem~\ref{thm-mdgtd-wk} in the previous subsection for MD-GTD (see also the explanation of the proof of Theorem~\ref{thm-gtd1-average} in Section~\ref{sec-aveite}).

\begin{thm} \label{thm-mdtd}
Consider the MD-TD algorithm (\ref{eq-mdtd}). Let Assumptions~\ref{cond-collective} and \ref{cond-mdtd} hold, and let $D_{\theta^*}$ satisfy Assumption~\ref{cond-mdgtd}(i) with $\theta_{\text{\rm TD}}$ in place of $\theta_{\text{\rm opt}}$.
Then, in the case of diminishing stepsize, with the stepsizes $\{\alpha_n\}$ satisfying those conditions given in Assumption~\ref{cond-large-stepsize} for the fast-time-scale stepsizes, the conclusions of Theorem~\ref{thm-gtd1-wk-dim} about the $\theta$-iterates hold with $\theta_{\text{\rm TD}}$ in place of $\Theta_{\text{\rm opt}}$. 
In the case of constant stepsize, the conclusions of both Theorem~\ref{thm-gtd1-wk-constant} (about the $\theta$-iterates) and Theorem~\ref{thm-gtd1-average} (about the averaged $\theta$-iterates) hold with $\theta_{\text{\rm TD}}$ in place of $\Theta_{\text{\rm opt}}$, and with obvious modifications to remove all the references to the fast-time-scale stepsize parameter $\beta$.
\end{thm}

\subsection{A Minimax Approach: Single-Time-Scale GTDa} \label{sec-gtds}

In this subsection we analyze a constrained single-time-scale GTDa algorithm and its biased variant.
As in Section~\ref{sec-mdtds}, we shall consider the regularized objective function $J_p(\theta) = J(\theta) + p(\theta)$, where $p(\cdot)$ is a differentiable convex function (for the reason that the analysis is the same with or without this regularizer $p(\cdot)$). The goal of the single-times-scale GTDa algorithm is, as before, to solve the minimization problem $\inf_{\theta \in B_\theta} J_p(\theta)$, where $B_\theta$ is the constraint set. It computes $\{(\theta_n, x_n)\}$ according to
\begin{align}
    \theta_{n+1} & = \Pi_{B_\theta} \Big( \theta_n + \alpha_n \,  \rho_n \big(\phi(S_n) - \gma_{n+1} \phi(S_{n+1}) \big) \cdot \e_n^\tr x_n  - \alpha_n \nabla p(\theta_n) \Big),  \label{eq-gtd1m-th}\\
   x_{n+1} & = \Pi_{B_x} \Big( x_n + \alpha_n \big(\e_n \delta_n(v_{\theta_n}) - \phi(S_n) \phi(S_n)^\tr x_n\big) \Big), \label{eq-gtd1m-x}
\end{align}
using the same stepsize sequence $\{\alpha_n\}$ in both iterations. Although the algorithm looks almost identical to its two-time-scale counterpart, it is based on a different, minimax approach to minimize $J_p(\theta)$, as first pointed out in \cite{pmtd3,pmrl}, and its behavior can thus also be quite different. 

To analyze the asymptotic behavior of the algorithm, we will focus on the solution properties of its associated mean ODE. 
Indeed, it is straightforward to apply stochastic approximation theory to this single-time-scale algorithm and prove that its average dynamics is characterized by the mean ODE, by using the analyses given in Sections~\ref{sec-property-stproc} and \ref{sec-gtd1-cond}. (The details of this proof step will be explained in Section~\ref{sec-gtds-conv} after we state the convergence theorem.) So the solution properties of the mean ODE will show us how the algorithm behaves and how the constraint sets can affect it.

The mean ODE associated with the algorithm (\ref{eq-gtd1m-th})-(\ref{eq-gtd1m-x}) is
\begin{equation} \label{eq-gtds-ode}
 \left( \! \begin{array}{c} \dot{\theta}(t) \\ \dot x(t) \end{array} \!\right) 
 =  \left( \! \begin{array}{c} \bar g\big(\theta(t), x(t)\big) - \nabla p\big(\theta(t)\big) \\
 \bar k \big(\theta(t), x(t) \big) \end{array} \!\right) +  \left( \! \begin{array}{l} z_1(t) \\ z_2(t) \end{array} \!\right),      \quad z_1(t) \in - \N_{B_\theta}(\theta(t)), \  z_2(t) \in - \N_{B_x}(x(t)),
\end{equation}
where $z_1(t), z_2(t)$ are the boundary reflection terms, and the functions $\bar g, \bar k$ are given by
\begin{equation} \label{eq-gtds-gk}
  \bar g(\theta, x) = \big( \Phi^\tr \Xi \, (I - \Pl) \Phi\big)^\tr x, \qquad \bar \k(\theta, x) = \Phi^\tr \Xi \, (\Tl v_\theta - v_\theta) - \Phi^\tr \Xi \, \Phi \, x.
\end{equation}
(The derivations of the above are the same as those given at the beginning of Section~\ref{sec-gtd1-odes} for GTDa.) 
In order to study its solutions, we will need to first clarify its relation with a minimax approach to solving the convex optimization problem $\inf_{\theta \in B_\theta} J_p(\theta)$.

\subsubsection{Solution properties of the associated mean ODE} \label{sec-gtds-ode}

Let us start by converting $\inf_{\theta \in B_\theta} \{ J(\theta) + p(\theta)\}$ to a minimax problem.
Notice the identity relation that for all $v \in \re^{|\S|}$, $\tfrac{1}{2} \| v \|_\xi^2  = \sup_y \big\{ \langle y, v \rangle_\xi - \tfrac{1}{2} \| y \|^2_\xi \big\}$, where
the maximization can be over the approximation subspace $\L_\phi$ only, if $v \in \L_\phi$:
\begin{align*}
 \tfrac{1}{2} \| v \|^2_\xi & = \sup_{ y \in  \L_\phi}  \big\{ \langle y, v \rangle_\xi - \tfrac{1}{2} \| y \|^2_\xi \, \big\} 
 =  \sup_{ x \in \re^d}  \big\{ \langle \Phi x, v \rangle_\xi - \tfrac{1}{2} \| \Phi x \|^2_\xi \, \big\} 
     =  \sup_{ x \in \sp \{ \phi(\S)\}} \!\! \big\{ \langle \Phi x, v \rangle_\xi - \tfrac{1}{2} \| \Phi x \|^2_\xi \, \big\}.
\end{align*} 
We have $J(\theta) = \tfrac{1}{2} \| v \|_\xi^2$ for $v = \Pi_\xi (\Tl v_\theta - v_\theta) \in \L_\phi$. Then, using also the fact that $\langle \Phi x, \Pi_\xi v \rangle_\xi = \langle \Phi x, v \rangle_\xi$ for any $v \in \re^{|\S|}$, we can write $\inf_{\theta \in B_\theta} \{ J(\theta) + p(\theta)\}$ equivalently as 
\begin{equation} \label{eq-minimax0}
  \inf_{\theta \in B_\theta} \sup_{x \in \re^d} \big\{ \langle \Phi x, \Tl v_\theta - v_\theta \rangle_\xi - \tfrac{1}{2} \| \Phi x \|^2_\xi  + p(\theta) \big\}.
\end{equation}  
Let us denote the function above by $\psi^o$, to simplify notation:
\begin{equation} \label{eq-psio}
 \psi^o(\theta, x) : = \langle \Phi x, \Tl v_\theta - v_\theta \rangle_\xi - \tfrac{1}{2} \| \Phi x \|^2_\xi  + p(\theta).
\end{equation} 
This function $\psi^o$ is a continuous convex-concave function: convex in $\theta$ for each $x$ and concave in $x$ for each $\theta$. 
Since $B_\theta$ is compact, by \cite[Corollary 37.3.2]{Roc70}, the minimax problem (\ref{eq-minimax0}) has a saddle-value.
Note also that due to the term $- \tfrac{1}{2} \| \Phi x\|^2_\xi$, the function $\inf_{\theta \in B_\theta} \psi^o(\theta, \cdot)$ (which is everywhere finite) is strictly concave and coercive on the subspace $\sp \{ \phi(\S)\}$, whereas it is constant over $\sp \{ \phi(\S)\}^{\perp}$. This implies that in $\sp \{\phi(\S) \}$ the minimax problem~(\ref{eq-minimax0}) has a unique dual optimal solution,
$$ x_{\text{\rm opt}} =  \argmax_{x \in \sp \{\phi(\S)\}} \big\{ \inf_{\theta\in B_\theta} \psi^o(\theta, x) \big\}.$$
Then by \cite[Lemma 36.2]{Roc70}, the set $\Theta_{\text{\rm opt}} \times \{ x_{\text{\rm opt}} \}$ is a subset of saddle points of the problem (\ref{eq-minimax0}), where   
$$ \Theta_{\text{\rm opt}} = \argmin_{\theta \in B_\theta} \{ J(\theta) + p(\theta) \} = \argmin_{\theta \in B_\theta} \big\{ \sup_{x \in \re^d} \psi^o(\theta, x) \big\}.$$

Instead of (\ref{eq-minimax0}), below we shall focus on the minimax problem
\begin{equation} \label{eq-minimax1}
  \inf_{\theta \in B_\theta} \sup_{x \in B_x}  \psi^o(\theta, x), 
\end{equation}  
where we have added the compact constraint set $B_x$ for the maximizer, the same constraint set in the algorithm and its mean ODE. 
Its mean ODE~(\ref{eq-gtds-ode}) is related to the minimax problem (\ref{eq-minimax1}), as we will show below.

By \cite[Corollary 37.6.2]{Roc70}, the minimax problem (\ref{eq-minimax1}) has a nonempty set of saddle points, which we denote by $D_{\theta} \times D_{x}$. The set $D_{\theta}$ is just the compact set of optimal solutions to the minimization problem $\inf_{\theta \in B_\theta} \{ \sup_{x \in B_x}  \psi^o(\theta, x) \}$, whereas the set $D_x$ is the compact set of optimal solutions to the dual maximization problem: $\sup_{x \in B_x} \{ \inf_{\theta \in B_\theta} \psi^o(\theta, x) \}$. 
The minimax problem~(\ref{eq-minimax1}) has a unique dual optimal solution $\bar x$ in $\sp \{\phi(\S) \}$ (in other words, $D_x \cap \sp \{\phi(\S) \} = \{\bar x\}$). This follows from the strict concavity and coercivity of the function $\inf_{\theta \in B_\theta} \psi^o(\theta, \cdot)$ on $\sp \{ \phi(\S)\}$ and its constancy over $\sp \{ \phi(\S)\}^\perp$ mentioned above. 
These saddle points of (\ref{eq-minimax1}) can be related to the saddle points of the original minimax problem (\ref{eq-minimax0}) as follows.

\begin{lem} \label{lem-saddle-relation}
If $x_{\text{\rm opt}} \in \text{\rm int}(B_x)$ (the interior of $B_x$), then $D_\theta \times \{\bar x\} =\Theta_{\text{\rm opt}} \times \{ x_{\text{\rm opt}} \}$.
\end{lem} 
 
The proof of this lemma is short, but we will give it at the end of this subsection, where we will have all the definitions needed in the proof. The main effort of our analysis is to prove that $D_\theta \times \{\bar x \}$ is the limit set of the ODE~(\ref{eq-gtds-ode}), as stated in the following proposition. We will spend the rest of this subsection to prove it.

\begin{prop} \label{prop-sd-limitset}
For initial conditions $x(0) \in \sp \{\phi(\S)\}$, the limit set of the ODE (\ref{eq-gtds-ode}) is the subset of saddle points, $D_\theta \times \{\bar x \}$, of the minimax problem (\ref{eq-minimax1}), where $\bar x$ is the unique dual optimal solution of (\ref{eq-minimax1}) in $\sp \{\phi(\S)\}$. 
\end{prop}

To setup the stage to prove Prop.~\ref{prop-sd-limitset}, we shall need some properties of the ODE (\ref{eq-gtds-ode}) as well as some basic properties of the minimax problem (\ref{eq-minimax1}). 
First, to apply the theory of convex analysis to (\ref{eq-minimax1}), it is more convenient to represent the constraint sets by functions and rewrite (\ref{eq-minimax1}) equivalently as 
$$\inf_{\theta \in \re^d} \sup_{x \in \re^d} \psi(\theta, x)$$ 
for an extended-real-valued convex-concave function $\psi(\theta, x)$ given by
\begin{align*}
 \psi(\theta, x)  & := \psi^o(\theta, x) + \delta_{B_\theta}(\theta) - \delta_{B_x}(x) \\
   & \, \, = \langle \Phi x, \Tl v_\theta - v_\theta \rangle_\xi - \tfrac{1}{2} \| \Phi x \|^2_\xi  + p(\theta)+ \delta_{B_\theta}(\theta) - \delta_{B_x}(x),
\end{align*}
where for a set $D$, $\delta_D(\cdot)$ is the indicator function that takes the value $0$ on $D$ and $+\infty$ outside $D$.%\footnote starts
\footnote{Outside $B_\theta \times B_x$, $\psi$ takes either the value $+\infty$ or $-\infty$, and there are two consistent ways to assign these values to $\psi$ there; see \cite[Section 33, p.\ 349]{Roc70}. Either way can be taken here as the convention to handle the difference $\delta_{B_\theta}(\theta) - \delta_{B_x}(x)$ when both terms are $+\infty$. Such choices do not affect the subsequent analysis because the effective domain of $\psi$ is $B_\theta \times B_x$ (see \cite[Section 34]{Roc70} on equivalent saddle-functions).}
%footnote ends 
The set of subgradients of $\psi$ at $(\theta, x)$ is 
$$\partial \psi(\theta, x) = \partial_\theta \psi(\theta, x) \times \partial_x \psi(\theta, x)$$ 
where, expressed in terms of the function $\psi^o$ or the functions $\bar g$, $\bar k$ in (\ref{eq-gtds-gk}),
\begin{align}
   \partial_\theta \psi(\theta, x) & = \nabla_\theta \psi^o(\theta, x) + \N_{B_\theta}(\theta) = - \bar g(\theta, z) + \nabla p(\theta) + \N_{B_\theta}(\theta),  \label{eq-psi-subgrad1}\\ 
    \partial_x \psi(\theta, x) & = \nabla_x \psi^o(\theta, x) - \N_{B_x}(x) = \bar k(\theta, x) - \N_{B_x}(x). \label{eq-psi-subgrad2}
\end{align}    
(If either $\theta \not\in B_\theta$ or $x \not\in B_x$, $\partial \psi(\theta, x) = \emptyset$ by definition.) 
Note that $(\theta, x)$ is a saddle-point of (\ref{eq-minimax1}) if and only if $(0,0) \in \partial \psi(\theta, x)$. 
We shall also need  the set-valued mapping $\A$ from $\re^d \times \re^d$ to $\re^d \times \re^d$ defined by 
$$\A(\theta, x) : = \partial_\theta \psi(\theta, x) \times \big( - \partial_x \psi(\theta, x) \big).$$ 
By \cite[Corollary 37.5.2]{Roc70}, $\A(\cdot)$ is a maximal monotone mapping.%footnote starts
\footnote{A set-value mapping $G$ from $\re^m$ to $\re^m$ is monotone if $\langle x - x', y - y' \rangle \geq 0$ for every $y \in G(x)$ and $y' \in G(x')$. 
It is maximal monotone if for every $(x', y')$ that satisfies $\langle x - x', y - y' \rangle \geq 0$ for all $(x,y) \in \text{gph} \, G$, we have $(x', y') \in \text{gph} \, G$, where $\text{gph} \, G = \{(x, y) \mid y \in G(x), x \in \re^m \}$ is the graph of $G$.}
%footnote ends
We will soon use a general result on differential inclusions with such mappings to help us in analyzing the ODE (\ref{eq-gtds-ode}). 
To this end, let us first relate the mapping $\A$ to the ODE.

Comparing the r.h.s.\ of the ODE (\ref{eq-gtds-ode}) with the expressions (\ref{eq-psi-subgrad1})-(\ref{eq-psi-subgrad2}) of the subgradients, we see that if $(\theta(t), x(t)) = (\theta, x)$, then
\begin{equation} \label{eq-sd-prf-a0}
     \bar g(\theta, x) - \nabla p(\theta) + z_1  \in - \partial_\theta \psi(\theta, x), \qquad
        \bar k(\theta, x) + z_2  \in \partial_x \psi(\theta, x),
\end{equation}
where $z_1, z_2$ are the boundary reflection terms in (\ref{eq-gtds-ode}). Moreover, using Lemma~\ref{lem-decomp-vec-cones}, we can relate the above two vectors in the r.h.s.\ of the ODE~(\ref{eq-gtds-ode}) to the set $\A(\theta,x)$ as follows:

\begin{lem} \label{lem-sd-minvec}
Given $(\theta, x) \in B_\theta \times B_x$, let $z_1 = - \Pi_{D_1} ( \bar g(\theta, x) - \nabla p(\theta) )$ for $D_1 = \N_{B_\theta}(\theta)$, and let $z_2 = - \Pi_{D_2}  \bar k(\theta, x)$ for $D_2 = \N_{B_x}(x)$. 
Then 
$$(\bar g(\theta, x) - \nabla p(\theta) + z_1, \, \bar k(\theta, x) + z_2) \in - \A(\theta, x)$$ 
and it is the vector in $- \A(\theta, x)$ that has the minimal $\|\cdot\|_2$-norm.
\end{lem}

\begin{proof}
Applying Lemma~\ref{lem-decomp-vec-cones} with $D=B_\theta$, $h = \bar g(\theta,x) - \nabla p(\theta)$ and $y=-z_1$, and using also the expression (\ref{eq-psi-subgrad1}), we see that $\bar g(\theta, x) - \nabla p(\theta) + z_1$ is the minimal-norm vector in the set $- \partial_\theta \psi(\theta, x)$. Then applying Lemma~\ref{lem-decomp-vec-cones} with $D=B_x$, $\theta = x$, $h = \bar k(\theta,x)$ and $y=-z_2$, and using also the expression (\ref{eq-psi-subgrad2}), we see that $\bar k(\theta, x) + z_2$ is the minimal-norm vector in the set $\partial_x \psi(\theta, x)$. It then follows that $(\bar g(\theta, x) - \nabla p(\theta) + z_1, \, \bar k(\theta, x) + z_2)$ is the minimal-norm vector in $- \A(\theta, x)$. 
\end{proof}

We are now ready to study the limit set of the ODE~(\ref{eq-gtds-ode}). For this part of the proof, we use $y(\cdot)$ to denote a solution $(\theta(\cdot), x(\cdot))$ of the ODE, in order to simplify notation. 

We shall use a very general result from \cite[Chap.\ 3]{AuC85} concerning solutions to differential inclusions with maximal monotone mappings. 
By \cite[Theorem 1, Chap.\ 3]{AuC85}, for a maximal monotone mapping~$\A$, the differential inclusion $\dot y(t) \in - \A(y(t))$ has a unique absolutely continuous solution for any initial condition in the domain of $\A$, and this solution has the following properties:\vspace*{-3pt}
\begin{itemize}
\item[(i)] For almost all $t$, $ \dot y(t) $ is the minimal-norm (w.r.t.\ the Euclidean norm $\|\cdot\|_2$) vector in $- \A(y(t))$.
\item[(ii)] The norm of the minimal-norm vector of $\A(y(t))$ is nonincreasing as $t$ increases.
\item[(iii)] If $y_1(\cdot)$, $y_2(\cdot)$ are two solutions corresponding to two different initial conditions, then
$$\| y_1(t) -  y_2(t)  \|_2 \leq \| y_1(0) - y_2(0) \|_2, \qquad \forall \, t \geq 0.$$
\end{itemize}

For the maximal monotone mapping $\A$ in our case, we see that the solutions to the differential inclusion with $\A$ are just the solutions to the mean ODE (\ref{eq-gtds-ode}), so the solutions $y(\cdot)$ of the ODE have the properties listed above. We now use these properties (especially, the third one) to study the limit set of the ODE~(\ref{eq-gtds-ode}). By definition this is the set $\cap_{\tau \geq 0} \, \cl \big\{ y(t) \mid y(0) \in B_\theta \times B_x, t \geq \tau \big\}$ as we recall.
For any set $D \in B_\theta \times B_x$ and $t \geq 0$, let $F^t(D) = \big\{ y(t) \mid y(0) \in D \big\}$, i.e., the union of $y(t)$ for all initial conditions in $D$. We say $D$ is invariant (w.r.t.~the ODE) if $F^t(D) = D$ for all $t \geq 0$.

\begin{lem} \label{lem-invset0}
The limit set of the ODE (\ref{eq-gtds-ode}) is $E : = \cap_{t \geq 0} F^t(B_\theta \times B_x)$ (a nonempty compact set), and it is the largest invariant subset of $B_\theta \times B_x$.
\end{lem}

\begin{proof}
If $t_1 \leq t_2$, $F^{t_2}(B_\theta \times B_x) \subset F^{t_1}(B_\theta \times B_x)$. By the property (iii) above, for a compact set $D$, $F^t(D)$ is closed for each $t$. Hence for any $\tau \geq 0$, 
$$\cl \big\{ y(t) \mid y(0) \in B_\theta \times B_x, t \geq \tau \big\} = \cl \big( \cup_{t \geq \tau} \!F^t(B_\theta \times B_x) \big) = F^{\tau}(B_\theta \times B_x),$$ 
so the limit set of the ODE is $\cap_{\tau \geq 0} F^{\tau}(B_\theta \times B_x) = E$.  As the intersection of nonempty compact sets, this set is nonempty and compact.
That $E$ is the largest invariant subset of $B_\theta \times B_x$ follows from \cite[Ex.\ 6(c), p.\ 110]{Aki93} (which is a general result on semiflows).%footnote starts
\footnote{Alternatively, one can prove the invariance of $E$ by using the solution property (iii) of the differential inclusion listed above. Specifically, for each $t \geq 0$, by the definition of $E$, it is clear that $F^t(E) \subset E$. To prove the reverse inclusion, suppose for some $t \geq 0$ and $\bar y \in E$, $\bar y \not\in F^t(E)$. Fix that $t$, and let $\Delta$ be the set of all initial conditions from which $\bar y$ can be reached at time $t$; i.e., $\Delta = \{ \hat y \mid \hat y \in B_\theta \times B_x \ \text{and} \ y(t) = \bar y \ \text{if} \ y(0) = \hat y \}$. Then $\Delta \cap E = \emptyset$. By the property (iii) listed above, $\Delta$ is closed and therefore compact. Since $E$ is by definition the intersection of nested nonempty compact sets, $\cap_{\tau \geq 0} F^{\tau}(B_\theta \times B_x)$,  in order to have $\Delta \cap E = \emptyset$, we must have $\Delta \cap F^{\tau}(B_\theta \times B_x) = \emptyset$ for some $\tau \geq 0$.  This in turn implies that $\bar y \not\in F^{t+\tau}(B_\theta \times B_x)$ and hence $\bar y$ cannot be in $E$, a contradiction. Thus we must also have $F^{t}(E) \supset E$ for all $t \geq 0$, proving that $F^t(E)=E$ for all $t \geq 0$.}
%footnote ends
\end{proof}

By the above lemma $F^t(E) = E$ for all $t \geq 0$. The next lemma is a consequence of this invariance of $E$ and the solution property (iii) above. Its implications, especially, the following Cor.~\ref{cor-invset}, will be useful for our study of what points can be in $E$.  

\begin{lem} \label{lem-invset}
Let $\bar y \in E$, and let $y(\cdot)$ be the solution of the ODE (\ref{eq-gtds-ode}) with the initial condition $y(0)=\bar y$. Then there exist a sequence of times $\tau_i \to \infty$ such that $y(\tau_i) \to \bar y$ as $i \to \infty$.
\end{lem}

\begin{proof}
In this proof, abusing notation, we let $F^t$ also stand for the mapping that maps an initial condition $\hat y$ of the ODE~(\ref{eq-gtds-ode}) to its solution at time $t$.
Thus, in particular, $F^t(\hat y) = \bar y$ means $y(t)=\bar y$ if $y(0) = \hat y$.

Since $E$ is invariant by Lemma~\ref{lem-invset0}, for each $t \geq 0$, there exists some $y_t \in E$ such that $F^t(y_t) = \bar y$. Since $E$ is compact, we can choose an increasing sequence $t_i \to \infty$ such that $y_{t_i} \to y_{\infty} \in E$. Then
$$\| F^{t_i}(y_\infty) - \bar y \|_2 = \| F^{t_i}(y_\infty) - F^{t_{i}}(y_{t_{i}}) \|_2 \leq \| y_\infty - y_{t_{i}} \|_2 \ \to 0 \ \text{as} \ i \to \infty,$$
where the last inequality follows from the solution property (iii) of the differential inclusion given earlier. Thus $F^{t_i}(y_\infty) \to \bar y$ as $i \to \infty$.
For any $j > i$, by the solution property (iii) again,
\begin{align*}
  \| F^{t_j - t_i} (\bar y) - F^{t_j}(y_\infty) \|_2 =   \| F^{t_j - t_i} \big( F^{t_i} (y_{t_i} ) \big) - F^{t_j}(y_\infty) \|_2 \leq  \| y_{t_i} - y_\infty \|_2.
\end{align*}    
Now choose an increasing sequence of integers $j_i$ such that $j_i > i$ and $\tau_i :=t_{j_i} - t_i \to \infty$ as $i \to \infty$. We then have
\begin{align*}
\| F^{\tau_i}(\bar y) - \bar y \|_2 & \leq \| F^{\tau_i}(\bar y)  - F^{t_{j_i}}(y_\infty) \|_2 + \| F^{t_{j_i}}(y_\infty) - \bar y \|_2 \\
& \leq  \| y_{t_i} - y_\infty \|_2 + \| F^{t_{j_i}}(y_\infty) - \bar y \|_2 \ \to 0 \ \ \text{as} \  i \to \infty.
\end{align*} 
This proves the lemma.
\end{proof}

\begin{cor} \label{cor-invset}
Let $y(\cdot)$ be a solution of the ODE (\ref{eq-gtds-ode}) with $y(0) \in E$. Then, for any continuous function $V(\cdot)$ on $B_\theta \times B_x$ such that $V(y(t))$ is nonincreasing in $t$, $V(y(t))$ is constant over $t \geq 0$. 
\end{cor}

We are now ready to prove Prop.\ \ref{prop-sd-limitset}, which says essentially that $E \cap (\re^d \times \sp \{\phi(\S)\})$ is the subset of saddle points, $D_\theta \times \{\bar x\}$, of the minimax problem (\ref{eq-minimax1}), where $\bar x$ is its unique dual optimal solution in $\sp \{\phi(\S)\}$. We will use the preceding corollary together with the special properties of the convex-concave function $\psi^o$  to prove this. (In particular, besides the uniqueness of $\bar x$, one property of $\psi^o$ we will use is that the only term in $\psi^o$ in which $\theta$ and $x$ ``interact'' is a bilinear function of $(\theta, x)$; cf.\ (\ref{eq-psio}).)

\begin{proof}[Proof of Prop.\ \ref{prop-sd-limitset}]
Let $E^o$ be the limit set of the ODE for all initial conditions $\theta(0) \in B_\theta, x(0) \in B_x \cap \sp \{\phi(\S)\}$.
Any saddle point $(\tilde \theta, \tilde x)$ of the minimax problem (\ref{eq-minimax1}) lies in $E$. This is because $(0,0) \in \partial \psi(\tilde \theta, \tilde x)$ and so, in view of Lemma~\ref{lem-sd-minvec}, the solution of the ODE (\ref{eq-gtds-ode}) from $(\tilde \theta, \tilde x)$ is just $\theta(\cdot)\equiv \tilde \theta,  x(\cdot) \equiv \tilde x$, implying that $(\tilde \theta, \tilde x)$ is in the limit set of the ODE, which is $E$ by Lemma~\ref{lem-invset0}.
Thus, $D_\theta \times \{\bar x\} \subset E^o \subset E \cap (\re^d \times \sp \{\phi(\S)\})$. To prove the proposition, we now use Cor.~\ref{cor-invset} and proof by contradiction to show that $E \cap (\re^d \times \sp \{\phi(\S)\})$ cannot contain other points than those saddle points $D_\theta \times \{\bar x\}$.

First, suppose $(\hat \theta, \hat x) \in E$ where $\hat x \in \sp \{\phi(\S)\}$ and $\hat x \not = \bar x$. Let $y(\cdot)$ be the solution of the ODE (\ref{eq-gtds-ode}) with $y(0) = (\hat \theta, \hat x)$. Since $\hat x \in \sp \{\phi(\S)\}$, the $x$-component of $y(t)$ lies in $\sp \{\phi(\S)\}$ for all $t$; since $y(\cdot)$ is continuous, there is a nonzero $\bar t$ such that for all $t < \bar t$, the $x$-component of $y(t)$ never equals $\bar x$.
For any $t \geq 0$, consider $\dot V(y(t))$ for the function 
$V(\theta, x) = \tfrac{1}{2} \| \theta - \bar \theta \|_2^2 + \tfrac{1}{2} \| x - \bar x \|_2^2$, where $\bar \theta \in D_\theta$ so that $(\bar \theta, \bar x)$ is a saddle point of (\ref{eq-minimax1}).
To simplify notation, write 
$$\bar k(\theta, x) = A \theta + b - C x \qquad \text{and} \qquad \bar g(\theta, x) = - A^\tr x,$$
where $A = \Phi^\tr \Xi (P^{(\lambda)} - I ) \Phi$, $C = \Phi^\tr \Xi \,\Phi$, and the vector $b$ is the constant term%footnote starts
\footnote{Recall the expression (\ref{eq-Tl0}) of the operator $\Tl$: $\Tl v = r_\pi^{(\lambda)} + \Pl v$; the vector $b$ here is given by $b=\Phi^\tr \Xi \,  r_\pi^{(\lambda)}$.}
%footnote ends
in the affine function $\bar k(\theta, x)$ (cf.\ (\ref{eq-gtds-gk})).
Since $(\bar \theta, \bar x)$ is a saddle point of (\ref{eq-minimax1}), $(0,0) \in \partial \psi(\bar \theta, \bar x)$, so for some $\bar z_1 \in - \N_{B_\theta}(\bar \theta)$ and $\bar z_2 \in - \N_{B_x}(\bar x)$, 
\begin{equation} \label{eq-sdprf-a1}
     A^\tr \bar x + \nabla p(\bar \theta) - \bar z_1 = 0, \qquad  A \bar \theta + b - C \bar x + \bar z_2 = 0
\end{equation}  
(cf.\ (\ref{eq-psi-subgrad1})-(\ref{eq-psi-subgrad2})).
At a point $(\theta, x) = y(t)$, let $z_1 \in - \N_{B_\theta}(\theta), z_2 \in - \N_{B_x}(x) $ be the boundary reflection terms in (\ref{eq-gtds-ode}) at time $t$. We have   
\begin{align*}
   \dot V(y(t)) & = \langle \theta - \bar \theta, \, - A^\tr x - \nabla p(\theta) + z_1 \rangle  +  \langle x - \bar x, \, A \theta + b - C x + z_2  \rangle \leq 0,
\end{align*}
where the non-positivity follows from the monotonicity of the set-valued mapping $\A(\cdot)$ and the fact that $(0,0) \in -\A(\bar \theta, \bar x)$ and $(- A^\tr x - \nabla p(\theta) + z_1, \, A \theta + b - C x + z_2) \in -\A(\theta, x)$ (Lemma~\ref{lem-sd-minvec}).
Consider each term in the expression of $\dot V(y(t))$ above. Using (\ref{eq-sdprf-a1}), we have
\begin{align*}
 \langle \theta - \bar \theta, - A^\tr x - \nabla p(\theta) + z_1 \rangle & = \langle \theta - \bar \theta, \, - A^\tr x - \nabla p(\theta) + z_1 \rangle + \langle \theta - \bar \theta, \, A^\tr \bar x + \nabla p(\bar \theta) - \bar z_1 \rangle \\
 & = - \langle \theta - \bar \theta, \, A^\tr (x - \bar x) \rangle - \langle \theta - \bar \theta, \, \nabla p(\theta) - \nabla p(\bar \theta) \rangle +  \langle \theta - \bar \theta, \, z_1 - \bar z_1 \rangle, \\
 \langle x - \bar x, \, A \theta + b - C x + z_2 \rangle  & = \langle x - \bar x, \, A \theta + b - C x + z_2 \rangle - \langle x - \bar x, \, A \bar \theta + b - C \bar x + \bar z_2 \rangle \\
 & = \langle x - \bar x, \, A (\theta - \bar \theta) \rangle -  \langle x - \bar x, \, C (x - \bar x) \rangle + \langle x - \bar x, \, z_2 - \bar z_2 \rangle,
\end{align*} 
and hence 
\begin{align*}
   \dot V(y(t)) & = - \langle \theta - \bar \theta, \, \nabla p(\theta) - \nabla p(\bar \theta) \rangle +  \langle \theta - \bar \theta, \, z_1 - \bar z_1 \rangle 
   -  \langle x - \bar x, \, C (x - \bar x) \rangle  + \langle x - \bar x, \, z_2 - \bar z_2 \rangle.
\end{align*}
In the summation on the r.h.s., each term is nonpositive,%footnote starts
\footnote{For the first, second and forth terms, this is because $\nabla p(\cdot)$, $\N_{B_{\theta}}(\cdot)$ and $\N_{B_{x}}(\cdot)$ are all monotone set-valued mappings \cite[Corollary 31.5.2]{Roc70}. For the third term, it is because $-C$ is a symmetric negative semidefinite matrix.}
%footnote ends
but the third term $ -  \langle x - \bar x, C (x - \bar x) \rangle < 0$ if $x \in \sp \{ \phi(\S)\}$ and $x \not=\bar x$. 
Thus $V(y(t))$ is a nonincreasing function of $t$ and for almost all $t \in [0, \bar t \,)$, $\dot V(y(t)) < 0$, so $V(y(t))$ cannot remain constant. This contradicts Cor.~\ref{cor-invset}. Therefore, points like $(\hat \theta, \hat x)$ cannot lie in $E$; i.e., if $(\hat \theta, \hat x) \in E$ and $\hat x \in \sp \{\phi(\S)\}$, then $\hat x = \bar x$.

Next, suppose for some $\hat \theta \not \in D_\theta$, $(\hat \theta, \bar x) \in E$. 
Let $y(\cdot)$ be the solution with $y(0) = (\hat \theta, \bar x)$.
Since $y(t) \in E$ by the invariance of $E$ (Lemma~\ref{lem-invset0}) and its $x$-component $x(t) \in \sp \{ \phi(\S)\}$ for all $t$, we must have $x(\cdot) \equiv \bar x$ by the preceding proof. This means that for all $t$, $\dot x(t) = 0$ and therefore by (\ref{eq-gtds-ode}) and (\ref{eq-sd-prf-a0}),
\begin{equation} \label{eq-sdprf-a2}
   0 = \bar k(y(t)) + z_2(t) \in \partial \psi_x(y(t)),
\end{equation}  
where $z_2(t)$ is the boundary reflection term. Since $y(\cdot)$ is continuous and its $\theta$-component starts from $\hat \theta$ outside the compact set $D_\theta$,  
there is a nonzero time $\bar t$ such that for all $t < \bar t$, $y(t)$ are not saddle points of (\ref{eq-minimax1}).

Consider now $\dot V(y(t))$ for the differentiable function $V(\theta,x) =  \langle \bar x, A \theta \rangle + p(\theta)$.
At the point $(\theta, \bar x) = y(t)$, we have $\bar g(\theta, \bar x) = - A^\tr \bar x$, so with $z_1 \in - \N_{B_\theta}(\theta)$ being the boundary reflection term $z_1(t)$ in (\ref{eq-gtds-ode}), we have
\begin{align}
  \dot V(y(t)) & = \langle A^\tr \bar x + \nabla p(\theta), \,  - A^\tr \bar x  - \nabla p(\theta) + z_1 \rangle \notag \\
     & = - \| A^\tr \bar x + \nabla p(\theta) - z_1 \|_2^2 + \langle z_1, \, - A^\tr \bar x- \nabla p(\theta) + z_1 \rangle. \label{eq-sdprf-a3}
\end{align}
By Lemma~\ref{lem-decomp-vec-cones}, the second term is $0$. So $\dot V(y(t)) \leq 0$, implying that $V(y(t))$ is a nonincreasing function of $t$. 
For $t < \bar t$, since $(\theta, \bar x) = y(t)$ is not a saddle point of (\ref{eq-minimax1}), the first term in (\ref{eq-sdprf-a3}) is strictly less than $0$ (otherwise, in view of (\ref{eq-sdprf-a2}) and (\ref{eq-sd-prf-a0}), we would have $(0,0) \in \partial \psi(y(t))$, implying $y(t)$ is a saddle point, a contradiction). 
Thus $\dot V(y(t)) < 0$ for almost all $t$ in the interval $[0, \bar t\,)$ and consequently $V(y(t))$ cannot remain constant. This contradicts Cor.~\ref{cor-invset}. So we conclude that if $(\hat \theta, \bar x) \in E$, then $\hat \theta \in D_\theta$.
This completes the proof.
\end{proof}

Finally, let us give the proof of Lemma~\ref{lem-saddle-relation}, which says that if $x_{\rm \text{\rm opt}}$ is in the interior of the constraint set $B_x$, then the limit set $D_\theta \times \{\bar x\}$ coincides with the desired set $\Theta_{\text{\rm opt}} \times \{ x_{\text{\rm opt}}\}$.

\begin{proof}[Proof of Lemma~\ref{lem-saddle-relation}]
The point $x_{\rm \text{\rm opt}}$ is the unique optimal solution of the maximization problem $\sup_{x \in \re^d} \{ \inf_{\theta\in B_\theta} \psi^o(\theta, x) \}$ in the subspace $\sp \{\phi(\S)\}$.
So if $x_{\rm \text{\rm opt}} \in B_x$, we must have $\bar x = x_{\rm \text{\rm opt}}$. 
Now suppose $\bar x =  x_{\rm \text{\rm opt}} \in \text{int}(B_x)$, the interior of $B_x$.
By the definition of a saddle point, we have that $\bar \theta \in D_\theta$ if and only if
\begin{equation} \label{eq-sadd1}
 \psi^o(\theta, \bar x) \geq \psi^o(\bar \theta, \bar x) \geq \psi^o(\bar \theta, x), \qquad \forall \, \theta \in B_\theta, \ x \in B_x, 
\end{equation} 
whereas $\bar \theta \in \Theta_{\text{\rm opt}}$ if and only if
\begin{equation} \label{eq-sadd0}
  \psi^o(\theta, \bar x) \geq \psi^o(\bar \theta, \bar x) \geq \psi^o(\bar \theta, x), \qquad \forall \, \theta \in B_\theta, \ x \in \re^d.
\end{equation}
Then any $\bar \theta \in D_\theta$ must satisfy (\ref{eq-sadd0}), because the second inequality in (\ref{eq-sadd1}) together with the fact $\bar x \in \text{int}(B_x)$ implies that $\bar x$ attains a local maximum of the concave function $\psi^o(\bar \theta, \cdot)$, and hence $\bar x$ must attain the global maximum of $\psi^o(\bar \theta, \cdot)$. This shows $D_\theta \subset \Theta_{\text{\rm opt}}$. 
Of course, any $\bar \theta \in \Theta_{\text{\rm opt}}$ satisfies (\ref{eq-sadd1}) and is therefore in $D_\theta$. This shows $D_\theta \times \{\bar x\} =\Theta_{\text{\rm opt}} \times \{ x_{\text{\rm opt}} \}$ if $x_{\rm \text{\rm opt}} \in \text{int}(B_x)$.
\end{proof}

\subsubsection{Convergence properties} \label{sec-gtds-conv}

We can now state our convergence results for the single-time-scale GTDa algorithm. 
Recall that
$$\Theta_{\text{\rm opt}} = \argmin_{\theta \in B_\theta} \{ J(\theta) + p(\theta) \}, \qquad x_{\text{\rm opt}} =  \argmax_{x \in \sp \{\phi(\S)\}} \big\{ \inf_{\theta\in B_\theta} \psi^o(\theta, x) \big\}.$$
The set
$\Theta_{\text{\rm opt}} \times \{x_{\text{\rm opt}}\}$ is a subset of saddle points of the minimax problem (\ref{eq-minimax0}), which we want to solve and which places no constraints on the maximizer.  The set $D_\theta \times \{\bar x\}$ is a subset of saddle points of the minimax problem (\ref{eq-minimax1}) with the constraint set $B_x$ for the maximizer. 

\begin{thm} \label{thm-gtds}
Consider the GTDa algorithm (\ref{eq-gtd1m-th})-(\ref{eq-gtd1m-x}) under Assumption~\ref{cond-collective}. Let the initial $x(0), \e_0 \in \sp \{\phi(\S)\}$. Then for each initial condition of the algorithm, the following hold:\vspace*{-3pt}
\begin{itemize}
\item[\rm (i)] In the case of diminishing stepsize, let the stepsizes $\{\alpha_n\}$ satisfy the conditions in Assumption~\ref{cond-large-stepsize} for the fast-time-scale stepsizes. Then 
$\lim_{n \to \infty} \E \big[ \text{\rm dist} \big( (\theta_n, x_n), \, D_\theta \times \{\bar x\} \big)\big] = 0$. Moreover, there exists a sequence of positive numbers $T_n \to \infty$ such that for any $\epsilon > 0$,
\begin{equation} 
  \limsup_{n \to \infty} \Pr \big( (\theta_i, x_i) \not\in N_\epsilon(D_\theta \times \{\bar x\}), \ \text{some} \ i \in [n, m(n, T_n) ] \big)  = 0.
\end{equation}  
\item[\rm (ii)] In the case of constant stepsize $\alpha_n = \alpha$ for all $n$, for any integers $n_{\alpha}$ such that $\alpha \, n_{\alpha} \to \infty$ as $\alpha \to 0$,
there exist positive numbers $\{T_{\alpha} \mid \alpha > 0\}$ with $T_{\alpha} \to \infty$ as $\alpha \to 0$, such that for any $\epsilon > 0$,
\begin{equation} 
  \limsup_{\alpha \to 0} \Pr \big( (\theta^{\alpha}_{n_\alpha + i}, x^{\alpha}_{n_\alpha + i})  \not\in N_\epsilon(D_\theta \times \{\bar x\}), \ \text{some} \ i \in [0, T_\alpha/\alpha ] \big)  = 0.
\end{equation}  
\item[\rm (iii)] In the case of constant stepsize, for the averaged iterates $\{(\bar \theta^{\alpha}_n, \bar x^{\alpha}_n)\}$,%footnote starts
\footnote{Similar to (\ref{eq-aveite}), these averaged iterates are given by $\bar \theta^{\alpha}_n = \textstyle{\frac{1}{n - n_0} \sum_{i=n_0}^{n-1} \theta^{\alpha}_i}$,  $\bar x^{\alpha}_n = \textstyle{\frac{1}{n - n_0} \sum_{i=n_0}^{n-1} x^{\alpha}_i}$.} 
%footnote ends 
the conclusions of Theorem~\ref{thm-gtd1-average} hold with the set $D_\theta \times \{\bar x\}$ in place of $\Theta_{\text{\rm opt}}$, with the averaged $(\theta,x)$-iterates in place of the averaged $\theta$-iterates, and with obvious modifications to remove all the references to the fast-time-scale stepsize parameter $\beta$.\vspace*{-3pt}
\end{itemize}
If the constraint set $B_x$ contains $x_{\rm \text{\rm opt}}$ in its interior, then $D_\theta \times \{\bar x\} =\Theta_{\text{\rm opt}} \times \{ x_{\text{\rm opt}} \}$ in the above.
\end{thm}
%\smallskip

\begin{proof}[Outline of proof.]
For the parts (i) and (ii), we first apply stochastic approximation theory to show that the average dynamics of the algorithm is characterized by the mean ODE~(\ref{eq-gtds-ode}). Specifically, in the case of diminishing stepsize, we apply \cite[Theorem 8.2.3]{KuY03}, and in the case of constant stepsize, we apply \cite[Theorem 8.2.2]{KuY03}. The required conditions are straightforward to verify.%footnote starts
\footnote{In order to apply \cite[Theorems 8.2.2 and 8.2.3]{KuY03} to the algorithm (\ref{eq-gtd1m-th})-(\ref{eq-gtd1m-x}), the conditions we need to verify include uniform integrability, tightness, continuity conditions as well as an ``averaging condition,'' similar to those listed in Section~\ref{sec-gtd1-cond}. In particular, for diminishing stepsize, using the notations of Section~\ref{sec-gtd1-cond} for GTDa, we need the following sets to be uniformly integrable 
$$\big\{ \big( g(\theta_n, x_n, Z_n) - \nabla p(\theta_n), \, k(\theta_n, x_n, Z_n) + \e_n \omega_{n+1} \big) \big\}_{n \geq 0}, \ \
\big\{ \big( g(\theta, x, Z_n) - \nabla p(\theta), \, k(\theta, x, Z_n) \big) \mid n \geq 0, \theta \in B_\theta, x \in B_x \big\}.$$
Since the algorithm is constrained, the uniform integrability of these sets follows from the uniform integrability of $\{\e_n\}$, as explained in Section~\ref{sec-gtd1-cond}. For the same reason, the uniform integrability conditions are met in the case of constant stepsize (we do not write down those sets here, which are the same as the above except that they involve the $(\theta, x)$-iterates generated with all stepsize parameters $\alpha$.) 
The tightness condition on $\{Z_n\}$ is satisfied as before. The continuity condition requires the functions $g(\theta, x, z) - \nabla p(\theta)$ and $k(\theta, x, z)$ to be continuous in $(\theta, x)$ uniformly in $z$ in a compact set. It is satisfied clearly as before. The functions $\bar g(\theta, x) - \nabla p(\theta)$ and $\bar k(\theta, x)$ are required to be continuous, which are also clearly satisfied. Finally, the ``averaging condition'' requires that for each $(\theta, x)$ and each compact set $D \subset \Z$,
\begin{align*}
& \textstyle{ \lim_{m, n \to \infty} \frac{1}{m} \sum_{i=n}^{n+m-1} \E_n \big[ g(\theta, x,  Z_i) - \bar g(\theta, x) \big] \I(Z_n \in D) = 0} \ \ \ \text{in mean}, \\
& \textstyle{ \lim_{m, n \to \infty} \frac{1}{m} \sum_{i=n}^{n+m-1} \E_n \big[ k(\theta, x,  Z_i) - \bar k(\theta, x) \big] \I(Z_n \in D) = 0} \ \ \ \text{in mean}.
\end{align*}
This follows from the same proof of Lemma~\ref{lem-conv-mean-cond1}(i), since both $g(\theta, x, \cdot)$ and $k(\theta, x, \cdot)$ are Lipschitz continuous in the trace variable $\e$. Note that the regularizer $p(\theta)$ does not appear in the ``averaging condition'' at all, although $\nabla p(\theta)$ appears in the algorithm and its mean ODE. This is because $\nabla p(\theta)$ does not depend on $z$, so for fixed $\theta$, $\nabla p(\theta)$ and $\E_\zeta [\nabla p(\theta)]$ canceled each other out in the first convergence-in-mean requirement above. Thus, for the constrained algorithm, the convergence analysis is the same with or without the regularizer $p(\cdot)$. \label{footnote-gtds-cond}
}
%footnote ends
These theorems give us the conclusions stated in the part (i) and (ii), respectively, with the set involved being the limit set of the ODE associated with the algorithm, and we then use the expression of the limit set given by Prop.~\ref{prop-sd-limitset} to obtain the conclusions as stated. 

The procedure of deriving the part (iii) for the averaged $(\theta,x)$-iterates is similar to that described in Section~\ref{sec-aveite} after Theorem~\ref{thm-gtd1-average}. It uses the observation that with constant stepsize $\alpha$, the iterates $(\theta_n^\alpha, x^\alpha_n)$ jointly with the states/memory states and traces form a weak Feller Markov chain, and the proof of this part is essentially the same as that given in \cite[Section 4.3]{etd-wkconv}.

Finally, the last statement about when $D_\theta \times \{\bar x\} =\Theta_{\text{\rm opt}} \times \{ x_{\text{\rm opt}} \}$ follows from Lemma~\ref{lem-saddle-relation}.
\end{proof}

\begin{rem}[About the condition on $B_x$ and a comparison with two-time-scale GTD] \label{rem-Bx} \rm \hfill\\
Recall that for the two-time-scale GTDa and GTDb algorithms in Section~\ref{sec-3}, to ensure desired convergence properties, we required the constraint set $B_x$ to be large enough so that $B_x \supset \{x_\theta \mid \theta \in B_\theta\}$, where $x_\theta$ is the unique solution to the linear system of equations, $\bar k(\theta, x) =0, x \in \sp \{\phi(\S)\}$. 
For each $\theta$, an optimal solution to the maximization problem $\sup_{x} \psi^o(\theta, x)$ is just $x_\theta$. 
Therefore, $x_{\rm \text{\rm opt}}$ must be in the set $\{x_\theta \mid \theta \in B_\theta\}$, and if $\text{int}(B_x) \supset \{ x_\theta \mid \theta \in B_\theta\}$, then $x_{\rm \text{\rm opt}} \in \text{int}(B_x)$. This shows that, to ensure that the single-time-scale GTDa can approach the desired solution set $\Theta_{\text{\rm opt}}$, a sufficient condition for $B_x$ is one that is almost identical to the condition we imposed earlier for the two-time-scale GTD. 
But the latter sufficient condition is clearly far more stronger than the inclusion $x_{\rm \text{\rm opt}} \in \text{int}(B_x)$ needed by the single-time-scale GTDa. In particular, if $x_{\rm \text{\rm opt}} = 0$, as is the case when the optimal solutions in $\Theta_{\text{\rm opt}}$ satisfy that $J(\theta) = 0$ (i.e., they solve the projected Bellman equation), then we only need $0 \in \text{int}(B_x)$, which is always true. 
Another remark is that, as we will show in Section~\ref{sec-gtda-revisit}, the weaker condition $x_{\rm \text{\rm opt}} \in \text{int}(B_x)$ is in fact also sufficient for the two-time-scale GTDa algorithm, because it can be viewed as a two-time-scale algorithm for solving the minimax problem (\ref{eq-minimax1}), with the maximizer operating at a faster time-scale than the minimizer.
This suggests an advantage of the single- or two-time-scale constrained GTDa algorithm over the two-time-scale constrained GTDb algorithm.
\end{rem}

\begin{rem}[About setting $\lambda$ according to the composite scheme] \rm 
Theorem~\ref{thm-gtds} also holds for the single-time-scale GTDa algorithm when the $\lambda$-parameters are set according to the composite scheme described in Section~\ref{sec-cp-extension}. The analysis is essentially the same as the one given in this subsection, apart from that a different Bellman operator $\Tl$ is involved.
\end{rem}

\subsubsection{Biased variant} \label{sec-gtds-bias}

For the case of state-dependent $\lambda$, one can apply a biased variant of the single-time-scale GTDa algorithm, which is more robust. Like its two-time-scale counterpart (\ref{eq-gtd1b-th})-(\ref{eq-gtdb-x}) we discussed in Section~\ref{sec-bias-vrt}, this algorithm uses a bounded function $h(\e)$ of traces in computing $\{(\theta_n, x_n)\}$: 
\begin{align}
   \theta_{n+1} & = \Pi_{B_\theta} \Big( \theta_n + \alpha_n \,  \rho_n \big(\phi(S_n) - \gma_{n+1} \phi(S_{n+1}) \big) \cdot h(\e_n)^\tr x_n - \alpha_n \nabla p(\theta_n) \Big), \label{eq-gtdsv-th} \\
   x_{n+1} & = \Pi_{B_x} \Big( x_n + \alpha_n \big(h(\e_n) \cdot \delta_n(v_{\theta_n}) - \phi(S_n) \phi(S_n)^\tr x_n\big) \Big). \label{eq-gtdsv-x}
\end{align} 
We shall assume that $h$ is a bounded Lipschitz continuous function. 
As in Section~\ref{sec-bias-vrt}, we consider a family of bounded Lipschitz continuous functions $\{h_K | K > 0\}$ that satisfy (\ref{eq-h})-(\ref{eq-h2}), and we analyze the behavior of the above algorithm for $h=h_K$, as $K$ increases (and the degree of bias correspondingly decreases). We will show that the algorithm solves minimax problems that are approximations of the minimax problem $\inf_{\theta \in B_\theta} \sup_{x \in B_x} \psi^o(\theta, x)$ solved by GTDa.

The mean ODE associated  with the above biased GTDa has the same form as the ODE (\ref{eq-gtds-ode}) of GTDa, 
\begin{equation} \label{eq-gtdsv-ode}
  \left( \! \begin{array}{c} \dot{\theta}(t) \\ \dot x(t) \end{array} \!\right) 
 =  \left( \! \begin{array}{c} \bar g_h\big(\theta(t), x(t)\big) - \nabla p\big(\theta(t)\big) \\
 \bar k_h \big(\theta(t), x(t) \big) \end{array} \!\right) +  \left( \! \begin{array}{l} z_1(t) \\ z_2(t) \end{array} \!\right),      \quad z_1(t) \in - \N_{B_\theta}(\theta(t)), \  z_2(t) \in - \N_{B_x}(x(t)),
\end{equation} 
where the functions $\bar g_h, \bar k_h$ are defined according to (\ref{eq-gtd1b-th})-(\ref{eq-gtdb-x}) as
\begin{equation} \label{eq-gtdsvrt-gk}
  \bar g_h(\theta, x) =   \E_\zeta \big[ \, \rho_0 \big(\phi(S_0) - \gma_{1}  \phi(S_{1}) \big) \cdot h(\e_0)^\tr x \big], \quad   \bar k_h(\theta, x) =  \E_\zeta \big[ h(\e_0) \, \bar{\delta}_0(v_\theta) \big] - \Phi^\tr \Xi \, \Phi \, x.
\end{equation}
Recall $\bar \delta_0(v_\theta) = \rho_0 r(S_0, S_1) + \rho_0 \big(\gma_{1}  \phi(S_{1}) - \phi(S_0) \big)^\tr\! \theta$. Denote $C = \Phi^\tr \Xi \, \Phi$, and define matrix $A_h$ and vector $b_h$ by
\begin{equation} \label{eq-Ahbh}
  A_h = \E_\zeta \big[ h(\e_0) \, \cdot \rho_0 \big(\gma_{1}  \phi(S_{1}) - \phi(S_0) \big)^\tr \big], \qquad b_h = \E_\zeta \big[ h(\e_0) \, \cdot \rho_0 r(S_0, S_1) \big].
\end{equation}
Then
\begin{equation} \label{eq-ghkh}
    \bar g_h(\theta, x) = -A_h^\tr x, \qquad   \bar k_h(\theta, x) = A_h \theta + b_h  - C x.
\end{equation}
Note that as shown in the proof of Lemma~\ref{lem-vrt-xth}, under the condition (\ref{eq-h2}) on the function $h(\e)$ (which requires $h(\e) \in \sp \{\phi(\S)\}$ for $\e \in \sp \{\phi(\S)\}$),
$b_h \in \sp\{\phi(\S)\}$ and the column space of $A_h$ is contained in $\sp\{\phi(\S)\}$.

Now consider the minimax problem 
\begin{equation} \label{eq-minimax-vrt}
  \inf_{\theta \in B_\theta} \sup_{x \in B_x} \psi^o_h(\theta, x),
\end{equation}  
for the convex-concave function
\begin{equation} \label{eq-psih-vrt}
   \psi^o_h(\theta, x) : =  x^\tr  ( A_h \theta + b_h) -  \tfrac{1}{2} \, x^\tr C x + p(\theta).
\end{equation}
The structure of this minimax problem and its relation with the ODE~(\ref{eq-gtdsv-ode}) are exactly the same as that of the minimax problem (\ref{eq-minimax1}) and its relation with the ODE~(\ref{eq-gtds-ode}). Thus, the same analysis given in Section~\ref{sec-gtds-ode} for GTDa applies to the biased variant here.
In particular, Prop.~\ref{prop-sd-limitset} holds in this case, with the minimax problem involved being (\ref{eq-minimax-vrt}).
That is, for all initial $x(0) \in \sp \{\phi(\S)\}$, the limit set of the ODE (\ref{eq-gtdsv-ode}) is the subset of saddle points, $D^h_\theta \times \{ \bar x^h\}$, of the minimax problem (\ref{eq-minimax-vrt}), where $D^h_\theta$ consists of its primal optimal solutions and $\bar x^h$ is its unique dual optimal solution in $\sp \{\phi(\S)\}$, i.e.,
$$ D^h_\theta = \argmin_{\theta \in B_\theta} \big\{ \sup_{x \in B_x} \psi^o_h(\theta, x) \big\}, \quad \bar x^h =  \argmax_{x \in \sp \{\phi(\S)\}} \big\{ \inf_{\theta\in B_\theta} \psi^o_h(\theta, x) \big\}.$$
Moreover, regarding the convergence properties of the biased GTDa algorithm, the same conclusions of Theorem~\ref{thm-gtds}(i)-(iii) hold with $D^h_\theta \times \{ \bar x^h\}$ in place of $D_\theta \times \{\bar x\}$.

We now relate the minimax problem (\ref{eq-minimax-vrt}) and its saddle points $D^h_\theta \times \{ \bar x^h\}$ to the minimax problem (\ref{eq-minimax1}) solved by GTDa and its saddle points $D_\theta \times \{\bar x\}$.

\begin{lem}[Approximation property of (\ref{eq-minimax-vrt})] \label{lem-bias-apprsadd-lim}
Let $\{ h_K \mid K > 0\}$ be a family of bounded Lipschitz continuous functions that satisfy (\ref{eq-h})-(\ref{eq-h2}). Then
$\sup_{\theta \in B_\theta, x \in B_x}  \big| \psi^o_{h_K}(\theta, x) - \psi^o(\theta, x) \big| \to 0$ as $K \to \infty$,
and for any $\epsilon > 0$, there exist $K_\epsilon > 0$ such that
$$D^{h_K}_\theta \times \{ \bar x^{h_K}\} \subset N_\epsilon \big( D_\theta \times \{ \bar x\} \big), \qquad \forall \, K \geq K_\epsilon.$$
\end{lem}

\begin{proof}
The proof is similar to that of Lemma~\ref{lem-bias-appr-lim}. We can express the convex-concave function $\psi^o$ as 
$\psi^o(\theta, x) = x^\tr (A \theta + b) - \tfrac{1}{2} x^\tr C x + p(\theta)$, where
$A = \E_\zeta \big[ \e_0 \, \cdot \rho_0 \big(\gma_{1}  \phi(S_{1}) - \phi(S_0) \big)^\tr \big]$ and $b = \E_\zeta \big[ \e_0 \, \cdot \rho_0 r(S_0, S_1) \big]$.
Then 
$$ \psi^o_{h_K}(\theta, x) - \psi^o(\theta, x)  = x^\tr (A_{h_K} - A) \theta +  x^\tr (b_{h_K} - b).$$
To prove $\sup_{\theta \in B_\theta, x \in B_x}  \big| \psi^o_{h_K}(\theta, x) - \psi^o(\theta, x) \big| \to 0$ as $K \to \infty$, since $\theta$ and $x$ are confined in bounded sets, it is sufficient to prove that 
$ \| A_{h_K} - A \| \to 0$ and $\| b_{h_K} - b \| \to 0$ as $K \to \infty$. In turn, in view of the expressions of these matrices and vectors, it is sufficient to prove
$ E_\zeta \big[ \big\|  h_K(\e_0) - \e_0  \big\|  \big]  \to 0$ as $K \to \infty$, which we already proved in the proof of Lemma~\ref{lem-bias-appr-lim}.

For the second part of the lemma, we use proof by contradiction. Suppose the statement does not hold. Then there exist some $\epsilon > 0$, a sequence $\{K_i\}$ with $K_i \to \infty$ as $i \to \infty$, and a sequence of points $(\theta_i, x_i) \in D^{h_i}_\theta \times \{ \bar x^{h_i}\}$ with $h_i = h_{K_i}$, such that $(\theta_i, x_i) \not\in N_\epsilon \big( D_\theta \times \{ \bar x\} \big)$ for all $i$.
Since all the points lie in the compact set $B_\theta \times B_x$,
by choosing a subsequence if necessary, we can assume that the sequence $\{(\theta_i, x_i)\}$ converges to $(\theta_\infty, x_\infty) \in B_\theta \times B_x$.
Then $x_\infty \in \sp \{\phi(\S)\}$ (since all the $x_i$'s lie in this subspace), and $(\theta_\infty, x_\infty) \not\in D_\theta \times \{ \bar x\}$. 

On the other hand, each $(\theta_i, x_i)$ is a saddle point of the minimax problem (\ref{eq-minimax-vrt}) for the corresponding function $\psi^o_{h_i}$. Therefore, 
$(0,0) \in \partial \psi_{h_i}(\theta_i, x_i)$, where $\psi_{h_i}(\theta, x) = \psi^o_{h_i}(\theta, x) + \delta_{B_\theta}(\theta) - \delta_{B_x}(x)$; or equivalently,
\begin{equation} \label{eq-gtdsv-prf0}
  - \nabla_\theta \psi^o_{h_i}(\theta_i, x_i) \in \N_{B_\theta}(\theta_i), \qquad \nabla_x \psi^o_{h_i}(\theta_i, x_i) \in \N_{B_x}(x_i).
\end{equation}  
By \cite[Theorem 35.7]{Roc70}, the convergence of $\psi^o_{h_i}$ to $\psi^o$ proved above implies that as $(\theta_i, x_i) \to (\theta_\infty, x_\infty)$,
\begin{equation} \label{eq-gtdsv-prf1}
\nabla \psi^o_{h_i}(\theta_i, x_i) \to \nabla \psi^o(\theta_\infty,x_\infty).
\end{equation}
Then by the outer semicontinuity of the set-value mappings $\N_{B_\theta}(\cdot)$ and $\N_{B_x}(\cdot)$ \cite[Proposition 6.6]{RoW98}, (\ref{eq-gtdsv-prf0}) together with (\ref{eq-gtdsv-prf1}) implies that
$$ - \nabla_\theta \psi^o(\theta_\infty, x_\infty) \in \N_{B_\theta}(\theta_\infty), \qquad \nabla_x \psi^o(\theta_\infty, x_\infty) \in \N_{B_x}(x_\infty)$$ 
or equivalently, $(0,0) \in \partial \psi(\theta_\infty, x_\infty)$ (where $\psi(\theta, x) = \psi^o(\theta, x) + \delta_{B_\theta}(\theta) - \delta_{B_x}(x)$ as we recall). This shows that $(\theta_\infty, x_\infty)$ is a saddle point of the minimax problem (\ref{eq-minimax1}), so, since $x_\infty \in \sp \{\phi(\S)\}$, $(\theta_\infty, x_\infty) \in D_\theta \times \{ \bar x\}$, a contradiction. This proves the second part of the lemma.
\end{proof}

Combining the preceding lemma with the convergence properties of the algorithm mentioned earlier, we obtain the following convergence theorem.

\begin{thm} \label{thm-gtds-bvrt}
For the case of state-dependent $\lambda$, consider the biased single-time-scale GTDa algorithm (\ref{eq-gtdsv-th})-(\ref{eq-gtdsv-x}) under Assumption~\ref{cond-collective}. In the algorithm, let the function $h \in \{ h_K \mid K > 0\}$, a family of bounded Lipschitz continuous functions that satisfy (\ref{eq-h})-(\ref{eq-h2}), and let the initial $x(0), \e_0 \in \sp \{\phi(\S)\}$. Then for each $\epsilon > 0$, there exists $K_\epsilon > 0$ such that if $h=h_K$ with $K \geq K_\epsilon$, the conclusions of Theorem~\ref{thm-gtds}(i)-(iii) hold for the algorithm (with ``for any $\epsilon > 0$'' removed from the statements). 
Moreover, as in Theorem~\ref{thm-gtds}, if the constraint set $B_x$ contains $x_{\rm \text{\rm opt}}$ in its interior, then in the conclusions, the set $D_\theta \times \{\bar x\} =\Theta_{\text{\rm opt}} \times \{ x_{\text{\rm opt}} \}$.
\end{thm}

\subsubsection{Further remarks} \label{sec-gtds-eta}

Finally, we remark on another variation of the GTDa algorithm and show how the analysis given in this subsection applies. In \cite{maei11,gtd09} the unconstrained single-time-scale GTDa algorithm can use an additional parameter $\eta > 0$ to scale the stepsizes $\alpha_n$ for the $x$-iterates:
\begin{align}
   \theta_{n+1} & = \theta_n + \alpha_n \,  \rho_n \big(\phi(S_n) - \gma_{n+1} \phi(S_{n+1}) \big) \cdot \e_n^\tr x_n,  \label{gtd1-th-eta}\\
   x_{n+1} & = x_n + \eta \alpha_n \big(\e_n \delta_n(v_{\theta_n}) - \phi(S_n) \phi(S_n)^\tr x_n\big). \label{gtd1-x-eta}
\end{align} 
The analyses \cite{maei11,gtd09} of this linear update rule equate it to the following recursion via a change of variable $\tilde x = x/\sqrt{\eta}$ or $x =  \sqrt{\eta} \, \tilde x$:
\begin{align}
   \theta_{n+1} & = \theta_n +  \alpha_n \, \sqrt{\eta} \rho_n \big(\phi(S_n) - \gma_{n+1} \phi(S_{n+1}) \big) \cdot \e_n^\tr \tilde x_n, \label{gtd1-th-eta-alt} \\
   \tilde x_{n+1} & = \tilde x_n + \alpha_n \sqrt{\eta} \big(\e_n \delta_n(v_{\theta_n}) - \sqrt{\eta} \phi(S_n) \phi(S_n)^\tr \tilde x_n\big). \label{gtd1-x-eta-alt}
\end{align} 

A constrained version of the above algorithm, for minimizing the regularized objective function $J(\theta) + p(\theta)$, is
\begin{align}
   \theta_{n+1} & = \Pi_{B_\theta} \Big( \theta_n + \alpha_n \,  \rho_n \big(\phi(S_n) - \gma_{n+1} \phi(S_{n+1}) \big) \cdot \e_n^\tr x_n - \alpha_n \nabla p(\theta_n) \Big),  \label{gtd1-th-eta-cnstr}\\
   x_{n+1} & = \Pi_{B_x} \Big( x_n + \eta \alpha_n \big(\e_n \delta_n(v_{\theta_n}) - \phi(S_n) \phi(S_n)^\tr x_n\big) \Big). \label{gtd1-x-eta-cnstr}
\end{align} 
Scaling the coordinates by $1/\sqrt{\eta}$ for the $x$-iterates as before, and with $\tilde B_x = B_x/\sqrt{\eta}$ and with the correspondence $x_n = \sqrt{\eta} \, \tilde x_n$,  we see that (\ref{gtd1-th-eta-cnstr})-(\ref{gtd1-x-eta-cnstr}) are equivalent to 
\begin{align}
    \theta_{n+1} & = \Pi_{B_\theta} \Big(  \theta_n +  \alpha_n \, \sqrt{\eta} \rho_n \big(\phi(S_n) - \gma_{n+1} \phi(S_{n+1}) \big) \cdot \e_n^\tr \tilde x_n - \alpha_n \nabla p(\theta_n) \Big), \label{gtd1-th-eta-cnstr-alt} \\
   \tilde x_{n+1} & = \Pi_{\tilde B_x} \Big(  \tilde x_n + \alpha_n \sqrt{\eta} \big(\e_n \delta_n(v_{\theta_n}) - \sqrt{\eta} \phi(S_n) \phi(S_n)^\tr \tilde x_n\big) \Big). \label{gtd1-x-eta-cnstr-alt}
\end{align} 
Thus one can carry out this constrained algorithm in either form, and to analyze its behavior, it suffices to analyze the behavior of (\ref{gtd1-th-eta-cnstr-alt})-(\ref{gtd1-x-eta-cnstr-alt}). The latter can be understood by recognizing the minimax problem associated with (\ref{gtd1-th-eta-cnstr-alt})-(\ref{gtd1-x-eta-cnstr-alt}). 

This minimax problem is similar to the minimax problem (\ref{eq-minimax1}) discussed earlier and has the same structure. Similar to the derivation of (\ref{eq-minimax0}) and (\ref{eq-minimax1}) in Section~\ref{sec-gtds-ode}, for any $a > 0$, we can express $\tfrac{1}{a} J(\theta)$ as the optimal value of a maximization problem using the identity relation  
$$\tfrac{1}{2 a} \| v \|_\xi^2  = \sup_y \big\{ \langle y, v \rangle_\xi - \tfrac{a}{2} \| y \|^2_\xi \big\}, \qquad \forall \, v \in \re^{|\S|},$$
and then convert the minimization problem $\inf_{\theta \in B_\theta} \{ J(\theta) + p(\theta) \} = \inf_{\theta \in B_\theta} \{ a \cdot \tfrac{1}{a} J(\theta) + p(\theta)\}$ to the equivalent minimax problem
\begin{equation} \label{eq-minimax0-eta}
  \inf_{\theta \in B_\theta} \sup_{\tilde x \in \re^d} \Big\{ a \, \langle \Phi \tilde x, \, \Tl v_\theta - v_\theta \rangle_\xi - \frac{a^2}{2} \| \Phi \tilde x \|^2_\xi  + p(\theta) \Big\}.
\end{equation} 
Adding the constraint set $\tilde B_x$ for the maximizer then gives the minimax problem associated with (\ref{gtd1-th-eta-cnstr-alt})-(\ref{gtd1-x-eta-cnstr-alt}):
\begin{equation} \label{eq-minimax1-eta}
  \inf_{\theta \in B_\theta} \sup_{\tilde x \in \tilde B_x} \Big\{ a \, \langle \Phi \tilde x, \, \Tl v_\theta - v_\theta \rangle_\xi - \frac{a^2}{2} \| \Phi \tilde x \|^2_\xi  + p(\theta) \Big\} \quad \text{where} \  a = \sqrt{\eta}.
\end{equation} 

The two minimax problems above have the same structure as the minimax problems (\ref{eq-minimax0}) and (\ref{eq-minimax1}) discussed in Section~\ref{sec-gtds-ode}, and the reasoning we gave there applies here.
In particular, the mean ODE associated with (\ref{gtd1-th-eta-cnstr-alt})-(\ref{gtd1-x-eta-cnstr-alt}) corresponds to differential inclusion with the maximal monotone mapping associated with the minimax problem (\ref{eq-minimax1-eta}), and all the results given in Sections~\ref{sec-gtds-ode}-\ref{sec-gtds-conv} for the single-time-scale GTDa (where $\eta = 1$) hold for the algorithm (\ref{gtd1-th-eta-cnstr-alt})-(\ref{gtd1-x-eta-cnstr-alt}), with the minimax problems (\ref{eq-minimax0-eta}) and (\ref{eq-minimax1-eta}) in place of (\ref{eq-minimax0}) and (\ref{eq-minimax1}), respectively. (These results then translate easily to the alternative form of the algorithm, (\ref{gtd1-th-eta-cnstr})-(\ref{gtd1-x-eta-cnstr}), via the correspondence $x_n = \sqrt{\eta} \, \tilde x_n$.) Similarly, the results given in Section~\ref{sec-gtds-bias}, which relate the biased variant of GTDa to GTDa in the case of $\eta = 1$, also hold for the biased variant of the algorithm (\ref{gtd1-th-eta-cnstr-alt})-(\ref{gtd1-x-eta-cnstr-alt}) or (\ref{gtd1-th-eta-cnstr})-(\ref{gtd1-x-eta-cnstr}) with $\eta > 0$, except that now the minimax problems involved are the problem (\ref{eq-minimax1-eta}) and its approximations.

Although the line of convergence analysis is the same for different values of $\eta$, the behavior of the algorithm is certainly affected by $\eta$. This is reflected by the different minimax problems corresponding to different $\eta$, as shown above. In addition, we can also see the differences by considering two extreme cases. Suppose we run the algorithm with a small constant stepsize, and let us examine the algorithm of the form (\ref{gtd1-th-eta-cnstr})-(\ref{gtd1-x-eta-cnstr}). If $\eta$ is very large, the algorithm resembles the two-time-scale GTDa algorithm with the $x$-iterates (which correspond to the maximizer) running at the fast time-scale. If $\eta < < 1$, however, then the algorithm again resembles a two-time-scale algorithm but with the $\theta$-iterates (which correspond to the minimizer) running at the fast time-scale.

\subsection{Two-Time-Scale GTDa Revisited} \label{sec-gtda-revisit}

In this subsection we revisit the two-time-scale GTDa algorithm and its biased variant. Recall that in Section~\ref{sec-3} we regarded them as (stochastic) gradient-descent and approximate gradient-descent algorithms for minimizing $J(\theta)$. There, for GTDa, we required that the constraint set $B_x$ satisfies the condition~(\ref{eq-cond-Bx0}): $B_x \supset \{x_\theta \mid \theta \in B_\theta\}$, in order to ensure that the gradients $\nabla J(\theta)$ can be computed by the algorithm for all $\theta \in B_\theta$. This condition is, however, unnecessarily strong, as we already remarked (cf.\ Remark~\ref{rem-Bx} near the end of Section~\ref{sec-gtds-conv}). Using the minimax problem formulation, we can now address the behavior of the two-time-scale GTDa algorithm when $B_x \not\supset \{x_\theta \mid \theta \in B_\theta\}$ and the algorithm cannot compute all the gradients $\nabla J(\theta)$. In this case, as will be shown below, the algorithm is using gradient-ascent to solve the inner maximization problem in the minimax problem (\ref{eq-minimax1}) at the fast time-scale, and using gradient-descent to solve the outer minimization problem in (\ref{eq-minimax1}) at the slow time-scale. Thus we can weaken the condition on $B_x$ and require merely that $B_x$ contains $x_{\rm \text{\rm opt}}$ in its interior, without losing the convergence properties of GTDa derived earlier in Section~\ref{sec-3}. Similarly, for the two-time-scale biased variant of GTDa, the condition on $B_x$ can also be weakened. 

The result is summarized in the theorem below. The objective function we consider here is the regularized objective function $J_p$ as in Section~\ref{sec-gtds}, so the algorithms we refer to in the theorem have, in their $\theta$-iterations, the gradient term $- \alpha_n \nabla p(\theta_n)$, just as their single-time-scale counterparts discussed in Section~\ref{sec-gtds}. 

\begin{thm} \label{thm-gtda-relaxBx}
For the two-time-scale GTDa algorithm and its biased variant, Theorems~\ref{thm-gtd1-wk-dim}-\ref{thm-gtd1-average} and \ref{thm-biasgtd1-wk}-\ref{thm-composite-gtd} still hold if we replace the condition on $B_x$ in these theorems by the weaker condition that $x_{\rm \text{\rm opt}} \in \text{\rm int}(B_x)$, where $x_{\rm \text{\rm opt}}$ is the unique dual optimal solution of the minimax problem (\ref{eq-minimax0}) in $\sp \{ \phi(\S)\}$.
\end{thm}

In the rest of this subsection, we provide the proof details. To prove this theorem, we only need to re-characterize, under the new condition on $B_x$, the limit sets of the mean ODEs associated with the two algorithms. We will show that the fast-time-scale mean ODEs still have solution properties that satisfy the requirements in order to apply the ODE-based stochastic approximation analysis given earlier in Section~\ref{sec-3}. We will then consider the limit sets of the slow-time-scale mean ODEs, and show that in the case of GTDa, the limit set is the desired optimal solution set $\Theta_{\text{opt}}$, and in the case of the biased variant, it is a set contained in a neighborhood of $\Theta_{\text{opt}}$ whose size decreases as the bias introduced by the algorithm decreases. Theorem~\ref{thm-gtda-relaxBx} then follow from these arguments.

Let us start with GTDa. Recall the minimax problem (\ref{eq-minimax1}): $\inf_{\theta \in B_\theta} \sup_{x \in B_x} \psi^o(\theta, x)$, where $B_x$ is a ball centered at the origin as before and with an arbitrary radius. 
For each $\theta$, let
\begin{equation} \label{eq-tJp}
 \tilde J_p(\theta) = \sup_{x \in B_x} \psi^o(\theta, x) = \sup_{x \in B_x} \big\{ \langle \Phi x, \Tl v_\theta - v_\theta \rangle_\xi - \tfrac{1}{2} \| \Phi x \|^2_\xi \big\} + p(\theta).
\end{equation} 
For the maximization problem above, denote by $\tilde x(\theta)$ its unique optimal solution in $\sp \{\phi(\S)\}$. 

\begin{lem} \label{lem-gtda-rev1}
The convex function $\tilde J_p(\theta)$ is differentiable with $\nabla \tilde J_p(\theta) = \nabla_\theta \psi^o(\theta, \tilde x(\theta))$, and $\tilde x(\theta)$ is a continuous function of $\theta$.
\end{lem}

\begin{proof}
For any optimal solution $\hat x$ of the maximization problem (\ref{eq-tJp}), the projection of $\hat x$ on $\sp \{\phi(\S)\}$ coincides with $\tilde x(\theta)$, so $\Phi \hat x = \Phi \tilde x(\theta)$. Consequently, $\nabla_\theta \psi^o(\theta, \hat x) = \nabla_\theta \psi^o(\theta, \tilde x(\theta))$ for all optimal solutions $\hat x$ of (\ref{eq-tJp}). It then follows from \cite[Corollary 10.14(b)]{RoW98} that the function $\tilde J_p(\theta)$ is differentiable at $\theta$ with $\nabla \tilde J_p(\theta) = \nabla_\theta \psi^o(\theta, \tilde x(\theta))$, for all $\theta \in \re^d$.

For the continuity of $\tilde x(\theta)$, note that as $\theta \to \bar \theta$, $\tilde J_p(\theta) = \psi^o(\theta, \tilde x(\theta)) \to \tilde J_p(\bar \theta) =  \psi^o(\bar \theta, \tilde x(\bar \theta))$ by the continuity of the function $\tilde J_p$. Let $x_\infty \in B_x \cap \sp \{ \phi(\S)\}$ be the limit of any convergent subsequence of $\tilde x(\theta)$ as $\theta \to \bar \theta$. By the continuity of the functions $\psi^o$ and $\tilde J_p$, we have $\tilde J_p(\bar \theta) = \psi^o(\bar \theta, x_\infty)$ and hence both $x_\infty$ and $\tilde x(\bar \theta)$ are optimal solutions in $\sp \{ \phi(\S)\}$ for the maximization problem (\ref{eq-tJp}) corresponding to $\bar \theta$. Since the latter has a unique optimal solution in $\sp \{ \phi (\S)\}$, we must have $x_\infty = \tilde x(\bar \theta)$. This proves that $\tilde x(\theta) \to \tilde x(\bar \theta)$ as $\theta \to \bar \theta$.
\end{proof}

We now show that at the fast time-scale of GTDa, the associated mean ODE (\ref{eq-gtd1-odefast}) has solutions $(\theta(\cdot), x(\cdot))$ with $\theta(\cdot) \equiv \theta(0)$ and $x(t) \to \tilde x(\theta(0))$ as $t \to \infty$. This is a counterpart of Lemma~\ref{lem-Bx} when the condition (\ref{eq-cond-Bx0}) on $B_x$ has been removed, and the proof is very similar to that of Lemma~\ref{lem-Bx}.

\begin{lem} \label{lem-Bx2}
For all initial $x(0) \in B_x \cap \sp \{\phi(\S)\}$ and $\theta(0) \in B_\theta$,
the limit set of the ODE (\ref{eq-gtd1-odefast}) is $\big\{ \big(\theta, \tilde x(\theta)\big) \mid \theta \in B_\theta \big\}$.
\end{lem}

\begin{proof}
For each $\theta \in B_\theta$, with $\theta(\cdot) \equiv \theta$, consider a solution $x(\cdot)$ of the ODE (\ref{eq-gtd1-odefast}) with $x(0) \in \sp \{\phi(\S)\}$. Let $V_\theta(x) = \tfrac{1}{2} \| x - \tilde x(\theta) \|_2^2$, and we calculate $\dot V_\theta(x(t))$. Note that by the optimality of $\tilde x(\theta)$, there exists some $\tilde z \in - \N_{B_x}(\tilde x(\theta))$ such that $\nabla_x \psi^o(\theta, \tilde x(\theta)) + \tilde z = 0$. This is the same as $b_\theta - C \tilde x(\theta) + \tilde z = 0$ for $C = \Phi^\tr \Xi \, \Phi$ and $b_\theta = \Phi^\tr \Xi \, (\Tl v_\theta - v_\theta)$. Since $\bar k(\theta, x) = \nabla_x \psi^o(\theta, x) = b_\theta - C x$, we have 
\begin{align*}
  \dot V_\theta(x(t)) & = \langle x(t) - \tilde x(\theta), \, b_\theta - C x(t) + z(t) \rangle \\
  & =  \big\langle x(t) - \tilde x(\theta), \, b_\theta - C x(t) + z(t)  - \big( b_\theta - C \tilde x(\theta) + \tilde z\big) \big\rangle \\
  & = - \langle x(t) - \tilde x(\theta), C(x(t) - \tilde x(\theta)) \rangle + \langle x(t) - \tilde x(\theta), \, z(t) - \tilde z \rangle.
\end{align*}   
Since $x(t) \in \sp \{\phi(\S)\}$, the first term on the r.h.s.\ is bounded above by $- c \, \| x(t) - \tilde x(\theta) \|_2^2$, where $c > 0$ is the smallest nonzero eigenvalue of the symmetric positive semidefinite matrix $C$. The second term on the r.h.s.\ is nonpositive by the monotonicity of the set-valued mapping $\N_{B_x}(\cdot)$~\cite[Corollary 31.5.2]{Roc70}.
Thus, we have $\dot V_\theta(x(t)) \leq - c \| x(t) - \tilde x(\theta) \|_2^2$. Since on $B_x$, $\{V_\theta \mid \theta \in B_\theta\}$ are uniformly bounded, this implies that for any $\epsilon > 0$, there is a time $t_\epsilon$ independent of $\theta$ such that $\| x(t) - \tilde x(\theta) \|_2^2  \leq \epsilon$ for all $t \geq t_\epsilon$. The conclusion of the lemma then follows. 
\end{proof}

Lemma~\ref{lem-Bx2}, together with the continuity of $\tilde x(\theta)$ given in Lemma~\ref{lem-gtda-rev1}, shows that the fast-time-scale mean ODE satisfies the requirements in the convergence proofs given earlier in Sections~\ref{sec-gtd1-cond}-\ref{sec-gtd1conv-constant}. We now consider the mean ODE associated with the slow time-scale. Since at the fast time-scale the $x$-iterates of GTDa ``track'' $\tilde x(\theta)$ instead of $x_\theta$ for the slowly varying $\theta$-iterates, instead of the ODE~(\ref{eq-gtd1-odeslow}), the mean ODE for the slow time-scale is now given by
$$ \dot \theta(t) = \bar g\big(\theta(t), \tilde x(\theta(t))\big) - \nabla p(\theta(t))  + z(t),  \quad z(t) \in - \N_{B_\theta}\big(\theta(t)\big),$$
where $z(t)$ is the boundary reflection term. (The reasoning is the same as that given in Section~\ref{sec-gtd1-odes} when deriving (\ref{eq-gtd1-bg}) and (\ref{eq-gtd1-odeslow}).)
Since $\nabla_\theta \psi^o(\theta, \tilde x(\theta))  = - \bar g(\theta, \tilde x(\theta)) + \nabla p(\theta)$ (cf.\ (\ref{eq-gtds-gk})), by Lemma~\ref{lem-gtda-rev1}, the above ODE is the same as
\begin{equation} \label{eq-gtda-new-odeslow}
 \dot \theta(t) = - \nabla \tilde J_p(\theta(t))  + z(t),  \quad z(t) \in - \N_{B_\theta}\big(\theta(t)\big),
\end{equation}
and it corresponds to using gradient-descent to minimize $\tilde J_p$ on $B_\theta$. So by the same proof of Lemma~\ref{lem-gtd1-odeslow} (with the function $\tilde J_p$ in place of $J$), we see that the limit set of the ODE~(\ref{eq-gtda-new-odeslow}) is the solution set $\argmin_{\theta \in B_\theta} \tilde J_p(\theta)$, which, by Lemma~\ref{lem-saddle-relation}, coincides with the desired set $\Theta_{\text{opt}}$ when $x_{\rm \text{\rm opt}} \in \text{int}(B_x)$. 
The first part of Theorem~\ref{thm-gtda-relaxBx} for GTDa then follows; that is, Theorems~\ref{thm-gtd1-wk-dim}-\ref{thm-gtd1-average} and \ref{thm-composite-gtd} still hold for GTDa under the much weaker condition on $B_x$.

We now consider the two-time-scale biased variant of GTDa and apply the same reasoning to relax the condition (\ref{eq-cond-vrtBx0}) on $B_x$ in Theorem~\ref{thm-biasgtd1-wk} for this algorithm. Recall the minimax problem~(\ref{eq-minimax-vrt}) associated with the biased variant of the single-time-scale GTDa discussed in Section~\ref{sec-gtds-bias}. 
We can view the two-time-scale biased variant algorithm as a two-time-scale algorithm for solving the same minimax problem, but with the maximizer operating at the fast time-scale. In particular, for each $\theta$, recall that the inner maximization problem appearing in (\ref{eq-minimax-vrt}) is
\begin{equation} \label{eq-tJph}
 \tilde J_p^h (\theta) : = \sup_{x \in B_x} \psi^o_h(\theta, x) = \sup_{x \in B_x}  \big\{ x^\tr  ( A_h \theta + b_h) -  \tfrac{1}{2} \, x^\tr C x \big\} + p(\theta),\end{equation}
where $A_h, b_h$ are given by (\ref{eq-Ahbh}) and $C = \Phi^\tr \Xi \, \Phi$.
This maximization problem has a unique optimal solution in $\sp \{\phi(\S)\}$ (because $b_h \in \sp \{\phi(\S)\}$ and the column space of $A_h$ is contained in $\sp \{\phi(\S)\}$; cf.\ Section~\ref{sec-gtds-bias}). Let us denote this optimal solution by $\tilde x_h(\theta)$.

\begin{lem} \label{lem-vrtgtda-rev2}
The convex function $\tilde J_p^h(\theta)$ is differentiable with $\nabla \tilde J_p^h(\theta) = \nabla_\theta \psi^o_h(\theta, \tilde x_h(\theta))$, and $\tilde x_h(\theta)$ is a continuous function of $\theta$.
\end{lem}

The proof of the lemma is the same as that of Lemma~\ref{lem-gtda-rev1}, except that in showing that $\nabla_\theta \psi^o_h(\theta, \hat x)$ is the same for all optimal solutions $\hat x$ of (\ref{eq-tJph}), we use the fact that the column space of $A_h$ is contained in $\sp \{\phi(\S)\}$ (cf.\ the discussion immediately after (\ref{eq-ghkh})).

Consider now the mean ODE (\ref{eq-gtdvrt-odefast}) associated with the fast time-scale of the biased variant of GTDa. The next lemma characterizes the limit set of this ODE, and it is a counterpart of Lemma~\ref{lem-vrtBx} when the condition (\ref{eq-cond-vrtBx0}) on $B_x$ has been removed. The proof is the same as that of Lemma~\ref{lem-Bx2} (with $\tilde x_h(\theta)$ in place of $\tilde x(\theta)$, $A_h \theta + b_h$ in place of $b_\theta$, and $\bar k_h(\theta, x) = \nabla_x \psi^o_h(\theta, x)$ in place of $\bar k(\theta, x) = \nabla_x \psi^o(\theta, x)$).

\begin{lem} \label{lem-vrtBx2}
For all initial $x(0) \in B_x \cap \sp \{\phi(\S)\}$ and $\theta(0) \in B_\theta$,
the limit set of the ODE (\ref{eq-gtdvrt-odefast}) is
$\big\{ \big(\theta, \tilde x_h(\theta)\big) \mid \theta \in B_\theta \big\}.$
\end{lem}

Finally, we consider the mean ODE associated with the slow time-scale of the biased variant of GTDa. In view of Lemma~\ref{lem-vrtBx2}, instead of the ODE (\ref{eq-gtdvrt-odeslow}), this mean ODE is now given by
$$ \dot \theta(t) = \bar g_h\big(\theta(t), \tilde x_h(\theta(t))\big) - \nabla p(\theta(t))  + z(t),  \quad z(t) \in - \N_{B_\theta}\big(\theta(t)\big),$$
where the function $\bar g_h(\theta, x)$ is as defined in (\ref{eq-gtdsvrt-gk}), and $z(t)$ is the boundary reflection term.
(The reasoning is the same as that given in Section~\ref{sec-vrt-ode} when deriving (\ref{eq-gtdvrt-odeslow}).)
Since $\nabla_\theta \psi^o_h(\theta, \tilde x_h(\theta))  = - \bar g_h(\theta, \tilde x_h(\theta)) + \nabla p(\theta)$ (cf.\ (\ref{eq-ghkh}) and (\ref{eq-psih-vrt})), by Lemma~\ref{lem-vrtgtda-rev2}, the above ODE is the same as
\begin{equation} \label{eq-gtdavrt-new-odeslow}
 \dot \theta(t) = - \nabla \tilde J_p^h(\theta(t))  + z(t),  \quad z(t) \in - \N_{B_\theta}\big(\theta(t)\big).
\end{equation}
It corresponds to using gradient-descent to minimize $\tilde J_p^h$ on $B_\theta$. Hence the limit set of this ODE is just the set of optimal solutions of the minimization problem, $D^h_\theta = \argmin_{\theta \in B_\theta} \tilde J_p^h(\theta)$ (this follows from the same proof of Lemma~\ref{lem-gtd1-odeslow} with the function $\tilde J_p^h$ in place of $J$). 
We now use the approximation property given in Lemma~\ref{lem-bias-apprsadd-lim} to relate the set $D^h_\theta$ to the set $D_\theta = \argmin_{\theta \in B_\theta}  \{ \sup_x \psi^o(\theta, x)\}$, like what we did in Section~\ref{sec-gtds-bias}.  Lemma~\ref{lem-bias-apprsadd-lim} shows that for $h \in \{ h_K \mid K > 0\}$ (the family of functions defined earlier for the biased variant algorithm), given any $\epsilon > 0$, we have $D^{h_K}_\theta \subset N_\epsilon (D_\theta)$ for sufficiently large $K$ (equivalently, when the bias introduced by the algorithm is sufficiently small). Therefore, when $x_{\rm \text{\rm opt}} \in \text{int}(B_x)$ so that $D_\theta = \Theta_{\text{opt}}$ (Lemma~\ref{lem-saddle-relation}), $D^{h_K}_\theta \subset N_\epsilon (D_\theta)$ for all $K$ sufficiently large. Consequently, Theorem~\ref{thm-biasgtd1-wk} remains valid for the two-time-scale biased variant of GTDa if we replace its condition on $B_x$ by the weaker condition $x_{\rm \text{\rm opt}} \in \text{int}(B_x)$. This establishes the second part of Theorem~\ref{thm-gtda-relaxBx}.

\section{Convergence Analyses III: Almost Sure Convergence} \label{sec-5}

In this section we analyze the asymptotic behavior of the gradient-based TD algorithms with respect to a stronger mode of convergence: almost sure convergence.
We consider primarily those constrained algorithms discussed in the previous sections, for diminishing stepsizes that satisfy the square-summable condition $\sum_{n \geq 0} \alpha^2_n < \infty$ and $\sum_{n \geq 0} \beta^2_n < \infty$, in addition to Assumption~\ref{cond-large-stepsize}.
We analyze two-time-scale algorithms in Section~\ref{sec-as-twotime} and single-time-scale algorithms in Section~\ref{sec-as-singletime}.
The convergence results we obtained are more limited in the case of state-dependent $\lambda$, and more satisfactory in the case of history-dependent $\lambda$ where the trace iterates $\{\e_n\}$ are bounded by the choice of $\lambda$. 
Finally, in Section~\ref{sec-as-singletime-unconstr}, we give a result on the unconstrained single-time-scale GTDa algorithm for the case of history-dependent $\lambda$. 

As discussed in Section~\ref{sec-gtds-eta}, one can analyze single-time-scale GTDa with a scaling parameter $\eta > 0$ for the $x$-iterates by using the same arguments for the algorithm without such scaling. So for notational simplicity, we will focus on the latter case where $\eta = 1$ and give detailed arguments for this case only. The implications for the case with scaling are as explained at the end of Section~\ref{sec-gtds-eta} for constrained GTDa and in Remark~\ref{rem-ugtds-eta}, to be given in Section~\ref{sec-as-singletime-unconstr}, for unconstrained GTDa.

The results of this section are listed here for quick access:\vspace*{-3pt}
\begin{itemize}
\item two-time-scale GTD and MD-GTD with history-dependent $\lambda$: Theorems~\ref{thm-gtd-as-conv} and \ref{thm-mdgtd-as-conv};
\item single-time-scale GTDa with history-dependent $\lambda$: Theorem~\ref{thm-gtds-as-conv};
\item single-time-scale GTDa and its biased variant with state-dependent $\lambda$: Theorem~\ref{thm-gtds-as-conv2};
\item single-time-scale unconstrained GTDa with history-dependent $\lambda$: Theorem~\ref{thm-ugtds-as-conv}.
\end{itemize}

\subsection{Two-Time-Scale GTD and MD-GTD} \label{sec-as-twotime}

We consider the regularized objective function $J_p(\theta) = J(\theta) + p(\theta)$ as in Section~\ref{sec-4}, where $p(\cdot)$ is a convex differentiable function on $\re^d$. Let 
$\Theta_{\text{\rm opt}} = \textstyle{\argmin_{\theta \in B_\theta}} J_p(\theta)$.
Recall that the GTD algorithms for minimizing $J_p$ are the same as given before except that there is an extra gradient term $-\alpha_n \nabla p(\theta_n)$ in the $\theta$-iteration (cf.\ Section~\ref{sec-4} or Section~\ref{subsec-asconv-gtd} below).

For two-time-scale GTD/MD-GTD algorithms, we only have almost sure convergence results for the case of history-dependent $\lambda$ (including the use of the composite scheme), where the trace iterates $\{\e_n\}$ are kept bounded by the choice of the $\lambda$-parameters (cf.~Condition~\ref{cond-lambda}(ii)). This, of course, includes the special case of state-dependent $\lambda$ in which one bounds the traces by using sufficiently small $\lambda$.

Before stating our convergence theorems, we collect the conditions of the theorems in the two lists below, for the GTD and MD-GTD algorithms, respectively. For GTDa and GTDb, the conditions are the same except for the condition on the constraint set $B_x$, which can be made weaker for GTDa (cf.\ Section~\ref{sec-gtda-revisit}). The weaker condition, as we recall, involves the point $x_{\text{\rm opt}}$, which is the unique dual optimal solution of the minimax problem (\ref{eq-minimax0}) in $\sp \{ \phi(\S)\}$.

\begin{assumption} \label{as-cond-gtd}
For the two-time-scale GTD algorithms:\vspace*{-3pt}
\begin{itemize}
\item[\rm (i)] Condition~\ref{cond-pol} holds.
\item[\rm (ii)] The $\lambda$-parameters are set according to either (\ref{eq-histlambda}) or (\ref{eq-cp-histlambda}), where the memory states satisfy Condition~\ref{cond-mem}, and the function $\lambda(\cdot)$ in (\ref{eq-histlambda}) or each function $\lambda^{(i)}(\cdot), i \leq \ell$ in (\ref{eq-cp-histlambda}) satisfies Condition~\ref{cond-lambda}.
\item[\rm (iii)] The stepsizes $\{\alpha_n\}$ and $\{\beta_n\}$ satisfy Assumption~\ref{cond-large-stepsize} and the square-summable condition, $\sum_{n \geq 0} \alpha^2_n < \infty$, $\sum_{n \geq 0} \beta^2_n < \infty$.
\item[\rm (iv)] In the case of GTDa, the constraint set $B_x$ contains $x_{\text{\rm opt}}$ in its interior; and in the case of GTDb, the constraint set $B_x$ satisfies the condition (\ref{eq-cond-Bx0}).
\item[\rm (v)] The initial $x_0, \e_0 \in \sp  \{\phi(\S)\}$.
\end{itemize}
\end{assumption}

%\smallskip
\begin{thm} \label{thm-gtd-as-conv}
Let $\{(\theta_n, x_n)\}$ be generated by any of the two-time-scale GTD algorithms discussed in Sections~\ref{sec-gtd1},~\ref{sec-gtd2} and~\ref{sec-cp-extension}.  
Then under Assumption~\ref{as-cond-gtd}, for each given initial condition, $\{\theta_n\}$ converges almost surely to the set $\Theta_{\text{\rm opt}}$.
\end{thm}

For the two-time-scale MD-GTD algorithms, we impose the same conditions on the constraint sets and on the minimization problem $\inf_\theta J_p(\theta)$ as in Section~\ref{sec-mdgtd}:

\begin{assumption} \label{as-cond-mdgtd}
For the two-time-scale MD-GTD algorithms:\vspace*{-3pt}
\begin{itemize}
\item[\rm (i)] The conditions in Assumption~\ref{as-cond-gtd}(i)-(iii) and (v) hold.
\item[\rm (ii)] The constraint sets $D_{\theta^*}$ and $B_x$ satisfy Assumption~\ref{cond-mdgtd}.
\item[\rm (iii)] The minimization problem $\inf_\theta J_p(\theta)$ has a unique optimal solution $\theta_{\text{\rm opt}}$.\vspace*{5pt} 
\end{itemize}
\end{assumption}

%\smallskip
\begin{thm} \label{thm-mdgtd-as-conv}
Let $\{(\theta_n, x_n)\}$ be generated by any of the two-time-scale MD-GTD algorithms discussed in Section~\ref{sec-mdgtd}.
Then under Assumption~\ref{as-cond-mdgtd}, for each given initial condition, $\{\theta_n\}$ converges almost surely to $\theta_{\text{\rm opt}}$.
\end{thm}
%\smallskip

In the rest of this subsection, we prove the two theorems above. 
The proof arguments for MD-GTD are essentially the same as those for GTD; we will explain this after we give the proof details for GTD.

\subsubsection{Convergence proof for GTD} \label{subsec-asconv-gtd}

In this subsection we prove Theorem~\ref{thm-gtd-as-conv} for the GTD algorithms. As before, for notational simplicity, we will use GTDa to show how we carry out the analysis. Because the condition on the constraint set $B_x$ is stronger for GTDb and weaker for GTDa, and because the proof for GTDb is essentially the same as that for GTDa under the stronger condition on $B_x$, we will analyze GTDa under both conditions.

Recall the two-time-scale GTDa algorithm for minimizing $J_p(\theta)$:
\begin{align*}
   \theta_{n+1} & = \Pi_{B_\theta} \Big( \theta_n + \alpha_n \,  \rho_n \big(\phi(S_n) - \gma_{n+1} \phi(S_{n+1}) \big) \cdot \e_n^\tr x_n - \alpha_n \nabla p(\theta_n) \Big),  \\   
   x_{n+1} & = \Pi_{B_x} \Big( x_n + \beta_n \big(\e_n \delta_n(v_{\theta_n}) - \phi(S_n) \phi(S_n)^\tr x_n\big) \Big). 
\end{align*} 
In terms of the functions $g, k$ defined in~(\ref{eq-gtd1-gk}), we can write it as
\begin{align*}
     \theta_{n+1} & = \Pi_{B_\theta} \big( \theta_n + \alpha_n \, g(\theta_n, x_n, Z_n) - \alpha_n \nabla p(\theta_n)\big), \\
      x_{n+1} & = \Pi_{B_x} \big( x_n + \beta_n (k(\theta_n, x_n, Z_n) + \e_n \omega_{n+1} ) \big), 
\end{align*}
where $\omega_{n+1} = \rho_n \big(R_{n+1} - r(S_n, S_{n+1}) \big)$.
Similar to the proof given earlier, we first consider the fast time-scale determined by the stepsize $\{\beta_n\}$, and write the algorithm in the form of a single-time-scale algorithm as 
\begin{align} 
   \left( \! \begin{array}{c}  \theta_{n+1} \\ x_{n+1} \end{array} \!\right) = \Pi_{B_\theta \times B_x}
  \! \left( \! \begin{array}{l} \theta_n + \beta_n \, (\alpha_n / \beta_n) \cdot \big( g(\theta_n, x_n, Z_n) - \nabla p(\theta_n) \big) \\
   x_n + \beta_n \big( k(\theta_n, x_n, Z_n) + \e_n \omega_{n+1} \big) \end{array} \! \right).  \label{eq-gtd1-asf}
\end{align}
We want to show that its average dynamics is characterized by the mean ODE:
\begin{equation} \label{eq-gtd1as-odefast}
\qquad \left( \! \begin{array}{c} \dot{\theta}(t) \\ \dot x(t) \end{array} \!\right) 
 =  \left( \! \begin{array}{l} 0 \\
 \bar k \big(\theta(t), x(t) \big) + z(t)  \end{array} \!\right), \qquad \theta(0) \in B_\theta, \  z(t) \in - \N_{B_x}(x(t)),
\end{equation} 
where the function $\bar k(\cdot)$ is given by (\ref{eq-gtd1-bk}) as before, $z(t)$ is the boundary reflection term. 

We then consider the slow time-scale. For each $\theta$, let us define $\bar x(\theta)$ and $\bar g(\theta)$, according to the condition placed on the constraint set $B_x$:\vspace*{-0.1cm} 
\begin{itemize}
\item Under the stronger condition (\ref{eq-cond-Bx0}) on $B_x$, let $\bar x(\theta) = x_\theta$ and $\bar g(\theta) = - \nabla J(\theta)$ as in Section~\ref{sec-gtd1-odes} (cf.\ (\ref{eq-xtheta3})-(\ref{eq-gtd1-bg})).
\item Under the weaker condition $x_{\text{opt}} \in \text{int} (B_x)$, let $\bar x(\theta) = \tilde x(\theta)$ and $\bar g(\theta) = \bar g(\theta, \tilde x(\theta))$ as in Section~\ref{sec-gtda-revisit}.
\end{itemize}
With these definitions, we write the $\theta$-iterates at the slow time-scale equivalently as
\begin{align} 
 \theta_{n+1} & = \Pi_{B_\theta} \big( \theta_n + \alpha_n (g(\theta_n, \bar x (\theta_n), Z_n) - \nabla p(\theta_n) + \Delta_n ) \big),   \label{eq-gtd1-ass}\\
\text{where} \qquad \Delta_n & = g(\theta_n, x_n, Z_n) - g(\theta_n, \bar x (\theta_n), Z_n).  \label{eq-gtd1-ass-noise}
\end{align}
Treating $\Delta_n$ as noise terms, we want to show that the average dynamics of (\ref{eq-gtd1-ass}) is characterized by the mean ODE:
\begin{equation}  \label{eq-gtd1as-odeslow}
  \dot \theta(t) = \bar g (\theta(t)) - \nabla p ( \theta(t))+ z(t),  \quad z(t) \in - \N_{B_\theta}\big(\theta(t)\big),
\end{equation}
where $z(t)$ is the boundary reflection term.

The ODEs (\ref{eq-gtd1as-odefast}) and (\ref{eq-gtd1-ass}) above are the same mean ODEs for GTDa as discussed before in Sections~\ref{sec-gtd1-odes} and \ref{sec-gtda-revisit}. However, because the notion of convergence here is different and stronger, to show that the average dynamics of the algorithm is characterized by these mean ODEs, we need to verify a different set of conditions, which are more demanding (the reason that we required the traces to be bounded).

\bigskip
\noindent {\bf Conditions to verify}
\smallskip

We will apply \cite[Theorem 6.1.1]{KuY03} in our convergence proof. This theorem concerns single-time-scale stochastic approximation algorithms with ``exogenous noises.'' We will apply it twice, first to the fast-time-scale iterates (\ref{eq-gtd1-asf}) and then to the slow-time-scale iterates (\ref{eq-gtd1-ass}).
In order to apply \cite[Theorem 6.1.1]{KuY03}, we need to verify the conditions listed below. Let us first explain some notations in the conditions. The ODE-based analysis works with continuous-time processes that are piecewise linear or constant interpolations of the discrete-time iterates. There is a correspondence between the continuous-time and the discrete-time, which is defined by the stepsizes used in the iterations.
Imagine placing consecutive iterations at moments on the continuous timeline that are separated by the stepsizes used in these iterations.
With respect to the stepsizes $\{\beta_n\}$, define, for each $t \geq 0$,
$$m_\beta(t) : = \min \big\{ n \mid \textstyle{\sum_{i=0}^n} \beta_i > t \big\}, \qquad  m_\beta^-(t) : = \max \{ 0, m_\beta(t) - 1\}.$$
Thus $m_\beta(t)$ points to the first iteration occurred after time $t$, and $m_\beta^-(t)$ the latest iteration before time $t$ has elapsed. 
Abusing notation, we also denote
$$m_\beta(n,t) = \min \big\{ j \geq n \mid \textstyle{\ \sum_{i=n}^j} \beta_i > t \big\}, \qquad m_\beta^-(n,t) = \max \{ n, m_\beta(n,t) - 1\}.$$
These are similar to the ``iteration pointers'' above, but work with a continuous timeline that has the $n$-th iteration placed at time $0$.
The above pointers with the subscript $\beta$ will appear in the conditions for the fast-time-scale iterates. 
Likewise, with respect to the stepsizes $\{\alpha_n\}$, we define ``iteration pointers'' $m_\alpha(t), m_\alpha^-(t)$, etc.\
These pointers will appear in the conditions for the slow-time-scale iterates. 
Note that for any $\tau \geq 0$, we have $\sum_{i=m_\beta(\tau)}^{m^-_\beta(\tau+t)} \beta_i \leq t$ and $\sum_{i=m_\alpha(\tau)}^{m^-_\alpha(\tau + t)} \alpha_i \leq t$.

The following conditions correspond to the conditions (A.6.1.1) and (A.6.1.3)-(A.6.1.7) in \cite[Chap.~6.1]{KuY03}.%footnote starts
\footnote{In our case, its condition (A.6.1.2) simply requires the function $k(\theta, x, z)$ to be continuous in $(\theta, x)$ for each $z$, and the function $g(\theta, \bar x (\theta), z)$ to be continuous in $\theta$ for each $z$, which are clearly satisfied.}
%footnote ends

\medskip
\noindent {\bf Conditions for the fast time-scale:}\vspace*{-3pt}
\begin{itemize}
\item[(i)] $\sup_{n \geq 0} \E \big[ \big\| k(\theta_n, x_n, Z_n) + \e_n \omega_{n+1}  \big\| \big] < \infty$, 
$\sup_{n \geq 0} \E  \big[ (\alpha_n/\beta_n) \cdot \| g(\theta_n, x_n, Z_n) - \nabla p(\theta_n) \| \big] < \infty$.
\item[(ii)] For each $(\theta, x) \in B_\theta \times B_x$, $\eta > 0$ and $T > 0$,
\begin{equation} \label{eq-ascondf2a}
 \lim_{n \to \infty} \Pr \left\{ \sup_{j \geq n} \max_{0 \leq t \leq T} \Big\|  \sum_{i=m_\beta(jT)}^{m^-_\beta(jT+t)} \beta_i \big( k(\theta, x, Z_i) - \bar k(\theta, x) \big) \Big\| \geq \eta \right\} = 0,
\end{equation} 
\begin{equation} \label{eq-ascondf2b}
 \lim_{n \to \infty} \Pr \left\{ \sup_{j \geq n} \max_{0 \leq t \leq T} \Big\|  \sum_{i=m_\beta(jT)}^{m^-_\beta(jT+t)} \alpha_i  \big( g(\theta, x, Z_i) - \nabla p(\theta) \big)  \Big\| \geq \eta \right\} = 0.
\end{equation} 
\item[(iii)] For each $\eta > 0$ and $T > 0$,
\begin{equation}
    \lim_{n \to \infty} \Pr \left\{ \sup_{j \geq n} \max_{0 \leq t \leq T} \Big\|  \sum_{i=m_\beta(jT)}^{m^-_\beta(jT+t)} \beta_i  \e_i \omega_{i+1} \Big\| \geq \eta \right\} = 0. \notag
\end{equation}
\item[(iv)] There are nonnegative measurable functions $f_1(\theta, x)$, $f_2(z)$ such that $\|k(\theta, x, z)\| \leq f_1(\theta,x) f_2(z)$ for all $(\theta, x) \in B_\theta \times B_x$, where
$f_1$ is bounded on $B_\theta \times B_x$. For each $\eta > 0$,
\begin{equation}
 \lim_{t \to 0} \lim_{n \to \infty} \Pr \left\{ \sup_{j \geq n}   \sum_{i=m_\beta(jt)}^{m^-_\beta(jt+t)} \beta_i f_2( Z_i) \geq \eta \right\} = 0. \notag
\end{equation} 
\item[(v)] There are nonnegative measurable functions $f_3(\theta, x)$, $f_4(z)$ such that 
$$\| k(\theta, x, z) - k(\theta', x', z) \| \leq f_3(\theta - \theta', x - x') f_4(z), \quad \forall \, (\theta, x), (\theta', x') \in B_\theta \times B_x,$$ 
where $f_3$ is bounded on $B_\theta \times B_x$ and $f_3(\theta, x) \to 0$ as $(\theta, x) \to 0$, and 
\begin{equation}
   \Pr \left\{ \limsup_{n \to \infty}   \sum_{i=n}^{m_\beta(n,t)} \beta_i f_4( Z_i) < \infty \right\} = 1, \quad \text{for some} \ t > 0. \notag
\end{equation}
\end{itemize}

\smallskip
\noindent {\bf Conditions for the slow time-scale:}\vspace*{-3pt}
\begin{itemize}
\item[(i)] $\sup_{n \geq 0} \E  \big[ \| g(\theta_n, x_n, Z_n) - \nabla p(\theta_n) \| \big] < \infty$.
\item[(ii)] For each $\theta \in B_\theta$, $\eta > 0$ and $T > 0$,
\begin{equation} \label{eq-asconds2}
 \lim_{n \to \infty} \Pr \left\{ \sup_{j \geq n} \max_{0 \leq t \leq T} \Big\|  \sum_{i=m_\alpha(jT)}^{m^-_\alpha(jT+t)} \alpha_i \big( g(\theta, \bar x(\theta), Z_i) - \bar g(\theta) \big) \Big\| \geq \eta \right\} = 0.
\end{equation} 
\item[(iii)] For each $\eta > 0$ and $T > 0$,
\begin{equation} \label{eq-asconds3}
 \lim_{n \to \infty} \Pr \left\{ \sup_{j \geq n} \max_{0 \leq t \leq T} \Big\|  \sum_{i=m_\alpha(jT)}^{m^-_\alpha(jT+t)} \alpha_i \big( g(\theta_i, x_i, Z_i)  - g(\theta_i, \bar x(\theta_i), Z_i)  \big) \Big\| \geq \eta \right\} = 0.
\end{equation}
\item[(iv)] There are nonnegative measurable functions $f_1(\theta)$, $f_2(z)$ such that 
$$\|g(\theta, \bar x(\theta), z) - \nabla p(\theta) \| \leq f_1(\theta) f_2(z), \quad \forall \, \theta \in B_\theta,$$ 
where $f_1$ is bounded on $B_\theta$. For each $\eta > 0$ and $T > 0$,
\begin{equation}
 \lim_{t \to 0} \lim_{n \to \infty} \Pr \left\{ \sup_{j \geq n}   \sum_{i=m_\alpha(jt)}^{m^-_\alpha(jt+t)} \alpha_i f_2( Z_i) \geq \eta \right\} = 0. \notag
\end{equation} 
\item[(v)] There are nonnegative measurable functions $f_3(\theta)$, $f_4(z)$ such that 
$$\big\| g(\theta, \bar x(\theta), z) - \nabla p(\theta) - \big(g(\theta', \bar x(\theta'), z) - \nabla p(\theta') \big) \big\| \leq f_3(\theta - \theta') f_4(z), \quad \forall \, \theta, \theta' \in B_\theta,$$ 
where $f_3$ is bounded on $B_\theta$ and $f_3(\theta) \to 0$ as $\theta \to 0$, and 
\begin{equation}
   \Pr \left\{ \limsup_{n \to \infty}   \sum_{i=n}^{m_\alpha(n, t)} \alpha_i f_4( Z_i) < \infty \right\} = 1, \quad \text{for some} \ t > 0. \notag
\end{equation}
\end{itemize}

Most of the above conditions can be quickly verified. 
Because we restrict attention to the case of history-dependent $\lambda$, by our choice of the $\lambda$'s (cf.~Condition~\ref{cond-lambda}(ii)), $\{\e_n\}$ lies in a pre-determined bounded set. Since the iterates $(\theta_n, x_n)$ are also confined in bounded sets, the conditions~(i) for both time-scales hold obviously. 
Consider now the conditions (iv)-(v) for the fast time-scale.
By the definition of $ k(\cdot)$ (cf.\ (\ref{eq-gtd1-gk})), we have for some constant $c > 0$,
\begin{align} 
 \| k(\theta, x, z) \| & \leq  c \, (\| \theta \| +  \| x\| +  1)  \cdot (  \| \e \|  + 1)  \label{eq-asprf-k1}, \\
\| k(\theta, x, z) - k(\theta', x', z) \| & \leq c \, (  \| \theta - \theta'\| +\| x - x'\|  ) \cdot (\| \e\| + 1), \label{eq-asprf-k2} 
\end{align}
so in these conditions, we can let $f_1(\theta, x) = \| \theta \| +  \| x\| +  1$, $f_3(\theta, x) = \| \theta \| +  \| x\|$, and 
$f_2(z) = f_4(z) =c \, ( \| \e\| + 1)$. 
Since the traces $\{\e_n\}$ lie in a pre-determined bounded set, $f_2(Z_i)$ and $f_4(Z_i)$ are also bounded for all $i$. 
Then, in view of the fact that $\sum_{i=m_\beta(jt)}^{m^-_\beta(jt+t)} \beta_i \leq t$ and $\sum_{i=n}^{m_\beta(n,t)} \beta_i \leq t$, it is trivially true that the conditions (iv)-(v) for the fast time-scale are satisfied. 

The conditions (iv)-(v) for the slow time-scale can be similarly checked. 
Using the definition of $g(\cdot)$ (cf.\ (\ref{eq-gtd1-gk})), we have that for some constant $c > 0$ and for all $\theta \in B_\theta$,
\begin{equation} \label{eq-asprf-g1}
  \|g(\theta, \bar x(\theta), z) - \nabla p(\theta) \| \leq c \, (\| \bar x(\theta) \| + \| \nabla p(\theta) \| + 1) \cdot (\| \e \| + 1).
\end{equation}  
Correspondingly, in the condition (iv) for the slow time-scale, we let $f_1(\theta)=\| \bar x(\theta) \| + \| \nabla p(\theta) \| + 1$ and $f_2(z) = c \, (\| \e\| + 1)$. 
Since $\{\e_n\}$ lie in a pre-determined bounded set and $\sum_{i=m_\alpha(jt)}^{m^-_\alpha(jt+t)} \alpha_i \leq t$ , we see that the condition (iv) holds trivially.

Consider now the condition (v) for the slow time-scale. 
Note that $\nabla p(\theta)$ is uniformly continuous on the compact set $B_\theta$ (since $\nabla p(\cdot)$ is continuous on $\re^d$).
This implies that the function $u_p: \re_+ \to \re_+$,  
\begin{equation}  \label{eq-asprf-g2a}
    u_p(\epsilon) : = \sup_{\theta, \theta' \in B_\theta, \, \| \theta - \theta'\| \leq \epsilon} \| \nabla p(\theta) - \nabla p(\theta')\|
\end{equation}
satisfies that $u_p(\epsilon) \to 0$ as $\epsilon \to 0$. Similarly, since $\bar x(\theta)$ is continuous in $\theta$,%footnote starts
\footnote{In the case $\bar x(\theta) = \tilde x(\theta)$, the continuity of $\bar x(\cdot)$ is shown by Lemma~\ref{lem-gtda-rev1}. In the case $\bar x(\theta) = x_\theta$, one can use a simpler argument to show that $\bar x(\cdot)$ is Lipschitz continuous. In this case, $\bar x(\theta) - \bar x(\theta')$ is the unique solution to the linear system of equations (in $x$):
$x \in \sp \{\phi(\S)\},  \Phi^\tr \Xi \, \Phi x = \Phi^\tr \Xi \, (\Pl - I) \Phi (\theta - \theta') = 0$.
Therefore, $\| \bar x(\theta) - \bar x(\theta') \|_2 \leq  \| \Phi^\tr \Xi \, (\Pl - I) \Phi (\theta - \theta') \|_2 /c $, where $c$ is the smallest nonzero eigenvalue of the matrix $\Phi^\tr \Xi \, \Phi$. From this inequality it follows that $\bar x(\cdot)$ is Lipschitz continuous.}
%footnote ends
it is uniformly continuous on $B_x$ and hence the function $u_{\bar x}: \re_+ \to \re_+$,
\begin{equation}  \label{eq-asprf-g2a}
    u_{\bar x}(\epsilon) : = \sup_{\theta, \theta' \in B_\theta, \, \| \theta - \theta'\| \leq \epsilon} \| \bar x(\theta) - \bar x(\theta')\|
\end{equation} 
satisfies that $u_{\bar x}(\epsilon) \to 0$ as $\epsilon \to 0$. We then have that for some constant $c > 0$ and for all $\theta, \theta' \in B_\theta$,
\begin{align} 
\big\| g(\theta, \bar x(\theta), z) - \nabla p(\theta) - \big(g(\theta', \bar x(\theta'), z) - \nabla p(\theta') \big) \big\| & \leq \big\| g(\theta, \bar x(\theta), z) - g(\theta', \bar x(\theta'), z) \big\| \notag \\
& \quad \, + \big\|  \nabla p(\theta) -  \nabla p(\theta')  \big\| \notag \\
& \leq c  \| \bar x(\theta) - \bar x(\theta') \| \cdot \| \e \| + u_p( \| \theta - \theta'\|)  \label{eq-asprf-g2}  \\
& \leq c \big(u_{\bar x}(\| \theta - \theta' \|)  + u_p( \| \theta - \theta'\|) \big) \cdot (\| \e \| + 1).   \notag %\label{eq-asprf-g3}
\end{align} 
Now, in the condition~(v) for the slow time-scale, let $f_3(\theta) = u_{\bar x}(\| \theta \|) + u_p(\| \theta\|)$ and $f_4(z) =c  (\| \e\| + 1)$.
We see that as required, $f_3(\theta) \to 0$ as $\theta \to 0$, and moreover, since $\{\e_n\}$ lie in a pre-determined bounded set and $\sum_{i=n}^{m_\alpha(n,t)} \alpha_i \leq t$, the conditions~(v) also holds trivially.

That the algorithm satisfies the condition (iii) for the fast time-scale is a consequence of the fact that $\{\e_i \omega_{i+1}\}$ are martingale differences with bounded variances and $\{\beta_i\}$ are square-summable (see \cite[Example 1, Chap.\ 2.2, p.\ 30]{KuC78}).
Consider now (\ref{eq-ascondf2b}) in the condition (ii) for the fast time-scale. In view of the boundedness of $\{\e_n\}$, $\| g(\theta, x, Z_i) - \nabla p(\theta) \| \leq c$ for some constant $c$, so for any $j \geq n$,
$$ \max_{0 \leq t \leq T} \Big\|  \sum_{i=m_\beta(jT)}^{m^-_\beta(jT+t)} \alpha_i \big( g(\theta, x, Z_i) - \nabla p(\theta)\big) \Big\| \leq c \sum_{i=m_\beta(jT)}^{m^-_\beta(jT+T)} \alpha_i \leq c \sum_{i=m_\beta(jT)}^{m^-_\beta(jT+T)} \beta_i \cdot \frac{\alpha_i}{\beta_i} \leq c \, T \cdot \sup_{i \geq m_\beta(nT)} \frac{\alpha_i}{\beta_i},$$
which converges to $0$ as $n \to \infty$ (since $\alpha_i/\beta_i \to 0$ as $i \to \infty$). Therefore, (\ref{eq-ascondf2b}) is satisfied.

We now prove that the algorithm satisfies (\ref{eq-ascondf2a}) and (\ref{eq-asconds2}) in the conditions (ii) as well. This proof is longer and uses a result from \cite{Bor08} on the convergence of certain random probability measures. Let us first explain this result.

\bigskip
\noindent {\bf A result for verifying the ``averaging conditions'' (\ref{eq-ascondf2a}) and (\ref{eq-asconds2}):}
\smallskip

Focus, for now, on the case of history-dependent $\lambda$ under Assumption~\ref{cond-collective} (the composite scheme of setting $\lambda$ will be treated later using the results we are going to derive first). Define random variables $\mu^{\beta,t}_{\tau},  \mu^{\alpha,t}_{\tau}, \tau \geq 0, t > 0$, that take values in the space of probability measures on $\Z$ as follows: for Borel subsets $D \subset \Z$,
$$  \mu^{\beta,t}_{\tau}(D) = \frac{1}{\sum_{i=m_\beta(\tau)}^{m^-_\beta(\tau+t)} \beta_i } \sum_{i=m_\beta(\tau)}^{m^-_\beta(\tau+t)} \beta_i \, \I(Z_i \in D), \qquad
\mu^{\alpha,t}_{\tau}(D) = \frac{1}{\sum_{i=m_\alpha(\tau)}^{m^-_\alpha(\tau+t)} \alpha_i} \sum_{i=m_\alpha(\tau)}^{m^-_\alpha(\tau+t)} \alpha_i \, \I(Z_i \in D).$$
Recall that the Markov chain $\{Z_i\}$ has a unique invariant probability measure (which is just the probability distribution of $Z_0$ in the stationary state-trace process).
The following lemma concerns the convergence of these random probability measures to the invariant probability measure of $\{Z_i\}$ as $\tau \to \infty$. It is implied by \cite[Lemma 6, Chap.\ 6.3]{Bor08} and Theorem~\ref{thm-erg}(i).%footnote starts
\footnote{A condition in \cite[Lemma 6, Chap.\ 6.3]{Bor08} is that the set of random probability measures $\{\mu^{\beta,t}_{\tau} \mid \tau \geq 0, t > 0 \}$ is almost surely tight, which is trivially satisfied in our case since $\{Z_i\}$ lie in a bounded set. It is also assumed in \cite[Lemma 6, Chap.\ 6.3]{Bor08} that the stepsize $\beta_n$ is eventually nonincreasing; however, this assumption is used only once in the proof, as remarked in \cite{Bor08}. Where this assumption is used, the proof argument given in \cite{Bor08} also goes through if one assumes instead that $\beta_n$ satisfies $\lim_{n \to \infty} \, \sup_{0 \leq j \leq m_n} \big| \tfrac{\beta_{n+j}}{\beta_n} - 1 \big| = 0$ for some $m_n \to \infty$. Thus we can include the latter assumption as an alternative in Lemma~\ref{lem-occ-meas}. By \cite[Lemma 6, Chap.\ 6.3]{Bor08}, for each $t > 0$, any limit point of $\{\mu^{\beta,t}_{\tau}\}_{\tau \geq 0}$ must be an invariant probability measure of the weak Feller Markov chain $\{Z_i\}$. Since $\{Z_i\}$ has a unique invariant probability measure by Theorem~\ref{thm-erg}(i), we obtain Lemma~\ref{lem-occ-meas}.}
%footnote ends
In the lemma $\{\beta_n\}$ is a generic stepsize sequence, so the lemma also applies to $\mu^{\alpha,t}_\tau$.

\begin{lem} \label{lem-occ-meas}
Consider the case of history-dependent $\lambda$ under Assumption~\ref{cond-collective}. Let $\{\beta_n\}$ be positive stepsizes such that $\sum_{n \geq 0} \beta_n = \infty$, $\sum_{n \geq 0} \beta^2_n < \infty$, and either $\{\beta_n\}$ is eventually nonincreasing or for some sequence of integers $m_n \to \infty$, 
$\lim_{n \to \infty} \, \sup_{0 \leq j \leq m_n} \big| \tfrac{\beta_{n+j}}{\beta_n} - 1 \big| = 0$.
Then for each $t > 0$, almost surely, as $\tau \to \infty$, $\mu^{\beta,t}_{\tau}$ converges weakly to the unique invariant probability measure of $\{Z_i\}$.
\end{lem}

We use this lemma to establish the following result, which leads directly to the verification of the ``averaging conditions'' (\ref{eq-ascondf2a}) and (\ref{eq-asconds2}) for both time-scales.

\begin{prop} \label{prop-trace-averaging2}
Let the conditions in Lemma~\ref{lem-occ-meas} hold. Let $f(z)$ be a vector-valued function on $\Z$ that is continuous in the trace variable $\e$, and let $\bar f = \E_\zeta [ f(Z_0) ]$. Then for each $\eta > 0$ and $T > 0$,
\begin{equation} \label{eq-ave2}
 \lim_{n \to \infty} \Pr \left\{ \sup_{j \geq n} \max_{0 \leq t \leq T} \Big\|  \sum_{i=m_\beta(jT)}^{m^-_\beta(jT+t)} \beta_i \big( f(Z_i) - \bar f \, \big) \Big\| \geq \eta \right\} = 0.
\end{equation} 
\end{prop}

\begin{proof}
Fix $\eta$ and $T$. Denote
$\Delta_{j,t} = \sum_{i=m_\beta(jT)}^{m^-_\beta(jT+t) } \beta_i ( f(Z_i) - \bar f \, )$,
so (\ref{eq-ave2}) is the same as
\begin{equation} \label{eq-asprfa1}
 \lim_{n \to \infty} \Pr \left\{ \sup_{j \geq n} \max_{0 \leq t \leq T} \| \Delta_{j,t} \| \geq \eta \right\} = 0.
\end{equation}
Since $\{\e_n\}$ lies in a pre-determined bounded set by our choice of history-dependent $\lambda$ under Assumption~\ref{cond-collective}, for all $i$, $\|f(Z_i) - \bar f \| \leq c$ for some constant $c$. 
Let us replace the range of $t$ in (\ref{eq-asprfa1}) by a finite set of values. Let $m$ be a positive integer such that $c\, T/m \leq \eta/2$. If $ |t_1 - t_2| \leq T/m$, then 
$\| \Delta_{j, t_1} - \Delta_{j, t_2} \| \leq c \, T/m \leq \eta/2$. Therefore, to prove (\ref{eq-asprfa1}), it is sufficient to prove for the finite set of $t$, $E_m=\{ i T/m \mid i = 1, \ldots, m\}$, that
\begin{equation} \label{eq-asprfa2}
   \lim_{n \to \infty} \Pr \left\{ \sup_{j \geq n} \max_{ t \in E_m} \| \Delta_{j,t} \| \geq \eta/2 \right\} = 0.
\end{equation}
By the Markov inequality,
\begin{equation} \label{eq-asprfa2b}
 \lim_{n \to \infty} \Pr \left\{ \sup_{j \geq n} \max_{ t \in E_m} \| \Delta_{j,t} \| \geq \eta/2 \right\} 
 \leq \lim_{n \to \infty} (2/\eta) \cdot  \E \Big[ \sup_{j \geq n} \max_{ t \in E_m} \| \Delta_{j,t} \| \Big].
\end{equation}
Since $\sup_{j \geq n} \max_{ t \in E_m} \| \Delta_{j,t} \| \leq c \, T$, the relation (\ref{eq-asprfa2}) will follow from (\ref{eq-asprfa2b}) and the bounded convergence theorem if we show that
$\lim_{n \to \infty} \sup_{j \geq n} \max_{ t \in E_m} \| \Delta_{j,t} \| = 0$ a.s., i.e.,  
\begin{equation} \label{eq-asprfa3}
  \limsup_{n \to \infty} \max_{t \in E_m} \| \Delta_{n,t} \| = 0 \ \ \ \text{a.s.}
\end{equation}
So let us prove (\ref{eq-asprfa3}).
 
Denote by $\zeta'$ the unique invariant probability measure of $\{Z_i\}$. By Lemma~\ref{lem-occ-meas}, with probability~$1$, 
it holds for all $t \in E_m$ that $\mu^{\beta,t}_{\tau}$ converges weakly to $\zeta'$ as $\tau \to \infty$. 
Since $\{Z_n\}$ lie in a compact set by our choice of $\lambda$, the supports of all these probably measures are contained in the same compact set.
Then for the continuous function $f(z)$, which is (component-wise) bounded on the compact set, we have, by the definition of weak convergence, that almost surely, for all $t \in E_m$,
\begin{equation} \label{eq-asprfa4}
  \int f(z) d \mu^{\beta,t}_{\tau} - \int f(z) d\zeta' \to 0 \quad \text{as} \ \tau \to \infty.
\end{equation}  
We have $\int f(z) d\zeta' = \E_\zeta [ f(Z_0) ] = \bar f$. For $\tau_n = nT$, by the definition of $\mu^{\beta,t}_{\tau_n}$,
$$  \int f(z) d \mu^{\beta,t}_{\tau_n} = \frac{\sum_{i=m_\beta(nT)}^{m^-_\beta(nT+t)} \beta_i f(Z_i)}{\sum_{i=m_\beta(nT)}^{m^-_\beta(nT+t) } \beta_i}.$$
Then, since for all $t \leq T$, $\sum_{i=m_\beta(nT)}^{m^-_\beta(nT+t)} \beta_i \leq T$, 
we obtain from (\ref{eq-asprfa4}) that, almost surely, for all $t \in E_m$,
$$ \Big(\sum_{i=m_\beta(nT)}^{m^-_\beta(nT+t)} \beta_i \Big) \cdot \left\| \int f(z) d \mu^{\beta,t}_{\tau_n} - \bar f \, \right\| \to 0 \quad \text{as} \ n \to \infty,$$
which is the same as 
$$   \Delta_{n, t} = \sum_{i=m_\beta(nT)}^{m^-_\beta(nT+t)} \beta_i \big( f(Z_i) - \bar f) \to 0 \quad \text{as} \ n \to \infty.$$
Consequently, $\max_{t \in E_m} \| \Delta_{n, t} \| \to 0$ as $n \to \infty$, almost surely. This proves (\ref{eq-asprfa3}), from which the desired relation (\ref{eq-ave2}) follows, as shown earlier.
\end{proof}

We can now apply Prop.~\ref{prop-trace-averaging2} to show that the algorithm satisfies the required ``averaging conditions'' (\ref{eq-ascondf2a})  and (\ref{eq-asconds2}). We will do this below, inside the end part of our proof for Theorem~\ref{thm-gtd-as-conv}.

\bigskip
\noindent {\bf Completion of the proof of Theorem~\ref{thm-gtd-as-conv}}
\smallskip

In the rest of this proof of Theorem~\ref{thm-gtd-as-conv}, we discuss separately the two cases of setting $\lambda$: set $\lambda$ according to (\ref{eq-histlambda}), and set $\lambda$ according to the composite scheme (\ref{eq-cp-histlambda}). This is because the first case is simpler and provides the basis for treating the second case, as the analysis in Section~\ref{sec-cp-extension} already showed.

\medskip
\noindent {\bf Case I:} history-dependent $\lambda$'s set according to (\ref{eq-histlambda}).
In this case, for each $(\theta, x)$, applying Prop.~\ref{prop-trace-averaging2} with $f(\cdot)=k(\theta, x, \cdot)$ and stepsizes $\{\beta_n\}$, we see that the algorithm satisfies the averaging condition (\ref{eq-ascondf2a}) for the fast time-scale. For each $\theta$, applying the proposition with $f(\cdot) =g(\theta, \bar x(\theta), \cdot)$ and stepsizes $\{\alpha_n\}$, we see that the algorithm satisfies the averaging condition (\ref{eq-ascondf2a}) for the slow time-scale. 
All the conditions required for the fast time-scale have now been verified. So we can apply \cite[Theorem 6.1.1]{KuY03} to the equivalent form (\ref{eq-gtd1-asf}) of the algorithm at the fast time-scale, and conclude that almost surely, $(\theta_n, x_n)$ converges to the limit set of the ODE (\ref{eq-gtd1as-odefast}) in $B_\theta \times B_x$ for the initial condition $x(0) \in \sp \{ \phi(\S)\}$. This limit set is $\{(\theta, \bar x(\theta) ) \mid \theta \in B_\theta \}$, as shown by Lemma~\ref{lem-Bx} if $B_x$ satisfies the stronger condition (\ref{eq-cond-Bx0}), and by Lemma~\ref{lem-Bx2} otherwise.

We now consider the $\theta$-iterates (\ref{eq-gtd1-ass}) at the slow time-scale. Among the conditions listed above, we still need to verify the condition (\ref{eq-asconds3}) in (iii), which we do now. Using the almost sure convergence of $(\theta_n, x_n)$ to $\{(\theta, \bar x(\theta) ) \mid \theta \in B_\theta \}$ just proved, and using also the uniform continuity of $\bar x(\theta)$ on $B_\theta$, we have $x_n - \bar x(\theta_n) \asto 0$. Then, since $\{Z_n\}$ lie in a bounded set and since for all $z$ in that set and $\theta \in B_\theta$, $g(\theta, x, z)$ is Lipschitz continuous in $x$ uniformly w.r.t.\ $(\theta, z)$, there is a constant $c$ such that for all $i$,
$$ \big\| g(\theta_i, x_i, Z_i) - g(\theta_i, \bar x(\theta_i), Z_i) \big\| \leq c \| x_i - \bar x(\theta_i) \| \asto 0.$$
Consequently, 
\begin{equation} \label{eq-asprfb1}
  \sup_{j \geq n} \max_{0 \leq t \leq T} \Big\|  \sum_{i=m_\alpha(jT)}^{m^-_\alpha(jT+t)} \alpha_i \big( g(\theta_i, x_i, Z_i)  - g(\theta_i, \bar x(\theta_i), Z_i)  \big) \Big\|  \leq c \, T \cdot\!\! \sup_{i \geq m_\alpha(nT)} \| x_i - \bar x(\theta_i) \| \asto 0.
\end{equation}  
The l.h.s.\ above is bounded by some constant because $\{Z_n\}$ and $\{(\theta_n, x_n)\}$ are all bounded. So, using the same reasoning that led to (\ref{eq-asprfa3}) in the proof of Prop.~\ref{prop-trace-averaging2}, (\ref{eq-asconds3}) follows from the Markov inequality, the almost sure convergence to zero in (\ref{eq-asprfb1}), and the bounded convergence theorem.

Thus all the conditions required for the slow time-scale are met. We can now apply \cite[Theorem 6.1.1]{KuY03} to the $\theta$-iterates (\ref{eq-gtd1-ass}) to conclude that almost surely, $\theta_n$ converges to the limit set of the ODE~(\ref{eq-gtd1as-odeslow}) in $B_\theta$. This limit set is $\Theta_{\text{opt}}$, as we showed already in Section~\ref{sec-gtd1-odes} and Section~\ref{sec-gtda-revisit} for the two conditions on $B_x$. 
So we obtain the conclusions of Theorem~\ref{thm-gtd-as-conv} in the case of history-dependent $\lambda$'s set according to (\ref{eq-histlambda}), for the GTDa algorithm, as well as for the GTDb algorithm (since its proof under the stronger condition (\ref{eq-cond-Bx0}) on $B_x$ is essentially the same as given above).

\medskip
\noindent {\bf Case II:} history-dependent $\lambda$'s set according to the composite scheme (\ref{eq-cp-histlambda}). We proceed to verify (\ref{eq-ascondf2a}) and (\ref{eq-asconds2}) in the conditions (ii) for both time-scales. Recall from Section~\ref{sec-cp-extension} that in this case we have $\ell$ concurrent state-trace processes, $\big\{ (S_n, y_n, \e_n^{(q)} )\big\}_{n \geq 0}$, $q \leq \ell$, each of which has a unique invariant probability measure $\zeta^{(q)}$. As showed in the proof of Lemma~\ref{lem-composite-conv-mean-cond}, we can decompose $k(\cdot)$ and $g(\cdot)$ as the sum of $\ell$ functions and write, for each $(\theta, x)$,
\begin{equation} \label{eq-asprfc0}
  \textstyle{k(\theta, x, Z_n) = \sum_{q=1}^\ell k^o(\theta, x, Z_n^{(q)}), \qquad g(\theta, x, Z_n) = \sum_{q = 1}^\ell g(\theta, x, Z_n^{(q)})},
\end{equation}  
where $k^o(\cdot)$ is a continuous function and $Z_n^{(q)} = (S_n, y_n, \e_n^{(q)}, S_{n+1})$, and we also have that
\begin{align*}
 \bar k(\theta, x) & = \textstyle{  \sum_{q=1}^\ell} \bar k_q(\theta, x),  \quad  \text{where} \ \ \bar k_q(\theta, x) = \E_{\zeta^{(q)}} \big[ \, k^o(\theta, x, Z_0^{(q)}) \big], \\
\bar g(\theta) & = \textstyle{  \sum_{q=1}^\ell} \bar g_q(\theta), \qquad \text{where} \ \ \bar g_q(\theta) = \E_{\zeta^{(q)}} \big[ \, g(\theta, \bar x(\theta), Z_0^{(q)}) \big].
\end{align*}
We now apply Prop.~\ref{prop-trace-averaging2} to the $\ell$ state-trace processes individually. For each $q \leq \ell$, applying Prop.~\ref{prop-trace-averaging2} to the $q$-th process with $f(\cdot)=k^o(\theta, x, \cdot)$ and stepsizes $\{\beta_n\}$, we obtain that for each $\eta, T > 0$,
\begin{equation} \label{eq-asprfc1}
 \lim_{n \to \infty} \Pr \left\{ \sup_{j \geq n} \max_{0 \leq t \leq T} \Big\|  \sum_{i=m_\beta(jT)}^{m^-_\beta(jT+t)} \beta_i \big( k^o(\theta, x, Z^{(q)}_i) - \bar k^o_q(\theta, x) \big) \Big\| \geq \eta \right\} = 0, \qquad \forall \, q \leq \ell;
\end{equation}
and applying Prop.~\ref{prop-trace-averaging2} to the $q$-th process with $f(\cdot) = g(\theta, \bar x(\theta), \cdot)$ and stepsizes $\{\alpha_n\}$, we obtain that for each $\eta, T > 0$,
\begin{equation} \label{eq-asprfc2}
 \lim_{n \to \infty} \Pr \left\{ \sup_{j \geq n} \max_{0 \leq t \leq T} \Big\|  \sum_{i=m_\alpha(jT)}^{m^-_\alpha(jT+t)} \alpha_i \big( g(\theta, \bar x(\theta), Z^{(q)}_i) - \bar g_q(\theta) \big) \Big\| \geq \eta \right\} = 0, \qquad \forall \, q \leq \ell.
\end{equation} 
Clearly, that (\ref{eq-asprfc1}) and (\ref{eq-asprfc2}) hold for each $\eta, T > 0$ imply, respectively, that the desired relations (\ref{eq-ascondf2a}) and (\ref{eq-asconds2}) also hold for each $\eta, T > 0$. We have thus verified the ``averaging condition'' (ii) for both time-scales. From this point on, the rest of the proof is the same as that given above for the first case of history-dependent $\lambda$. 
This completes our proof of Theorem~\ref{thm-gtd-as-conv}.

\subsubsection{Convergence proof for MD-GTD}

The proof of Theorem~\ref{thm-mdgtd-as-conv} for MD-GTD follows the same steps and uses the same arguments as the proof for GTDa we just gave. So we will only briefly explain the proof outline.
In particular, take MD-GTDa as an example. 
The associated mean ODEs for the fast and slow time-scales are the same ODEs (\ref{eq-mdgtd1-odefast}) and (\ref{eq-mdgtd1-odeslow}) discussed in Section~\ref{sec-mdgtd}. 
As in the previous GTDa case, we apply \cite[Theorem 6.1.1]{KuY03} twice to analyze the iterates $\{(\theta^*_n, x_n)\}$, first to the fast time-scale and then to the slow time-scale, to establish that the average dynamics of $\{(\theta^*_n, x_n)\}$ at the fast time-scale and the average dynamics of $\{\theta^*_n\}$ at the slow time-scale are both characterized by the respective mean ODEs. Within this proof step, we first combine the conclusions obtained about the fast-time-scale iterates with the characterization of the limit set of the fast-time-scale ODE (\ref{eq-mdgtd1-odefast}) under the assumed condition on $B_x$. This leads to the conclusion that $x_n - x_{\theta_n} \asto 0$, where $\theta_n = \nabla \psi^*(\theta^*_n)$. We then use this in the analysis of the slow-time-scale $\theta^*$-iterates, and by applying \cite[Theorem 6.1.1]{KuY03} again, we obtain a convergence result about the $\theta^*$-iterates---in particular, their almost sure convergence to the limit set of the slow-time-scale ODE~(\ref{eq-mdgtd1-odeslow}). 
Then, as in Section~\ref{sec-mdgtd}, we translate this convergence of the $\theta^*$-iterates to the convergence of the $\theta$-iterates by using the relation~(\ref{eq-mdgtd-thstar2th}).

To apply \cite[Theorem 6.1.1]{KuY03} to MD-GTDa, we need to verify two sets of conditions for the fast and slow time-scales, similar to those given in the previous subsection for GTDa. The conditions are now on the $(\theta^*, x)$-space instead of the $(\theta, x)$-space, but they are essentially the same as those for GTDa, after we make the change of variable $\theta = \nabla \psi^*(\theta^*)$, similar to the case with the weak convergence proof method (cf.\ Footnote~\ref{footnote-mdgtd-cond} in Section~\ref{sec-mdgtd}). We therefore omit the details about the verification of these conditions.

\subsection{Single-Time-Scale GTDa and Biased GTDa} \label{sec-as-singletime}

We consider now the single-time-scale GTDa algorithm (\ref{eq-gtd1m-th})-(\ref{eq-gtd1m-x}) and its biased variant discussed in Section~\ref{sec-gtds}. We analyze separately the case with history-dependent $\lambda$, where the trace iterates lie in a pre-determined bounded set, and the case with state-dependent $\lambda$, where the trace iterates can be unbounded.
In the former case, we have more satisfactory results, obtained with an analysis similar to the one given in the previous subsection. 

\begin{assumption} \label{as-cond-gtds}
For the single-time-scale GTDa algorithm:\vspace*{-3pt}
\begin{itemize}
\item[\rm (i)] Condition~\ref{cond-pol} holds.
\item[\rm (ii)] The $\lambda$-parameters are set according to either (\ref{eq-histlambda}) or (\ref{eq-cp-histlambda}), where the memory states satisfy Condition~\ref{cond-mem}, and the function $\lambda(\cdot)$ in (\ref{eq-histlambda}) or each function $\lambda^{(i)}(\cdot), i \leq \ell$ in (\ref{eq-cp-histlambda}) satisfies Condition~\ref{cond-lambda}.
\item[\rm (iii)] The stepsizes $\{\alpha_n\}$ satisfy the conditions in Assumption~\ref{cond-large-stepsize} for the fast-time-scale stepsizes, as well as the square-summable condition, $\sum_{n \geq 0} \alpha^2_n < \infty$.
\item[\rm (iv)] The initial $x_0, \e_0 \in \sp  \{\phi(\S)\}$.
\end{itemize}
\end{assumption}

In the convergence theorem below, the definitions of the saddle points $D_\theta \times \{\bar x\}$ and $\Theta_{\text{\rm opt}} \times \{ x_{\text{\rm opt}} \}$ are as in Section~\ref{sec-gtds}. Recall that $\Theta_{\text{\rm opt}} = \argmin_{\theta \in B_\theta} \{ J(\theta) + p(\theta) \}$ and $x_{\text{\rm opt}}$ is the unique dual optimal solution in $\sp \{\phi(\S)\}$ of the minimax problem (\ref{eq-minimax0}) (which places no constraints on the maximizer).

\begin{thm} \label{thm-gtds-as-conv}
Let $\{(\theta_n, x_n)\}$ be generated by the single-time-scale GTDa algorithm (\ref{eq-gtd1m-th})-(\ref{eq-gtd1m-x}).
Then under Assumption~\ref{as-cond-gtds}, for each given initial condition, $\{(\theta_n, x_n)\}$ converges almost surely to the subset $D_\theta \times \{\bar x\}$ of saddle points of the minimax problem (\ref{eq-minimax1}). If the constraint set $B_x$ contains $x_{\text{\rm opt}}$ in its interior, then $D_\theta \times \{\bar x\}=\Theta_{\text{\rm opt}} \times \{ x_{\text{\rm opt}} \}$.
\end{thm}

For the case of state-dependent $\lambda$ (or the case of the composite scheme discussed in Section~\ref{sec-cp-extension} that involves state-dependent $\lambda$), in general, the trace iterates can be unbounded. Our convergence result for this case is limited: it is only for stepsizes $\alpha_n = O(1/n)$. (This is also the reason why we do not have a convergence theorem for the two-time-scale GTDa algorithm with state-dependent $\lambda$---if we let the stepsize be $O(1/n)$ for the fast time-scale, we do not have a choice of stepsize left for the slow time-scale.)

\begin{assumption} \label{as-cond-gtds2}
For the single-time-scale GTDa algorithm and its biased variant:\vspace*{-3pt}
\begin{itemize}
\item[\rm (i)] The conditions in Assumption~\ref{as-cond-gtds}(i) and (iv) hold.
\item[\rm (ii)] The stepsizes $\{\alpha_n\}$ satisfy $\alpha_n = O(1/n)$ and $(\alpha_n - \alpha_{n+1})/\alpha_n = O(1/n)$.
\end{itemize}
\end{assumption}

In the convergence theorem below, the family of functions $\{ h_K \mid K > 0\}$ for the biased GTDa algorithm are the same as in Theorem~\ref{thm-gtds-bvrt}.

\begin{thm} \label{thm-gtds-as-conv2}
Consider the case of state-dependent $\lambda$. Let Assumption~\ref{as-cond-gtds2} hold, and let $\{(\theta_n, x_n)\}$ be generated by the single-time-scale GTDa algorithm (\ref{eq-gtd1m-th})-(\ref{eq-gtd1m-x}) or by its biased variant (\ref{eq-gtdsv-th})-(\ref{eq-gtdsv-x}) with $h \in \{ h_K \mid K > 0\}$. Then for each given initial condition, the following hold:\vspace*{-3pt} 
\begin{itemize}
\item[\rm (i)] In the case of GTDa, $\{(\theta_n, x_n)\}$ converges almost surely to the subset $D_\theta \times \{\bar x\}$ of saddle points of the minimax problem (\ref{eq-minimax1}). 
\item[\rm (ii)] In the case of biased GTDa, for each $\epsilon > 0$, there exists $K_\epsilon > 0$ such that if $h=h_K$ with $K \geq K_\epsilon$, then $\{(\theta_n, x_n)\}$ converges almost surely to $N_\epsilon ( D_\theta \times \{\bar x\})$.
\end{itemize}
If the constraint set $B_x$ contains $x_{\text{\rm opt}}$ in its interior, then $D_\theta \times \{\bar x\}=\Theta_{\text{\rm opt}} \times \{ x_{\text{\rm opt}} \}$ in the above.
\end{thm}

\subsubsection{Convergence proofs}

To prove the preceding convergence theorems, we will again apply \cite[Theorem 6.1.1]{KuY03}. The mean ODEs associated with the single-time-scale GTDa and its biased variant are, respectively, the ODEs~(\ref{eq-gtds-ode}) and (\ref{eq-gtdsv-ode}), which we analyzed in Section~\ref{sec-gtds}. We will use \cite[Theorem 6.1.1]{KuY03} to obtain the conclusion that the iterates produced by each algorithm converge almost surely to the limit set of its associated ODE, thereby proving the convergence theorems above.

Thus, all we need to do is to show that the algorithms satisfy the conditions of \cite[Theorem 6.1.1]{KuY03} under the assumptions of each theorem. For the single-time-scale GTDa, we need to verify the five conditions listed below, which correspond to the conditions (A.6.1.1), (A.6.1.3)-(A.6.1.4), and (A.6.1.6)-(A.6.1.7) in \cite[Chap.\ 6.1]{KuY03}. 
In the conditions, the functions $g, k, \bar g, \bar k$ are as defined in (\ref{eq-gtd1-gk}) and (\ref{eq-gtds-gk}).
The ``iteration pointers'' $m_\alpha(\cdot), m^-_\alpha(\cdot)$, etc.\ are defined w.r.t.\ the stepsizes $\{\alpha_n\}$ as in Section~\ref{subsec-asconv-gtd}.

The conditions for the biased GTDa variant are the same, except that the functions $g, k$ are replaced respectively by the functions $g_h, k_h$ (which are similar to $g, k$ but with $h(\e)$ in place of $\e$), and the functions $\bar g, \bar k$ are replaced respectively by the functions $\bar g_h, \bar k_h$ defined in (\ref{eq-gtdsvrt-gk}), and finally, in the conditions (i) and (iii), $\e_n \omega_{n+1}$ is replaced by $h(\e_n)\omega_{n+1}$.

\medskip
{\samepage \noindent {\bf Conditions for GTDa:}\vspace*{-3pt}
\begin{itemize}
\item[(i)] $\sup_{n \geq 0} \E  \big[ \| g(\theta_n, x_n, Z_n) - \nabla p(\theta_n) \| \big] < \infty$, $\sup_{n \geq 0} \E \big[ \big\| k(\theta_n, x_n, Z_n) + \e_n \omega_{n+1}  \big\| \big] < \infty$.
\item[(ii)] For each $(\theta, x) \in B_\theta \times B_x$, $\eta > 0$ and $T > 0$,
\begin{equation} \label{eq-ascondm2a}
 \lim_{n \to \infty} \Pr \left\{ \sup_{j \geq n} \max_{0 \leq t \leq T} \Big\|  \sum_{i=m_\alpha(jT)}^{m^-_\alpha(jT+t)} \alpha_i \big( g(\theta, x, Z_i) - \bar g(\theta, x) \big) \Big\| \geq \eta \right\} = 0,
\end{equation} 
\begin{equation} \label{eq-ascondm2b}
 \lim_{n \to \infty} \Pr \left\{ \sup_{j \geq n} \max_{0 \leq t \leq T} \Big\|  \sum_{i=m_\alpha(jT)}^{m^-_\alpha(jT+t)} \alpha_i \big( k(\theta, x, Z_i) - \bar k(\theta, x) \big) \Big\| \geq \eta \right\} = 0.
\end{equation} 
\item[(iii)] For each $\eta > 0$ and $T > 0$,
\begin{equation} \label{eq-ascondm3}
    \lim_{n \to \infty} \Pr \left\{ \sup_{j \geq n} \max_{0 \leq t \leq T} \Big\|  \sum_{i=m_\alpha(jT)}^{m^-_\alpha(jT+t)} \alpha_i  \e_i \omega_{i+1} \Big\| \geq \eta \right\} = 0.
\end{equation}
\item[(iv)] There are nonnegative functions $f_1(\theta, x)$ and $f_2(z)$ such that 
\begin{equation} \label{eq-ascond-sing40}
  \max\big\{ \|g(\theta, x, z) - \nabla p(\theta)\|, \, \|k(\theta, x, z)\|  \big\} \leq f_1(\theta,x) f_2(z),
\end{equation}  
where
$f_1$ is bounded on $B_\theta \times B_x$, and for each $\eta > 0$,
\begin{equation} \label{eq-ascond-sing4}
 \lim_{t \to 0} \lim_{n \to \infty} \Pr \left\{ \sup_{j \geq n}   \sum_{i=m_\alpha(jt)}^{m^-_\alpha(jt+t)} \alpha_i f_2( Z_i) \geq \eta \right\} = 0.
\end{equation} 
\item[(v)] There are nonnegative functions $f_3(\theta, x)$ and $f_4(z)$ such that 
$$ \big\| g(\theta, x, z) - \nabla p(\theta) - \big( g(\theta', x', z) - \nabla p(\theta') \big) \big\| \leq f_3(\theta - \theta', x - x') f_4(z),$$
$$ \big\| k(\theta, x, z) - k(\theta', x', z) \big\| \leq f_3(\theta - \theta', x - x') f_4(z),$$
where $f_3$ is bounded on $B_\theta \times B_x$ and $f_3(\theta, x) \to 0$ as $(\theta, x) \to 0$, and 
\begin{equation} \label{eq-ascond-sing5}
   \Pr \left\{ \limsup_{n \to \infty}   \sum_{i=n}^{m_\alpha(n,t)} \alpha_i f_4( Z_i) < \infty \right\} = 1, \quad \text{for some} \ t > 0.
\end{equation}
\end{itemize}}
\smallskip

Given the analysis in the previous section, Theorem~\ref{thm-gtds-as-conv} is evident:

\begin{proof}[Proof of Theorem~\ref{thm-gtds-as-conv}]
Assumption~\ref{as-cond-gtds}(ii) ensures that the trace iterates $\{\e_n\}$ lie in a pre-determined bounded set. It is then evident that in the previous Section~\ref{subsec-asconv-gtd}, we have already verified essentially all the conditions listed above, when analyzing the two-time-scale GTDa under the same choice of the $\lambda$-parameters. To avoid repetition, we will only discuss the condition (ii) here. Since $g(\theta, x, z)$ and $k(\theta, x, z)$ are continuous functions, applying Prop.~\ref{prop-trace-averaging2} to each of them (for fixed $(\theta, x)$) with stepsizes $\{\alpha_n\}$, we see that the condition (ii) is satisfied when we set the $\lambda$-parameters according to (\ref{eq-histlambda}). If the composite scheme with history-dependent $\lambda$ is used and the $\lambda$-parameters are set according to (\ref{eq-cp-histlambda}), then we proceed as in the analysis of Case II in the last part of the proof for Theorem~\ref{thm-gtd-as-conv}. In particular, we first decompose $g(\cdot)$ and $k(\cdot)$ as the sums of $\ell$ functions and then apply Prop.~\ref{prop-trace-averaging2} to each of the functions, before we combine the results (cf.\ the reasoning given in (\ref{eq-asprfc0})-(\ref{eq-asprfc2}) in Section~\ref{subsec-asconv-gtd}).
\end{proof}

We now proceed to prove Theorem~\ref{thm-gtds-as-conv2}. It concerns the case of state-dependent $\lambda$. In this case, without conditions to guarantee the boundedness of $\{\e_n\}$, a major difficulty for the algorithms is in satisfying the ``averaging condition'' (ii). We can only show that this condition is met for stepsizes satisfying $\alpha_n = O(1/n)$ and $(\alpha_n - \alpha_{n+1})/\alpha_n = O(1/n)$. The proof is given below. It utilizes the ergodicity property of the state-trace process given in Theorem~\ref{thm-erg}(ii), and it is similar to the proof given in \cite[Section 4.1]{Yu-siam-lstd} and \cite[Section 4.1 and Appendix B]{yu-etdarx} for the constrained off-policy TD and ETD algorithms respectively.

\begin{proof}[Proof of Theorem~\ref{thm-gtds-as-conv2}]
We give the proof for GTDa only, as the proof for the biased variant is essentially the same, except that it involves the functions $g_h, k_h, \bar g_h, \bar g_k$ instead of the functions $g, k, \bar g, \bar k$, as mentioned earlier.

We need to show that the five conditions listed above are all met. 
Similar to the derivations (\ref{eq-asprf-k1})-(\ref{eq-asprf-k2}) and (\ref{eq-asprf-g1})-(\ref{eq-asprf-g2}) in Section~\ref{subsec-asconv-gtd}, 
in the conditions (iv)-(v), 
we can take 
$$f_1(\theta, x) = \| x \| + \| \theta \| + \| \nabla p(\theta)\| + 1, \qquad f_3(\theta, x) = \| x \| + \| \theta \| + u_p(\| \theta \|),$$
and take $f_2(z), f_4(z)$ to be $c \, ( \| \e\| + 1)$ for some constant $c$. (Note that as required, $f_3(\theta, x) \to 0$ as $(\theta, x) \to 0$.)

Then, since the iterates $(\theta_n, x_n)$ lie in the bounded set $B_\theta \times B_x$, using the bound (\ref{eq-ascond-sing40}), we see that the two requirements in the condition (i) are implied by 
$\sup_{n \geq 0} \E [ \| \e_n\| ] < \infty$ and $\sup_{n \geq 0} \E [ \| \e_n \omega_{n+1} \| ] < \infty$. The latter inequality is itself implied by $\sup_{n \geq 0} \E [ \| \e_n \| ] < \infty$, because given the history $Z_i , R_i, i \leq n$, the noise term $\omega_{n+1}$ has zero mean and finite variance which depends only on the state transition $(S_n, S_{n+1})$.
Since Prop.~\ref{prop-trace}(ii) ensures $\sup_{n \geq 0} \E [ \| \e_n\| ] < \infty$, we see that the condition (i) is met.

We now focus on proving (\ref{eq-ascondm2a})-(\ref{eq-ascondm3}) and (\ref{eq-ascond-sing4})-(\ref{eq-ascond-sing5}) in the conditions listed above. 
Consider first (\ref{eq-ascond-sing4})-(\ref{eq-ascond-sing5}) in the conditions (iv)-(v). 
Based on the remarks in~\cite[p.\ 166]{KuY03}, to prove (\ref{eq-ascond-sing4}), it is sufficient to show that with $\bar f_2 = \E_\zeta [ f_2(Z_0) ]$,
\begin{equation} \label{eq-ascond-sing4a}
\lim_{n \to \infty} \Pr \left\{ \sup_{j \geq n} \max_{0 \leq t \leq T} \Big\|  \sum_{i=m_\alpha(jT)}^{m^-_\alpha(jT+t)} \alpha_i \big( f_2( Z_i) - \bar f_2 \big) \Big\| \geq \mu \right\} = 0;
\end{equation}
and to prove (\ref{eq-ascond-sing5}), it is sufficient to prove (\ref{eq-ascond-sing4a}) with $f_4$ in place of $f_2$ and $\bar f_4 = \E_\zeta [ f_4(Z_0) ]$ in place of $\bar f_2$. The above condition (\ref{eq-ascond-sing4a}) has the same form as (\ref{eq-ascondm2a})-(\ref{eq-ascondm3}) in the conditions (ii)-(iii). (They are known as the asymptotic rate of change conditions, or the Kushner-Clark conditions; cf.~\cite[Chap.\ 2.2]{KuC78}) By the examples in \cite[Chap.\ 6.2, p.\ 171]{KuY03}, one sufficient condition for this type of assumptions is that a strong law of large numbers holds. 
In particular, let $\psi_i$ represent any of the following terms involved in the conditions (ii)-(v): $g(\theta, x, Z_i) - \bar g(\theta, x)$, $k(\theta, x, Z_i) - \bar k(\theta, x)$, $\e_i \omega_{i+1}$, $f_2( Z_i) - \bar f_2$, and $f_4( Z_i) - \bar f_4$. Then for each of the conditions to hold, it is sufficient that for the respective $\psi_i$'s, we have $\alpha_n \sum_{i=0}^{n-1} \psi_i \asto 0$ and $(\alpha_n - \alpha_{n+1})/\alpha_n = O(\alpha_n)$.  

Since the functions $g(\theta, x, \cdot), k(\theta, x, \cdot), f_2(\cdot)$ and $f_4(\cdot)$ are Lipschitz continuous in the trace variable, by Theorem~\ref{thm-erg}(ii), for stepsizes $\alpha_n = O(1/n)$, we have $\alpha_n \sum_{i=0}^{n-1} \psi_i \asto 0$ for the corresponding $\psi_i$'s. Consequently the relations (\ref{eq-ascondm2a})-(\ref{eq-ascondm2b}) and  (\ref{eq-ascond-sing4})-(\ref{eq-ascond-sing5}) hold.
For (\ref{eq-ascondm3}) in the condition (iii), the proof that $\frac{1}{n} \sum_{i=0}^{n-1} \e_i \omega_{i+1} \asto 0$ was given in the context of ETD in \cite[Prop.\ 3.1(ii), Appendix~A.4]{yu-etdarx}, and the same proof applies here. 
Thus all the required conditions are met, so we can apply \cite[Theorem~6.1.1]{KuY03}, and using its conclusions, together with Prop.~\ref{prop-sd-limitset} and Lemma~\ref{lem-saddle-relation}, we then obtain the conclusions of the theorem.
\end{proof}

\subsection{Unconstrained Single-Time-Scale GTDa} \label{sec-as-singletime-unconstr}

Finally, we consider the unconstrained single-time-scale GTDa algorithm that uses the minimax approach to solve the problem $\inf_\theta J_p(\theta)$, for the regularized convex, differentiable objective function $J_p(\theta) = J(\theta) + p(\theta)$ as before. The algorithm is given by 
\begin{align}
    \theta_{n+1} & = \theta_n + \alpha_n \,  \rho_n \big(\phi(S_n) - \gma_{n+1} \phi(S_{n+1}) \big) \cdot \e_n^\tr x_n  - \alpha_n \nabla p(\theta_n),  \label{eq-ugtd1m-th}\\
   x_{n+1} & =  x_n + \alpha_n \big(\e_n \delta_n(v_{\theta_n}) - \phi(S_n) \phi(S_n)^\tr x_n\big). \label{eq-ugtd1m-x}
\end{align}
For unconstrained algorithms, keeping the iterates bounded is a difficulty in theory as well as in practice. We will impose stronger conditions on the above algorithm than those 
on the constrained GTDa in the previous subsection. We shall consider only the case of history-dependent $\lambda$ (including the composite scheme with history-dependent $\lambda$), where the traces $\{\e_n\}$ lie in a pre-determined bounded set by the choice of $\lambda$. To obtain a convergence theorem, we will apply stability and convergence results from \cite[Chap.\ 6.3]{Bor08}, and to this end, we shall also impose stronger conditions on the function $p(\cdot)$ in the regularized objective function.

\begin{assumption}[Conditions on the objective function] \label{as-cond-p} \hfill 
\vspace*{-3pt}
\begin{itemize}
\item[\rm (a)] The function $p(\cdot)$ is convex and differentiable, and its gradient mapping $\nabla p(\cdot)$ is Lipschitz continuous. Moreover, there is a convex differentiable function $p_\infty(\cdot)$ such that for each $\theta$, $\lim_{a \to \infty} p(a \theta)/a^2 = p_\infty(\theta)$. 
\item[\rm (b)] The set of optimal solutions to $\inf_\theta J_p(\theta)$ is nonempty.
\item[\rm (c)] The origin is the unique optimal solution of $\inf_{\theta} \big\{ \tfrac{1}{2} \| \Pi_\xi (\Pl - I) v_\theta \|_\xi^2 + p_\infty(\theta) \big\}$. 
\item[\rm (c$'$)] Alternative to (c), assume that $\nabla p(\cdot)$ is such that $\nabla p(\theta) \in \sp \{\phi(\S)\}$ for all $\theta \in \sp \{ \phi(\S)\}$, and that the origin is the uniquq optimal solution of $\inf_{\theta} \big\{ \tfrac{1}{2} \| \Pi_\xi (\Pl - I) v_\theta \|_\xi^2 + p_\infty(\theta) \big\}$ in $\sp \{ \phi(\S)\}$. 
\end{itemize}
\end{assumption}

We remark that Assumption~\ref{as-cond-p}(a) is especially strong. The requirement on $\nabla p(\cdot)$ being Lipschitz continuous rules out regularizers such as $p(\theta) = \sum_{j=1}^d | \theta_j|^c$, $c \in (1, 2)$, which can be useful if sparse solutions are preferred. The requirement on the convergence to $p_\infty$ also seems strong, since if $\nabla p(\cdot)$ is Lipschitz continuous, then for all $a > 0$, the functions $p(a \theta)/a^2$ are pointwise bounded and there always exists a sequence of them converging to a finite convex function. It may be possible to relax this requirement if one can make sure that the proof arguments in \cite[Chap.\ 6.3]{Bor08} still go through under weakened conditions.
We also note that it can be shown that Assumption~\ref{as-cond-p}(a) and (c) imply Assumption~\ref{as-cond-p}(b). However, the latter is a more natural assumption on the problem that the unconstrained algorithm aims to solve, since $p_\infty$ is introduced only for the purpose of analysis.

The next two lemmas are two implications of the preceding assumptions. They concern the optimal solutions of the minimization problem $\inf_{\theta} J_p(\theta)$ and the saddle points of its equivalent minimax problem. We will need them to state our convergence result.

\begin{lem} \label{lem-asugtd-optset}
Under Assumption~\ref{as-cond-p}(a)-(c), the set of optimal solutions of $\inf_{\theta} J_p(\theta)$ is compact. If instead of Assumption~\ref{as-cond-p}(c), Assumption~\ref{as-cond-p}(c$'\!$) holds, then the subset of those optimal solutions of  $\inf_{\theta} J_p(\theta)$ in $\sp \{ \phi(\S)\}$ is nonempty and compact.
\end{lem}

\begin{proof}
Proof by contradiction. Suppose that Assumption~\ref{as-cond-p}(a)-(c) hold and yet $\inf_{\theta} J_p(\theta)$ has an unbounded (closed) set $\Theta_{\text{\rm opt}}$ of optimal solutions. 
Let $\theta_i \in \Theta_{\text{\rm opt}}$ for $i \geq 0$, with $a_i:= \| \theta_i \|_2 \to + \infty$ as $i \to +\infty$. 
Recall 
$$J_p(\theta) = \tfrac{1}{2} \big\| \Pi_\xi (\Tl v_{\theta} - v_{\theta} ) \big\|_\xi^2 + p(\theta)  =  \tfrac{1}{2} \big\| \Pi_\xi (\Pl  - I) \Phi \theta  + \Pi_\xi r_\pi^{(\lambda)}  \big\|_\xi^2 + p(\theta).$$
By the optimality of the $\theta_i$'s, $J_p(\theta_i) \leq J_p(\theta)$ for all $\theta \in \re^d$. 
Therefore, with $\tilde \theta_i = \theta_i/a_i = \theta_i / \| \theta_i\|_2$, we have
\begin{align}
 \frac{1}{a_i^2} J_p(a_i \tilde \theta_i)  =   \frac{1}{a_i^2} J_p(\theta_i) 
  & =   \tfrac{1}{2} \big\| \Pi_\xi (\Pl  - I) \Phi \tilde \theta_i  + \tfrac{1}{a_i} \Pi_\xi r_\pi^{(\lambda)}  \big\|_\xi^2 + \frac{1}{a_i^2} p(a_i \tilde \theta_i) \notag \\
  & \leq  \frac{1}{a_i^2}  J_p(a_i \theta)   \label{eq-ugtdsprf-1}
\end{align}
for all $\theta \in \re^d$.
By choosing a subsequence if necessary, we can suppose that the sequence $\{\tilde \theta_i\}$ converges to a point $\tilde \theta_\infty$ on the unit ball. 
Then, take $i \to +\infty$ in both sides of (\ref{eq-ugtdsprf-1}), and we have, by Assumption~\ref{as-cond-p}(a), that 
\begin{align}
 \lim_{i \to \infty}  \frac{1}{a_i^2} J_p(a_i \tilde \theta_i)  & =  \tfrac{1}{2} \big\| \Pi_\xi (\Pl  - I) \Phi \tilde \theta_\infty \big\|_\xi^2 + p_\infty(\tilde \theta_\infty)  \notag \\
    & \leq \tfrac{1}{2} \big\| \Pi_\xi (\Pl  - I) \Phi \theta \big\|_\xi^2 + p_\infty(\theta) \label{eq-ugtdsprf-2} 
\end{align}
for all $\theta \in \re^d$, where, to derive the inequality (\ref{eq-ugtdsprf-2}), we used the pointwise convergence of the convex functions $ \frac{1}{a_i^2} p(a_i \, \cdot)$ to $p_\infty(\cdot)$, and to derive the first equality, we used the fact that this pointwise convergence of $ \frac{1}{a_i^2} p(a_i \, \cdot)$ to $p_\infty(\cdot)$ implies that the convergence is uniform on a neighborhood of $\tilde \theta_\infty$ \cite[Theorem 10.8]{Roc70}. 
The inequality~(\ref{eq-ugtdsprf-2}) implies that the point $\tilde \theta_\infty \not=0$ is an optimal solution of $\inf_{\theta} \big\{ \tfrac{1}{2} \| \Pi_\xi (\Pl - I) v_\theta \|_\xi^2 + p_\infty(\theta) \big\}$. 
This contradicts Assumption~\ref{as-cond-p}(c), under which the latter optimization problem has the origin as its unique optimal solution. Therefore, the set of optimal solutions to $\inf_{\theta} J_p(\theta)$ must be compact. 

For the second part of the lemma, note that if $\nabla p(\theta) \in \sp \{\phi(\S)\}$ for all $\theta \in \sp \{ \phi(\S)\}$, then $\nabla J_p(\theta) \in  \sp \{\phi(\S)\}$ for all $\theta \in \sp \{ \phi(\S)\}$. This implies that the subset $\tilde \Theta_{\text{\rm opt}}$ of those optimal solutions of $\inf_{\theta} J_p(\theta)$ in $\sp \{ \phi(\S)\}$ is nonempty.%footnote starts
\footnote{To see this, let $\theta_{\text{opt}}$ be an optimal solution of $\inf_{\theta} J_p(\theta)$, and let $\bar \theta$ be the projection of  $\theta_{\text{opt}}$ on $\sp \{ \phi(\S)\}$. Then by the convexity of $J_p$, $J_p(\theta_{\text{opt}}) \geq J_p(\bar \theta) + \langle \nabla J_p(\bar \theta), \theta_{\text{opt}} - \bar \theta \rangle = J_p(\bar \theta)$, implying that $\bar \theta$ is also an optimal solution.} 
%footnote ends
Now apply the preceding proof argument with the set $\tilde \Theta_{\text{\rm opt}}$ in place of $\Theta_{\text{\rm opt}}$, and it follows that $\tilde \Theta_{\text{\rm opt}}$ must be compact.
\end{proof}

Recall from the discussion in Section~\ref{sec-gtds-ode} that the problem $\inf_\theta J_p(\theta)$ can be equivalently written as the minimax problem 
$$ \inf_\theta \sup_x \psi^o(\theta, x) = \inf_\theta \sup_{x \in \sp \{\phi(\S) \}} \psi^o(\theta, x) $$
for the convex-concave function $ \psi^o(\theta, x)  = \langle \Phi x, \Tl v_\theta - v_\theta \rangle_\xi - \tfrac{1}{2} \| \Phi x \|^2_\xi  + p(\theta)$ (cf.\ (\ref{eq-psio})),
which is constant over the subspace $\sp \{\phi(\S)\}^\perp$ for each $\theta$. Let $\Theta_{\text{\rm opt}} = \argmax_{\theta} J_p(\theta)$.

\begin{lem} \label{lem-asugtd-sadd}
Under Assumption~\ref{as-cond-p}(b), the maximization problem $\sup_x \{ \inf_\theta \psi^o(\theta, x)\}$ has a unique optimal solution $x_{\text{\rm opt}}$ in $\sp \{\phi(\S)\}$, and the set of saddle points of $\inf_\theta \sup_{x \in \sp \{\phi(\S) \}} \psi^o(\theta, x)$ is $\Theta_{\text{\rm opt}} \times \{x_{\text{\rm opt}}\}$.
\end{lem}

\begin{proof}
By the definition of a saddle point, $(\bar \theta, \bar x)$ is a saddle point of $\inf_\theta \sup_x \psi^o(\theta, x)$ if and only if $\nabla \psi^o(\bar \theta, \bar x) = (0, 0)$.
The latter equation is the same as
\begin{equation} \label{eq-ugtdsprf-3}  
 \big( \Phi^\tr \Xi \, (\Pl- I) \Phi\big)^\tr \bar x + \nabla p(\bar \theta) = 0, \qquad  \Phi^\tr \Xi \, \Phi \bar x = \Phi^\tr \Xi \, (\Tl v_{\bar \theta} - v_{\bar \theta}). 
\end{equation} 
Consider any $\bar \theta \in \Theta_{\text{\rm opt}}$ (which is nonempty by our assumption). The second equation in (\ref{eq-ugtdsprf-3}) has a unique solution $x_{\bar \theta}$ in $\sp \{\phi(\S)\}$, as discussed earlier (cf.\ (\ref{eq-xtheta})-(\ref{eq-xtheta2})). 
By the optimality of $\bar \theta$ for $\inf_\theta J_p(\theta)$, $\nabla J_p(\bar \theta) = \nabla J(\bar \theta) + \nabla p(\bar \theta) = 0$, 
and then using the expression (\ref{eq-Jgrad1}) of $\nabla J(\bar \theta)$, we see that $(\bar \theta, x_{\bar \theta})$ satisfies the first equation in (\ref{eq-ugtdsprf-3}). Thus $(\bar \theta, x_{\bar \theta})$ is a saddle point. It then follows from \cite[Lemma 36.2]{Roc70} that $\inf_\theta \sup_x \psi^o(\theta, x)$ has a finite saddle-value and its saddle points are the points in $\Theta_{\text{\rm opt}} \times X_{\text{\rm opt}}$, where $X_{\text{\rm opt}} = \argmax_x \{\inf_\theta \psi^o(\theta, x)\}$ is the set of dual optimal solutions.
Obviously, since in the preceding proof the point $x_{\bar \theta} \in \sp \{\phi(\S)\}$, $X_{\text{\rm opt}} \cap \sp \{\phi(\S)\} \not= \emptyset$. Suppose we can find in this set two different points $x_1, x_2$. Then for any $\bar \theta \in \Theta_{\text{\rm opt}}$, since both $(\bar \theta, x_1)$ and $(\bar \theta, x_2)$ are saddle points, they must both satisfy (\ref{eq-ugtdsprf-3}). But the second equation in (\ref{eq-ugtdsprf-3}) has a unique solution in $\sp \{\phi(\S)\}$ for each $\bar \theta$, a contradiction. Therefore, $X_{\text{\rm opt}} \cap \sp \{\phi(\S)\}$ must be a singleton.
\end{proof}

Let $\tilde \Theta_{\text{\rm opt}} = \argmax_{\theta} J_p(\theta) \, \cap \, \sp \{\phi(\S)\}$ if Assumption~\ref{as-cond-p}(a)-(b) and (c$'\!$) hold; it is nonempty and compact by Lemma~\ref{lem-asugtd-optset}. By the preceding Lemma~\ref{lem-asugtd-sadd}, 
$\tilde \Theta_{\text{\rm opt}} \times \{x_{\text{\rm opt}}\}$ is the subset of those saddle points in $\sp \{\phi(\S)\} \times \sp \{\phi(\S)\}$ for the minimax problem $\inf_\theta \sup_x \psi^o(\theta, x)$.
 
Here is our convergence result for the unconstrained GTDa:

\begin{thm} \label{thm-ugtds-as-conv}
Let $\{(\theta_n, x_n)\}$ be generated by the unconstrained single-time-scale GTDa algorithm (\ref{eq-ugtd1m-th})-(\ref{eq-ugtd1m-x}).
Then under Assumption~\ref{as-cond-gtds} and Assumption~\ref{as-cond-p}(a)-(c), for each given initial condition, $\{(\theta_n, x_n)\}$ converges almost surely to the set $\Theta_{\text{\rm opt}} \times \{ x_{\text{\rm opt}}\}$. If instead of Assumption~\ref{as-cond-p}(c), Assumption~\ref{as-cond-p}(c$'\!$) holds, then with $\theta_0 \in \sp \{ \phi(\S)\}$,  $\{(\theta_n, x_n)\}$ converges almost surely to the set $\tilde \Theta_{\text{\rm opt}}\times \{ x_{\text{\rm opt}}\}$.
\end{thm}

%\smallskip
\begin{rem}[About unconstrained GTDa with a scaling parameter $\eta$] \label{rem-ugtds-eta} \rm \hfill \\
A version of Theorem~\ref{thm-ugtds-as-conv} also holds for the unconstrained GTDa algorithm (\ref{gtd1-th-eta-alt})-(\ref{gtd1-x-eta-alt}) or equivalently, (\ref{gtd1-th-eta})-(\ref{gtd1-x-eta}), which uses a scaling parameter $\eta > 0$. In this case, the conditions are the same as stated by Theorem~\ref{thm-ugtds-as-conv}. In the conclusion, we only need to replace the minimax problem $\inf_\theta \sup_x \psi^o(\theta, x)$ by the minimax problem corresponding to $\eta$, which is 
\begin{equation} \label{eq-minimax-ugtds-eta}
 \inf_{\theta \in \re^d} \sup_{\tilde x \in \re^d} \Big\{ a \, \langle \Phi \tilde x, \, \Tl v_\theta - v_\theta \rangle_\xi - \frac{a^2}{2} \| \Phi \tilde x \|^2_\xi  + p(\theta) \Big\}, \quad \text{for} \  a = \sqrt{\eta},
\end{equation} 
and use its unique dual optimal solution $\tilde x_{\text{opt}}$ in the subspace $\sp \{ \phi(\S)\}$ in replace of the point $x_{\text{opt}}$. The convergence conclusion for (\ref{gtd1-th-eta-alt})-(\ref{gtd1-x-eta-alt}) then translates to that for (\ref{gtd1-th-eta})-(\ref{gtd1-x-eta}) via the correspondence $x = \sqrt{\eta} \,\tilde x$. The reasoning is essentially the same as that given in Section~\ref{sec-gtds-eta} for the constrained GTDa, which shows why the case $\eta > 0$ can be treated by the same line of analysis for $\eta = 1$. In particular, the derivation of the preceding minimax problem (\ref{eq-minimax-ugtds-eta}) is similar to that of~(\ref{eq-minimax0-eta}), and since the two minimax problems, (\ref{eq-minimax-ugtds-eta}) and $\inf_\theta \sup_x \psi^o(\theta, x)$, have the same structure, the arguments we give in this subsection for proving Theorem~\ref{thm-ugtds-as-conv} carry over to the general case $\eta > 0$.
\end{rem} 

\subsubsection{Convergence proof}

We apply the ``natural averaging'' argument given in \cite[Chaps.\ 6.2-6.3]{Bor08} to prove Theorem~\ref{thm-ugtds-as-conv}.
Assumption~\ref{as-cond-p} is indeed introduced for this purpose. Let us derive more implications of it, which will be needed to show that the conditions required by the analysis in \cite{Bor08} are satisfied.

First, consider the desired mean ODE for the unconstrained GTDa algorithm (\ref{eq-ugtd1m-th})-(\ref{eq-ugtd1m-x}),
\begin{equation} \label{eq-ugtds-ode}
 \dot{\theta}(t) = \bar g (\theta(t), x(t)) - \nabla p(\theta(t)), \qquad   \dot x(t) = \bar k (\theta(t), x(t)),
\end{equation}
where the functions $\bar g, \bar k$ are defined as in (\ref{eq-gtds-gk}):
$$ \bar g(\theta, x) = \big( \Phi^\tr \Xi \, (I - \Pl) \Phi\big)^\tr x, \qquad \bar \k(\theta, x) = \Phi^\tr \Xi \, (\Tl v_\theta - v_\theta) - \Phi^\tr \Xi \, \Phi \, x.$$
Note that, as before, from any initial $x(0) \in \sp \{\phi(\S)\}$, the $x$-component of the solution of (\ref{eq-ugtds-ode}) lies entirely in $\sp \{\phi(\S)\}$.
Note also that if, as assumed in Assumption~\ref{as-cond-p}(c$'$), $\nabla p(\theta) \in \sp \{ \Phi(\S)\}$ for all $\theta \in \sp \{ \Phi(\S)\}$, then from any initial $\theta(0) \in \sp \{ \Phi(\S)\}$, the $\theta$-component of the solution of (\ref{eq-ugtds-ode}) also lies entirely in $\sp \{\phi(\S)\}$.
We can use the same analysis given in Section~\ref{sec-gtds-ode} for the constrained GTDa to show that the above ODE has the following solution properties:

\begin{lem} \label{lem-asugtd-1} \hfill \vspace*{-3pt}
\begin{itemize} 
\item[\rm (i)] Under Assumption~\ref{as-cond-p}(a)-(c), consider the ODE~(\ref{eq-ugtds-ode}) on the space $\re^d \times \sp \{\phi(\S)\}$. Then w.r.t.\ the latter space, the set $\Theta_{\text{\rm opt}} \times \{ x_{\text{\rm opt}}\}$ is globally asymptotically stable.
\item[\rm (ii)] Under Assumption~\ref{as-cond-p}(a)-(b) and (c$'\!$), consider the ODE~(\ref{eq-ugtds-ode}) on the space $\sp \{\phi(\S)\} \times \sp \{\phi(\S)\}$. Then w.r.t.\ the latter space, the set $\tilde \Theta_{\text{\rm opt}} \times \{ x_{\text{\rm opt}}\}$ is globally asymptotically stable.
\end{itemize}
\end{lem} 
 
\begin{proof}
Recall the solution properties of differential inclusion with a maximal monotone mapping $\A$ given by \cite[Theorem 1, Chap.\ 3]{AuC85} and described in Section~\ref{sec-gtds-ode}. We apply them here with $\A$ being the maximal monotone mapping $(\theta, x) \mapsto \big(\nabla_\theta \psi^o(\theta, x), - \nabla_x \psi^o(\theta, x) \big)$. 
With $y(t) = (\theta(t), x(t))$, the ODE~(\ref{eq-ugtds-ode}) is the same as $\dot y(t) = - \A(y(t))$. For any $\bar \theta \in \Theta_{\text{\rm opt}}$ and $\bar x = x_{\text{\rm opt}}$, $y(\cdot) \equiv (\bar \theta, \bar x)$ is a solution of (\ref{eq-ugtds-ode}). 
Then by the solution property (iii) for differential inclusion with $\A$ described in Section~\ref{sec-gtds-ode}, for any $\epsilon > 0$, starting from the $\epsilon$-neighborhood of $\Theta_{\text{\rm opt}} \times \{ x_{\text{\rm opt}}\}$, the solutions of (\ref{eq-ugtds-ode}) remain in that same neighborhood. 

Now consider an arbitrary solution $\bar y(\cdot)$ of (\ref{eq-ugtds-ode}) with the initial $\bar y(0) \in \re^d \times \sp \{\phi(\S)\}$.
Define a set $D \subset \re^d \times \sp \{\phi(\S)\}$ by
$$D = \Big\{ (\theta, x) \, \Big| \, \text{dist}\big( (\theta, x), \, \Theta_{\text{\rm opt}} \times \{ x_{\text{\rm opt}}\} \big) \leq \text{dist}\big( \bar y(0), \, \Theta_{\text{\rm opt}} \times \{ x_{\text{\rm opt}}\} \big)  \Big\} \, \cap  \big( \re^d \times \sp \{\phi(\S)\} \big).$$
Note that $D$ is compact by Lemma~\ref{lem-asugtd-optset}, and any solution of (\ref{eq-ugtds-ode}) starting from $D$ remains in $D$. 
Then, as $t \to + \infty$, $\bar y(t)$ must converge to the set $\cap_{\bar t \geq 0} \, \cl \{ y(t) \mid y(0) \in D, t \geq \bar t \, \}$, which is the limit set of the ODE~(\ref{eq-ugtds-ode}) for all initial conditions in $D$. 

We now apply the same proof arguments given in Section~\ref{sec-gtds-ode} for analyzing the limit set of the ODE associated with the constrained GTDa. In particular, with the set $D$ in place of the constraint set $B_\theta \times B_x$, Lemma~\ref{lem-invset0} and its proof remain valid for the case here, and Lemma~\ref{lem-invset} and Cor.~\ref{cor-invset} then also hold for the ODE~(\ref{eq-ugtds-ode}) here. In turn, the proof of Prop.~\ref{prop-sd-limitset} (which uses Cor.~\ref{cor-invset}, as we recall) also applies here, if we now take the sets $B_\theta$ and $B_x$ appearing in that proof to be the entire space $\re^d$, and let all the boundary reflection terms of the ODE be zero in that proof. The result of applying the proof of Prop.~\ref{prop-sd-limitset} here is that the limit set $\cap_{\bar t \geq 0} \, \cl \{ y(t) \mid y(0) \in D, t \geq \bar t \, \} = \Theta_{\text{\rm opt}} \times \{ x_{\text{\rm opt}}\}$, analogous to Prop.~\ref{prop-sd-limitset}. This shows that any solution $\bar y(t)$ must converge to $\Theta_{\text{\rm opt}} \times \{ x_{\text{\rm opt}}\}$ as $t \to +\infty$. So w.r.t.\ the space $\re^d \times \sp \{\phi(\S)\}$, the set $\Theta_{\text{\rm opt}} \times \{ x_{\text{\rm opt}}\}$ is globally asymptotically stable.

The second part of the lemma follows from the same argument, with $\sp \{\phi(\S)\} \times \sp \{\phi(\S)\}$ in place of the space $\re^d \times \sp \{\phi(\S)\}$.
\end{proof}

Consider now the minimization problem $\inf_{\theta} \big\{ \tfrac{1}{2} \| \Pi_\xi (\Pl - I) v_\theta \|_\xi^2 + p_\infty(\theta) \big\}$ appearing in Assumption~\ref{as-cond-p}. 
The same reasoning given before (\ref{eq-minimax0}) in Section~\ref{sec-gtds-ode} shows that it is equivalent to the minimax problem
\begin{equation}  \label{eq-minimax-infty}
 \quad \qquad \inf_\theta \sup_x \psi_\infty(\theta, x), \qquad \text{where} \ \psi_\infty(\theta, x) : = \langle \Phi x,  (\Pl - I) v_\theta \rangle_\xi - \tfrac{1}{2} \| \Phi x \|^2_\xi  + p_\infty(\theta).
\end{equation} 
Analogously to (\ref{eq-ugtds-ode}), consider the following ODE for solving this minimax problem:
\begin{equation} \label{eq-ode-infty}
\dot{\theta}(t)  = - \nabla_\theta \psi_\infty(\theta(t), x(t)), \qquad  \dot x(t) =  \nabla_x \psi_\infty(\theta(t), x(t)),
\end{equation}
where
$$ \nabla_\theta \psi_\infty(\theta, x) =  \big( \Phi^\tr \Xi \, (\Pl - I) \Phi\big)^\tr x + \nabla p_\infty(\theta),   \quad  \nabla_x \psi_\infty(\theta, x) =   \Phi^\tr \Xi \, (\Pl - I) v_\theta - \Phi^\tr \Xi \, \Phi \, x. $$

We need the following two implications of Assumption~\ref{as-cond-p}, which will be used in establishing the boundedness of the iterates of the unconstrained GTDa algorithm.

\begin{lem} \label{lem-asugtd-2}
Under Assumption~\ref{as-cond-p}(a), for each $(\theta, x)$,
$$\lim_{a \to +\infty} \tfrac{1}{a} \,  \bar g(a\theta, a x) - \tfrac{1}{a} \, \nabla p(a \theta) = - \nabla_\theta \psi_\infty(\theta, x), \qquad \lim_{a \to +\infty} \tfrac{1}{a} \, \bar k(a\theta, a x) =  \nabla_x \psi_\infty(\theta, x),$$ 
and the convergence is uniform on any compact set of $(\theta, x)$.
\end{lem}

\begin{proof}
It is clear that $\tfrac{1}{a} \,  \bar g(a\theta, a x) = \big( \Phi^\tr \Xi \, (I - \Pl) \Phi\big)^\tr x$ for any $a > 0$, and 
$\lim_{a \to +\infty} \tfrac{1}{a} \, \bar k(a\theta, a x) = \Phi^\tr \Xi \, (\Pl - I) v_\theta - \Phi^\tr \Xi \, \Phi \, x = \nabla_x \psi_\infty(\theta, x)$ and the convergence is uniform on any compact set of $(\theta, x)$. 
Now note that $\tfrac{1}{a} \, \nabla p(a \theta)$ is the gradient of the convex function $\tfrac{1}{a^2} \, p(a \theta)$ at $\theta$. 
By Assumption~\ref{as-cond-p}(a), as $a \to + \infty$, the function $\tfrac{1}{a^2} \, p(a \theta)$ converges pointwise to the convex differentiable function $p_\infty(\theta)$.
Then by \cite[Theorem 35.10]{Roc70}, as $a \to + \infty$, the gradients $\tfrac{1}{a} \, \nabla p(a \theta)$ of $\tfrac{1}{a^2} \, p(a \theta)$ converge to $\nabla p_\infty(\theta)$ uniformly on any compact set of $\theta$.
The lemma then follows.
\end{proof}

\begin{lem} \label{lem-asugtd-3}
Let Assumption~\ref{as-cond-p}(a) hold in what follows.\vspace*{-3pt}
\begin{itemize} 
\item[\rm (i)] Under Assumption~\ref{as-cond-p}(c) or (c$'$\!), the point $x = 0$ is the unique dual optimal solution in $\sp \{\phi(\S)\}$ for the minimax problem (\ref{eq-minimax-infty}), and the point $(\theta, x) = (0,0)$ is its saddle point. 
\item[\rm (ii)] If Assumption~\ref{as-cond-p}(c) (respectively, (c$'$\!)) holds, then for the ODE (\ref{eq-ode-infty}) on the space $\re^d \times \sp \{\phi(\S)\}$ (respectively, the space $\sp \{\phi(\S)\} \times \sp \{\phi(\S)\}$), the point $(\theta, x) = (0,0)$ is globally asymptotically stable.
\end{itemize}
\end{lem}

\begin{proof}
The first part follows from the same proof of Lemma~\ref{lem-asugtd-sadd} with the function $\psi_\infty$ in place of $\psi^o$ and with the point $\theta_{\text{\rm opt}} = 0$ in place of the set $\Theta_{\text{\rm opt}}$. 

For the second part, the case of Assumption~\ref{as-cond-p}(c) follows from the proof of Lemma~\ref{lem-asugtd-1} with $\psi_\infty$ in place of $\psi^o$ and with the point $(\theta_{\text{\rm opt}}, x_{\text{\rm opt}}) = (0,0)$ in place of the set $\Theta_{\text{\rm opt}} \times \{ x_{\text{\rm opt}}\}$.
For the case of Assumption~\ref{as-cond-p}(c$'$), note first that $\tfrac{1}{a} \, \nabla p(a \theta) \to \nabla p_\infty (\theta)$ as shown in the proof of Lemma~\ref{lem-asugtd-2}. Then, since $\nabla p(a \theta) \in \sp \{ \phi(\S)\}$ for all $\theta \in \sp \{ \phi(\S)\}$ under Assumption~\ref{as-cond-p}(c$'$), we have $\nabla p_\infty (\theta) \in \sp \{ \phi(\S)\}$ for all $\theta \in \sp \{ \phi(\S)\}$. 
The desired conclusion then follows from the first part of the proof of Lemma~\ref{lem-asugtd-1}, with $\psi_\infty$ in place of $\psi^o$, with the point $(\theta_{\text{\rm opt}}, x_{\text{\rm opt}}) = (0,0)$ in place of the set $\Theta_{\text{\rm opt}} \times \{ x_{\text{\rm opt}}\}$, and with the space 
$\sp \{\phi(\S)\} \times \sp \{\phi(\S)\}$ in place of the space $\re^d \times \sp \{\phi(\S)\}$.
\end{proof}

Finally, we are ready to prove Theorem~\ref{thm-ugtds-as-conv}:

\begin{proof}[Proof of Theorem~\ref{thm-ugtds-as-conv}]
First, let us note the space we work with in the proof. In the case where Assumption~\ref{as-cond-p}(c) or (c$'$) holds, we restrict attention to the space $\re^d \times \sp \{\phi(\S)\}$ or the space $\sp \{\phi(\S)\} \times \sp \{\phi(\S)\}$, respectively. Recall that in each of these two cases, it is ensured by our assumption on the initial condition $x_0, \e_0$ and $\theta_0$ for the algorithm that all the iterates $\{(\theta_n, x_n)\}$ lie in the corresponding space $\re^d \times \sp \{\phi(\S)\}$ or $\sp \{\phi(\S)\} \times \sp \{\phi(\S)\}$.

Under Assumption~\ref{as-cond-gtds}, by Prop.~\ref{thm-erg}(i), the process $\{Z_n\}$ is an ergodic weak Feller Markov chain with a unique invariant probability measure. To prove the theorem, we will apply the ``natural averaging'' results of \cite[Chaps.\ 6.2-6.3]{Bor08} for the ``uncontrolled'' case, which involves an ergodic weak Feller Markov chain whose evolution is not affected by the $(\theta, x)$-iterates of the algorithm. Under our Assumption~\ref{as-cond-gtds}(ii), the traces $\{\e_n\}$ lie in a pre-determined compact set by the choice of $\lambda$. With the trace component $\e$ of $z$ restricted to this compact set, the continuous functions $g(\theta, x, z), k(\theta, x, z)$ in our GTDa algorithm are Lipschitz continuous in $(\theta, x)$ uniformly w.r.t.\ $z$. By Assumption~\ref{as-cond-p}(i), the gradient mapping $\nabla p(\theta)$ is Lipschitz continuous. These two facts together fulfill the condition on the algorithm in \cite[Chaps.\ 6.2, p.\ 68]{Bor08}. 

With the boundedness of $\{\e_n\}$ just mentioned, it also holds trivially that the set of random probability measures $\mu_\tau^{\alpha, t}, t \geq 0$ defined before Lemma~\ref{lem-occ-meas} in Section~\ref{subsec-asconv-gtd} is tight for each $\tau \geq 0$. This fulfills one major assumption in the analysis in \cite[Chaps.\ 6.2-6.3]{Bor08} (cf.\ page 71 therein). The second major assumption is for proving that the iterates of the algorithm are almost surely bounded, and it is the condition in \cite[Theorem 9, Chap.\ 6.3]{Bor08}. This condition is the same as the properties stated in our Lemma~\ref{lem-asugtd-2} and Lemma~\ref{lem-asugtd-3}(ii) above, and which is thus fulfilled by these lemmas. 

Then, applying \cite[Theorem 9, Chap.\ 6.3]{Bor08}, we obtain that the iterates $(\theta_n, x_n)$ are almost surely bounded, and in turn, applying \cite[Theorem 7 and Corollary 8, Chap.\ 6.3]{Bor08}, we obtain that $\{ (\theta_n, x_n) \}$ converges almost surely to an internally chain transitive invariant set of the ODE~(\ref{eq-ugtds-ode}), for the space $\re^d \times \sp \{\phi(\S)\}$ or the space $\sp \{\phi(\S)\} \times \sp \{\phi(\S)\}$ that we work with (which depends on whether Assumption~\ref{as-cond-p}(c) or (c$'$) holds, as noted at the beginning of the proof). Finally, combining this conclusion with Lemma~\ref{lem-asugtd-1}, we obtain the conclusions of Theorem~\ref{thm-ugtds-as-conv}.
\end{proof}

\section{Discussion} \label{sec-6}

We have presented a collection of convergence results for a number of gradient-based off-policy TD algorithms, including the two-time-scale GTD and mirror-descent GTD algorithms, the single-time-scale GTDa algorithm, as well as the biased variants of these algorithms. For slowly diminishing stepsize and constant stepsize, using the weak convergence method, we have shown that w.r.t.\ a sense of convergence that is weaker than the almost sure convergence, all these algorithms have comparable, satisfactory convergence properties, whether they employ state-dependent or history-dependent $\lambda$. 
For the standard type of diminishing stepsize, however, concerning the almost sure convergence of the algorithms, our results are more satisfactory  in the case of history-dependent $\lambda$ where the traces are bounded, and more limited in the case of state-dependent $\lambda$ (including constant $\lambda$) where the traces can be unbounded. We believe this is not entirely an artifact of the analysis: It could be that with general state-dependent (or constant) $\lambda$, in certain situations, the algorithms need not converge almost surely under the standard diminishing stepsize condition, whereas in some other situations, our current results can be strengthened. Further investigation is needed in order to clarify this issue.

In this paper we have focused on one case of history-dependent $\lambda$ where the way $\lambda(\cdot)$ and memory states are defined ensures that the state-trace process has a unique invariant probability measure. Another worthwhile subject for future research, as mentioned in the introduction, is to use stochastic approximation theory for differential inclusions to study the asymptotic behavior of the TD algorithms when the $\lambda$-parameters are chosen in a history-dependent way that is more flexible than the one we considered and when the state-trace process can have multiple invariant probability measures.

\section*{Acknowledgement}
I thank Professor Richard Sutton for helpful discussion. This research was supported by a grant from Alberta Innovates---Technology Futures.

\clearpage
\addcontentsline{toc}{section}{References}
\bibliographystyle{apa}
\let\oldbibliography\thebibliography
\renewcommand{\thebibliography}[1]{%
  \oldbibliography{#1}%
  \setlength{\itemsep}{0pt}%
}
{\fontsize{9}{11} \selectfont
\bibliography{gtd_bib}
}

\end{document}